\documentclass[onefignum,onetabnum]{siamart171218}

\usepackage{xfrac}  
\usepackage{lipsum}
\usepackage{amsfonts,dsfont}
\usepackage{graphicx}
\usepackage{epstopdf}
\usepackage{algorithmic}
\usepackage{tikz}
\usepackage{graphicx}

\usepackage{enumitem,comment}
\usepackage{subcaption}
\usepackage{graphicx}
\ifpdf
\hypersetup{
  pdftitle={Inverse Transport},
  pdfauthor={blahlblah..., and Tan Bui-Thanh}
}
\usepackage{multirow}
\usepackage{colortbl}
\fi
\usetikzlibrary{positioning}

\ifpdf
  \DeclareGraphicsExtensions{.eps,.pdf,.png,.jpg}
\else
  \DeclareGraphicsExtensions{.eps}
\fi
\usepackage{comment}

\usepackage{tikz}
\usetikzlibrary{arrows,shapes}
\usetikzlibrary{decorations.pathreplacing,calligraphy}
\usetikzlibrary{patterns}
\usetikzlibrary{arrows,shapes}
\usetikzlibrary{decorations.pathreplacing,calligraphy}

\newsiamremark{remark}{Remark}
\newsiamremark{hypothesis}{Hypothesis}
\crefname{hypothesis}{Hypothesis}{Hypotheses}
\newsiamthm{claim}{Claim}

\title{A model-constrained tangent slope learning approach for dynamical systems}
\author{Hai V. Nguyen\thanks{Department of Aerospace Engineering and
  Engineering Mechanics, the University of Texas at Austin, Texas (\email{hainguyen@utexas.edu}, \email{https://orcid.org/0000-0002-0720-7237})} \and Tan Bui-Thanh\thanks{Department of Aerospace Engineering and
  Engineering Mechanics, The Oden Institute for Computational
  Engineering and Sciences,  the University of Texas at Austin, Texas
  (\email{tanbui@oden.utexas.edu},
  \url{https://users.oden.utexas.edu/\string~tanbui/}).}
}

\usepackage{amsopn}

\makeatletter
\newcommand*{\addFileDependency}[1]{% argument=file name and extension
  \typeout{(#1)}% latexmk will find this if $recorder=0 (however, in that case, it will ignore #1 if it is a .aux or .pdf file etc and it exists! if it doesn't exist, it will appear in the list of dependents regardless)
  \@addtofilelist{#1}% if you want it to appear in \listfiles, not really necessary and latexmk doesn't use this
  \IfFileExists{#1}{}{\typeout{No file #1.}}% latexmk will find this message if #1 doesn't exist (yet)
}
\makeatother

\newcommand{\Grad} {\ensuremath{\nabla}}  % Gradient
\newcommand{\nor}[1]{\left\| #1 \right\|} % norm
\newcommand{\LRp}[1]{\left( #1 \right)} % adaptive left and right parentheses
\newcommand{\LRs}[1]{\left[ #1 \right]} % adaptive left and right square brackets
\newcommand{\LRc}[1]{\left\{ #1 \right\}} % adaptive left and right curly brackets
\newcommand{\pp}[2]{\frac{\partial #1}{\partial #2}} % adaptive partial derivatives

\newcommand{\mc}[1]{\mathcal{#1}} %mathcal
\newcommand{\mb}[1]{\mathbf{#1}} %math boldface
\newcommand{\half}{\frac{1}{2}}

\newcommand{\F}{\bs{G}}
\newcommand{\W}{\bs{W}}

\newcommand{\bs}[1]{\boldsymbol{#1}}
\renewcommand{\P}{U}

\newcommand{\Pbar}{\overline{\P}}

\newcommand{\bb}{{\bf b}}

\newcommand{\yb}{\bs{y}}

\newcommand{\One}{\mathds{1}}

\newcommand{\R}{{\bs{\mathbb{R}}}}

\newcommand{\G}{G}

\newcommand{\pb}{\bs{u}}
\newcommand{\ub}{\pb}

\newcommand{\pbbar}{\overline{\pb}}
\newcommand{\ubbar}{\pbbar}

\newcommand{\figlab}[1]{\label{fig:#1}}
\newcommand{\eqnlab}[1]{\label{eq:#1}}
\newcommand{\theolab}[1]{\label{theo:#1}}

\newcommand{\lemlab}[1]{\label{lem:#1}}

\newcommand{\tablab}[1]{\label{tab:#1}}

\newcommand{\theoref}[1]{\ref{theo:#1}}

\newcommand{\lemref}[1]{\ref{lem:#1}}
\newcommand{\eqnref}[1]{\eqref{eq:#1}}

\newcommand{\seclab}[1]{\label{sect:#1}}
\newcommand{\secref}[1]{\ref{sect:#1}}

\newcommand{\dt}{\Delta t}
\newcommand{\Ib}{ \textbf{I}}

\newcommand{\MSE}[1]{ \nor{#1}_2^2}
\newcommand{\MSEsqrt}[1]{ \nor{#1}_2}
\newcommand{\NN}[1]{ \Psi \LRp{#1}}
\newcommand{\ui}[1]{ {\ub}^{{#1}}}

\newcommand{\ut}[1]{ \tilde{\ub}^{{#1}}}
\newcommand{\ubar}[1]{ \bar{\ub}^{{#1}}}

\newcommand{\eval}[2][\right]{\relax \ifx#1\right\relax \left.\fi#2#1\rvert}

\newcommand{\xb}{\bs{x}}
\renewcommand{\d}{d}
\renewcommand{\u}{u}

\newcommand{\eg}[1]{ {\varepsilon}^{{#1}}}
\newcommand{\ev}[1]{ {\bs{e}_{\text{ML}}^{{#1}}}}
\renewcommand{\epsilon}{\varepsilon}
\newcommand{\vb}{\bs{v}}
\newcommand{\ubr}{\vb}
\newcommand{\epsb}{\bs{\epsilon}}

\begin{document}

\maketitle

% REQUIRED
\begin{abstract}
Real-time accurate solutions of large-scale complex dynamical systems are in critical need for control, optimization, uncertainty quantification, and decision-making in practical engineering and science applications, especially digital twin applications. This paper contributes in this direction
a model-constrained tangent slope learning (\texttt{mcTangent}) approach. At the heart of \texttt{mcTangent} is the synergy of several desirable strategies: i) a tangent slope learning to take advantage of the neural network speed and the time-accurate nature of the method of lines;
ii) a model-constrained approach to encode the neural network tangent slope with the underlying governing equations; iii) sequential learning strategies to promote long-time stability and accuracy; and iv) data randomization approach to implicitly enforce the smoothness of the neural network tangent slope and its likeliness to the truth tangent slope up second order derivatives in order to further enhance the stability  and accuracy of \texttt{mcTangent} solutions. Rigorous results are provided to analyze and justify the proposed approach. Several numerical results for transport equation, viscous Burgers equation, and Navier-Stokes equation are presented to study and demonstrate the robustness and long-time accuracy of the proposed \texttt{mcTangent} learning approach.
\end{abstract}

% REQUIRED

\section{Introduction}
Dynamical systems are pervasive in engineering and science applications. They are typically  time-dependent systems of ordinary differential equations (ODEs) or partial differential equations (PDEs). The latter is not different from the former from the method of lines viewpoint in which a PDE reduces to a system of ODEs after a spatial discretization. For practical settings, simulating a dynamical system could be challenging due to a large number of degrees of freedom, and hence the number of ODEs, interdependent on each other in a highly nonlinear manner. For multi-scale or stiff systems of ODEs, explicit time discretization schemes, though straightforward, are not efficient to due time stepsize limitation to ensure stability. Implicit schemes, on the other hand, are stable but computationally expensive as a large linear system of equations needs to be solved at each time step. Though currently infeasible, real-time accurate approximate solutions for the practical complex dynamical system are highly desirable for control, optimization, uncertainty quantification, and decision-making. 

Towards achieving real-time solutions for dynamical systems, various pure data-driven deep learning attempts have been made.
Autoencoder architecture has been explored to simulate fluid flows \cite{kim2019deep}. Autoencoder  with physics-informed regularization to improve accuracy has been proposed to predict the future sea surface temperature given past series of measurements \cite{de2019deep}. 
% The work in  \cite{sanchez2018graph} proposes a graph network-based model to approximate the forward map and inference model, which is then used to speed up control algorithms.  
{In \cite{sanchez2018graph}, a graph network-based model is trained to approximate the forward map and inference model, and then used to speed up control algorithms.}
As an effort to combine traditional and machine learning approaches, {the authors in} \cite{morton2018deep} introduce a deep Koopman model\textemdash an auto-encoder architecture of convolution neural network\textemdash to predict the dynamics of airflow over a cylinder. Comprehensive overviews of  machine learning methods for forecasting dynamical systems can be found in \cite{lim2021time} and \cite{benidis2018deep}. The work in  \cite{duraisamy2021perspectives} presents a review and aspects of using machine learning techniques to simulate turbulent flows. %Pure machine learning surrogate model augmentation approaches and field inversion approaches are core methods in the field.

%{TAKE A LOOK AT THIS PARAGRAPH, TAN. I FEEL IT NO SMOOTH AT ".... In a different effort, ...}

Instead of replacing traditional computational approaches with pure data-driven machine learning models, which is debatable and an active research direction, one can use machine learning methods to speed up only computationally demanding modules. This could maintain desirable physics constraints as in traditional approaches while gaining computational time. Indeed, a convolution neural network (CNN) can be trained to learn the numerical error between high-resolution and low-resolution simulations \cite{pathak2020using}. Combining the CNN prediction with low-resolution simulations can then achieve similar high-resolution accuracy while being faster and at that the same time not compromising the physics. {In a different effort, neural networks are trained to replace components/terms severely affected by a low-resolution grid \cite{kochkov2021machine}}. The predictions from  neural networks are then unrolled over multiple time steps to improve long-time inference performance.  A recurrent neural network can also be used to enhance the effectiveness of geometric multigrid methods for simulating  Navier-Stokes equations \cite{margenberg2022neural}.

Completely replacing traditional methods while respecting governing equations, we argue, is highly desirable for machine learning methods because fast but nonphysical solutions are undesirable. A popular deep-learning approach aiming to accomplish this goal is  the physics-informed neural
network (PINN) \cite{raissi2017physics1}. Similar to least squares finite element
methods, PINN trains deep learning solution
constrained by the PDE residual through a regularization
\cite{raissi2017physics1, RaissiEtAl2019, RaissiKarniadakis2018,
  RaissiEtAl2017, YangPerdikaris2019, TripathyBilionis2018}).
  %to learn
%the solution $u$ of the forward problem \eqnref{forward} as a function
%of the parameter. 
PINN can learn solutions that attempt to make the
PDE residual small. However, the PINN approach directly approximates
the PDE solution in infinite dimensional spaces. While universal
approximation results (see, e.g.,
\cite{Cybenko1989,hornik1989multilayer,Zhou17,johnson2018deep}) could
ensure any desired accuracy with a sufficiently large number of neurons,
practical network architectures are moderate in both depth and width,
and hence the number of weights and biases,
%, due to computation
the
accuracy of PINN can be limited. Moreover, PINN requires a retrain for new scenarios such as new boundary conditions, or new initial conditions, or new values of the underlying parameters. A physics-informed recurrent neural network has also been studied in \cite{jia2019physics}. In order to produce physically consistent and better prediction results, energy flow and density-depth constraint laws are integrated into the loss function.

Instead of learning the infinite-dimensional solution as in PINN, learning discretized solutions of dynamical systems is equally popular. 
{The work in}
\cite{zhuang2021model} uses a neural network to approximate the derivative of the system state in reduced projected subspace. The neural network is then combined with  forward Euler and Runge-Kutta time discretization schemes to achieve high-accuracy solutions. Alternatively, {a feed forward neural network can be used to directly learn the map from the  solution at the current time step to the solution in the next time step \cite{pan2018long}}. The stability and accuracy of long-time prediction are reinforced by introducing a Jacobian regularization into the loss function. Realizing several drawbacks of the direct learning approach,
{the authors in \cite{Wang1998RungeKuttaNN} propose to learn the tangent slope with Runge-Kutta schemes}. Once trained, the learned tangent slope can be used with any time discretization schemes and any time step size. 
%Learning discretized solutions is most convenient with differential programming, which typically requires a new implementation of existing numerical approaches. 
In
\cite{um2020solver}, the authors propose to learn a correction neural network that lifts low-resolution solutions to high-resolution accuracy, and the training procedure includes low-resolution differentiable codes. Similarly, differential molecular dynamics simulations \cite{jaxmd2020} have been implemented in \texttt{Jax}
\cite{jax2018github}. Alternatively,
{the authors in }\cite{hu2019difftaichi} develop a differentiable simulations package that wraps  a numerical simulator as a gradient kernel for end-to-end back-propagation used in optimization algorithms. Similar to \cite{jaxmd2020}, {a differentiable physic simulations package equipped with the adjoint method for backpropagation is developed in \cite{holl2020phiflow}, which enables the embedding  of physical forward model into the training process.}

In this paper, aiming at simulating dynamical systems in real-time, we propose a model-constrained tangent slope deep learning (\texttt{mcTangent}) approach  that has several appealing features over existing methods. First, it operates on finite dimensional systems and is thus in principle easier to train. However, it is spatial discretization-dependent for systems governed by PDEs. Second, it learns the underlying tangent slope and thus is semi-discrete in nature. Once trained, it can be deployed with any time discretization schemes with any time step size. The next three features are the main advances
beyond the work in \cite{Wang1998RungeKuttaNN}.
Third, it aims to fulfill the governing equations by  constraining a fully discrete system in the loss function during training. Fourth, it is equipped with sequential  learning strategies and thus promotes stability and accuracy in simulating the underlying dynamical systems far beyond the training time horizon. Fifth, our approach imposes regularizations on the smoothness of the neural network tangent and its derivatives implicitly via data randomization. This provides extra stability and accuracy for \texttt{mcTangent} solutions.

The paper is organized as follows. Section \secref{mcTangent} introduces an abstract dynamical system and a model-constrained tangent slope learning (\texttt{mcTangent}) approach. Both sequential machine learning and sequential model-constrained strategies will be discussed in detail in \cref{sect:mc} and \cref{sect:sMC}. Data randomization approach then follows with an in-depth semi-heuristic argument to reveal its implicit regularization nature in \cref{sect:noise_data_sec}. In particular, data randomization induces smoothness regularization for the underlying neural network via the standard machine learning loss. The beauty of the model-constrained loss term is that it not only enforces the likeliness of the neural network and the truth tangent {slopes} but also implicitly constrains their likeliness up to second-order derivatives via data randomization.  \cref{sect:error} provides a rigorous estimation for  prediction error using \texttt{mcTangent} approach. Several numerical results using the proposed \texttt{mcTangent} approach for  transport equation, viscous Burger's equation, and Navier-Stokes equation are presented in \cref{sect:numerics}.
{We also provide detailed information on parameter tuning, randomness setting, and the cost for both training and testing.}
%Moreover, in \cref{sect:numerics}, the computational cost for training and computation benefits of solving and predictions are provided}.  
Section \secref{conclusions} concludes the paper with future work.

\section{Model-constrained tangent {slope} deep learning solutions for dynamical systems}
\seclab{mcTangent}
\subsection{Motivation}
For the concreteness and simplicity of the exposition, let us consider an abstract  dynamical system governed by the following time-dependent scalar PDE equation of the form
\begin{equation}
    \eqnlab{eq_base}
    \pp{u}{t} = \mc{G}\LRp{u, \Grad u, \hdots} \quad \text{ in } \Omega \subset \R^\d,
    %\LRp{u, \pp{u}{x}, \pp{u^2}{x^2}, u\pp{u}{x} , \hdots}, 
\end{equation}
where $t \in \LRs{0,T}$, $u\LRp{\xb} \in \R$ for any $\xb \in \R^d$, and $\d \in \LRc{1,2,3}$. We also assume \eqnref{eq_base} is equipped with appropriate initial conditions and boundary conditions to ensure its well-posedness.

In this paper, we are interested in {\em parametrized PDEs}. For downstream tasks such as design, control, optimization, inference, and
uncertainty quantification, these PDEs  need to be solved many times. As such,
we wish to approximate solutions of \eqnref{eq_base} in real time for
any parameters (e.g. initial conditions or boundary conditions, or some parameter). 
Training a PINN
together with parameters (either by themselves or their neural
networks weights and biases as another set of optimization variables)
\cite{Chen20, RAISSI2019686, DeepXDE21,lu2021physicsinformed} may not
be efficient as a new solution (corresponding to new 
parameters) requires a retrain.  We note that attempts using pure
data-driven deep learning to learn the parameter-to-solution map have
been explored (see, e.g., \cite{Kojima17, WHITE20191118,
  Pestourie2020, Tahersima2019, Peurifoyeaar4206, Kojima17,
  DNNInverseNanoPhotonnics20, Jiang2020}).  On the other hand,
standard numerical methods such as finite difference, finite volume,
and finite elements \cite{smith1985numerical, leveque2002finite, johnson2012numerical} discretize \eqnref{eq_base} both in time
and space. One of the most popular approaches is perhaps the method of lines (see, e.g., \cite{schiesser2012numerical})
in which one performs spatial discretization first to obtain a system
of (possibly nonlinear) ordinary differential equations of the form
\begin{equation}
  \eqnlab{MoLines}
  \pp{\ub}{t} =  \F \LRp{\ub},
\end{equation}
where $\ub$ and $\F$ are vector representations of finite dimensional
approximations of $\u$ and $\mc{G}$, respectively. Now, either an
explicit or implicit (or their combination) can be deployed to
discretize the temporal derivative. For the former, the most expensive
operation is the evaluation of tangent {slope\footnote{We call the right hand side $\F \LRp{\ub}$ as the tangent slope as it is a generalization of the tangent slope field in scalar ordinary differential equation.}} $\F \LRp{\ub}$. For
the latter, evaluating both $\F \LRp{\ub}$ and its (possibly
approximate) Jacobian for each time step play a vital
role. Implementing the Jacobian, even with the adjoint method \cite{tromp2005seismic}, is
a significant part of the programming effort. Automatic
differentiation can mitigate this programming burden at the expense
of more memory bandwidth. In summary, computing $\F\LRp{\ub}$ and its
Jacobian is a major part, both in implementation and computational time, of existing numerical methods.

To overcome the time burden of estimating the tangent slope and its
Jacobian, we present a model-constrained tangent slope deep learning approach (\texttt{mcTangent}) inspired by  the semi-discrete nature of the method of
lines. In particular, we first learn the tangent slope $\F
\LRp{\ub}$ using neural network and then use a time discretization to
solve for approximations of $\ub$. Our approach thus aims to approximate only the spatial discretization and leaves the temporal discretization for traditional time integrators. At the heart of our approach is the incorporation of the governing equations into the
neural network tangent by constraining the learning task to respect a temporal
discretization of \eqnref{MoLines}.  Again, unlike PINN and its siblings which learn the
infinite-dimensional solution $u$, our approach learns the tangent
slope of the finite-dimensional approximation $\ub$. Furthermore,
we constrain the physics on the discrete level. Clearly, our approach
is discretization-dependent while PINN requires neither spatial discretization nor temporal discretization. 

\subsection{Model-constrained neural network approach with sequential data learning}
\seclab{mc}

In this section, we construct a model-constrained neural network
$\NN{\ub}$ to learn $\F\LRp{\ub}$. This is done in tandem with a
time discretization of \eqnref{MoLines}. For clarity, we limit our presentation to forward Euler method 
\begin{equation}
    \eqnlab{ubFE}
     \ui{k+1} = \ui{k} + \dt \, \F \LRp{\ui{k}},
\end{equation}
as it is straightforward
to extend the approach to any time discretization scheme, and we provide a brief discussion at the end of the section. The task at hand is to train $\NN{\ub}$ on a certain spatial mesh $\mc{T}$ corresponding to a spatial discretization. To begin, let us
denote the numerical solutions of \eqnref{ubFE} at $N_t + 1$ time
steps on a finer mesh $\mc{T}^f$ as
\[
\LRc{\ui{0}, \ui{1}, \hdots, \ui{N_t}}.
\]
which are then down-sampled on $\mc{T}$ for training $\NN{\ub}$. Doing so
has proved to yield more accurate predictions than training directly
on the solutions on $\mc{T}$ \cite{pathak2020using, kochkov2021machine,
  zhuang2021learned}. This is not surprising as the down-sampled
training data on $\mc{T}$ is more accurate than the solution on
$\mc{T}$. 

The next idea that we like to incorporate into our approach is sequential training. The key is to feed the predictions back to the neural network model to enable a better long-time predictive capability. Using this idea \cite{wu2022learning} deploys a mixture of graph neural network and 3D-U-Net neural network  to model fluid flows. Similarly, in 
\cite{zhuang2021learned} %the predictions are fed back to neural network for subsequent 10 time steps and then compare with ground-truth as part of  the loss functions. A neural network 
sequential learning is used to train a network to obtain the optimal finite difference coefficients from the high-resolution training data. %The convolution network is embedded inside an end-to-end differentiable program in a form of hard constraints to ensure physic consistency requirement. 
In the context of atmosphere modeling, \cite{brenowitz2018prognostic} introduces a stable and highly accurate long-time prediction loss function with sequential training.
%in which a predicted sate is the input for the next step and fed back for a number of time steps. 
Following \cite{pathak2020using, kochkov2021machine,wu2022learning,
  zhuang2021learned}, we partition the training data in $N_t -S$ overlapping subsets %\textcolor{blue}{(note that it is not the batch of training samples implemented in SGD algorithm).}
\[
\mc{U} := \LRc{\LRp{\ui{0}, \ui{1}, \hdots, \ui{S+1}}, \LRp{\ui{1}, \ui{2}, \hdots, \ui{S+2}}, \hdots, \LRp{\ui{{N_t-S-1}}, \ui{{N_t-S}}, \hdots, \ui{N_t}}}.
\]
For convenience in the exposition, we enumerate these $N_t -S$ subsets as
\begin{multline*}
      \mc{U} := \big \{ \LRp{\ui{{0,0}}, \ui{{0,1}}, \hdots, \ui{{0,S+1}}}, \LRp{\ui{{1,0}}, \ui{{1,1}}, \hdots, \ui{{1, S+1}}},  \\ 
     \hdots, \LRp{\ui{{N_t-S-1, 0}}, \ui{{N_t-S-1, 1}}, \hdots, \ui{{N_t-S-1, S+1}}} \big \},
\end{multline*}
where the second superscript denotes the local index in each subset. To distinguish from the sequential model-constrained learning in \cref{sect:sMC}, let us call the machine learning approach based on these overlapping subsets as {\em sequential data learning}. 

We next discuss how we use each subset in our model-constrained
approach. Consider the $k$th subset $\LRp{\ui{{k,0}}, \ui{{k,1}},
  \hdots, \ui{{k,S+1}}}$, for $k = 0, \hdots, N_t-S-1$.  Starting from
$\ut{k,0} = \ui{k,0}$, we can write the sequence of approximate
solutions $\LRc{\ut{k,i}}_{i=1}^{S+1}$ for \eqnref{MoLines} using
forward Euler time discretization with the neural network tangent
$\NN{\ub}$ as
\begin{equation}
    \eqnlab{ml_pred}
    \ut{k,i+1} = \ut{k,i} + \dt \NN{\ut{k,i}}, \quad i = 0, \hdots, S. 
\end{equation}
On the other hand, if we feed $\ut{k,i}$ through the forward Euler discretization \eqnref{ubFE}  we obtain
\begin{equation}
    \eqnlab{u_mc_pred}
    \ubar{k,i+1} = \ut{k,i} + \dt \, \F \LRp{\ut{k,i}}, \quad i = 0, \hdots, S.
\end{equation}
As can be seen $\ut{k,i+1} \ne \ubar{k,i+1}$, though we wish they are the same.
{\em If they were, the approximate solutions using neural network tangent would respect the governing discretized equation exactly}.  
  Obviously, this is not feasible in general.
%   Note that $\ubar{k,1} = \ui{k,1}$ for the noise-free data, but it is not the case of corrupted data. 
  Thus, we resort to requiring $\ut{k,i+1} $ as close as possible to $ \ubar{k,i+1}$.
One way to accomplish this is to consider the following loss function
for the $k$th batch:
\begin{equation}
\eqnlab{loss_1}
    \mc{J}_k := \quad \frac{1}{S+1} \sum_{i=1}^{S+1} \LRp{ \MSE{\ui{k,i} - \ut{k,i}} + \alpha \MSE{ \ubar{k,i} - \ut{k,i}}},
\end{equation}
where $\alpha$ is a model-constrained penalty (or regularization) parameter{, which controls the magnitude of the model-constrained loss (relative to the machine learning loss and). Parameter tuning in \cref{sect:Comp_cost} shows that a single value $\alpha = 10^5$ works well for all numerical examples in \cref{sect:numerics}}. The first term of the loss \eqnref{loss_1}\textemdash the ML Loss in \cref{fig:NN_architecture}\textemdash ensures the data
consistency, while the second term\textemdash the MC Loss in \cref{fig:NN_architecture}\textemdash is to force approximate solutions
of \eqnref{MoLines} using neural network tangent $\NN{\ub}$ to best fit the underlying space-time discretization \eqnref{ubFE}. The schematic of the \texttt{mcTangent} architecture with sequential data learning for the $k$th data subset and $S = 1$ is illustrated in \cref{fig:NN_architecture}. {We note that, unlike SINDy \cite{brunton2016discovering}, which discovers the dynamic systems from a dictionary of common differential operators, \texttt{mcTangent} aims to approximate high dimensional complex nonlinear tangent slope $\F$ operator using neural network.} %We then use the learned network for accelerating the forward solver and predict far beyond training time horizon. As we will see, using $\F$ as a model-constrained (hard-constraint) for training the neural network $\Psi$ enhances significantly the accuracy of the learned network.}

\begin{remark}
 Note that it is not essential that
$\ubar{i}$ must be obtained by the  forward Euler scheme \eqnref{ubFE}. In fact, our approach is flexible in the sense that any one-step explicit scheme, denoted as $\mc{F}$, (including explicit Runge-Kutta) is admissible. In such a case, our neural network can be considered as learning the forward Euler approximation of the ground-truth scheme.
\end{remark}

% {If my understanding, and hence the new indices, is correct, I don't get why Figure \figref{NN_architecture} corresponds to $S = 1$?}
% \textcolor{blue}{If we start $\ut{k,0} = \ui{k,0}$, we need to skip one first model-constrained term since $\ubar{k,1}$ is the same as true $\ui{k,1}$. Hence, both loss terms are the same in the case of $i = 1$ in \eqnref{loss_1}. I picked S = 1 to show case compute $\ut{k,2}$ from $\ut{k,1}$, S= 0 $\ut{k,1}$ from $\ui{k,0}$, and so on. Should I rearrange based on your new indices, or modify back to my old indices?}

\def\layersep{1.7cm}
\def\nodeinlayersep{0.8cm}
\usetikzlibrary{calc}

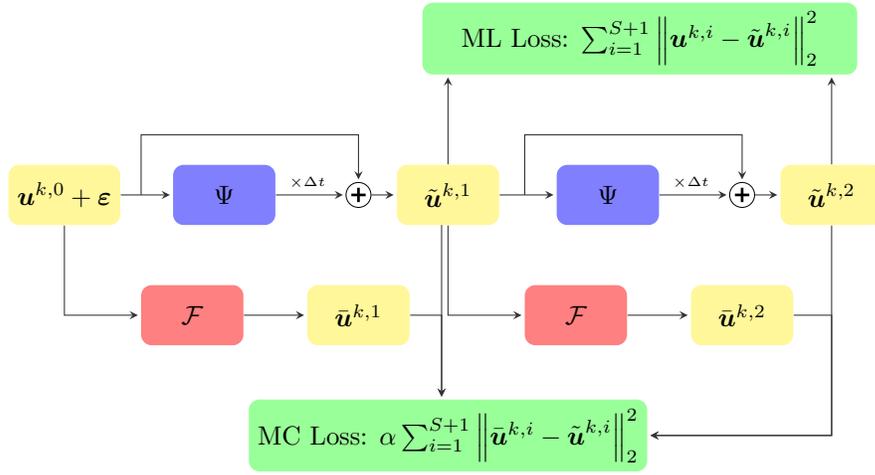
\begin{figure}[htb!]
\centering
\begin{tikzpicture}[
    node distance=\layersep,
    edge/.style={-stealth,shorten >=1pt, draw=black!80,thin},
    neuron/.style={circle,fill=black!25,minimum size=10pt,inner sep=0pt},
    operator/.style={rectangle,fill=green!,minimum height= \nodeinlayersep, minimum width= 0.8 * \layersep, inner sep=0pt, rounded corners},
    input neuron/.style={neuron, fill=green!50,minimum size=12pt},
    output neuron/.style={neuron, fill=green!50,minimum size=12pt},
    hidden neuron/.style={neuron, fill=blue!50},
    Forward map/.style={operator, fill=red!50},
    annot/.style={text width=4em, text centered},
    every node/.style={scale=1.0},
    node1/.style={scale=2.0},
    cross/.style={path picture={ \draw[black, shorten <=2pt, shorten >=2pt, line width=1pt] (path picture bounding box.south) -- (path picture bounding box.north); \draw[black, shorten <=2pt, shorten >=2pt, line width=1pt] (path picture bounding box.west) -- (path picture bounding box.east);}}
]

\node[rectangle,fill=yellow!50,minimum height= \nodeinlayersep, minimum width= 0.8 * \layersep, rounded corners] (u_in) at (0*\layersep, -2*\nodeinlayersep) {$\ui{{k,0}} + \epsb$};
\node[rectangle,fill=yellow!50,minimum height= \nodeinlayersep, minimum width= 0.8 * \layersep, rounded corners] (u_ml_1) at (3.0*\layersep, -2*\nodeinlayersep) {$\ut{{k,1}}$};
\node[rectangle,fill=yellow!50,minimum height= \nodeinlayersep, minimum width= 0.8 * \layersep, rounded corners] (u_ml_2) at (6.0*\layersep, -2*\nodeinlayersep) {$\ut{{k,2}}$};
% \node[rectangle,fill=yellow!500,minimum height= \nodeinlayersep, minimum width= 0.5 * \layersep, rounded corners] (u_ml_3) at (7.*\layersep, -2*\nodeinlayersep) {$\hdots$};

\node[rectangle,fill=blue!50,minimum height= \nodeinlayersep, minimum width= 0.8 * \layersep, rounded corners] (NN1) at (1.25*\layersep, -2*\nodeinlayersep) {$\Psi$};
\node[rectangle,fill=blue!50,minimum height= \nodeinlayersep, minimum width= 0.8 * \layersep, rounded corners] (NN2) at (4.25*\layersep, -2*\nodeinlayersep) {$\Psi$};

\node[rectangle,minimum height= \nodeinlayersep, minimum width= 0.8 * \layersep, rounded corners] (dt) at (1.9*\layersep, -1.8*\nodeinlayersep) {\tiny $ \times \dt$};

\node[rectangle,minimum height= \nodeinlayersep, minimum width= 0.8 * \layersep, rounded corners] (dt) at (4.9*\layersep, -1.8*\nodeinlayersep) {\tiny $ \times \dt$};

\node [draw,circle,cross,minimum width=.02*\layersep,line width=.3pt](cross_1) at (2.3*\layersep, -2*\nodeinlayersep){}; 
\node [draw,circle,cross,minimum width=.1*\layersep,line width=.3pt](cross_2) at (5.3*\layersep, -2*\nodeinlayersep){};
% \node [draw,circle,cross,minimum width=.1*\layersep,line width=.3pt](cross_3) at (5.2*\layersep, -4*\nodeinlayersep){};

\node[rectangle,fill=red!50,minimum height= \nodeinlayersep, minimum width= 0.8 * \layersep, rounded corners] (forward_1) at (1.0*\layersep, -4*\nodeinlayersep) {$\mc{F}$};
\node[rectangle,fill=yellow!50,minimum height= \nodeinlayersep, minimum width= 0.8 * \layersep, rounded corners] (u_mc_1) at (2.3*\layersep, -4*\nodeinlayersep) {$\ubar{{k,1}}$};

\node[rectangle,fill=red!50,minimum height= \nodeinlayersep, minimum width= 0.8 * \layersep, rounded corners] (forward) at (4.0*\layersep, -4*\nodeinlayersep) {$\mc{F}$};
\node[rectangle,fill=yellow!50,minimum height= \nodeinlayersep, minimum width= 0.8 * \layersep, rounded corners] (u_mc) at (5.3*\layersep, -4*\nodeinlayersep) {$\ubar{{k,2}}$};

% \node[rectangle,fill=yellow!50,minimum height= \nodeinlayersep, minimum width= 0.8 * \layersep, rounded corners] (u_mc_2) at (8.0*\layersep, -4*\nodeinlayersep) {$\ubar{{j+2}}$};

\draw[edge,thin] (u_in) -- (NN1);
\draw[edge,thin] (NN1) -- (cross_1);
\draw[edge,thin] (cross_1) -- (u_ml_1);
\draw[edge,thin] (u_ml_1) -- (NN2);
\draw[edge,thin] (NN2) -- (cross_2);
\draw[edge,thin] (cross_2) -- (u_ml_2);
% \draw[edge,thin] (u_ml_2) -- (u_ml_3);
% \draw[edge,thin] (forward) -- (cross_3);
% \draw[edge,thin] (cross_3) -- (u_mc);

\draw[edge,thin] (.6*\layersep, -2*\nodeinlayersep) -- (.6*\layersep, -1.*\nodeinlayersep) -- (2.3*\layersep, -1.*\nodeinlayersep) -- (cross_1);

\draw[edge,thin] (3.6*\layersep, -2*\nodeinlayersep) -- (3.6*\layersep, -1.*\nodeinlayersep) -- (5.3*\layersep, -1.*\nodeinlayersep) -- (cross_2);

\draw[edge,thin] (u_in) -- (0.0*\layersep, -4.*\nodeinlayersep) -- (forward_1);
\draw[edge,thin] (forward_1) -- (u_mc_1);

\draw[edge,thin] (u_ml_1) -- (3.0*\layersep, -4.*\nodeinlayersep) -- (forward);
\draw[edge,thin] (forward) -- (u_mc);

% LOSS ML TERMS
\node[rectangle,fill=green!40,minimum height= \nodeinlayersep, minimum width= 3.4 * \layersep, rounded corners] (loss_ml) at (4.5*\layersep, .6*\nodeinlayersep) {ML Loss: $\sum_{i = 1}^{S+1} \MSE{\ui{{k,i}} - \ut{{k,i}}}$};

\draw[edge,thin] (u_ml_1) -- (3.0*\layersep, .5*\nodeinlayersep - 0.5*\nodeinlayersep);
\draw[edge,thin] (u_ml_2) -- (6.0*\layersep, .5*\nodeinlayersep - 0.5*\nodeinlayersep);

% LOSS MC TERMS
\node[rectangle,fill=green!40,minimum height= \nodeinlayersep, minimum width= 1.6 * \layersep, rounded corners] (loss_mc) at (3.0*\layersep, -6.0*\nodeinlayersep) {MC Loss: $\alpha \sum_{i = 1}^{S+1} \MSE{ \ubar{{k,i}} - \ut{{k,i}}}$};
% \draw[edge,thin] (u_ml_1) -- (3.0*\layersep, .5*\nodeinlayersep - 0.5*\nodeinlayersep);

\draw[edge,thin] (2.95*\layersep, -2.5*\nodeinlayersep) -- (2.95*\layersep, -5.4*\nodeinlayersep);
\draw[edge,thin] (u_mc_1) -- (2.95*\layersep, -4.0*\nodeinlayersep) -- (2.95*\layersep, -5.4*\nodeinlayersep);

\draw[edge,thin] (u_ml_2) -- (6.0*\layersep, -6.0*\nodeinlayersep) -- (loss_mc);
\draw[edge,thin] (u_mc) -- (6.0*\layersep, -4*\nodeinlayersep) -- (6.0*\layersep, -6.0*\nodeinlayersep) -- (loss_mc);

\end{tikzpicture}
\caption{The schematic of the \texttt{mcTangent} approach with sequential data learning with $S = 1$. For the data randomization approach in Section \secref{noise_data_sec}, the random noise vector, $\epsb$, is added to the first input of the neural network.}
\figlab{NN_architecture}
\end{figure}

Taking all the batches into account  yields the total loss function
\begin{equation}
    \eqnlab{loss_sum}
    \mc{J} :=  \frac{1}{\LRp{N_t - S}\LRp{S+1}} \sum_{k=0}^{N_t - S-1} \sum_{i=1}^{S+1} \LRp{  \MSE{\ui{k,i} - \ut{k,i}} + \alpha \MSE{ \ubar{k,i} - \ut{k,i}}}.
\end{equation}

To gain insight into our \texttt{mcTangent} approach, we consider a linear problem in which $\F\LRp{\ub} = \F \ub$, and a one-layer linear neural network $\NN{\ui{k,0}} =  \W \ui{k,0} + \bb$. Under a mild condition, our approach should exactly recover the underlying tangent slope, i.e. $\NN{\ub} = \F\ub$. %Indeed, let $S = 0$ and noise-free data is used, $\epsb = \mb{0}$. 
Indeed, let $S = 0$ so that the loss function \eqnref{loss_sum} becomes
\begin{equation}
    \eqnlab{Linear_1step}
    \begin{aligned}
         \mc{J} = & \frac{1+\alpha}{N_t} \sum_{k=0}^{N_t - 1} \MSE{\ui{k,1} - \ut{k,1}} = \frac{1+\alpha}{N_t} \nor{{U}^{1} - \Tilde{{U}}^{1}}_{F}^2 \\
        = & \frac{\LRp{1+\alpha}\dt^2}{N_t} \nor{\F {U}^{0} - \LRp{\W {U}^{0} + \bb \One^T}}_{F}^2
    \end{aligned}
\end{equation}
where ${U}^{t_i}$ and $\Tilde{{U}}^{t_i}$ are matrices with true and predictive solutions as columns, respectively, and $\One$ is the unit column vector.

\begin{lemma}
\lemlab{linearOptimality}
The optimal solution $\LRp{\W^*, \bb^*}$ for the training problem
\[
\min_{\W,\bb}\mc{J}
\]
 is given by
\begin{equation}
    \begin{aligned}
        \eqnlab{optimal_linear_NN}
        \W^* =  \F \Pbar \Pbar^{\dagger},\quad 
        \bb^* =  \F \LRp{\Ib - \Pbar \Pbar^{\dagger}} \ubbar, 
    \end{aligned}
\end{equation}
where $\pbbar := \frac{1}{N_t} U^{t_0} \One$ is the column-average of matrix ${U}^{t_0}$, $\Pbar := U^{t_0} - \pbbar\One^T$, and $^\dagger$ denotes the pseudo-inverse. Consequently, the optimal network reads
\[
\NN{\ub} = \F \Pbar \Pbar^{\dagger} \ub +  \F \LRp{\Ib - \Pbar \Pbar^{\dagger}} \ubbar.
\]
\end{lemma}

\begin{remark}
 Lemma \lemref{linearOptimality} tells us that the optimal network exactly recovers the true forward map $\F$ if $\Pbar$ is a full row rank matrix. (In that case, $ \Pbar \Pbar^{\dagger} = \Ib$.) This holds, for example, when the number of independent data samples is equal to the discretized dimension. We would like to point out that the MC loss term  is the same as the ML loss term (up to a constant), and thus does not provide any extra information in this simple case.
When $S > 0$, at the time of writing this paper, we are not able to find a closed form solution as in Lemma \lemref{linearOptimality}. We leave it as future work.
\end{remark}

\begin{remark}
Although we learn the tangent slope using the Forward Euler scheme, it is straightforward to use any explicit scheme, such as Adams–Bashforth and Runge-Kutta methods to accomplish our goal. For example (ignoring extra subscripts for sequential data learning for simplicity), using the two-step Adams-Bashforth scheme, \texttt{mcTangent} solutions read
\[
    \ut{i+1} = \ut{i} + \frac{3}{2}\, \dt \, \NN{\ut{i}} - \frac{1}{2}\, \dt \, \NN{\ut{i-1}},
\]
as oppose to the solutions using the truth tangent slope
\[
 \ui{i+1} = \ui{i} + \frac{3}{2}\, \dt \, \F\LRp{\ui{i}} - \frac{1}{2}\, \dt \, \F\LRp{\ui{i-1}}.
\]
Similarly,  \texttt{mcTangent} solutions based on the second-order Runge-Kutta scheme reads
\[
    \begin{aligned}
        \tilde{k}_1 & = \dt \, \NN{\ut{i}} \\
        \tilde{k}_2 & = \dt \, \NN{\ut{i} + k_1} \\
        \ut{i+1} & = \ut{i} + \frac{\tilde{k}_1 + \tilde{k}_2}{2},
    \end{aligned}
\]
as oppose to the solutions using the truth tangent slope
\[
 \begin{aligned}
        k_1 & = \dt \, \F\LRp{\ui{i}} \\
        k_2 & = \dt \, \F\LRp{\ui{i} + k_1} \\
        \ui{i+1} & = \ui{i} + \frac{k_1 + k_2}{2}.
    \end{aligned}
\]
Clearly, we have to modify the lost function accordingly, but the idea is the same as forward Euler approach that we have presented above.
\end{remark}

\begin{remark}
% ml_pred
% u_mc_pred
Note that we have used forward Euler time discretization for both \cref{eq:ml_pred} and \cref{eq:u_mc_pred} for simplicity, but this is not necessary. We recommend to use time discretizations with the same order of accuracy for both as  accuracy gain in incompatible discretizations may not be well paid-off by additional computational demand. For example, if we use low-order accuracy for \cref{eq:ml_pred} but higher-order accuracy for \cref{eq:u_mc_pred}, \texttt{mcTangent} solution could be more accurate with smaller constant in the order of accuracy (still low-order) since  it tries to match more accurate solution from \cref{eq:u_mc_pred}. However, the training cost could increase significantly due to several evaluations (and hence differentiations for back-propagation) of the truth tangent slope $
\F$ in \cref{eq:u_mc_pred}. Clearly high-order accurate approaches could tax the training time significantly. 
\end{remark}

\subsection{Model-constrained neural network approach both sequential data and sequential model learnings}
\seclab{sMC}
%{We need references on similar sequential ideas in this section} \textcolor{blue}{this idea is from the sequential training \cite{wu2022learning} and \cite{zhuang2021learned} (I listed in the introductions), where the predictions are fed back to neural network in rollout format. The $S$ value in section 2.2 is equivalent to sequential training idea. I have not seen papers that use the direct forward models like us, so there is no reference using sequential model-constrained idea, I believe. I invent this approach to force the network to obey the forward scheme more strongly, multiple time steps. I think we should point out the sequential training in the section 2.2 and then by the same inspiration, we invent the sequential model-constrained there. Also, inspiration from \eqnref{errglomctaylor}, we propose the sequential model-constrained}
%{by S you mean R?}
%ans{Yes. From S idea, I proposed R idea. S and R are different variables. S and R have the same idea of sequential training/feeding prediction back as the input for next step, for each S we have R sequential model-constrained, equation 2.10 is what I meant}

In \cref{sect:mc}, we present a sequential data learning approach for the proposed model-constrained neural network $\NN{\ub}$ to learn the tangent slope while being constrained to provide the best possible approximate solutions for \eqnref{ubFE} for each time step. In order to improve the long-time predictive capability and accuracy,
this section constructs, in addition to {\em sequential data learning}, a {\em sequential model learning} strategy for training the neural network $\NN{\ub}$ is proposed. Sequential model learning is designed to  promote the neural network solutions to respect the underlying discretization scheme for multiple time steps concurrently. In particular, starting from $\ut{k,i}$ we can carry out $R$ steps forward in time using the underlying discretization \eqnref{ubFE} as
\[
    \ubar{k,i, r} = \ubar{k,i, r-1} + \dt \, \F \LRp{\ubar{k,i,r-1}}, \quad r = 1, \hdots, R,
\]
and using the neural network approximation \eqnref{ml_pred} as
\[
    \ut{k,i, r} = \ut{k,i, r-1} + \dt \, \NN{\ut{k,i,r-1}}, \quad r = 1, \hdots, R,
\]
where $\ubar{k,i, 0} = \ut{k,i, 0}= \ut{k,i}$. Here the third superscript $r$ has been introduced to keep track of $R$ sequential forward steps starting from $\ut{k,i}$ for both exact and neural network tangent slopes. In order to ensure that these corresponding $R$ sequential predictions closely match each other, we consider the following loss function
\begin{equation}
    \eqnlab{loss_sum_seq_mc}
    \mc{J} :=  \frac{1}{\LRp{N_t - S} \LRp{S+1}} \sum_{k=0}^{N_t - S-1} \sum_{i=1}^{S+1}  \LRp{\MSE{\ui{k,i} - \ut{k,i}} +  \frac{\alpha}{R} \sum_{r = 1}^{R} \MSE{ \ubar{k,i,r} - \ut{k,i,r}}}.
\end{equation}

\def\layersep{1.3cm}
\def\nodeinlayersep{.5cm}
\usetikzlibrary{calc}
{\tiny
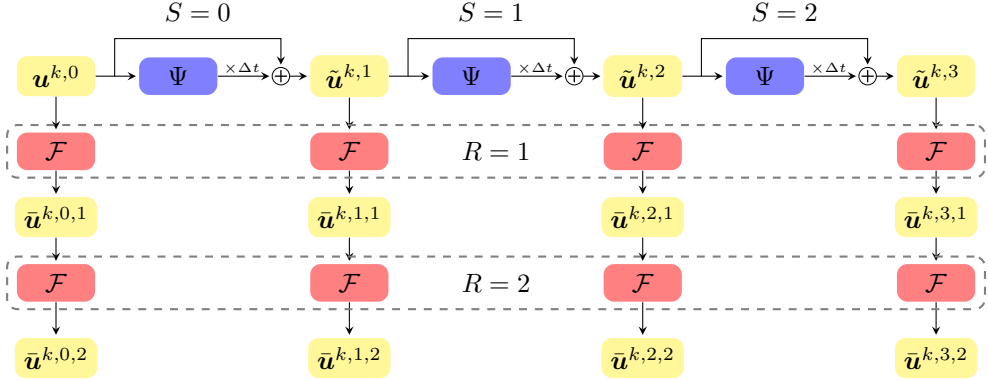
\begin{figure}[htb!]
\centering
\begin{tikzpicture}[
    node distance=\layersep,
    edge/.style={-stealth,shorten >=1pt, draw=black!100,thin},
    neuron/.style={circle,fill=black!25,minimum size=10pt,inner sep=0pt},
    operator/.style={rectangle,fill=green!,minimum height= \nodeinlayersep, minimum width= 0.8 * \layersep, inner sep=0pt, rounded corners},
    input neuron/.style={neuron, fill=green!50,minimum size=12pt},
    output neuron/.style={neuron, fill=green!50,minimum size=12pt},
    hidden neuron/.style={neuron, fill=blue!50},
    Forward map/.style={operator, fill=red!50},
    annot/.style={text width=4em, text centered},
    every node/.style={scale=1},
    node1/.style={scale=2.0},
    cross/.style={path picture={ \draw[black, shorten <=1pt, shorten >=1pt, line width=.5pt] (path picture bounding box.south) -- (path picture bounding box.north); \draw[black, shorten <=1pt, shorten >=1pt, line width=.5pt] (path picture bounding box.west) -- (path picture bounding box.east);}}
]

\node[rectangle,fill=yellow!50,minimum height= \nodeinlayersep, minimum width= 0.8 * \layersep, rounded corners] (u_in) at (0*\layersep, -2*\nodeinlayersep) {$\ui{{k,0}}$};
\node[rectangle,fill=yellow!50,minimum height= \nodeinlayersep, minimum width= 0.8 * \layersep, rounded corners] (u_ml_1) at (3.0*\layersep, -2*\nodeinlayersep) {$\ut{{k,1}}$};
\node[rectangle,fill=yellow!50,minimum height= \nodeinlayersep, minimum width= 0.8 * \layersep, rounded corners] (u_ml_2) at (6.0*\layersep, -2*\nodeinlayersep) {$\ut{{k,2}}$};
\node[rectangle,fill=yellow!50,minimum height= \nodeinlayersep, minimum width= 0.8 * \layersep, rounded corners] (u_ml_3) at (9.0*\layersep, -2*\nodeinlayersep) {$\ut{{k,3}}$};

\node[rectangle,fill=blue!50,minimum height= \nodeinlayersep, minimum width= 0.8 * \layersep, rounded corners] (NN1) at (1.25*\layersep, -2*\nodeinlayersep) {$\Psi$};
\node[rectangle,fill=blue!50,minimum height= \nodeinlayersep, minimum width= 0.8 * \layersep, rounded corners] (NN2) at (4.25*\layersep, -2*\nodeinlayersep) {$\Psi$};
\node[rectangle,fill=blue!50,minimum height= \nodeinlayersep, minimum width= 0.8 * \layersep, rounded corners] (NN3) at (7.25*\layersep, -2*\nodeinlayersep) {$\Psi$};

\node[rectangle,fill=blue!0,minimum height= \nodeinlayersep, minimum width= 0.8 * \layersep, rounded corners]  at (1.45*\layersep, -0.3*\nodeinlayersep) {$S=0$};
\node[rectangle,fill=blue!0,minimum height= \nodeinlayersep, minimum width= 0.8 * \layersep, rounded corners]  at (4.45*\layersep, -0.3*\nodeinlayersep) {$S = 1$};
\node[rectangle,fill=blue!0,minimum height= \nodeinlayersep, minimum width= 0.8 * \layersep, rounded corners]  at (7.45*\layersep, -0.3*\nodeinlayersep) {$S=2$};

\node[rectangle,minimum height= \nodeinlayersep, minimum width= 0.8 * \layersep, rounded corners] (dt) at (1.9*\layersep, -1.8*\nodeinlayersep) {\tiny $ \times \dt$};
\node[rectangle,minimum height= \nodeinlayersep, minimum width= 0.8 * \layersep, rounded corners] (dt) at (4.9*\layersep, -1.8*\nodeinlayersep) {\tiny $ \times \dt$};
\node[rectangle,minimum height= \nodeinlayersep, minimum width= 0.8 * \layersep, rounded corners] (dt) at (7.9*\layersep, -1.8*\nodeinlayersep) {\tiny $ \times \dt$};

\node[draw,circle,cross,minimum width=.01*\layersep,line width=.2pt, scale=0.7](cross_1) at (2.3*\layersep, -2*\nodeinlayersep){}; 
\node [draw,circle,cross,minimum width=.01*\layersep,line width=.2pt, scale=0.7](cross_2) at (5.3*\layersep, -2*\nodeinlayersep){};
\node [draw,circle,cross,minimum width=.01*\layersep,line width=.2pt, scale=0.7](cross_3) at (8.3*\layersep, -2*\nodeinlayersep){};

\draw[edge,thin] (u_in) -- (NN1);
\draw[edge,thin] (NN1) -- (cross_1);
\draw[edge,thin] (cross_1) -- (u_ml_1);
\draw[edge,thin] (u_ml_1) -- (NN2);
\draw[edge,thin] (NN2) -- (cross_2);
\draw[edge,thin] (cross_2) -- (u_ml_2);
\draw[edge,thin] (u_ml_2) -- (NN3);
\draw[edge,thin] (NN3) -- (cross_3);
\draw[edge,thin] (cross_3) -- (u_ml_3);
% \draw[edge,thin] (u_ml_3) -- (10*\layersep, -2.*\nodeinlayersep);

\draw[edge,thin] (.6*\layersep, -2*\nodeinlayersep) -- (.6*\layersep, -1.*\nodeinlayersep) -- (2.3*\layersep, -1.*\nodeinlayersep) -- (cross_1);
\draw[edge,thin] (3.6*\layersep, -2*\nodeinlayersep) -- (3.6*\layersep, -1.*\nodeinlayersep) -- (5.3*\layersep, -1.*\nodeinlayersep) -- (cross_2);
\draw[edge,thin] (6.6*\layersep, -2*\nodeinlayersep) -- (6.6*\layersep, -1.*\nodeinlayersep) -- (8.3*\layersep, -1.*\nodeinlayersep) -- (cross_3);

\node[rectangle,fill=red!50,minimum height= \nodeinlayersep, minimum width= 0.8 * \layersep, rounded corners] (forward_11) at (0*\layersep, -4.0*\nodeinlayersep) {$\mc{F}$};
\node[rectangle,fill=yellow!50,minimum height= \nodeinlayersep, minimum width= 0.8 * \layersep, rounded corners] (u_mc_11) at (0*\layersep, -5.75*\nodeinlayersep) {$\ubar{{k,0,1}}$};
\node[rectangle,fill=red!50,minimum height= \nodeinlayersep, minimum width= 0.8 * \layersep, rounded corners] (forward_12) at (0*\layersep, -7.5*\nodeinlayersep) {$\mc{F}$};
\node[rectangle,fill=yellow!50,minimum height= \nodeinlayersep, minimum width= 0.8 * \layersep, rounded corners] (u_mc_12) at (0*\layersep, -9.5*\nodeinlayersep) {$\ubar{{k,0,2}}$};

\node[rectangle,fill=red!50,minimum height= \nodeinlayersep, minimum width= 0.8 * \layersep, rounded corners] (forward_21) at (3*\layersep, -4.0*\nodeinlayersep) {$\mc{F}$};
\node[rectangle,fill=yellow!50,minimum height= \nodeinlayersep, minimum width= 0.8 * \layersep, rounded corners] (u_mc_21) at (3*\layersep, -5.75*\nodeinlayersep) {$\ubar{{k,1,1}}$};
\node[rectangle,fill=red!50,minimum height= \nodeinlayersep, minimum width= 0.8 * \layersep, rounded corners] (forward_22) at (3*\layersep, -7.5*\nodeinlayersep) {$\mc{F}$};
\node[rectangle,fill=yellow!50,minimum height= \nodeinlayersep, minimum width= 0.8 * \layersep, rounded corners] (u_mc_22) at (3*\layersep, -9.5*\nodeinlayersep) {$\ubar{{k,1,2}}$};

\node[rectangle,fill=red!50,minimum height= \nodeinlayersep, minimum width= 0.8 * \layersep, rounded corners] (forward_31) at (6*\layersep, -4.0*\nodeinlayersep) {$\mc{F}$};
\node[rectangle,fill=yellow!50,minimum height= \nodeinlayersep, minimum width= 0.8 * \layersep, rounded corners] (u_mc_31) at (6*\layersep, -5.75*\nodeinlayersep) {$\ubar{{k,2,1}}$};
\node[rectangle,fill=red!50,minimum height= \nodeinlayersep, minimum width= 0.8 * \layersep, rounded corners] (forward_32) at (6*\layersep, -7.5*\nodeinlayersep) {$\mc{F}$};
\node[rectangle,fill=yellow!50,minimum height= \nodeinlayersep, minimum width= 0.8 * \layersep, rounded corners] (u_mc_32) at (6*\layersep, -9.5*\nodeinlayersep) {$\ubar{{k,2,2}}$};

\node[rectangle,fill=red!50,minimum height= \nodeinlayersep, minimum width= 0.8 * \layersep, rounded corners] (forward_41) at (9*\layersep, -4.0*\nodeinlayersep) {$\mc{F}$};
\node[rectangle,fill=yellow!50,minimum height= \nodeinlayersep, minimum width= 0.8 * \layersep, rounded corners] (u_mc_41) at (9*\layersep, -5.75*\nodeinlayersep) {$\ubar{{k,3,1}}$};
\node[rectangle,fill=red!50,minimum height= \nodeinlayersep, minimum width= 0.8 * \layersep, rounded corners] (forward_42) at (9*\layersep, -7.5*\nodeinlayersep) {$\mc{F}$};
\node[rectangle,fill=yellow!50,minimum height= \nodeinlayersep, minimum width= 0.8 * \layersep, rounded corners] (u_mc_42) at (9*\layersep, -9.5*\nodeinlayersep) {$\ubar{{k,3,2}}$};

\draw[edge,thin] (u_in) -- (forward_11);
\draw[edge,thin] (forward_11) -- (u_mc_11);
\draw[edge,thin] (u_mc_11) -- (forward_12);
\draw[edge,thin] (forward_12) -- (u_mc_12);
% \draw[edge,thin] (u_mc_12) -- (0.*\layersep, -9.5*\nodeinlayersep);

\draw[edge,thin] (u_ml_1) -- (forward_21);
\draw[edge,thin] (forward_21) -- (u_mc_21);
\draw[edge,thin] (u_mc_21) -- (forward_22);
\draw[edge,thin] (forward_22) -- (u_mc_22);
% \draw[edge,thin] (u_mc_22) -- (3.*\layersep, -9.5*\nodeinlayersep);

\draw[edge,thin] (u_ml_2) -- (forward_31);
\draw[edge,thin] (forward_31) -- (u_mc_31);
\draw[edge,thin] (u_mc_31) -- (forward_32);
\draw[edge,thin] (forward_32) -- (u_mc_32);
% \draw[edge,thin] (u_mc_32) -- (6.*\layersep, -9.5*\nodeinlayersep);

\draw[edge,thin] (u_ml_3) -- (forward_41);
\draw[edge,thin] (forward_41) -- (u_mc_41);
\draw[edge,thin] (u_mc_41) -- (forward_42);
\draw[edge,thin] (forward_42) -- (u_mc_42);

% \node[rectangle,fill=red!0,minimum height= \nodeinlayersep, minimum width= 0.8 * \layersep, rounded corners]  at (4.5*\layersep, -4.0*\nodeinlayersep) {};
% \node[rectangle,fill=red!0,minimum height= \nodeinlayersep, minimum width= 0.8 * \layersep, rounded corners]  at (4.5*\layersep, -7.5*\nodeinlayersep) {$R = 2$};

\node[rectangle, draw = black!50, dashed, thick = 1pt ,minimum height= 1.4 * \nodeinlayersep, minimum width= 10 * \layersep, rounded corners] (block1) at (4.5*\layersep, -4.0*\nodeinlayersep) {$R = 1$};

\node[rectangle, draw = black!50, dashed, thick = 1pt ,minimum height= 1.4 * \nodeinlayersep, minimum width= 10 * \layersep, rounded corners] (block2) at (4.5*\layersep, -7.5*\nodeinlayersep) {$R = 2$};
    
\end{tikzpicture}
\caption{The schematic of the \texttt{mcTangent} approach with both sequential data and model learnings with  $S = 2, R = 2$.
%\hai{ARRANGE IN THE PLACE YOU THINK PROPER, TAN}
}
\figlab{NN_mc_architecture}
\end{figure}
}

The schematic of the \texttt{mcTangent} architecture with both sequential data and model learnings for the $k$th data subset and $S = 2, R = 2$ is depicted in \cref{fig:NN_mc_architecture}. Clearly, when $R = 1$ we recover \eqnref{loss_sum} from \eqnref{loss_sum_seq_mc}. In other words, \eqnref{loss_sum_seq_mc} is a generalization of \eqnref{loss_sum}. 
{Intuitively, larger values for $R, S$ increase the predictive capacity of \texttt{mcTangent} solutions, and as an example this will be demonstrated for the Burgers equation in \secref{Burger_eq}.}
%We would like to point out that large values of $R, S$ up to a saturated level increase predictive capacity of learned network, as shown comprehensively in Burger's equations problem \secref{Burger_eq}
However, it is computationally expensive to use large values  for both $S$ and $R$. In the numerical results in \cref{sect:numerics}, we study two combinations: $S \geq 1, R = 1$ and $S = 1, R \geq 1$. In order to have a deeper understanding of the role of the loss function \eqnref{loss_sum_seq_mc} in training the  neural network tangent and its predictive capability, we shall provide an in-depth heuristic argument in \cref{sect:noise_data_sec} and a rigorous error estimation for \texttt{mcTangent} predictions in \cref{sect:error}.

\subsection{Data randomization}
\seclab{noise_data_sec}
%{We need references on similar randomized/noisified idea in this section. We did discuss this with Tan Nguyen a while ago and there was some reference on adding noise is equivalent to a regularization. We need to summarize some of these work here and then transition to our adoption of a similar approach  in this section. I will come back to this when you are done.} \textcolor{blue}{Here is the reference \cite{sanchez2020learning} they use noisy data with sequential training to improve accuracy. I will read more on adding noise technique papers that you sent me before. Here are a few references}

 It has been observed \cite{sanchez2020learning} that adding a small amount of noise to training data not only increases the generalization on unseen data but also reduces accumulated errors in predictions. In fact,  clean noise-free data does not represent the accumulated error in the predictive state that is fed back to the network to produce subsequent predictions. Moreover, noisy data encourages neural network predictions to be more robust to noise-corrupted inputs and errors.
In order to investigate the significance of different noise additions (adding noise to the training inputs, weights of the neural network, and output labels) on the model generalization, \cite{an1996effects} demonstrates that the reasonable noise level in the outputs does not influence the trained network. Randomizing training data, on the one hand, prevents the neural network from overfitting data, and on the other hand, can make the network insensitive to noise in data in the validation phase. 

It is well-known that randomization induces  a regularization of the gradient of the loss function with respect to the inputs \cite{reed1992regularization}. %The magnitude of this regularization is the variance of the additive noise.
Consequently, the neural network, if a proper noise level is used, is regularized to be a smooth function of the input data. The smoothness reduces the sensitivity to the variation in the input \cite{matsuoka1992noise} and can enhance the stability of long-time predictions \cite{poggio1990networks}.
The work in \cite{bishop1995training} showed that adding noise to data is equivalent to introducing a Tikhonov regularization to the loss function (where the regularization parameter is the noise variance) and thus improving the model generalization. However, the analysis is only valid in the context of infinite training data set, as pointed out in \cite{an1996effects}.

Inspired by the aforementioned work, we randomize the input data for the model-constrained network. We shall  show that randomization induces regularizations not only to  promote the smoothness of the network but also to enhance the similarity of the derivatives of the network $\NN{\ub}$ and the true tangent slope $\F\LRp{\ub}$. As shall be seen, the numerical results in \cref{sect:numerics} reveal that 
randomization improves significantly the long-term stability and accuracy. 

In this paper, we randomize the input $\ub$ of the network as
\begin{equation}
    \eqnlab{add_noise}
    \ubr = \ub + \epsb,
    % \delta \, \max \LRp{\ub} \,
\end{equation}
where $\epsb$ is a normal random vector $\epsb \sim \mc{N}\LRp{\bs{0}, \delta^2 \Ib}$. Note that the following heuristic arguments also hold for any random vector with independent components, each of which is a random variable with zero mean and variances $\delta^2$. Let $\mathbb{E}\LRs{\cdot}$ denote the expectation with respect to $\epsb$. Following \cite{an1996effects}, for a generic loss function $\mc{L}\LRp{\ub}$ we have

%In general, it is well-known that adding noise to training data enhances the generalization of neural network. This technique is extremely useful in our framework since, in the prediction phase, we deal with the unseen data (predictions out of training data period). In the following, we further show that corrupted noise coupled with the model-constrained approach regularizes the smoothness of network to the optimal/ground-truth smoothness.
%{\bf Proof the stability by noisy data}
%As presented in \cite{an1996effects} and \cite{bishop1995training}, by introducing noise to input data is equivalent to adding penalty terms to the loss function. We consider the training loss function $\mc{L}\LRp{\ub}$ is train with noise-free data. We assume the noise vector is drawn from a distribution which is satisfied two conditions: (1) components of vector noise are independent; (2) the noise distribution is symmetric about zero. Then, the expectation of the loss function with the noise $\epsb$ reads \cite{an1996effects}

\begin{equation}
\eqnlab{Expect_form}
\begin{aligned}
    \mathbb{E}\LRs{\mc{L}\LRp{\ubr}} = & \mc{L}\LRp{\ub} + \mathbb{E}\LRs{\eval{\pp{\mc{L}}{\ub}}_{\ub} \epsb} +  \half \mathbb{E}\LRs{ \epsb^T \eval{\pp{^2 \mc{L}}{\ub^2}}_{\ub} \epsb} 
    + o\LRp{\nor{\epsb}^2}
    \\
    \approx & \mc{L}\LRp{\ub} +  \half \mathbb{E}\LRs{ \epsb^T \eval{\pp{^2 \mc{L}}{\ub^2}}_{\ub} \epsb},
    \end{aligned}
\end{equation}
where we have used sufficient small noise $\epsb$ (relatively to $\ub$) so that the high-order term $o\LRp{\nor{\epsb}^2}$, using the standard ``small o" notation, is assumed to be negligible.
%and thus can be ignored.
%$\ubh$ is some vector on the line segment between $\ub$ and $\vb$. 
%There are two satisfied options of noise distribution including uniform distribution and Gaussian distribution. In this research, we use latter with variance $\delta^2$, $\epsb \sim \mc{N}\LRp{0, \delta^2 \Ib}$. 
We consider $S = 0$ and $R = 1$. (For $S > 0$ and/or $R > 1$ , the sequential inputs to the network contain the error which may not satisfy the condition for \eqnref{Expect_form} to hold.) In this case, the loss function \eqnref{loss_sum} becomes
\begin{equation}
    \eqnlab{loss_noise}
    \begin{aligned}
         \mc{J} = & \frac{1}{N_t} \sum_{k=0}^{N_t - 1} {\underbrace{\MSE{\ui{k,1} - \ut{k,1}}}_{\mc{L}_{\text{ML}} \LRp{\ui{k,0} + \epsb}} + \alpha \underbrace{\MSE{\ubar{k,1} - \ut{k,1}}}_{\mc{L}_{\text{MC}}\LRp{\ui{k,0} + \epsb} } } 
        %  \\ = & \frac{1}{N_t} \sum_{k=0}^{N_t - 1} \LRp{\MSE{\ui{k,1} - \LRp{\ui{k,0} + \epsb + \dt \NN{\ui{k,0} + \epsb}} } + \alpha \MSE{\ut{k,1} - \ubar{k,1}} }
    \end{aligned}
\end{equation}
We now study the randomized ML loss term $\mc{L}_{\text{ML}} \LRp{\ui{k,0} + \epsb}$
and the randomized MC loss term
$\mc{L}_{\text{MC}}\LRp{\ui{k,0} + \epsb}$
to gain insights into the role of randomization. 

The machine learning loss term reads
\[
    {\mc{L}_{\text{ML}}\LRp{\ui{k,0} + \epsb}} =  \MSE{\ui{k,1} - \LRp{\ui{k,0} + \epsb + \dt \NN{\ui{k,0} + \epsb}}}
\]
which is a function of true input $\ui{k,0}$ plus a random noise vector $\epsb$.  {\em It is important to note that we do not randomize the true data $\ui{k,1}$ against which we compare the machine prediction $\ut{k,1}$}.
%It is noteworthy that we are not allowed to substitute $\ui{k,1} = \ui{k,0} + \dt \F \LRp{\ui{k,0}}$ since, during training process, we use $\ui{k,1}$ as desired outputs of network directly. Then, from \eqnref{Expect_form}, we have
Replacing $\mc{L}$ by $\mc{L}_{\text{ML}}$ in \eqnref{Expect_form} yields
\begin{equation}
\eqnlab{MLlossEx}
    \begin{aligned}
        \mathbb{E}\LRs{\mc{L}_{\text{ML}}\LRp{\ui{k,0} + \epsb}} \approx & \underbrace{\MSE{\ui{k,1} - \LRp{\ui{k,0} + \dt \NN{\ui{k,0}}}}}_{\mc{L}_{\text{ML}}\LRp{\ui{k,0}}} + 
%        \MSE{\ui{k,1} - \LRp{\ui{k,0} + \dt \NN{\ui{k,0}}}} +
\delta^2 \LRs{\mc{P}_1 \LRp{\ui{k,0}} + \mc{P}_2 \LRp{\ui{k,0}}},
    \end{aligned}
\end{equation}
where
\begin{equation}
    \eqnlab{P1}
    \mc{P}_1\LRp{\ui{k,0}} = \mb{Tr}\LRs{ \LRp{\Ib + \dt \eval{ \pp{\Psi}{\ub}}_{\ui{k,0}}}^T \LRp{\Ib + \dt \eval{ \pp{\Psi}{\ub}}_{\ui{k,0}}} },
\end{equation}
with $\mb{Tr}\LRs{\cdot}$ as the trace operator, and
\begin{equation}
    \begin{aligned}
        \mc{P}_2 \LRp{\ui{k,0}} = & \mb{Tr} \LRs{\dt {\eval{\frac{\partial^2 \Psi}{\partial \ub^2}}_{\ui{k,0}}}\odot \LRs{\LRp{\ui{k,0} + \dt \NN{\ui{k,0}}} - \ui{k,1}}},
        % = & 2 \mb{Tr} \LRs{\LRs{\eval{\frac{\partial^2 \Psi}{\partial \ub^2}}_{\ui{k,0}}}^T {\mc{L}_{\text{ML}}\LRp{\ui{k,0}}}}
    \end{aligned}
\end{equation}
where $\odot$ denotes the dot product of the third order tensor $\dt \eval{\frac{\partial^2 \Psi}{\partial \ub^2}}_{\ui{k,0}}$ and the vector $ \LRs{\LRp{\ui{k,0} + \dt \NN{\ui{k,0}}} - \ui{k,1}}$.

From \eqnref{MLlossEx}, three observations are in order. First, on average, the randomized ML loss term is approximately the original ML loss term plus two additional terms $\mc{P}_1$ and $\mc{P}_2$ scaled by the variance $\delta^2$ of the noise. Second, the first term $\mc{P}_1$ is positive and thus is a regularization. It enforces the boundedness of the gradient (and hence the smoothness) of the neural network. Third, the second term $\mc{P}_2$ can be either positive or negative. However, when the time step $\dt$ is small and/or the  ML misfit term $\LRs{\LRp{\ui{k,0} + \dt \NN{\ui{k,0}}} - \ui{k,1}}$ is  small (e.g. with sufficient training), the contribution of the second term is expected to be dominated by the first and thus is negligible. When neither of these two conditions is satisfied, if the training enforces small ``curvature" of the neural network (i.e. small $\eval{\frac{\partial^2 \Psi}{\partial \ub^2}}_{\ui{k,0}}$) then the second term is also negligible. When this happens, training with randomization provides extra smoothness to the network. 

%This can be achieved, for example, by choosing smooth activation function.  

%The penalty term regularizes the smoothness of the neural network function, and thus enhance the model generalization. In the meantime, the term $\mc{P}_1$ reduces the bound for the global error through limiting a reasonable gradient of network w.r.t. the input. The second penalty term is 

%which depends on both the residue machine learning misfit and the second-order of network. The effect of $\mc{P}_2$ on enhancing model generalization is guaranteed since the trace of product of two terms can be either positive or negative. However, in general, whenever the residue misfit converges to small value, the effect of $\mc{P}_2$ is dominated by $\mc{P}_1$, \cite{an1996effects}.

Next, from \eqnref{ml_pred} and \eqnref{u_mc_pred}, { the randomized MC loss term} can be written as
\begin{equation*}
    \begin{aligned}
        \mc{L}_{\text{MC}}\LRp{\ui{k,0} + \epsb} = \MSE{\ubar{k,1} - \ut{k,1}} 
        = \dt^2 \MSE{ \F\LRp{\ui{k,0}+ \epsb} - \NN{\ui{k,0}+ \epsb}}.
    \end{aligned}
\end{equation*}
%Similar to \eqnref{loss_noise}, applying the expectation operator, the model constrained loss with noise $\epsb$ reads
Applying \eqnref{Expect_form} with $\mc{L}_{\text{MC}}$ in place of $\mc{L}$ gives
\begin{equation}
\eqnlab{MClossEx}
    \begin{aligned}
        \mathbb{E}\LRs{\mc{L}_{\text{MC}}\LRp{\ui{k,0} + \epsb}} \approx \underbrace{\dt^2 \MSE{ \F\LRp{\ui{k,0}} - \NN{\ui{k,0}}}}_{\mc{L}_{\text{MC}}\LRp{\ui{k,0}}} + \delta^2 \LRs{\mc{Q}_1\LRp{\ui{k,0}} + \mc{Q}_2\LRp{\ui{k,0}}},
    \end{aligned}
\end{equation}
%The first penalty term is
where
\begin{equation}
    \mc{Q}_1 \LRp{\ui{k,0}} = \dt^2 \mb{Tr}\LRs{\LRp{\eval{\pp{\F}{\ub}}_{\ui{k,0}} - \eval{\pp{\Psi}{\ub}}_{\ui{k,0}}}^T \LRp{\eval{\pp{\F}{\ub}}_{\ui{k,0}} - \eval{\pp{\Psi}{\ub}}_{\ui{k,0}}} },
\end{equation}
%It can be seen that this term forces the of network to be not arbitrarily smooth, but closed to the smoothness of the well-posed forward map, the optimal smoothness. Therefore, $\mc{Q}_1$ has stronger benefits than $\mc{P}_1$, and thus improve generalization more effectively. Meanwhile, the second term by far more complicated
and
\begin{equation}
    \begin{aligned}
        \mc{Q}_2\LRp{\ui{k,0}} = & \mb{Tr} \LRs{ \dt \LRp{\eval{\frac{\partial^2 \F}{\partial \ub^2}}_{\ui{k,0}} - \eval{\frac{\partial^2 \Psi}{\partial \ub^2}}_{\ui{k,0}}}\odot \dt \LRp{\NN{\ui{k,0}} - \F\LRp{\ui{k,0}}} }.
        % = & 2 \mb{Tr} \LRs{ \dt \LRp{\eval{\frac{\partial^2 \F}{\partial \ub^2}}_{\ui{k,0}} - \eval{\frac{\partial^2 \Psi}{\partial \ub^2}}_{\ui{k,0}}}^T {\mc{L}_{\text{MC}}\LRp{\ui{k,0}}} },
    \end{aligned}
\end{equation}
As can be seen, the randomized MC loss term is approximately a sum of the original ML loss term and two additional terms. The first term $\mc{Q}_1$ is non-negative and behaves like a regularization to enforce the likeliness of the derivatives with respect to $\ub$ of the neural network tangent $\NN{\ub}$ and the true tangent $\F\LRp{\ub}$. The second term, though could be either negative or positive, can be negligible with sufficient training so that the MC misfit $\dt\LRp{\NN{\ui{k,0}} - \F\LRp{\ui{k,0}}}$ is relatively small.  Another possibility for the insignificance of the second term is when the difference in the ``curvature" of the neural network tangent  and the true tangent is sufficiently small. In that case, training with randomization promotes the closeness of not only  $\NN{\ub}$ and $\F\LRp{\ub}$ but their first and second derivatives with respect to $\ub$: {\em confirming the significant advantages obtained from data randomization}.
% which is the combination of model-constrained minimization term and the regularization of the curvature of the network to the one of the true map. Although both terms are benefits for the accuracy by themselves, the trace term is either positive or negative, and thus not always useful. However, in general, at the converged point in which the residual model-constrained term is small, the role of $\mc{Q}_1$ outweighs $\mc{Q}_2$.
{Next, combining \eqnref{loss_noise}, \eqnref{MLlossEx}, and \eqnref{MClossEx} yields the following result.
\begin{theorem}
Let the input of the neural network be randomized as in \cref{eq:add_noise}. Then
    \begin{align}
         &\mathbb{E}\LRs{\mc{J}} =  \frac{1}{N_t} \sum_{k=0}^{N_t - 1} \LRp{\mc{L}_{\text{ML}} \LRp{\ui{k,0} } + \alpha \mc{L}_{\text{MC}}\LRp{\ui{k,0}}} \eqnlab{loss_final_noise} \\
         & + \frac{\delta^2}{N_t} \sum_{k=0}^{N_t - 1}  \LRs{ \mc{P}_1 \LRp{\ui{k,0}} + \mc{P}_2\LRp{\ui{k,0}} + \alpha \LRp{\mc{Q}_1 \LRp{\ui{k,0}} + \mc{Q}_2\LRp{\ui{k,0}}} } +  o\LRp{\nor{\epsb}^2}. 
         \nonumber
    \end{align}
\end{theorem}}
The first sum in \cref{eq:loss_final_noise} is the original loss (without randomization) and the second sum consists of additional terms induced by data randomization. These additional terms play a vital role in stimulating the stability and accuracy of the neural network. Indeed, as discussed above, randomizing the machine learning loss term encourages the smoothness of the neural network tangent by penalizing its first and second derivatives implicitly. Note that explicitly penalizing the first derivative of a neural network as in \cite{pan2018long} is possible, but this could be computationally expensive and challenging. Doing so for both the first and second derivatives is not recommended. 
The above heuristic analysis of data randomization also reveals the {\em power of the model-constrained term} in training neural network: it promotes the agreement of the neural network tangent  and the true tangent up to second order that is otherwise not realizable using the standard data-driven approach with only machine learning loss term.

% Hence, unlike $\mc{P}_2$, there might be a higher chance that one can gain the benefits of better predictions.

% The proof also explains why the sequential training data or sequential model constrained approach improves the accuracy in predictions. To be more specific, the underlying reason is that the amount of error in predictions has an interchangeable role as the noise added to input data. However, the key difference is that the noise added to data is more controllable than the error in predictions. Indeed, in our architecture (figure \figref{NN_architecture}), when noisy data is used, the network is regularized at the right first step which is the most important step. By contrast, without noise in the input, there is no extra regularization in the loss term in the first. Afterward, the error in the first predictions plays the role of the noise for the second step, i.e., $S = 1$ and $R = 1$. However, unlike the specified noise level which can be tuned in advance, the amount of error in predictions is not manageable, so is the level of implicit regularization.

\subsection{Estimation of prediction errors}
\seclab{error}
In this section, we show how data randomization helps improve the stability and accuracy of long-time predictions. We are interested in predicting solutions  of the system \eqnref{MoLines} starting from an initial condition $\ui{0}$ that is not in the training set. To that end, it is natural to compare the \texttt{mcTangent} solutions $\ut{i}$ in \eqnref{ml_pred} with the solutions $\ui{i}$ obtained from the discretized system \eqnref{ubFE}. 
%We provide the error estimation for the simplest case with $S = 0$ and $R = 1$ and leave the analysis for general case 
% {short derivation / long derivation below}
%In our research, we validate the predictive ability based on the mean square error at time points, let consider the $\LRp{i+1}$th time step 
Let us define the neural prediction error as
\begin{equation}
    \eqnlab{err_v}
    \ev{}\LRp{\ut{i}} = \ui{i+1} - \LRs{\ut{i} + \dt \NN{\ut{i}}}, \quad  \eg{i+1} = \nor{\ev{}\LRp{\ut{i}}}_2.
\end{equation}
% \begin{equation}
%     \eqnlab{err_g}
%     \eg{i+1} = \MSE{\ui{i+1} - \ut{i+1}} = \MSE{\ev{i+1}\LRp{\ut{i}}},
% \end{equation}
% where, for the clarify of the exposition, we have defined

% and $\ut{i+1}, \ut{i}$ are sequential predictions obtained by \eqnref{ml_pred}.
%To begin with, we consider the error at $\LRp{i+1}$th time step
From \eqnref{ubFE}, \eqnref{ml_pred}, and \eqnref{err_v} we have
\begin{equation}
    \eqnlab{bound_err}
    \begin{aligned}
        \ev{}\LRp{\ut{i}} = & \LRp{\ui{i} + \dt \F \LRp{\ui{i}}} - \LRp{\ut{i} + \dt \NN{\ut{i}}} \\
        = & \dt \F \LRp{\ut{i} +  \ev{}\LRp{\ut{i-1}}} - \dt \NN{\ut{i}} + \ev{}\LRp{\ut{i-1}}
    \end{aligned}
\end{equation}
Applying the Taylor expansion for the first term gives
\begin{equation}
    \dt \F \LRp{\ut{i} +  \ev{}\LRp{\ut{i-1}}} = \dt \F \LRp{\ut{i}} + \dt \eval{\pp{\F}{\ub}}_{\ut{i}} \ev{}\LRp{\ut{i-1}} + o\LRp{\eg{i}}
\end{equation}
Substituting back to \eqnref{bound_err}, we have
\begin{equation}
    \eqnlab{bound_err_2}
    \begin{aligned}
        &\ev{}\LRp{\ut{i}} =  \dt \LRs{\F \LRp{\ut{i}} - \NN{\ut{i}}} + \dt \eval{\pp{\F}{\ub}}_{\ut{i}} \ev{}\LRp{\ut{i-1}} + \ev{}\LRp{\ut{i-1}} \\
        &    + o\LRp{\eg{i}} = \dt \LRs{\F \LRp{\ut{i}} - \NN{\ut{i}}} + \dt \LRs{\eval{\pp{\F}{\ub}}_{\ut{i}} -  \eval{\pp{\Psi}{\ub}}_{\ut{i}}} \ev{}\LRp{\ut{i-1}} \\
        &  + \LRs{ \Ib + \dt \eval{\pp{\Psi}{\ub}}_{\ut{i}}} \ev{}\LRp{\ut{i-1}} +  o\LRp{\eg{i}}
    \end{aligned}
\end{equation}
%With the assumption that $ \eg{i} \ll 1$ (this assumption is not unrealistic as $\eg{i}$ can be arbitrarily small if our learned network is stable and introduces insignificant error for each new step), the high order terms are insignificant. 
Applying triangle inequality and Cauchy–Schwarz inequality for \eqnref{bound_err_2} and using \eqnref{err_v} yields 
%the prediction at the $i$th time step 
\begin{equation}
    \eqnlab{final_err}
    \begin{aligned}
         \eg{i+1} \leq &  %\underbrace{\dt \MSEsqrt{\F\LRp{\ut{i}} - \NN{\ut{i}}}}_{\mc{L}_{\text{MC}} \LRp{\ut{i}}}  \\ 
         \dt\MSEsqrt{\F\LRp{\ut{i}} - \NN{\ut{i}}}\\
         & + \dt \MSEsqrt{\eval{\pp{\F}{\ub}}_{\ut{i}} -  \eval{\pp{\Psi}{\ub}}_{\ut{i}}} \eg{i} + \MSEsqrt{1 + \dt \eval{\pp{\Psi}{\ub}}_{\ut{i}}} \eg{i} + o\LRp{\eg{i}}, \quad i \ge 0.
    \end{aligned}
\end{equation}
We observe in \eqnref{final_err} that the first term on the right-hand side is  the model-constrained loss term 
being as small as possible at the training data. On the other hand, $\dt \MSEsqrt{\eval{\pp{\F}{\ub}}_{\ut{i}} -  \eval{\pp{\Psi}{\ub}}_{\ut{i}}}$ and $\MSEsqrt{1 + \dt \eval{\pp{\Psi}{\ub}}_{\ut{i}}}$ are regularized to be bounded and/or small by data randomization (see \cref{sect:noise_data_sec}). A heuristic argument reveals that the prediction error is under control at all times. Indeed, suppose $ \dt\MSEsqrt{\F\LRp{\ut{i}} - \NN{\ut{i}}}$, $\dt \MSEsqrt{\eval{\pp{\F}{\ub}}_{\ut{i}} -  \eval{\pp{\Psi}{\ub}}_{\ut{i}}}$, and $\MSEsqrt{1 + \dt \eval{\pp{\Psi}{\ub}}_{\ut{i}}}$ are bounded. Since $\eg{0} = 0$, $\eg{1}$ is bounded, and by induction $\eg{i}$ is also bounded for $i \ge 0$. A rigorous version of this argument is given in Theorem \theoref{mainTheo}.
\begin{theorem}
\theolab{mainTheo}
Assume that the second derivative of $\F\LRp{\ub}$ with respect to $\ub$ is uniformly bounded.
Let
\[
f^{i+1} := \dt\MSEsqrt{\F\LRp{\ut{i}} - \NN{\ut{i}}},
\]
and
\[
g^{i+1} :=  \dt \MSEsqrt{\eval{\pp{\F}{\ub}}_{\ut{i}} -  \eval{\pp{\Psi}{\ub}}_{\ut{i}}} + \MSEsqrt{1 + \dt \eval{\pp{\Psi}{\ub}}_{\ut{i}}} + c^i,
\]
where $c^i = \mc{O}\LRp{\eg{i}}$. Then,
the prediction error $\eg{n}$ at time $t_n$ satisfies
\[
\eg{n} \le \sum_{k = 1}^{n}\LRp{\Pi_{i = k+1}^{n}g^i} f^k.
\]
\end{theorem}
\begin{proof}
The proof is a simple application of a discrete Gronwall lemma on \eqnref{final_err}.
\end{proof}
\begin{remark}
Note that the boundedness of the second derivative of $\F\LRp{\ub}$ with respect to $\ub$ is valid for problem \eqnref{MoLines} with a smooth tangent slope. The boundedness of $f^i$ and $g^i$ is not too restricted if the prediction scenarios are close to the training data. Indeed, as argued in \cref{sect:noise_data_sec}, data randomization enforces the small values for  $f^i$ and $g^i$ at the training points. Now, due to the smoothness of $\NN{\ub}$ and $\F\LRp{\ub}$ and their closeness in both values and derivatives (again by randomization), the continuity guarantees the small values for $f^i$ and $g^i$ during the prediction.
\end{remark}

\begin{remark}
Theorem \theoref{mainTheo} allows us to bound the error between the neural network prediction with the exact solution of the original PDEs \eqnref{eq_base} provided that an error estimation of the solution of the discretized equation \eqnref{ubFE} is given. Indeed, suppose the error in the discretized solution $\ub^n$ and the exact solution $\ub\LRp{t_n}$ at time $t_n$ is bounded by $\mc{O}\LRp{\dt + h^p}$, where $h$ is the mesh size and $p$ is the order of accuracy of the underlying spatial discretization. Then by a simple application of triangle inequality we have
\[
\ut{n} - \ub\LRp{t_n} = \mc{O}\LRp{\dt + h^p + \sum_{k = 1}^{n}\LRp{\Pi_{i = k+1}^{n}g^i}f^k},
\]
which shows that in order to get the optimal accuracy and computational effort we ideally need to balance not only the temporal and spatial discretization errors but also the error in the neural network. Clearly, balancing the former two is not that difficult from a numerical analysis point of view, but balancing also the network error is challenging as it depends on the actual training process and randomization. 
\end{remark}

\section{Numerical results}
\seclab{numerics}
In this section, we present several numerical results using the proposed model-constrained tangent slope neural network (\texttt{mcTangent}) approach for transport equation (\cref{sect:1D_wave}), viscous Burger's equation (\cref{sect:Burger_eq}), and Navier-Stokes equation (\cref{sect:2D_NS}). As shall be shown, \texttt{mcTangent} solutions are\textemdash thanks to the model-constrained term and data randomization\textemdash stable and capable of producing  accurate approximations  far beyond the training time horizons. In \cref{sect:Comp_cost} we provide detailed information on parameter tuning, randomness setting, and the cost for both training and testing.

%{SHOULD WE PUT SOMETHING OF 3.4 HERE?} 

%WILL COME BACK TO THIS SECOND PARAGRAPH

%For brevity, we represent a simulation case 
%{what case?} \textcolor{blue}{I meant a specified simulation that the neural network is trained with d samples with noise level $\delta$, sequential training $S$, sequential model-constrained $R$, and regularization parameter $\alpha$. The combination is a set of numbers to represent the equivalent case. For example $\LRp{d600, 2\%, 1,1,0}$ means d600 data set with 2\% noise, S = 1, R = 1 and $\alpha = 0$. Otherwise, I need to explain the information by words about the specified neural network all the time} by the combination of features as (training data set $d$, noise level $\delta \%$, sequential training steps $S$, sequential model-constrained step $R$, model-constrained regularization parameter $\alpha$). 
Five hyperparameters of interest are the number of training samples,  noise level $\delta$, sequential machine learning steps $S$, sequential model-constrained learning steps $R$, and regularization parameter $\alpha$. For convenience, we shall conventionally write them in a group.
For example the $\LRp{d600, 2\%, 1,1,0}$ setting means we consider 600 training data samples, 2\% noise, S = 1, R = 1, and  $\alpha = 0$. In order to ensure the fairness between simulations and the comparison among approaches, we use fixed random keys 
for training and testing data generation, for adding noise, and for neural network parameter initialization
%in 
We implement our approach and perform all computations in \texttt{JAX} \cite{jax2018github}. 
%In particular, these random seeds are for train and test data generation, adding noise, neural network parameter initialization. 
%In all problems, 
%the neural network is always initialized with the same weight and bias parameters. %W
We would like to point out that all computations (training, testing, and predicting) are done with single precision accuracy. 
%all training runs are conducted in the single precision accuracy. {does it imply double precision for validation and prediction?} ans{all the time is single precision}

\subsection{One-dimensional (1D) wave/transport equation}
\seclab{1D_wave}
The 1D wave equation considered in this section is given by
\[
    \pp{u}{t} + c \pp{u}{x} = 0,
\]
with the wave speed $c = 1$, the spatial domain $x \in \LRp{0,1}$, and time horizon $t \in \LRp{0, T}$. The equation is equipped with an initial condition $u(x,0) = u_0(x)$ and periodic boundary condition. We are interested in real-time approximate solutions of the wave equation for any initial condition $u_0(x)$. 

{\bf Data generation.}
In this problem, the initial condition samples  are drawn from
\[
u_0(x) = \sum_{i=1}^{5} a_i \sin \LRp{2\pi x \, i} + \sum_{i=1}^{5}  b_i \cos \LRp{2\pi x \, i},
\]
where $a_i,b_i$ are distributed by the standard normal distribution with zero mean and unit variance, i.e., $a_i,b_i \sim \mc{N}\LRp{0,1}$. We solve the wave equation with the forward Euler scheme for the temporal derivative
%{what is this? why 0 and $5 \times 10^{-2}$? you mean the time step of time interval? YES it is interval}
and the first order upwind  finite difference scheme for spatial derivative. The time horizon is chosen as $T = 5 \times 10^{-2}$. A fine space-time mesh with $n_x = 10000$ points in space and $n_t = 2000$ points in time is deployed to achieve highly accurate solutions. The training data samples are obtained by extracting the high resolution solutions on a  coarser uniform space-time mesh with $n_x = 200$ and $n_t = 100$. In this simple problem, we generate a fixed training data set of 200 initial conditions. %A random test sample is generated in the same fashion. 
Note that we aim to predict long-time solutions, $t \in \LRp{0, 3}$, from the short-time training data in the interval $t \in \LRp{0, 5 \times 10^{-2}}$. %Thus, it is still acceptable to verify the efficiency of approaches with initial state samples within the training set. {I don't understand the last sentence} ans{I meant it is ok if a initial condition is in training data, since we will solve for beyond training period, but we draw independently new initial condition when we test }

{\bf Neural network architecture.} Because of the linear nature of the problem and the first order upwind  finite difference scheme, a linear neural work is sufficient to approximate the resulting tangent slope. The linear neural network is defined as 
\[\NN{\ui{i}} = \W \ui{i} + \bb, \]
where the weights $\W \in \mc{R}^{n_x \times n_x}$, and the bias $\bb \in \mc{R}^{n_x}$. To train, we use \texttt{ADAM} \cite{kingma2014adam} optimizer with default parameters and the learning rate of $10^{-3}$.
%{what is the best model? I am confused. There is only one trained neural network?} ans{Yes, there is only one learned network, but how we get it from training process. Because for each epoch, we have a new state of network. Thus, we need to point out how we pick the final network. The one gives the lowest accumulated mean square error on test data set is chosen. In other words, we solve a problem of all initial condition test samples in each epoch. Each epoch has a network, the network showing the lowest accumulated MSE for all test samples is chosen. Then, that learned network is the best one. By this way, we achieve the best network that an approach can offer}
We determine the best combination of weights and biases (and hence the final trained network) as the one that provides the lowest accumulated mean square error for $500$ time steps for the test sample. Specifically, during the training process, at each epoch, we solve for the predictions from the test initial condition with the current-epoch learned network. The accumulated mean-square error between predictive solutions and ground truth solutions is calculated at the $500$th time step to determine the ``optimal" network.

{\bf Long-time predictions.}
Shown in \cref{fig:1D_compare_all} is the mean-square error between true (high resolution) solutions and predicted ones obtained by various neural networks, each of which is trained with both randomized and noise-free training data. 
%It can be seen that training without the model-constrained term provides more accurate solutions, while the model-constrained approach shows a higher or equal error level against the low dimension ($n_x = 200$) finite difference solver.

 \begin{figure}[htb!]
    \centering
    \includegraphics[width=\textwidth]{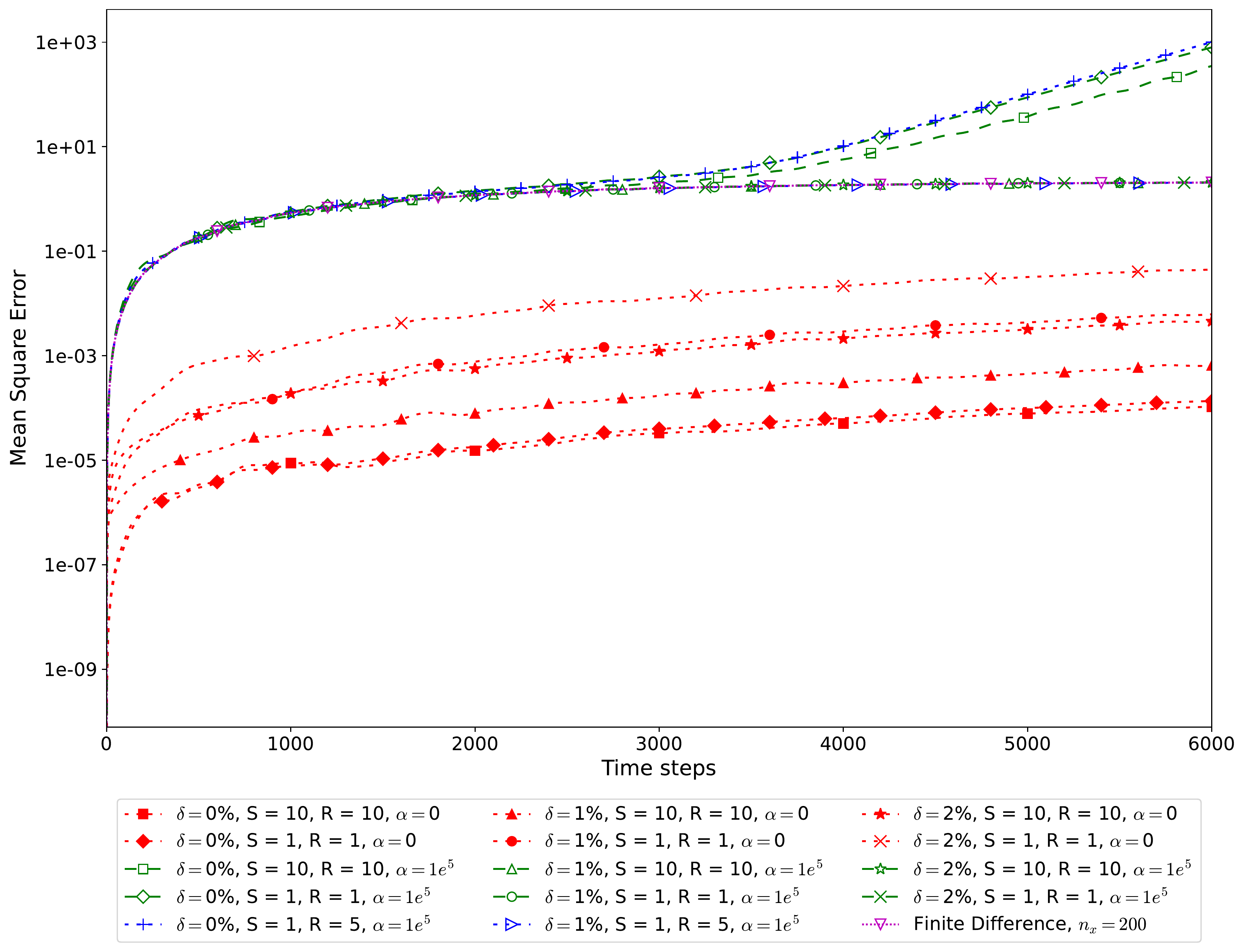}
    \caption{\textbf{Wave/transport equation}. Comparison between different neural network approaches with/without randomization.} 
    \figlab{1D_compare_all}
\end{figure}

 For pure data-driven machine learning networks ($\alpha = 0$, and thus no model-constrained term), we observe that noise-free data trained networks outperform those trained with noisy data ones. %It is worth noting that this phenomenon is not against the well-known discovery in the machine learning community that adding noise helps enhancing the generalization of neural networks. 
 This is not surprising as for this linear problem, as predicted by Lemma \lemref{linearOptimality}, one can obtain linear networks that accurately learn the tangent slope with sufficiently rich data.
  Therefore, the predictions by the learned linear networks are almost the same as the ground truth solutions. On the contrary, training with noisy data causes the neural network to predict solutions with a small amount of error such that it adapts to (possibly overfits) the amount of noise in the ground-truth solutions. 
  %Therefore, there is always a gap in the mean square error level between noisy data training networks and their counterparts trained without noise. The 
  \cref{fig:1D_wave_NN_params} presents the weight matrix and bias vector for two cases $\LRp{d200, 0\%, 10, 1, 0}$ and $\LRp{d200,0\%, 1, 1, 0}$ with noise-free data. It can be seen that both networks are almost identical and both have only an upper diagonal with a large magnitude. We also note that
  the bias vector is relatively small
%  ignoring the bias vector leads to insignificant changes in the predictions. It implies the bias terms are unimportant, 
and thus we ignore this bias vector in the subsequent comparisons. We present the test predicted solutions for the setting $\LRp{d200,0\%, 1, 1, 0}$  in \cref{fig:1D_preds_different_periods}. As the network fits the tangent slope for high-resolution data, accurate results are preserved far beyond
 the training time horizon, while finite difference results on the same coarse grid show a severe diffusion/dissipation effect. Furthermore, settings with a large number of sequential steps  such as  $\LRp{d200,0\%,10,1,0}$, $\LRp{d200,1\%,10,1,0}$ and  $\LRp{d200,2\%,10,1,0}$ yield more accurate neural networks than their counterparts with $S = 1$. The reason is that long sequential training reduces the prediction error.

 For model-constrained neural networks, we use  $\alpha = 10^5$ as  the regularization parameter   %for the model-constrained regularization parameter 
 for all cases. We tested with different values for $\alpha$ and  almost the same results are obtained for larger values, while smaller values make neural networks perform similarly to the pure data-driven machine learning networks. It can be seen in \cref{fig:1D_compare_all} that training with randomized data returns neural networks, regardless of $S, R$ values, as good as the coarse finite difference approximation with $n_x = 200$. This is expected as we constrain the training with a coarse finite difference model. The trained weight matrices and bias vectors for these neural networks corresponding to three settings $\LRp{d200,1\%, 10,1,10^5}$, $\LRp{d200,1\%, 1,1,10^5}$ and $\LRp{d200,1\%, 1,5,10^5}$ are shown in \cref{fig:1D_wave_FD_NN_params}. Again, the bias vectors do not have a significant role in the predictions. Note that, unlike those from purely data-driven in \cref{fig:1D_wave_NN_params} which have arbitrary structure, the model-constrained weight matrices, after ignoring small elements, have the same structure as the first-order upwind scheme matrix. Among these neural networks, the long sequential model-constrained network with $R = 5$ is closest to the first-order upwind scheme matrix. It is not surprising as the neural network is constrained to satisfy the first-order upwind scheme in multiple time steps. On the other hand, the neural network trained with noise-free data shows instability starting from the $2000$th time step in long-term predictions. This instability is due to the lack of regularizations
 as compared to the randomized cases for which regularizations are explicit via the model-constrained term and implicit via randomization  (see \cref{sect:noise_data_sec}).

% \begin{figure}[htb!]
%     \centering
%     \begin{tabular*}{\textwidth}{c c c}
%         \centering
%         \rotatebox[origin=c]{90}{\small $S = 10$} &
%         \raisebox{-0.5\height}{\includegraphics[width=.40 \textwidth]{figures/1D_wave/1D_W__0_0_10_1.pdf}} &
%         \raisebox{-0.5\height}{\includegraphics[width=.35 \textwidth]{figures/1D_wave/1D_b__0_0_10_1.pdf}} 
%         \\
%         \centering
%         \rotatebox[origin=c]{90}{\small $S = 1$} &
%         \raisebox{-0.5\height}{\includegraphics[width=.40 \textwidth]{figures/1D_wave/1D_W__0_0_1_1.pdf}} &
%         \raisebox{-0.5\height}{\includegraphics[width=.35 \textwidth]{figures/1D_wave/1D_b__0_0_1_1.pdf}}
        
%     \end{tabular*}
%     \caption{\textbf{Wave/transport equation}. Pure data-driven  trained linear neural network parameters: weight matrix heat maps (\textit{left column}) and bias vector magnitudes (\textit{right column}) with $\alpha = 0, \delta = 0\%$.} 
%     \figlab{1D_wave_NN_params}
% \end{figure}
{\tiny 
\begin{figure}[htb!]
    \centering
    \begin{tabular*}{\textwidth}{c c c}
        \centering
        \rotatebox[origin=c]{90}{\small $S = 10$} &
        \raisebox{-0.\height}{
            \begin{minipage}{.4\textwidth}
            \begin{tikzpicture}
                \tikzstyle{every node}=[font=\tiny, node distance=7.5mm]
                \node (img)  {\includegraphics[width = .95\textwidth]{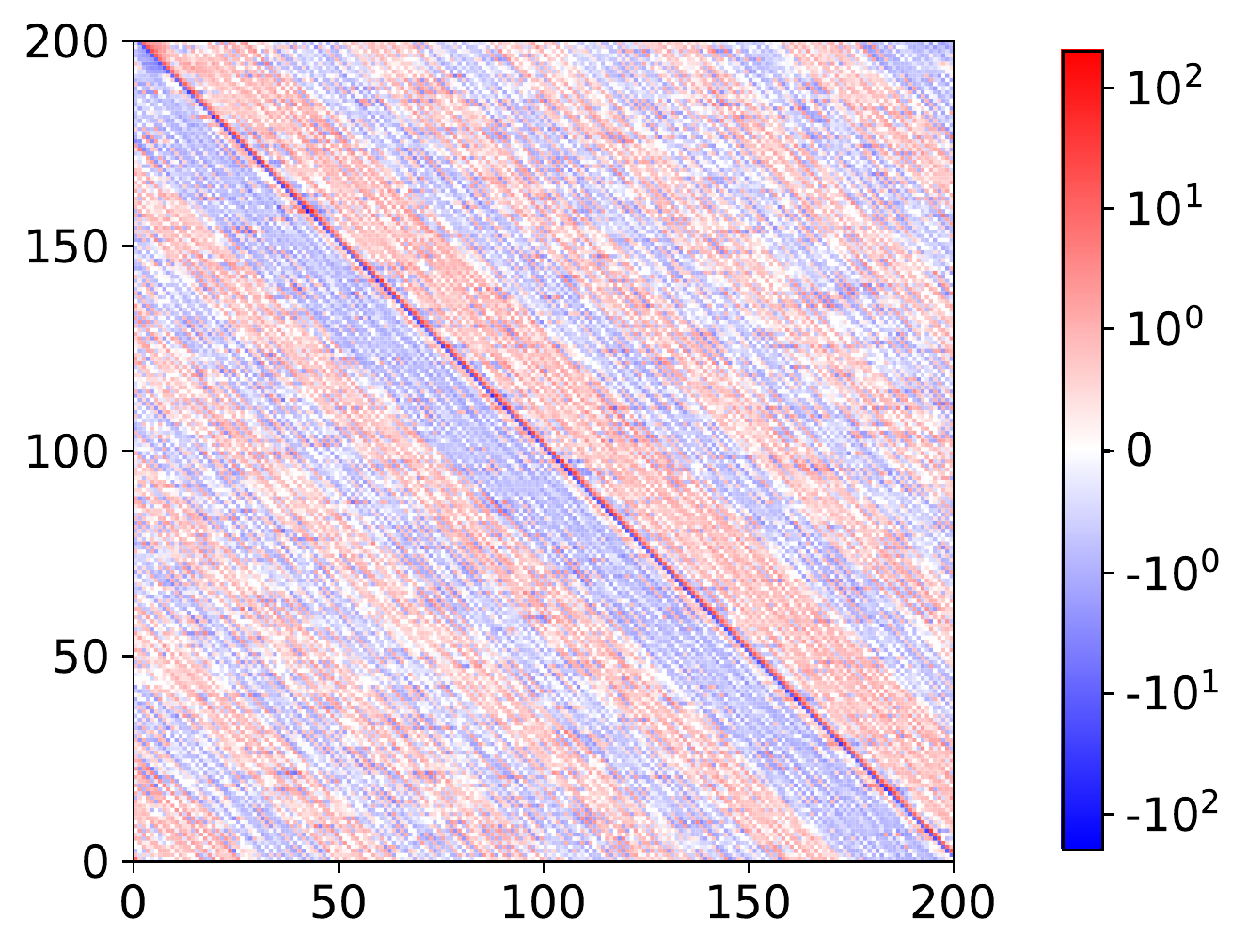}};
                \node[below=of img, node distance=0cm, xshift = -3mm, yshift=1cm] 
                {Column index of $\W$};
                \node[left=of img, node distance=0cm, rotate=90, anchor=center,yshift=-0.7cm] 
                {Row index of $\W$};
            \end{tikzpicture}
            \end{minipage}
        } &
        \raisebox{-0.\height}{
        \begin{minipage}{.35\textwidth}
            \begin{tikzpicture}
                \tikzstyle{every node}=[font=\tiny, node distance=7.5mm]
                \node (img)  {\includegraphics[width = 0.95\textwidth]{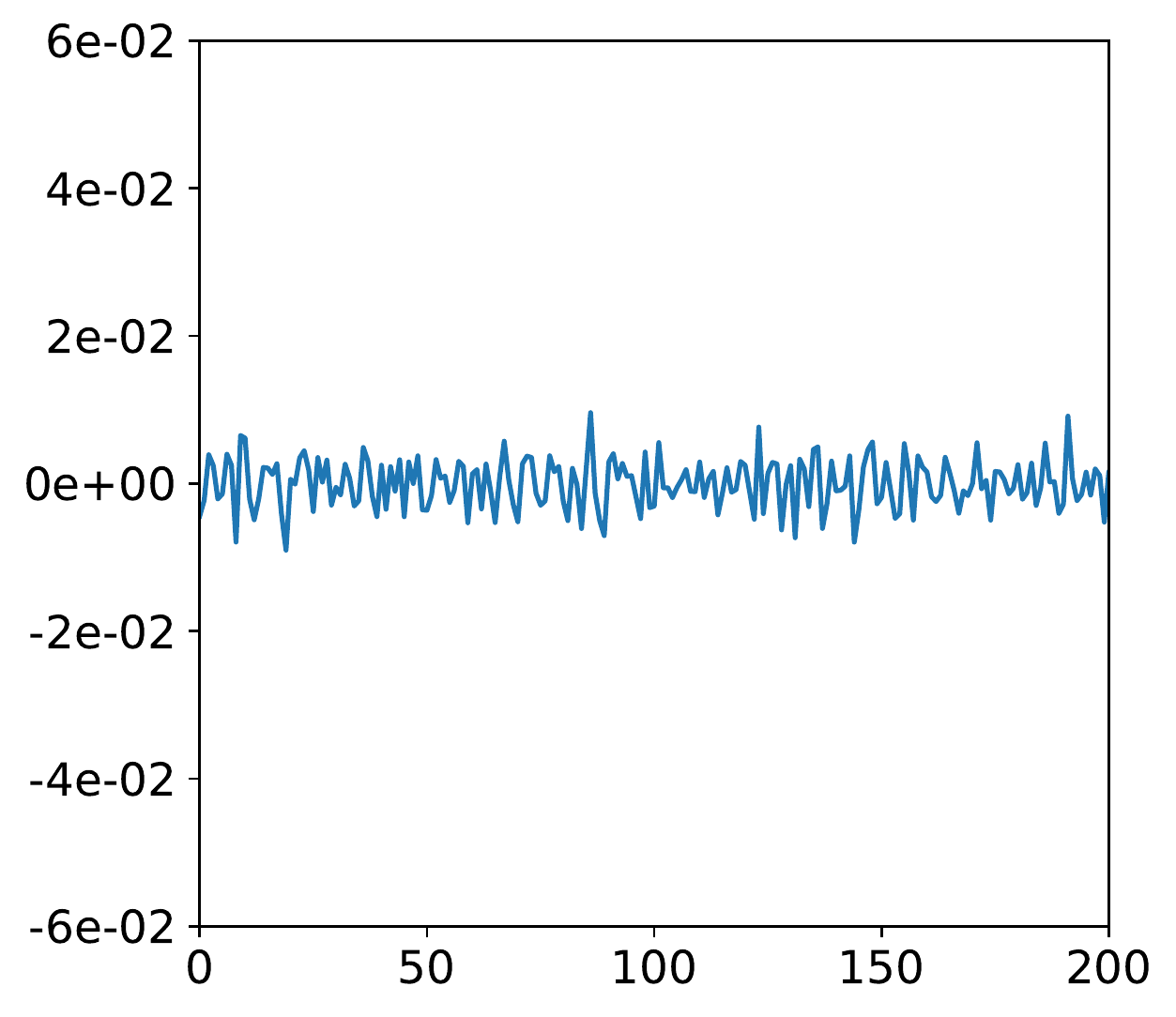}};
                \node[below=of img, node distance=0cm, yshift=1cm, xshift = 2mm] 
                {Index of bias vector $\bb$};
                \node[left=of img, node distance=0cm, rotate=90, anchor=center,yshift=-0.7cm] 
                {Magnitude $b_i$};
            \end{tikzpicture}
        \end{minipage}
        } 
        \\
        \centering
        \rotatebox[origin=c]{90}{\small $S = 1$} &
        \raisebox{-0.\height}{
            \begin{minipage}{.4\textwidth}
            \begin{tikzpicture}
                \tikzstyle{every node}=[font=\tiny, node distance=7.5mm]
                \node (img)  {\includegraphics[width = 0.95\textwidth]{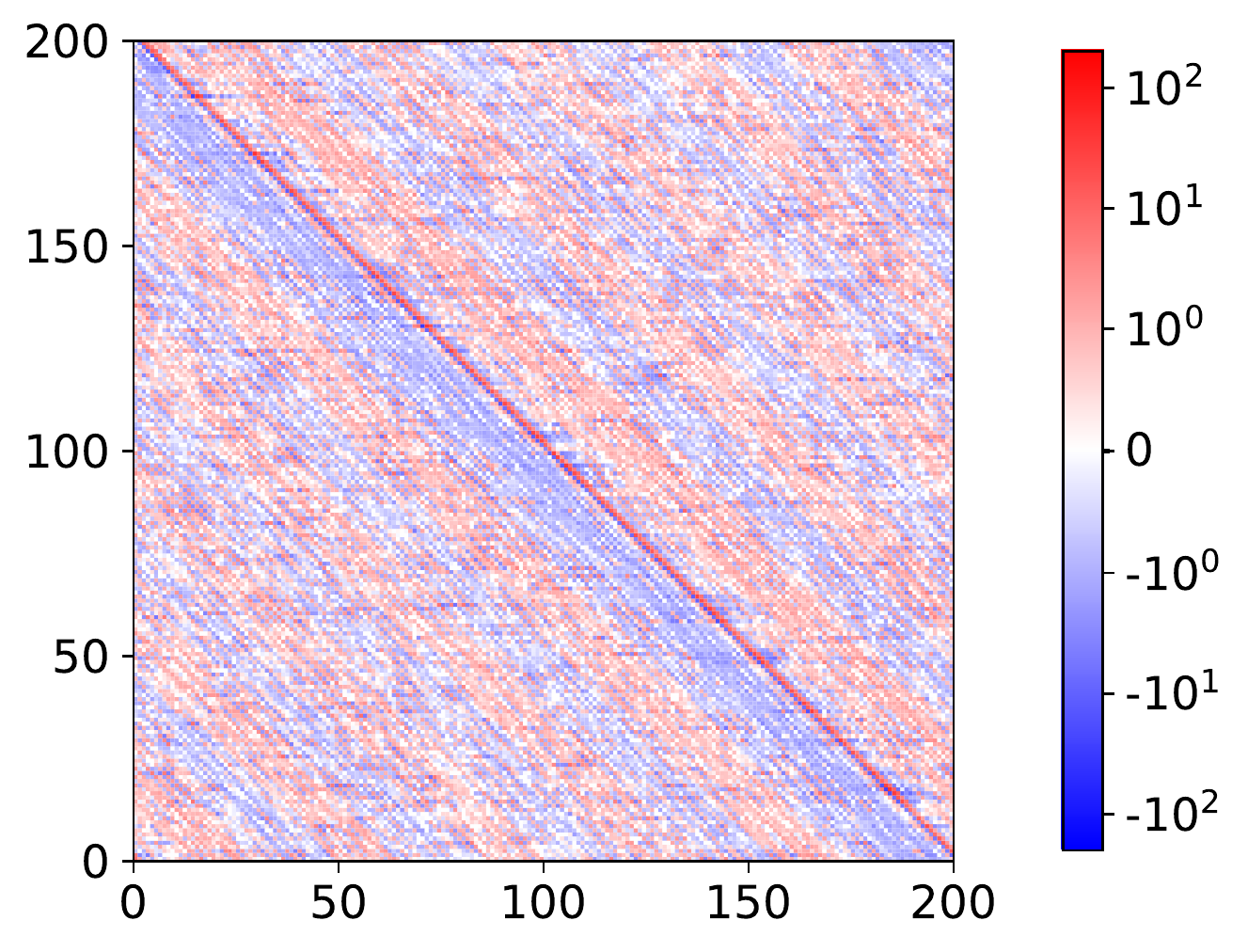}};
                \node[below=of img, node distance=0cm, xshift = -3mm, yshift=1cm] 
                {Column index of $\W$};
                \node[left=of img, node distance=0cm, rotate=90, anchor=center,yshift=-0.7cm] 
                {Row index of $\W$};
            \end{tikzpicture}
            \end{minipage}
            } &
        \raisebox{-0.\height}{
        \begin{minipage}{.35\textwidth}
            \begin{tikzpicture}
                \tikzstyle{every node}=[font=\tiny, node distance=7.5mm]
                \node (img)  {\includegraphics[width = 0.95\textwidth]{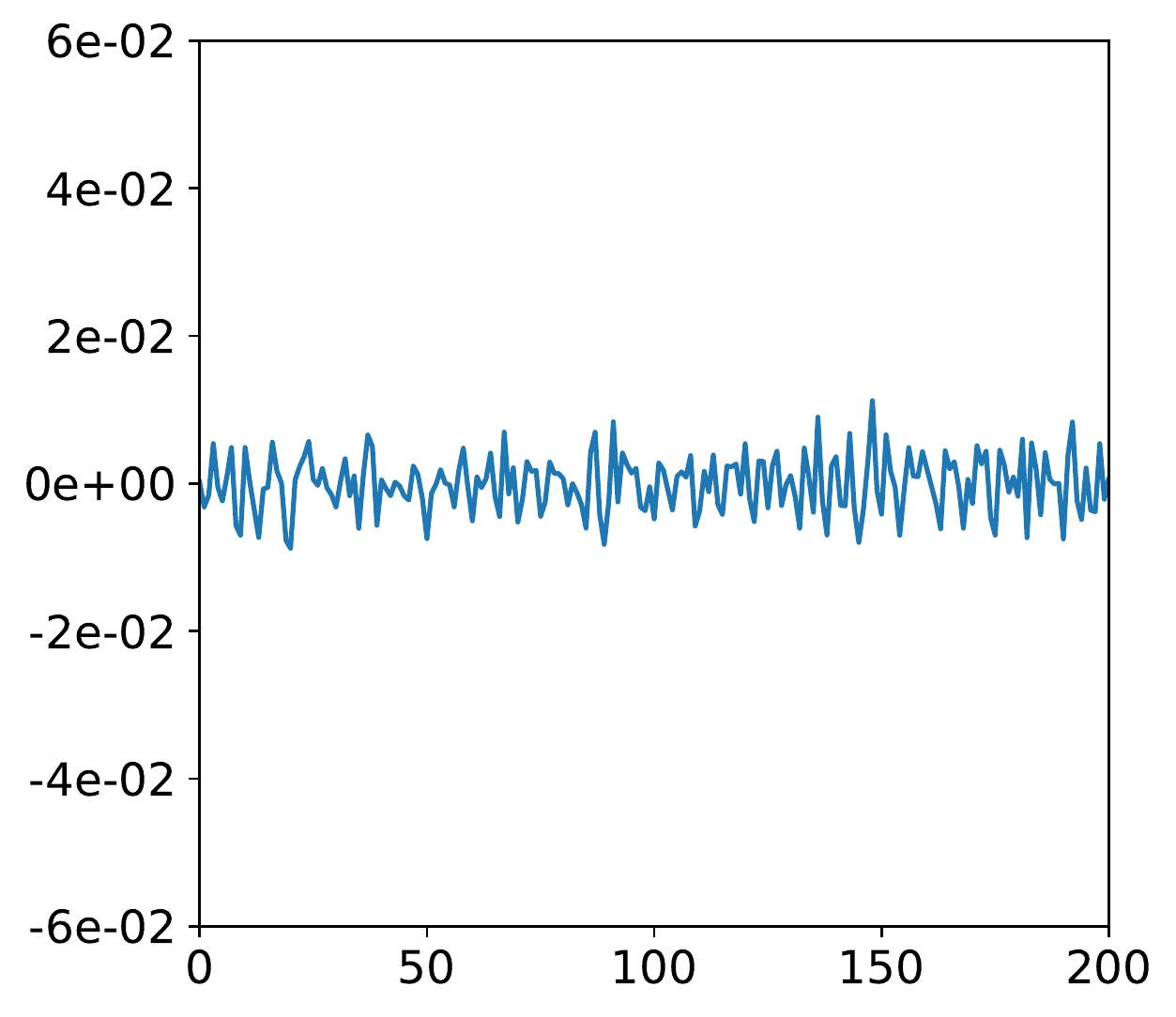}};
                \node[below=of img, node distance=0cm, yshift=1cm, xshift = 2mm] 
                {Index of bias vector $\bb$};
                \node[left=of img, node distance=0cm, rotate=90, anchor=center,yshift=-0.7cm] 
                {Magnitude $b_i$};
            \end{tikzpicture}
        \end{minipage}
        }
    \end{tabular*}
    \caption{\textbf{Wave/transport equation}. Pure data-driven  trained linear neural network parameters: weight matrix heat maps (\textit{left column}) and bias vector magnitudes (\textit{right column}) with $\alpha = 0, \delta = 0\%$.} 
    \figlab{1D_wave_NN_params}
\end{figure}
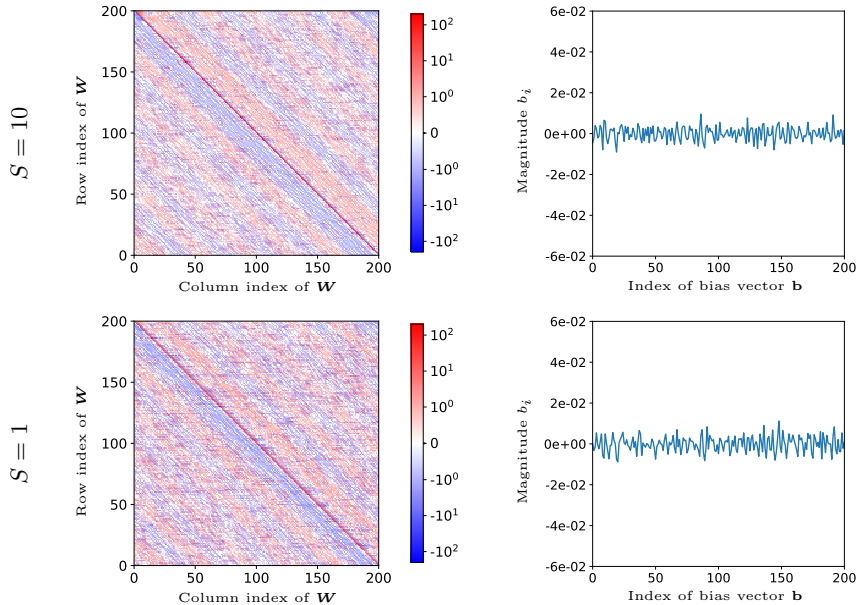
}

\begin{figure}[htb!]
    \centering
    \includegraphics[width=\textwidth]{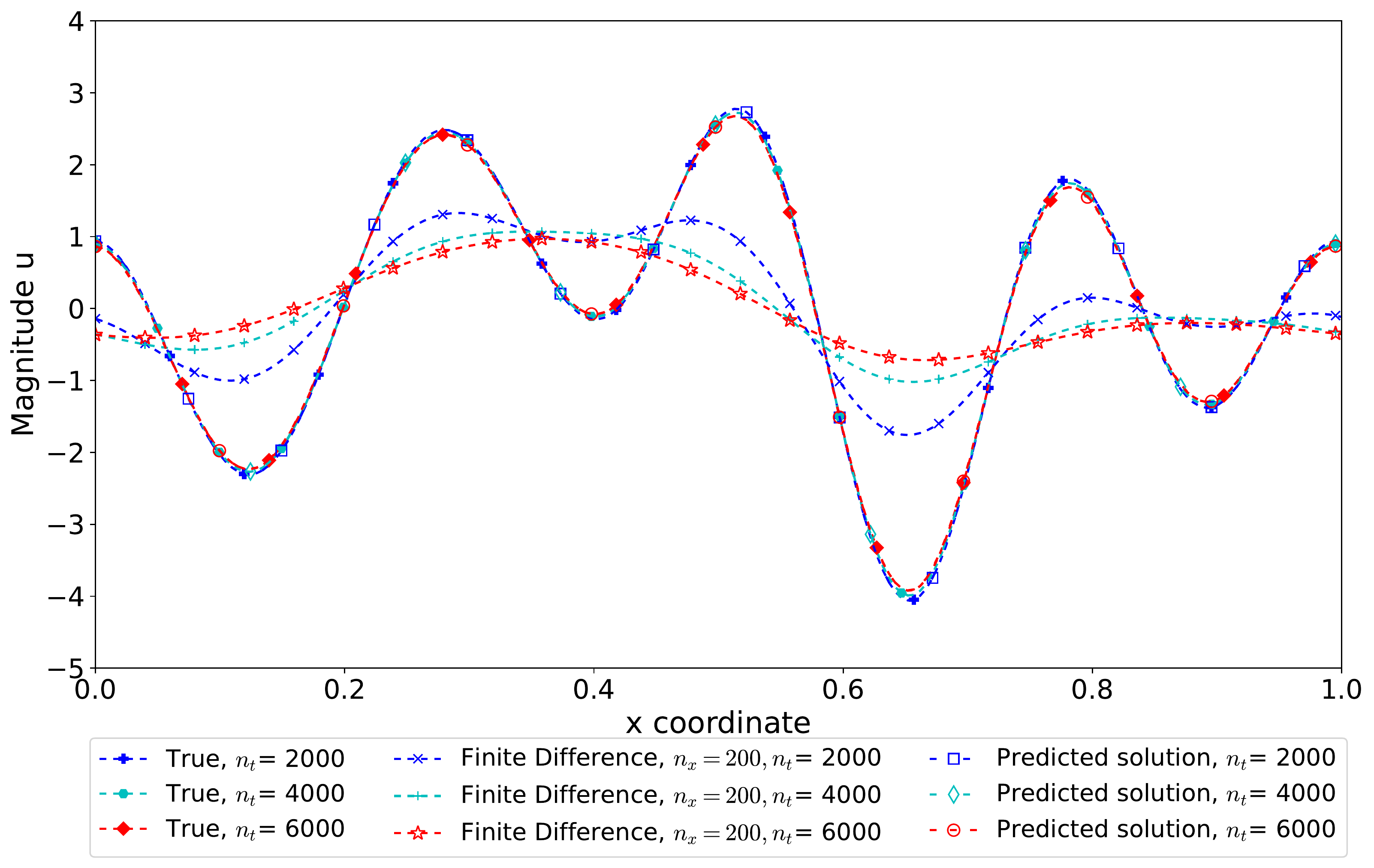}
    \caption{\textbf{Wave/transport equation}. The predicted solutions at time steps $n_t = 2000, 4000, 6000$ by learned neural network corresponding to ($d200, S = 1, R = 1, \alpha =0, \delta = 0\%$), finite difference  solutions on  coarse grid with $n_x = 200$, and the high resolution solutions (True).} 
    \figlab{1D_preds_different_periods}
\end{figure}
 
\begin{figure}[htb!]
    \centering
    \begin{tabular*}{\textwidth}{c c c}
        \centering
        \raisebox{-0.5\height}{\small $S = 10, R = 1$} &
        \raisebox{-0.5\height}{\small $S = 1, R = 1$} & 
        \raisebox{-0.5\height}{\small $S = 1, R = 5$} 
        \\
        \centering
        \raisebox{-0.\height}{
        \begin{minipage}{.28\textwidth}
            \begin{tikzpicture}
                \tikzstyle{every node}=[font=\tiny, node distance=7.5mm]
                \node (img)  {\includegraphics[width = 0.95\textwidth]{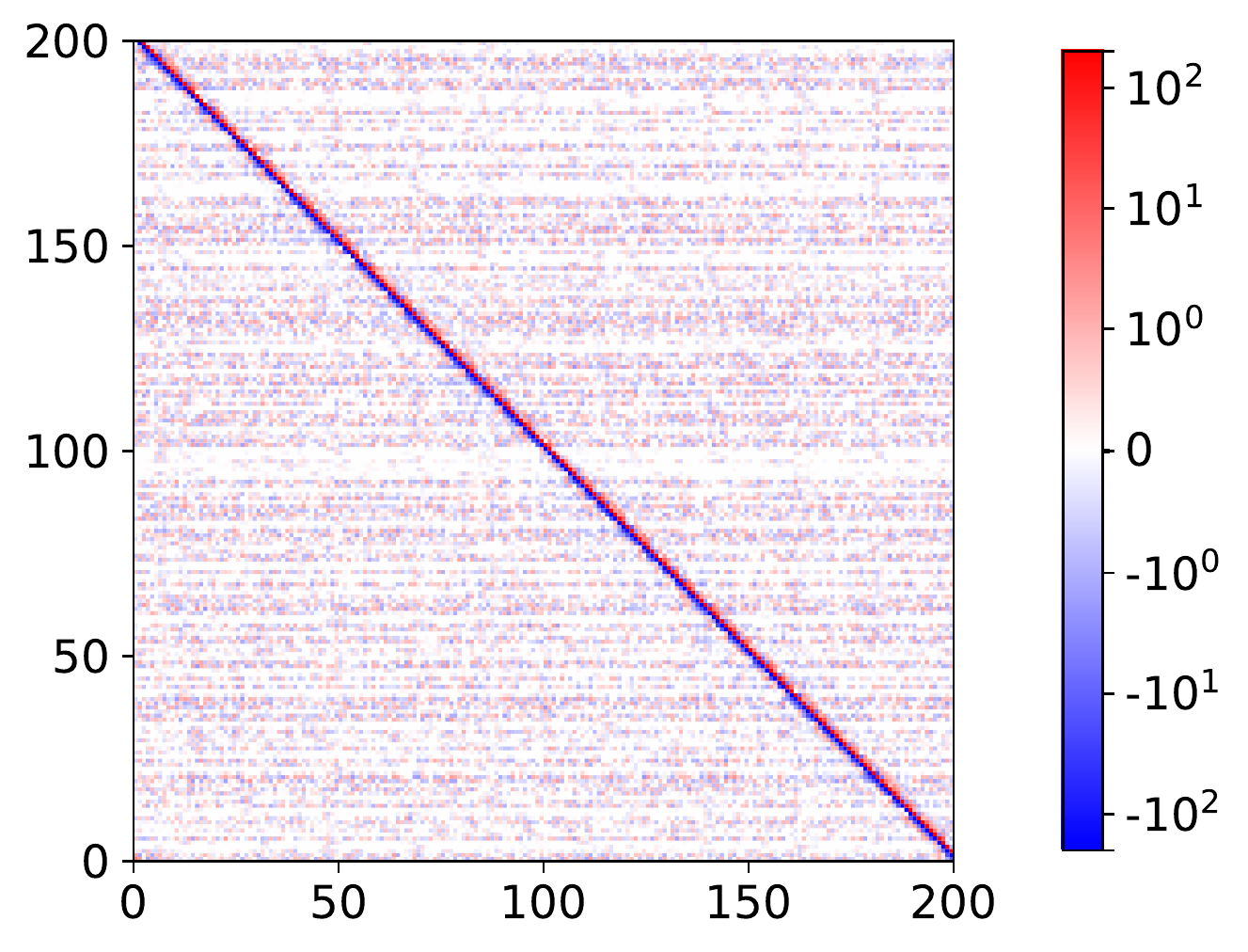}};
                \node[below=of img, node distance=0cm, xshift = -2.5mm, yshift=1cm] 
                {Column index of $\W$};
                \node[left=of img, node distance=0cm, rotate=90, anchor=center,yshift=-0.7cm] 
                {Row index of $\W$};
            \end{tikzpicture}
            \end{minipage}
        } &
        \raisebox{-0.\height}{
        \begin{minipage}{.28\textwidth}
            \begin{tikzpicture}
                \tikzstyle{every node}=[font=\tiny, node distance=7.5mm]
                \node (img)  {\includegraphics[width = 0.95\textwidth]{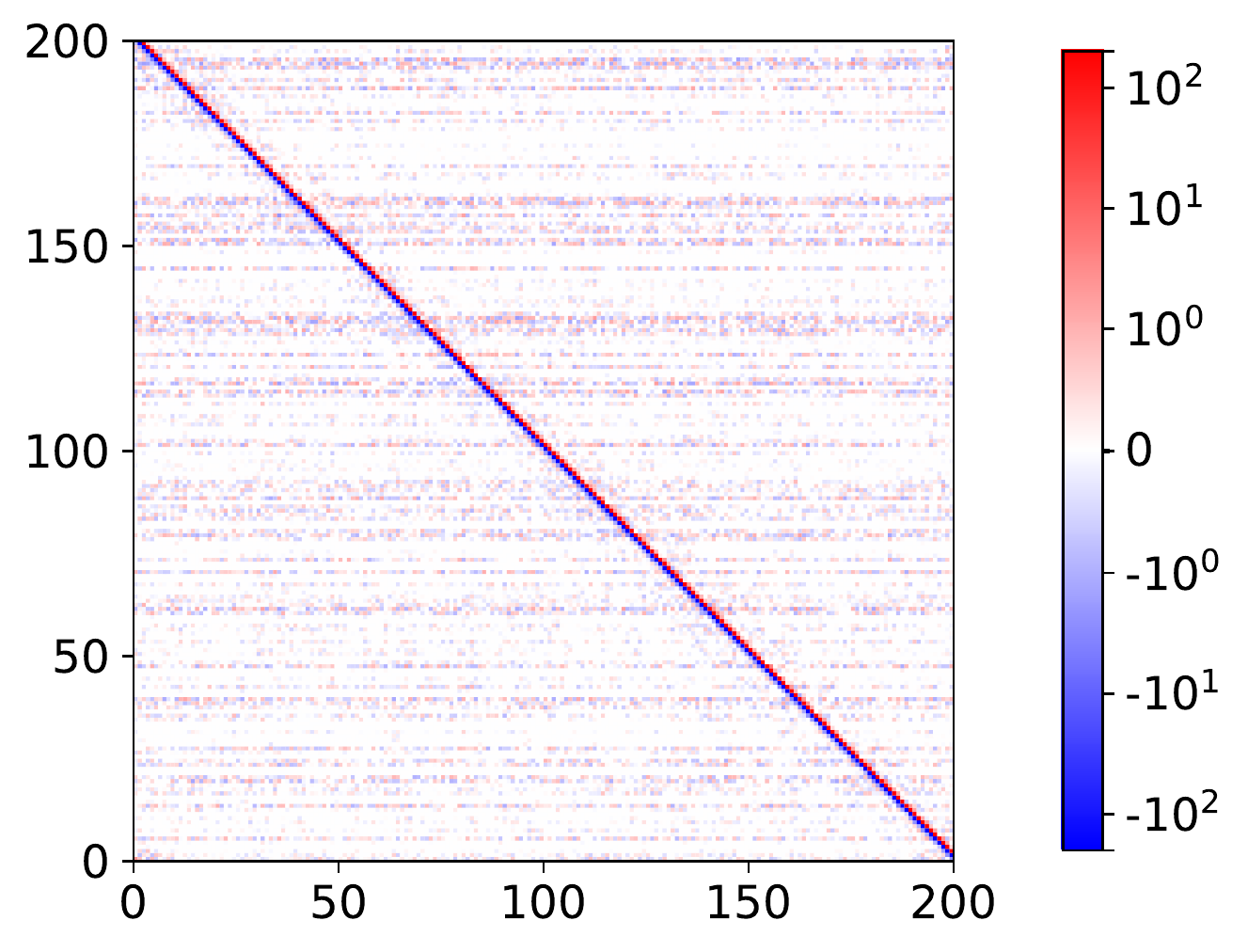}};
                \node[below=of img, node distance=0cm, xshift = -2.5mm, yshift=1cm] 
                {Column index of $\W$};
                \node[left=of img, node distance=0cm, rotate=90, anchor=center,yshift=-0.7cm] 
                {Row index of $\W$};
            \end{tikzpicture}
            \end{minipage}
        } & 
        \raisebox{-0.\height}{
        \begin{minipage}{.28\textwidth}
            \begin{tikzpicture}
                \tikzstyle{every node}=[font=\tiny, node distance=7.5mm]
                \node (img)  {\includegraphics[width = 0.95\textwidth]{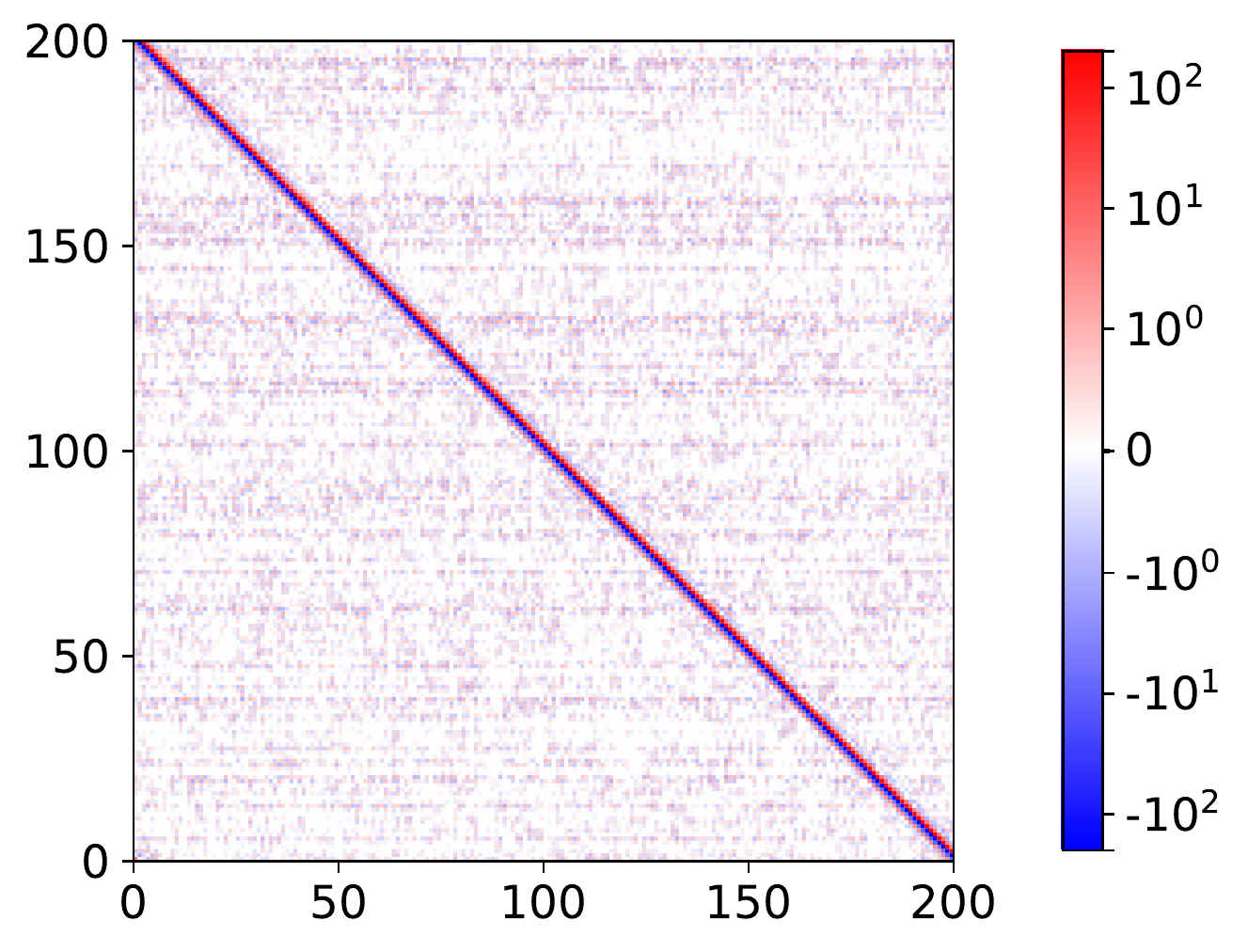}};
                \node[below=of img, node distance=0cm, xshift = -2.5mm, yshift=1cm] 
                {Column index of $\W$};
                \node[left=of img, node distance=0cm, rotate=90, anchor=center,yshift=-0.7cm] 
                {Row index of $\W$};
            \end{tikzpicture}
            \end{minipage}
        } 
        \\ ~ \\
        \centering
        \raisebox{-0.5\height}{
        \begin{minipage}{.28\textwidth}
            \begin{tikzpicture}
                \tikzstyle{every node}=[font=\tiny, node distance=7.5mm]
                \node (img)  {\includegraphics[width = 0.95\textwidth]{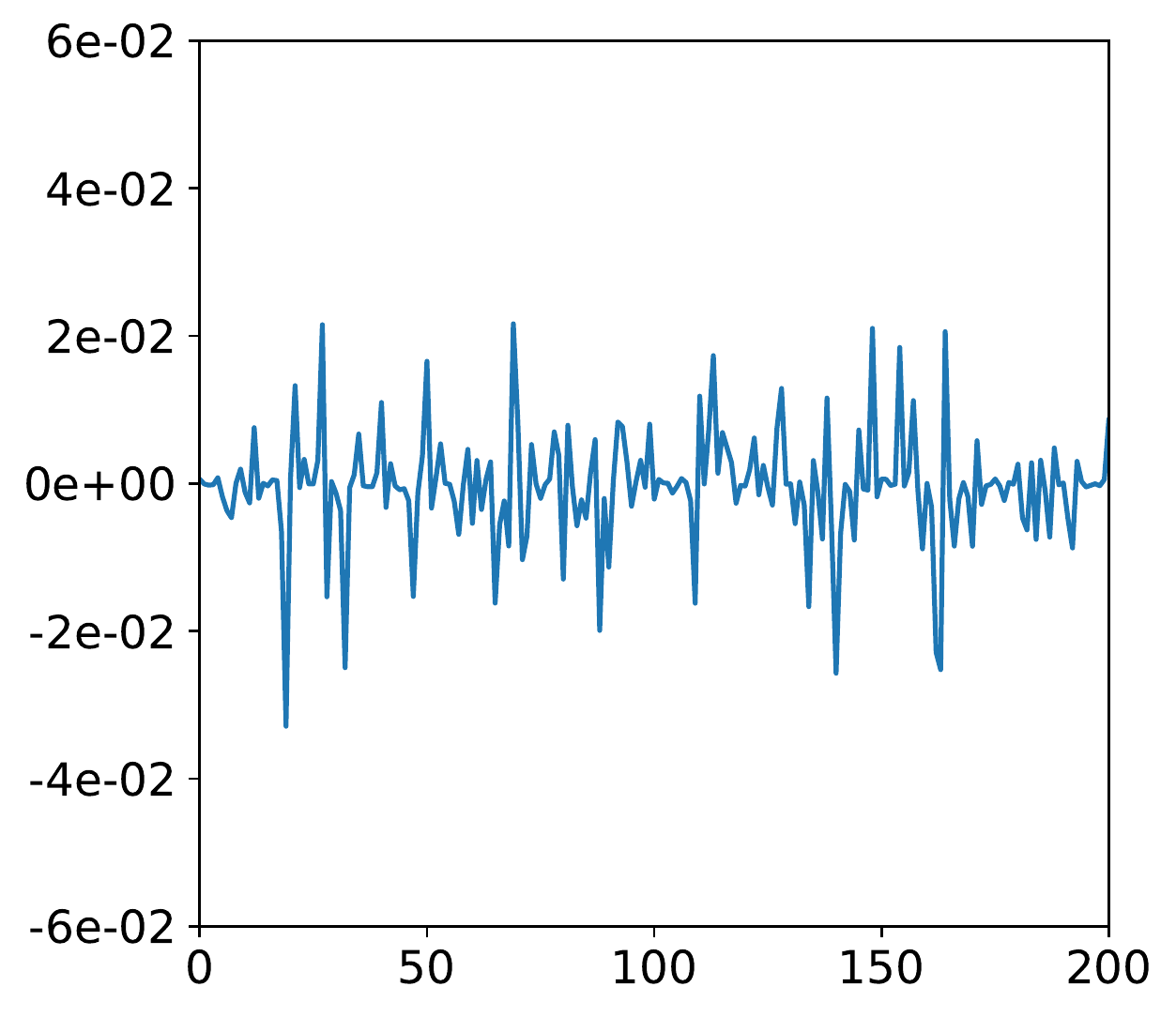}};
                \node[below=of img, node distance=0cm, yshift=1cm, xshift = 2mm] 
                {Index of bias vector $\bb$};
                \node[left=of img, node distance=0cm, rotate=90, anchor=center,yshift=-0.7cm] 
                {Magnitude $b_i$};
            \end{tikzpicture}
        \end{minipage}
        } &
        \raisebox{-0.5\height}{
        \begin{minipage}{.28\textwidth}
            \begin{tikzpicture}
                \tikzstyle{every node}=[font=\tiny, node distance=7.5mm]
                \node (img)  {\includegraphics[width = 0.95\textwidth]{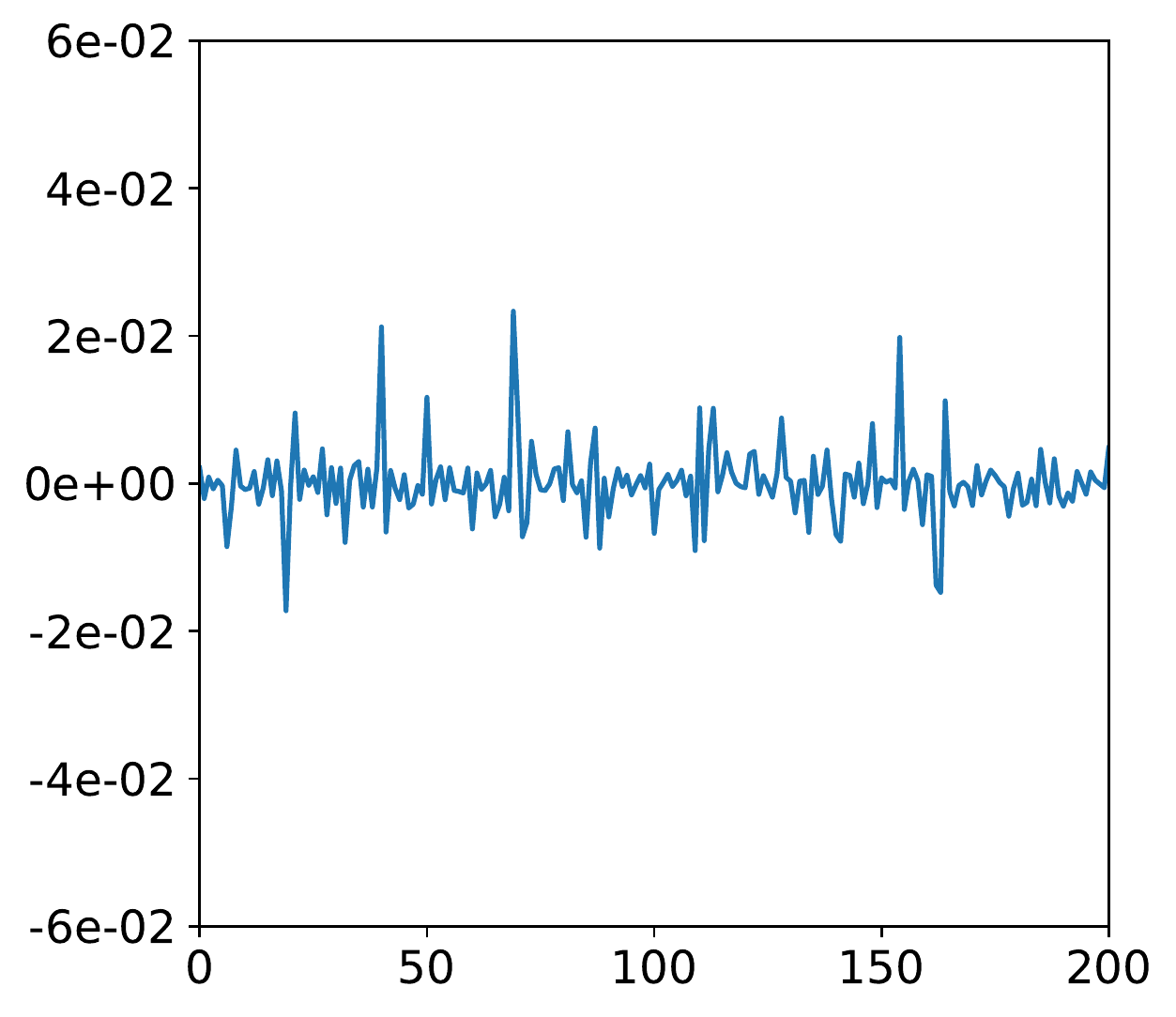}};
                \node[below=of img, node distance=0cm, yshift=1cm, xshift = 2mm] 
                {Index of bias vector $\bb$};
                \node[left=of img, node distance=0cm, rotate=90, anchor=center,yshift=-0.7cm] 
                {Magnitude $b_i$};
            \end{tikzpicture}
        \end{minipage}
        } & 
        \raisebox{-0.5\height}{
        \begin{minipage}{.28\textwidth}
            \begin{tikzpicture}
                \tikzstyle{every node}=[font=\tiny, node distance=7.5mm]
                \node (img)  {\includegraphics[width = 0.95\textwidth]{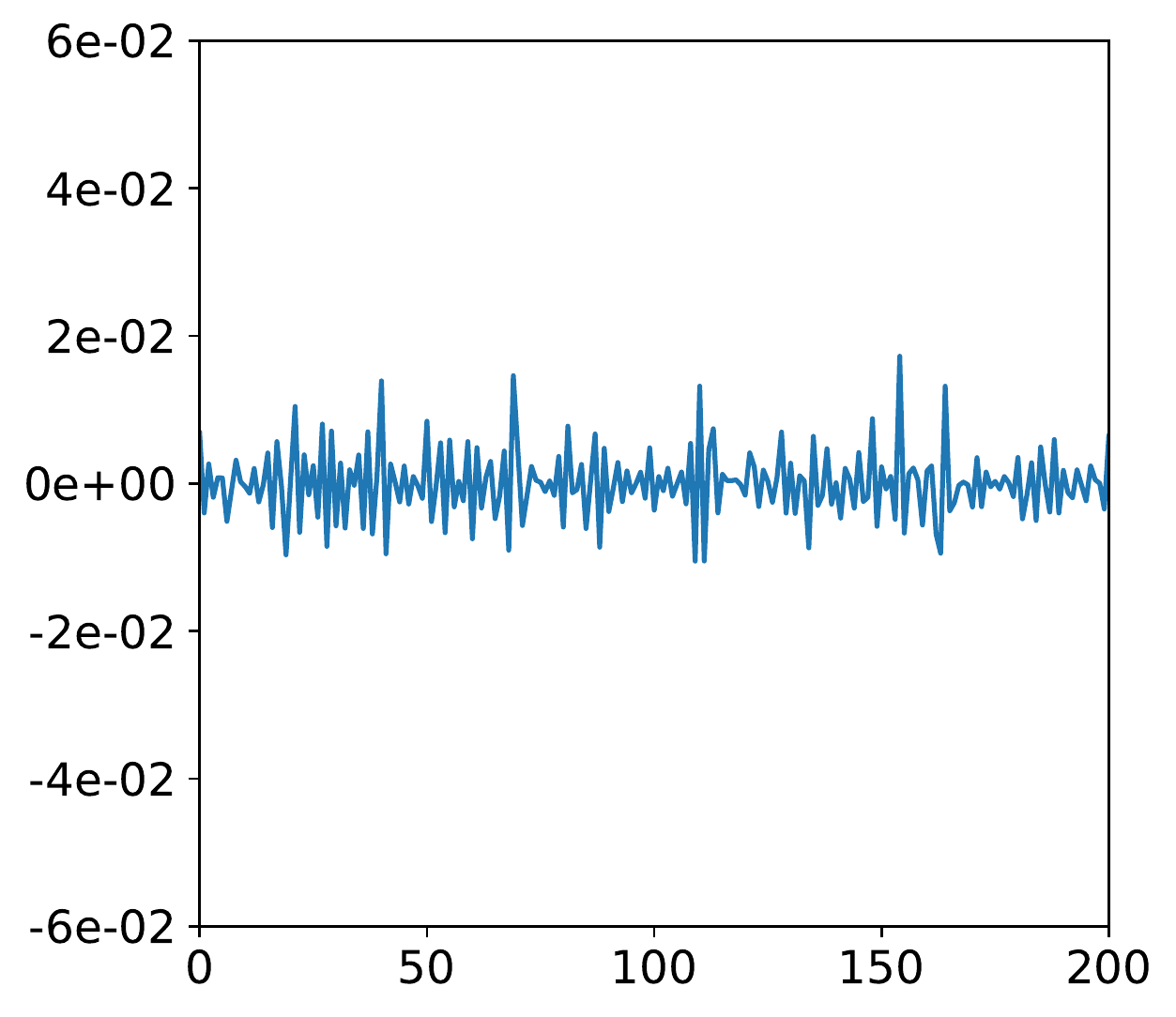}};
                \node[below=of img, node distance=0cm, yshift=1cm, xshift = 2mm] 
                {Index of bias vector $\bb$};
                \node[left=of img, node distance=0cm, rotate=90, anchor=center,yshift=-0.7cm] 
                {Magnitude $b_i$};
            \end{tikzpicture}
        \end{minipage}
        }
    \end{tabular*}
    \caption{\textbf{Wave/transport equation}. Trained model-constrained linear neural network parameters: weight matrix heat map (\textit{top row}) and bias vector magnitude (\textit{bottom row}) with  $\alpha = 1e^5, \delta = 1\%$.} 
    \figlab{1D_wave_FD_NN_params}
\end{figure}

{\bf Implicit time integration with learned network.}
\begin{figure}[htb!]
    \centering
    \includegraphics[width=\textwidth]{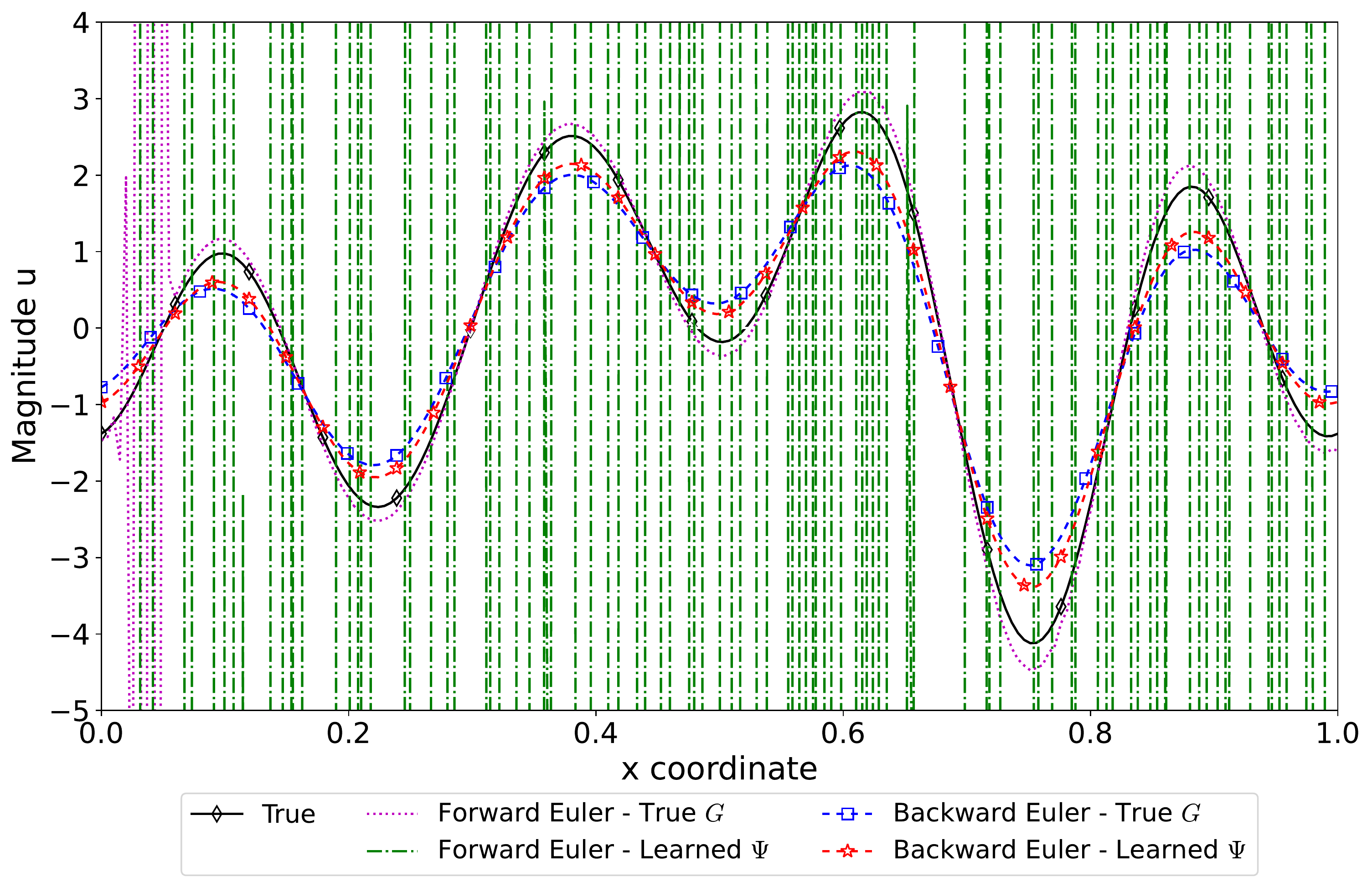}
    \caption{\textbf{Wave/transport equation}. Predicted solutions at  $t = 0.1$ obtained by forward Euler scheme and backward Euler scheme with the true tangent slope $\G$ and learned neural network counterparts with  time stepsize $\dt' = \frac{50}{3} \dt = \frac{25}{3} \times  10^{-3}$.} 
    \figlab{1D_preds_imp}
\end{figure}
One of the advantages of our proposed tangent slope learning approach is that once trained the learned tangent slope can be used with any time discretization method.
To demonstrate this, we use the learned neural network tangent for the setting $\LRp{d200,0\%,1,1,0}$
%, which results in the best results among networks, 
with both backward and forward  Euler schemes using a  time stepsize $\dt' = \frac{50}{3} \dt$ which is much larger than the training stepsize.  It can be seen in  \cref{fig:1D_preds_imp}  that the forward Euler solutions blow up for both learned and true tangent slopes, which is obvious as the time stepsize is much larger than the stable time stepsize. Both approaches are stable with implicit integration and the results are comparable (though the learned tangent slope was trained with a smaller time step size).

%Moreover, implicit solutions with  learned neural network tangent are better than those obtained by the true tangent slope on coarse mesh. This is due to the naive machine learning network encodes the high accuracy of fine-mesh training data.

{\bf Direct learning versus \texttt{mcTangent} slope learning.}
\begin{figure}[htb!]
    \centering
    \includegraphics[width=\textwidth]{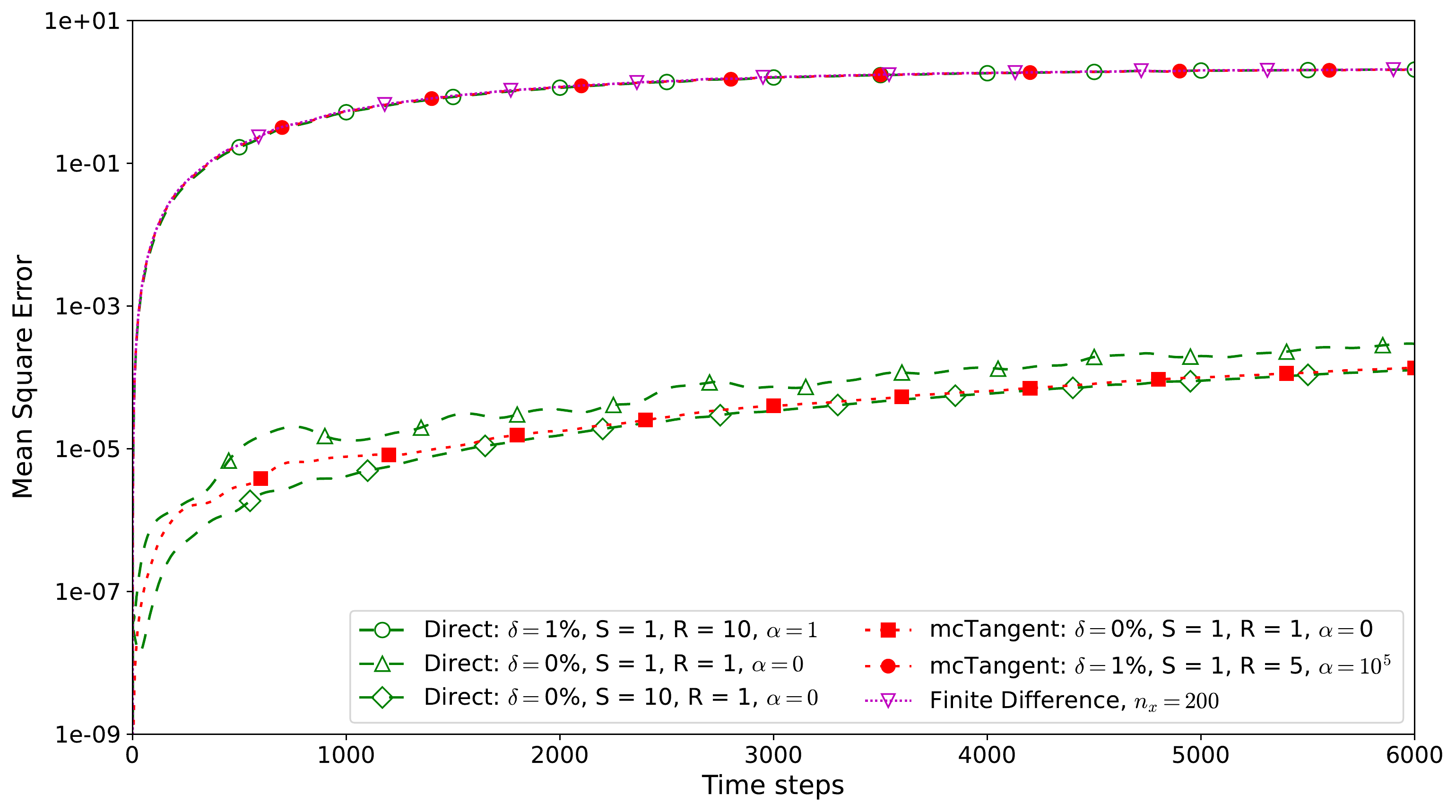}
    \caption{\textbf{Wave/transport equation}. Comparison between direct neural networks (Direct) and tangent slope neural networks (\texttt{mcTangent}).} 
    \figlab{1D_compare_Direct}
\end{figure}
\begin{figure}[htb!]
    \centering
    \begin{tabular*}{\textwidth}{c c c}
        \centering
        \raisebox{-0.5\height}{\small $S = 10, R = 1, \alpha = 0 $} &
        \raisebox{-0.5\height}{\small $S = 1, R = 1, \alpha = 0$} & 
        \raisebox{-0.5\height}{\small $S = 1, R = 5, \alpha = 1$} 
        \\
        \centering
        \raisebox{-0.5\height}{
        \begin{minipage}{.28\textwidth}
            \begin{tikzpicture}
                \tikzstyle{every node}=[font=\tiny, node distance=7.5mm]
                \node (img)  {\includegraphics[width = 0.95\textwidth]{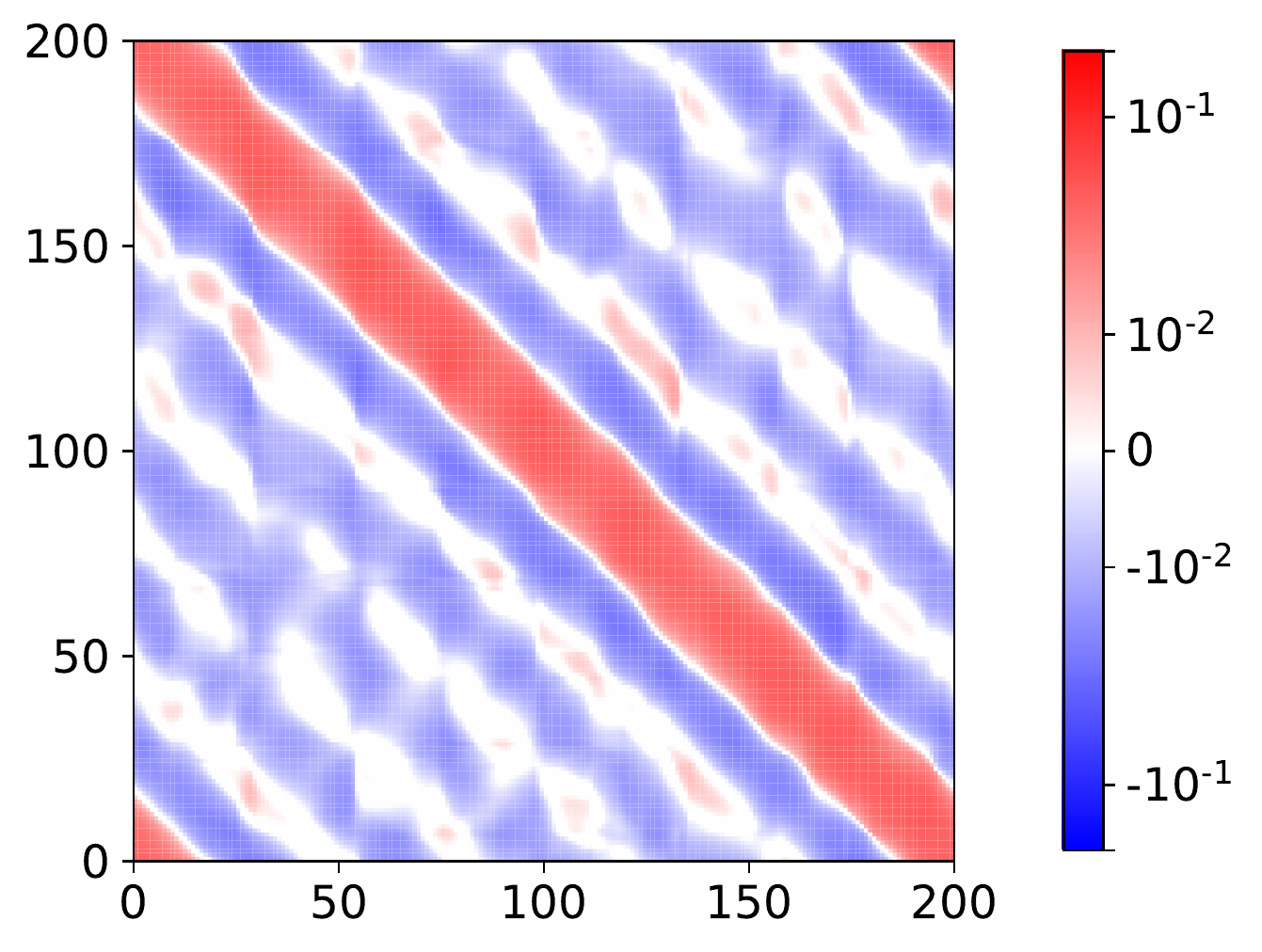}};
                \node[below=of img, node distance=0cm, xshift = -2.5mm, yshift=1cm] 
                {Column index of $\W$};
                \node[left=of img, node distance=0cm, rotate=90, anchor=center,yshift=-0.7cm] 
                {Row index of $\W$};
            \end{tikzpicture}
            \end{minipage}
        } &
        \raisebox{-0.5\height}{
        \begin{minipage}{.28\textwidth}
            \begin{tikzpicture}
                \tikzstyle{every node}=[font=\tiny, node distance=7.5mm]
                \node (img)  {\includegraphics[width = 0.95\textwidth]{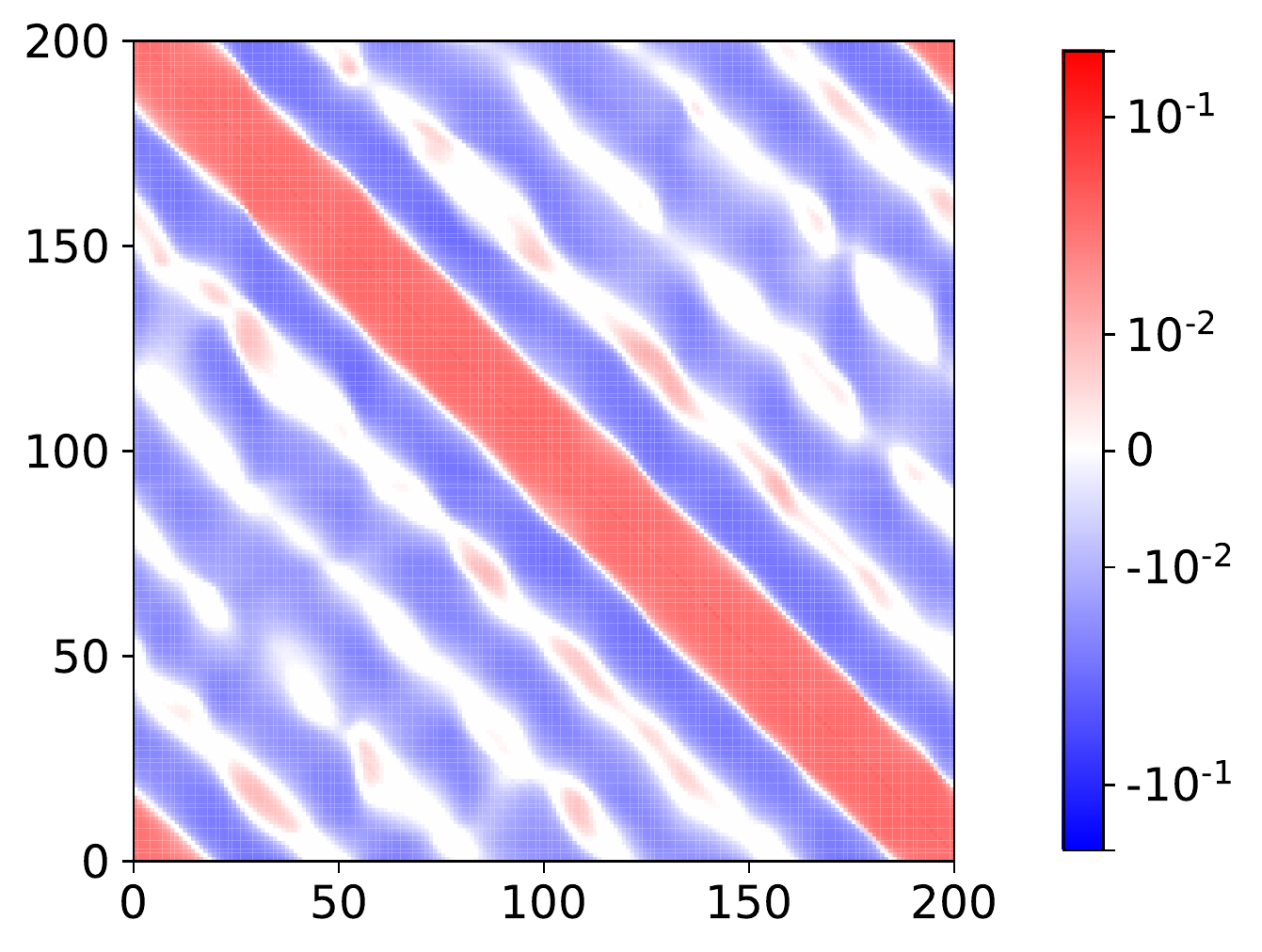}};
                \node[below=of img, node distance=0cm, xshift = -2.5mm, yshift=1cm] 
                {Column index of $\W$};
                \node[left=of img, node distance=0cm, rotate=90, anchor=center,yshift=-0.7cm] 
                {Row index of $\W$};
            \end{tikzpicture}
            \end{minipage}
        } & 
        \raisebox{-0.5\height}{
        \begin{minipage}{.28\textwidth}
            \begin{tikzpicture}
                \tikzstyle{every node}=[font=\tiny, node distance=7.5mm]
                \node (img)  {\includegraphics[width = 0.95\textwidth]{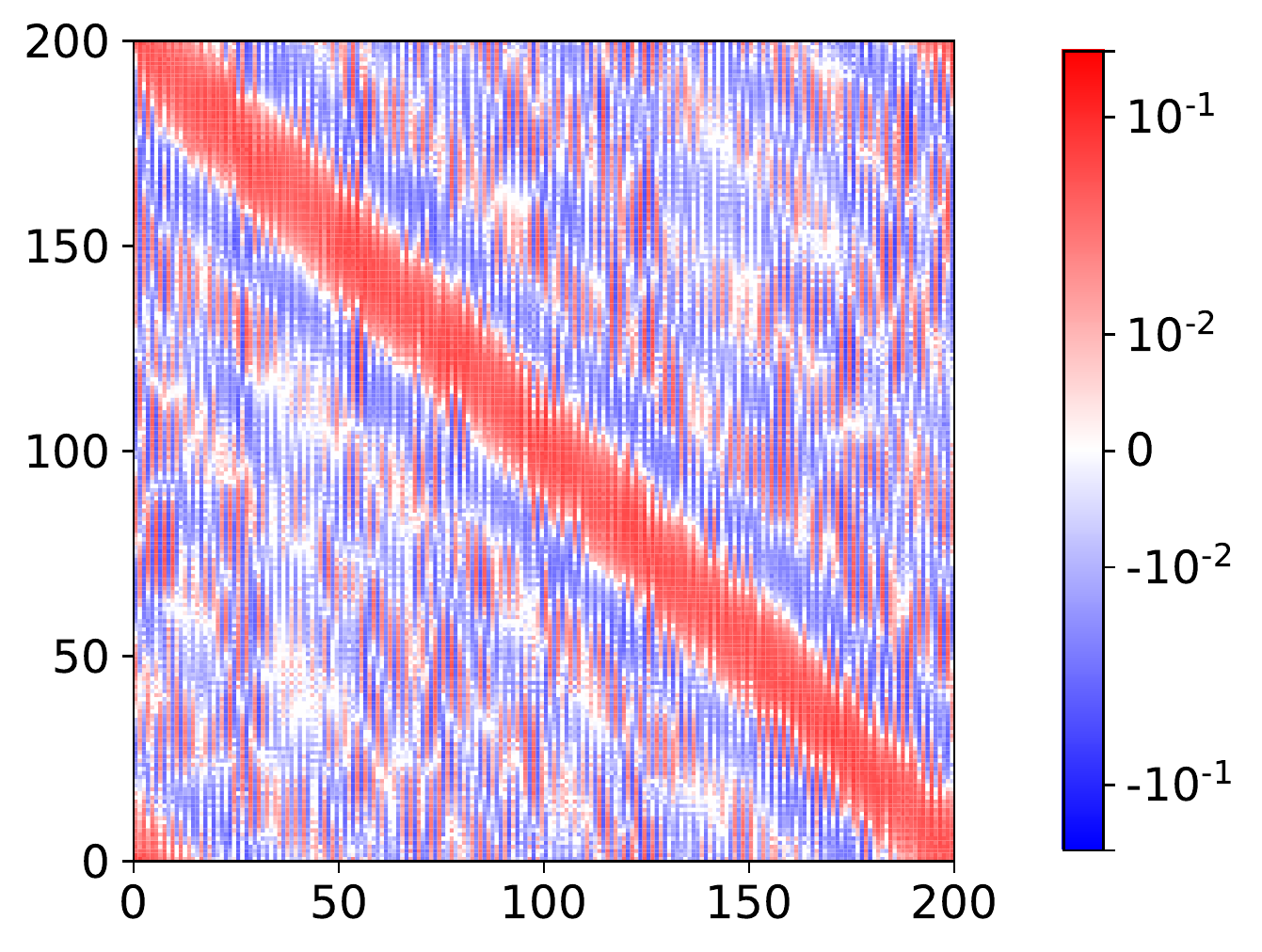}};
                \node[below=of img, node distance=0cm, xshift = -2.5mm, yshift=1cm] 
                {Column index of $\W$};
                \node[left=of img, node distance=0cm, rotate=90, anchor=center,yshift=-0.7cm] 
                {Row index of $\W$};
            \end{tikzpicture}
            \end{minipage}
        } 
    \end{tabular*}
    \caption{\textbf{Wave/transport equation}. Heat map of weight matrix  of direct linear neural networks.} 
    \figlab{1D_wave_FD_Direct_NN_params}
\end{figure}
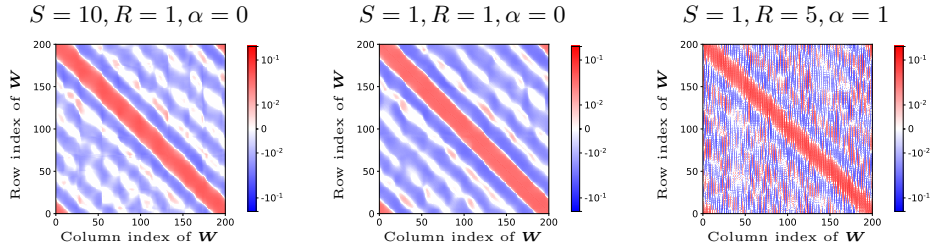
We now compare our tangent slope learning and direct learning. Here, by direct learning we mean learning the map from $\ui{i}$ to $\ui{i+1}$ for two consecutive time steps. Clearly, unlike the former, the latter  is tailored, and thus limited, to a particular space-time discretization.
%For direct neural network approach, we present the results with the case of the linear network consisting of only the weight matrix. We test both cases with and without a bias vector, we observe that the bias vector has an insignificant contribution.
To be fair, we also use the linear network with zero bias for the direct learning approach. \cref{fig:1D_compare_Direct} presents the mean-square error of predictions obtained by direct neural networks and tangent slope networks, both with and without model-constrained terms. As can be seen, both direct and tangent slope neural networks are comparable in terms of accuracy. However, the learned weight matrices of direct neural networks 
do not have the pattern of the underlying space-time discretization matrices, and this can be observed from \cref{fig:1D_wave_FD_Direct_NN_params}. That is, while our tangent slope approach preserves the structure of spatial discretization, the direct approach, which seems to be natural, does not.

%Meanwhile, the direct pure machine learning weight matrices have an interesting pattern. The reason is that learning a linear network that equivalent to $\Ib + \dt \mb{A}$ ($\Ib$ is the identity matrix and $\mb{A}$ is the approximated finite difference scheme) is challenging in the case of a small value of $\dt$. Specifically, if $\dt$ is small, so is $\dt \mb{A}$, thus $\Ib + \dt \mb{A}$ is slightly different from $\Ib$. We might be able to gain a perturbed $\Ib$, but it is very challenging to distinguish between perturbation and $\dt \mb{A}$. Therefore, the optimizer finds another optimal network instead. In summary, although both direct and RHS networks achieve almost the same level of accuracy, there is no correlation between their learned weight matrices, and hence no correlation between direct networks and the finite difference scheme. 
% Furthermore, it can be seen that the direct network approach does not allow for different time step size and implicit scheme approach, thus it is less flexible.

\subsection{2D Burger's equation}
\seclab{Burger_eq}
We consider the following viscous 2D Burger's equations
\begin{align*}
    & \pp{u}{t} + u \pp{u}{x} + v\pp{u}{y} = \nu \LRp{\pp{^2u}{x^2} + \pp{^2u}{y^2} } \\
    & \pp{v}{t} + u \pp{v}{x} + v\pp{v}{y} = \nu \LRp{\pp{^2v}{x^2} + \pp{^2v}{y^2} },
\end{align*}
where $x,y \in \LRs{0,1}$ and $ t \in (0,T] $. The boundary condition is periodic and the initial velocity is given by $v(x,y,0) = v_0\LRp{x,y} = 1$ and $u(x,y,0) = u_0\LRp{x,y}$. We take the viscosity coefficient to be $\nu = 10^{-2}$. We aim to predict $x$-velocity $u$ in the time interval $t \in \LRp{0, 1.5}$ given an initial velocity  $u_0(x,y)$ at $t = 0$.

{\bf Data generation.}
%To generate training data for learning the tangent slope, firstly, 
We draw periodic samples of $\ub$ using the truncated Karhunen-Loève expansion
\[u_0(x,y) = \exp \LRp{\sum_{i=1}^{15} \sqrt{\lambda_i} \, \omega_i(x,y) \, z_i},\]
where $\textbf{z} = \LRc{z_i}_{i=1}^{15} \sim \mathcal{N} \LRp{0, \textbf{I}}$, and $\LRp{\lambda, \omega}$ are eigenpairs obtained by the eigendecomposition of the covariance operator $7^{\frac{3}{2}} \LRp{-\Delta + 49 \textbf{I}}^{-2.5}$, where $\Delta$ is the Laplacian operator, with periodic boundary conditions. Training data corresponding to each initial velocity is generated from a $128 \times 128$ high-resolution spatial mesh
and $1000$ time steps for the time horizon $T = 0.1$
using finite difference method. These high resolution solutions are down-sampled on a coarser mesh of 100 time steps ($\dt = 10^{-3}$) and $32 \times 32$ spatial mesh. These down-sampled  solutions are treated as true solutions for the training process. Meanwhile, we draw 10 test initial velocity samples independently, and  the corresponding test data set of 10 samples is created in the same manner. However, the time horizon $T = 1.5$ for test samples is chosen\textemdash much larger than
 the trained time horizon\textemdash  with time stepsize $\dt = 10^{-3}$. This helps us test the accuracy and stability of neural network solutions beyond the training regime. 
 %, and thus we have 1500 snapshots for each initial velocity.

{\bf Neural network architecture.}
We use a shallow network of one layer with 5000 neurons for  all cases to approximate the tangent slope of Burger's equations.  
%This shallow network works well for a wide range of data sets. 
Note that we have compared the one-layer network with two- and three-layer networks with different numbers of neurons ranging from 100 to 5000. These deeper networks perform poorly with small data sets and are improved with  large data sets in which the shallow one has comparable performance. Note that one-layer neural network approximation capabilities are rigorously justified by past universal approximation theories (see, e.g. \cite{Cybenko1989,hornik1989multilayer,Zhou17,johnson2018deep}) and our current work \cite{BuiUniversal21}.
Thus we shall use a one-layer neural network for all numerical results.
In addition, ReLU \cite{nair2010rectified} is used as the activation function.
%to capture the nonlinearity of the map from state to the equivalent tangent slope. 
\texttt{ADAM} optimizer is used with the learning rate of $10^{-4}$ and the training batch size is 40 samples. 
%Although we verify the effectiveness of approaches based on the single-time-point mean square error, the best model is not selected based on this value since it is inconsistent over the predictive period. For example, the network giving a low mean square error at the $500$-th time point is neither necessarily good at other time points nor the best for overall. Instead of that, 
For this example,  reasonably optimal weights/biases are  the ones giving the lowest accumulated mean square error after $1500$ time steps for 10 test data. %The accumulated error is a more reliable measurement for long-term accuracy. 
%We verify that the neural network based on the accumulated mean square error at another time point, i.e., the $500$-th step, leads to almost the same results. 
We take $\alpha = 10^5$ for the regularization parameter as this gives the best results from our numerical experiments (not shown here).
%We train networks with other larger values $\alpha \leq 10^7$, and almost similar results are obtained.

\begin{figure}[htb!]
    \centering
    \includegraphics[width=\textwidth]{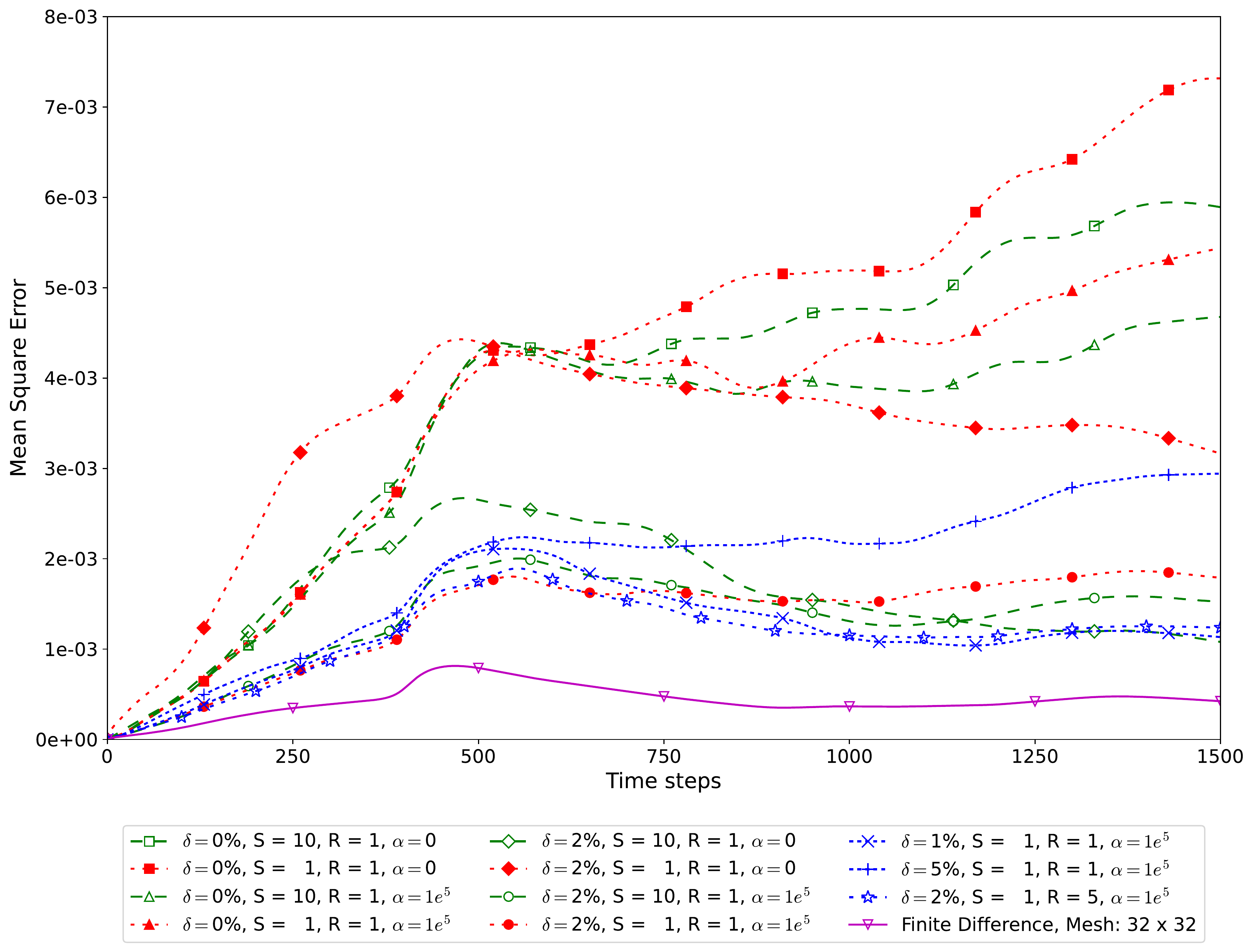}
    \caption{\textbf{Burger's equations}. Comparison of mean square error among different neural networks trained with $200$ data with/without noise. Recall that $\alpha = 0$ corresponds to the pure data-driven neural network training without model-constrained terms. Forward solver denotes the numerical solution on $32\times 32$ spatial mesh. } 
    \figlab{2D_Bur_compare_all}
\end{figure}

{\bf Comparison of different learned neural networks.}
%Unlike the linear wave equation problem, for Burger's equation, it is much more challenging to learn a nonlinear neural network that offers a high-resolution accuracy level from high-resolution training data. 
% For example, one might, in most cases, be only able to obtain a good local minimal for the set of a large number of parameters of the neural network, and hence the results are, in general, inaccurate. 
%In addition, inspired by the linear problem, introducing the model-constrained term allows the neural network to reach the same accuracy level as the embedded forward scheme. Therefore, we expect that our model-constrained neural network, if possible, gets as much accurate as a reasonable lower resolution forward solver. In this problem, the reasonable finite difference solver with the mesh grid $32 \times 32$ is used. It is noticeable that we should not use a too low-resolution solver since the model-constrained term degenerates the neural network to the same coarse-grained mesh accuracy level of the solver.
\cref{fig:2D_Bur_compare_all} presents the comparison of mean square error obtained by different learned neural networks  with the data set of 200 samples. It can be seen that, in general, the model-constrained neural networks are far better than their pure data-driven counterparts  (i.e. with $\alpha = 0$). Additionally, long sequential machine learning trainings with $S = 10$ provide slightly better accuracy than  $S = 1$, except for the noisy data with pure data-driven network in which the improvement is significant. 
%To be more specific, for the pure machine learning neural network method, the accuracy level of networks $\LRp{d200,0\%, 10, 1, 0}$ is a bit higher in the long-term period from $500$-th to $1500$-th time step than networks $\LRp{d200,0\%, 1, 1, 0}$. By contrast, training with noisy data, the $\LRp{2\%, 10, 1, 0}$ network is significant better than $\LRp{d200,2\%, 1, 1, 0}$ network over the same long-term period.

For model-constrained neural networks, long sequential training results with $S = 10$ in two settings $\LRp{d200,0\%, 10, 1, 10^5}$ and $\LRp{d200,2\%, 10, 1, 10^5}$ show an marginal improvement compared to short sequential training with $S = 1$ in two settings $\LRp{d200,0\%, 1, 1, 10^5}$ and $\LRp{d200,2\%, 1, 1, 10^5}$. Therefore, $S = 1$ is sufficient and we use it for the rest of numerical results with model-constrained neural networks. \cref{fig:2D_Bur_compare_all} shows that  using  $5\%$ noise causes the neural network corresponding to $\LRp{d200,5\%, 1, 1, 10^5}$ to perform poorly, while  $1\%$ noise gives almost the same accuracy as $2\%$ noise. It is noticeable that the long sequential model-constrained training with $R = 5$ ($d200,2\%, 1, 5, 10^5$) yields higher accuracy than the others.  However, large $R$ is more computationally expensive  since many passes through the back-propagation computational graph are needed.

\begin{figure}[htb!]
    \centering
    \includegraphics[width=\textwidth]{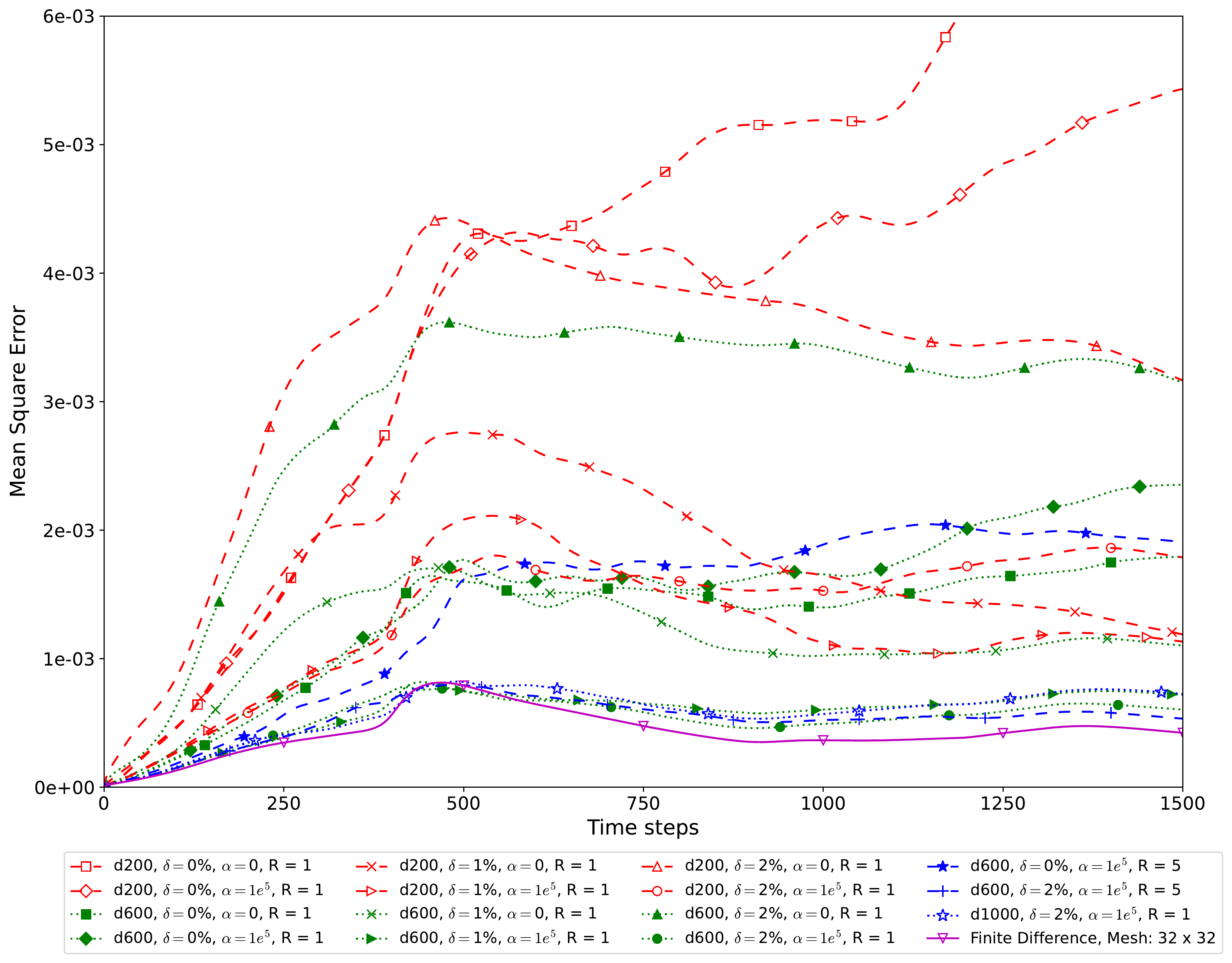}
    \caption{\textbf{Burger's equations}. The mean square error versus the number of time steps for various learned neural networks using 200, 600, and 1000 data samples with $S = 1$. } 
    \figlab{2D_Bur_d200d600}
\end{figure}
 \begin{figure}[htb!]
    \centering
    \begin{tabular*}{\textwidth}{c c c c c}
        \centering
         &
        \raisebox{-0.5\height}{\small $n_t = 0$} &
        \raisebox{-0.5\height}{\small $n_t = 100$} &
        \raisebox{-0.5\height}{\small $n_t = 500$} & 
        \raisebox{-0.5\height}{\small $n_t = 1500$} 
        \\
        \centering
        \rotatebox[origin=c]{90}{\small True $\ub$} &
        \raisebox{-0.5\height}{\includegraphics[width=.20 \textwidth]{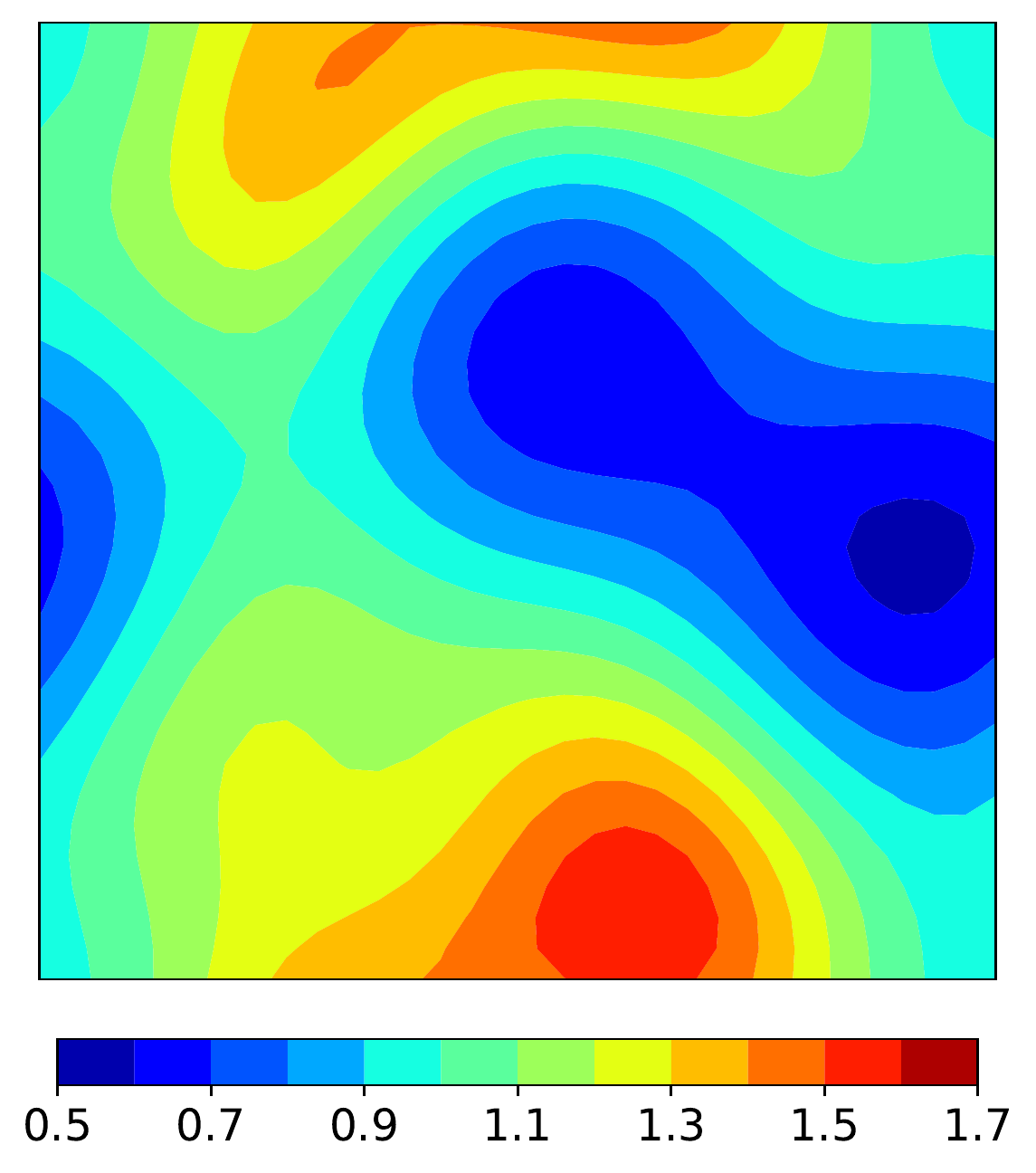}} &
        \raisebox{-0.5\height}{\includegraphics[width=.20 \textwidth]{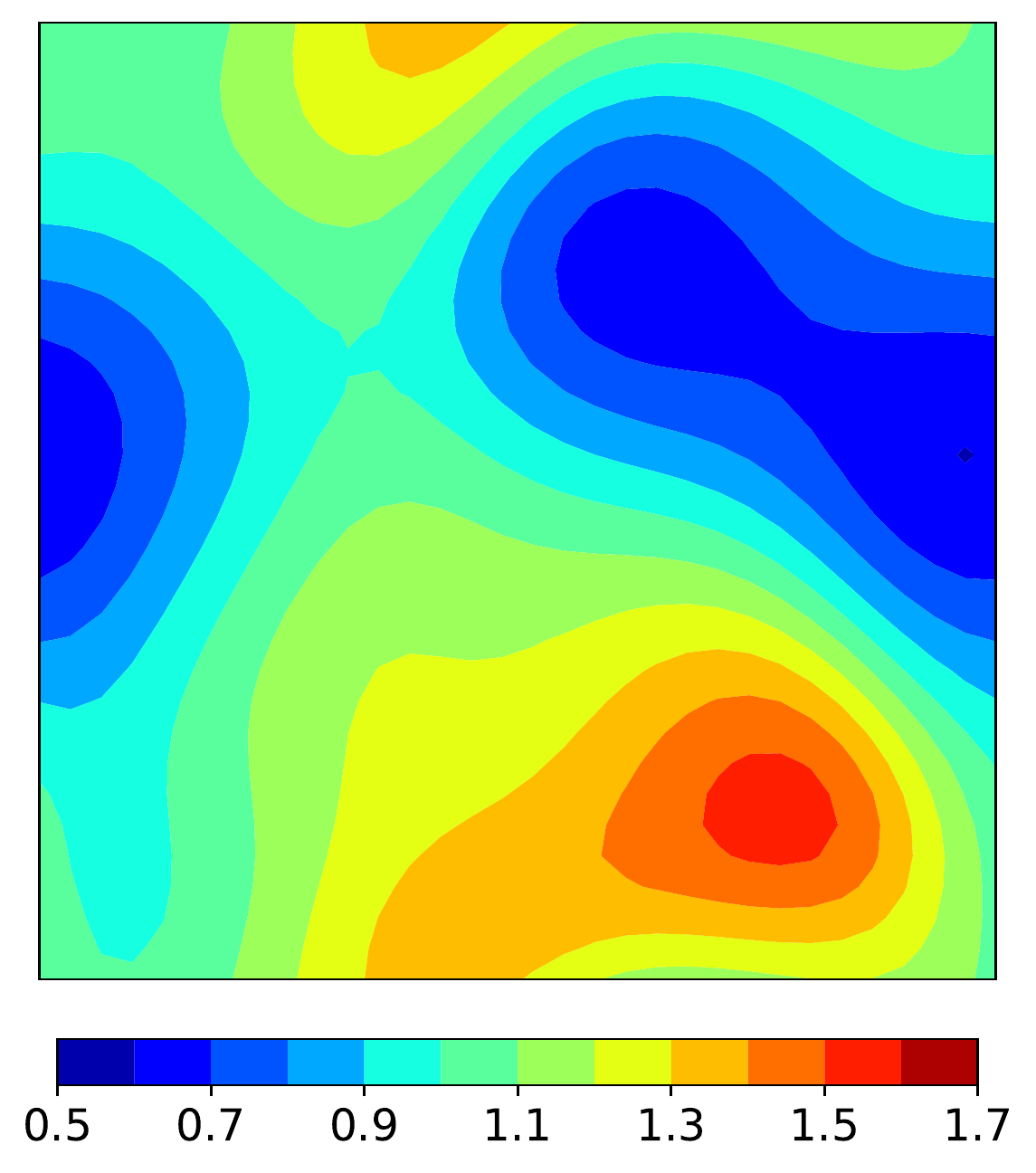}} &
        \raisebox{-0.5\height}{\includegraphics[width=.20 \textwidth]{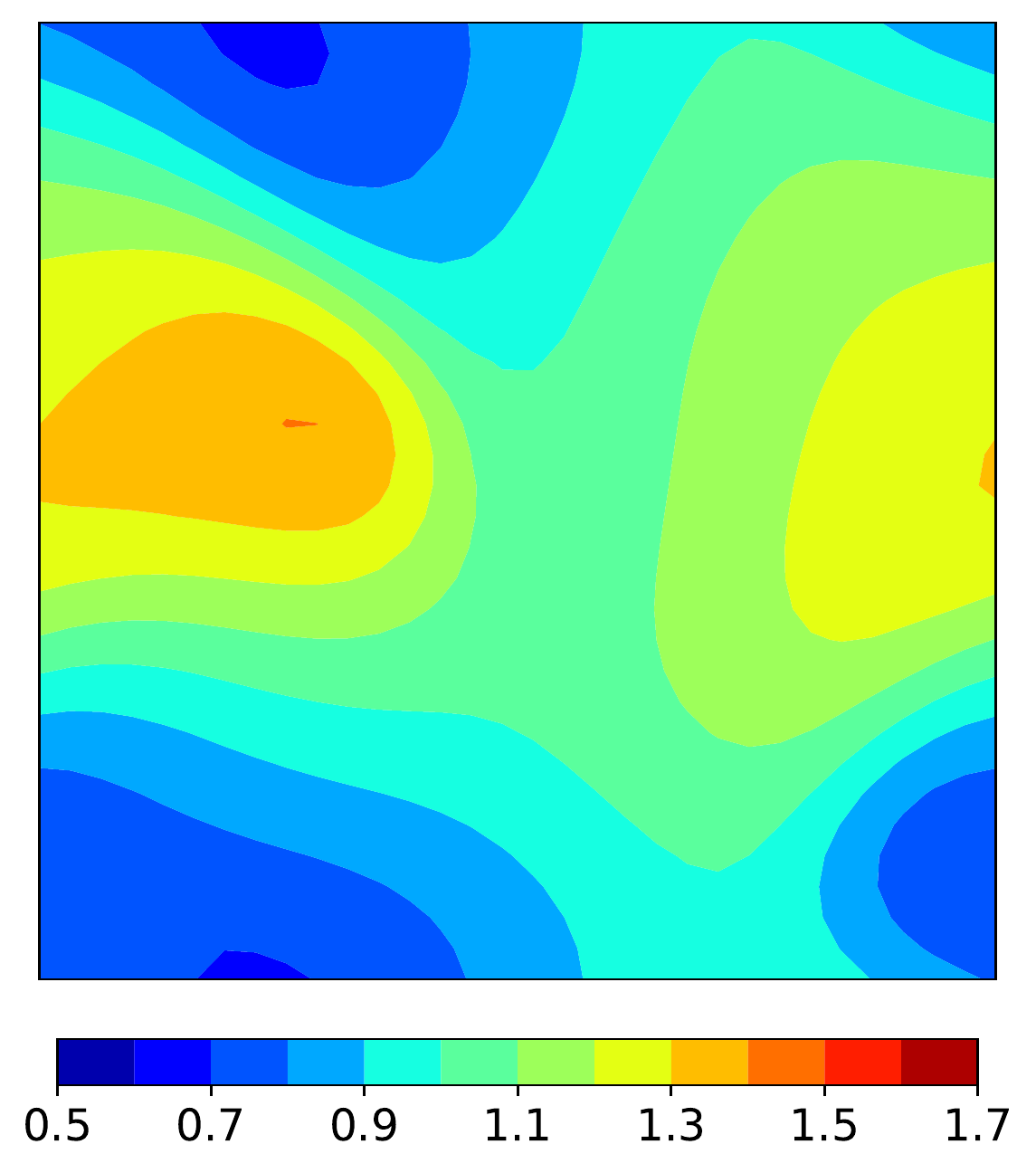}} & 
        \raisebox{-0.5\height}{\includegraphics[width=.20 \textwidth]{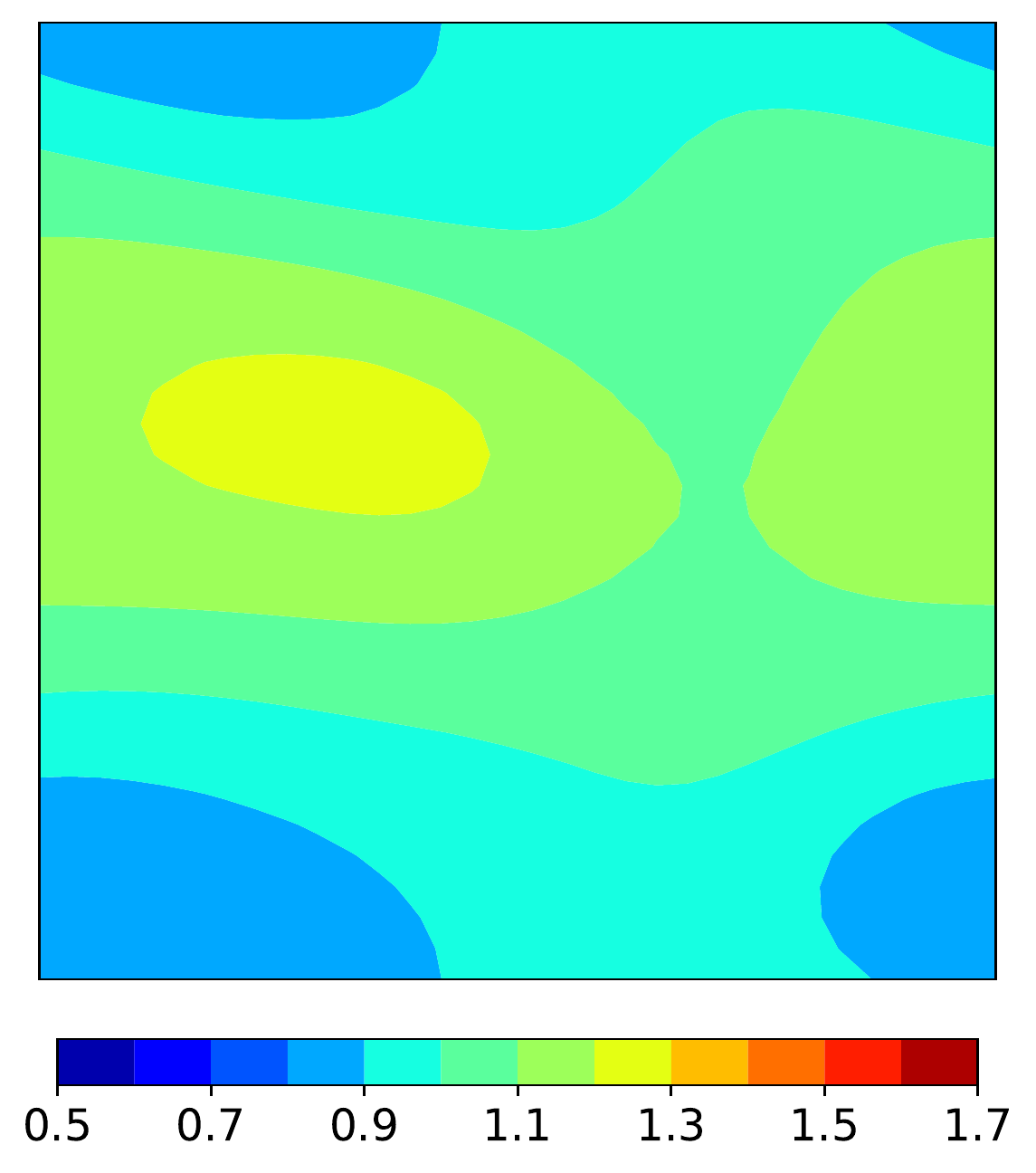}} 
        \\
        \centering
        \rotatebox[origin=c]{90}{\small $\alpha =0,\delta = 0\%$} &
        \raisebox{-0.5\height}{\includegraphics[width=.20 \textwidth]{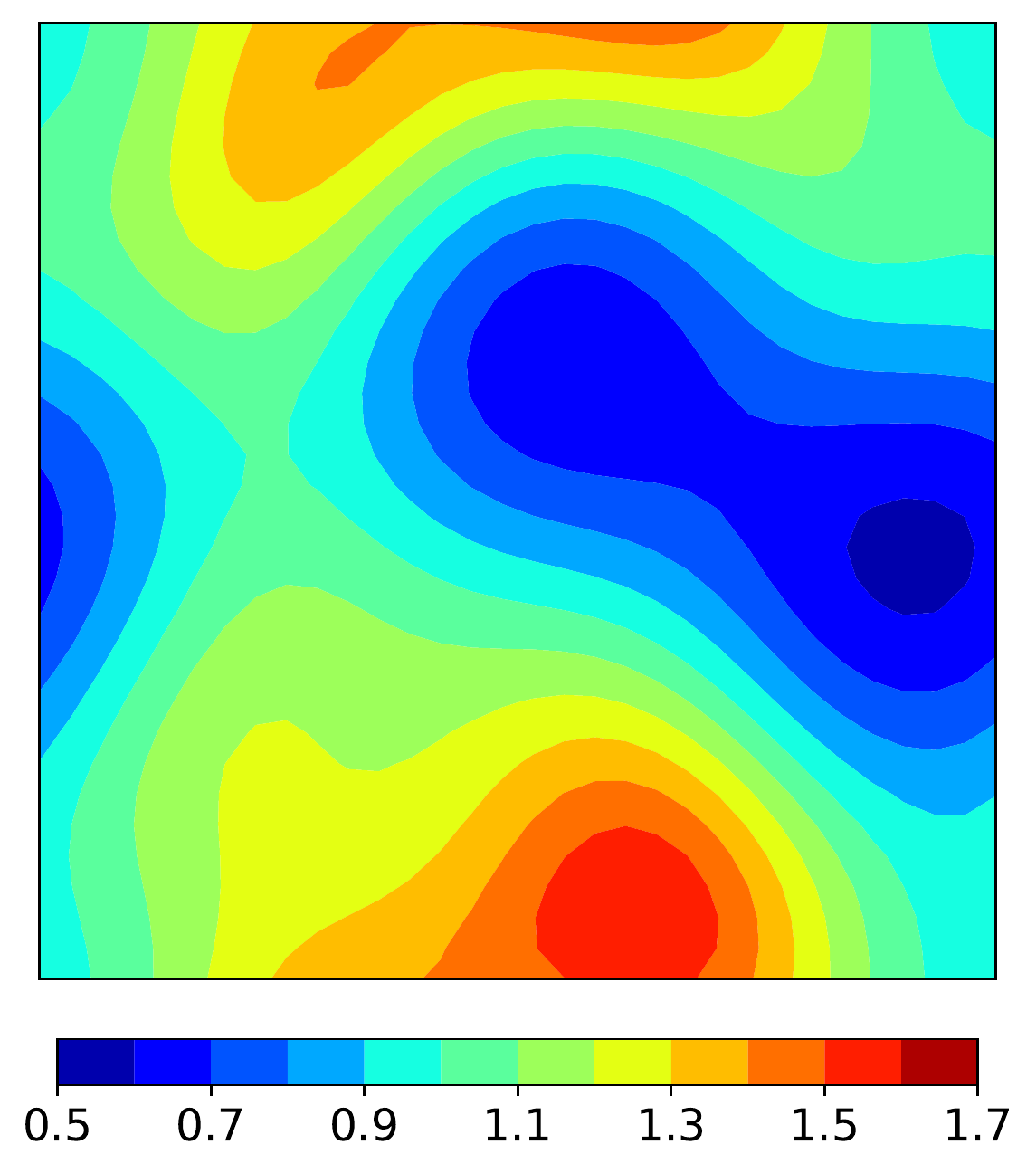}} &
        \raisebox{-0.5\height}{\includegraphics[width=.20 \textwidth]{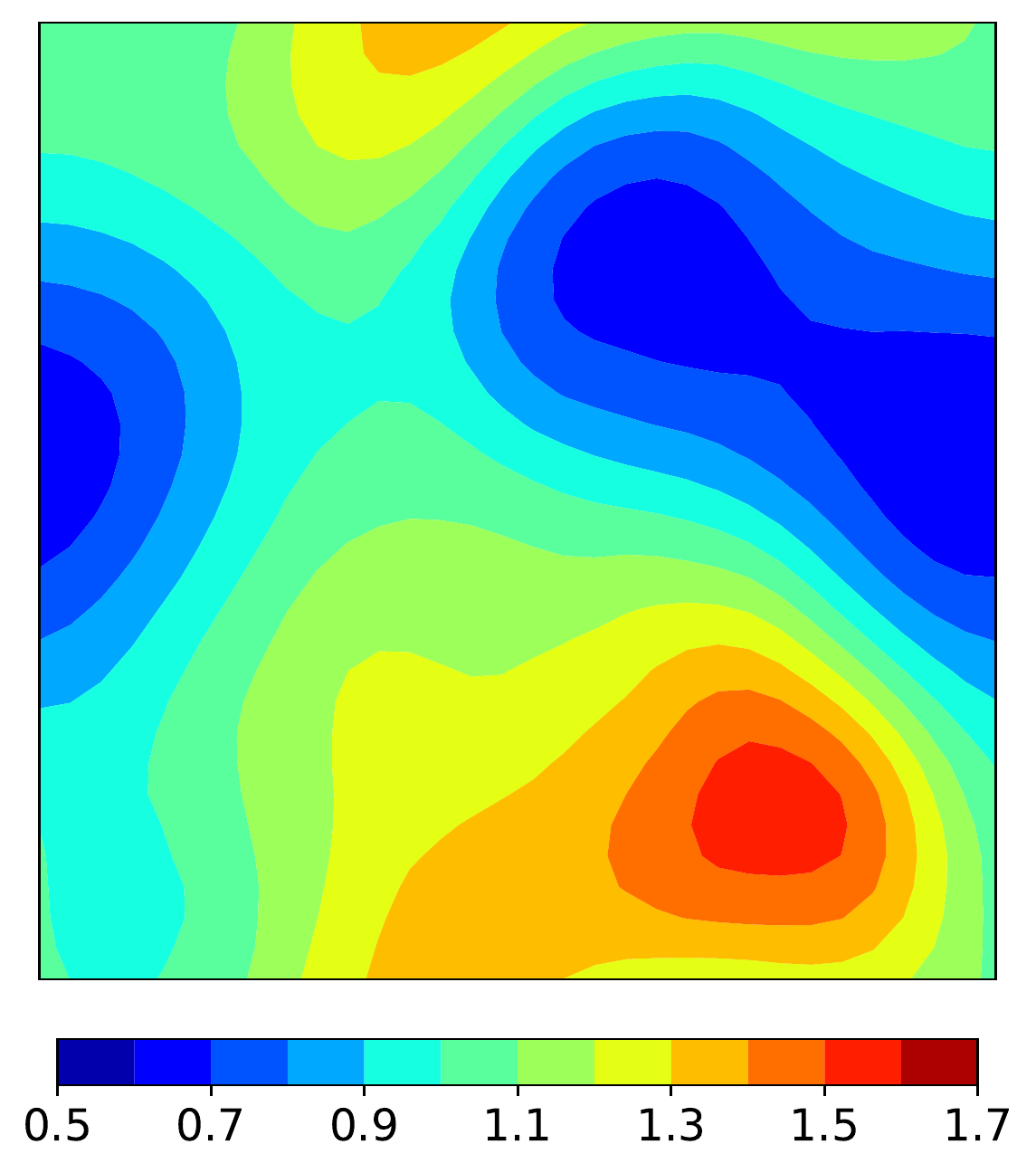}} &
        \raisebox{-0.5\height}{\includegraphics[width=.20 \textwidth]{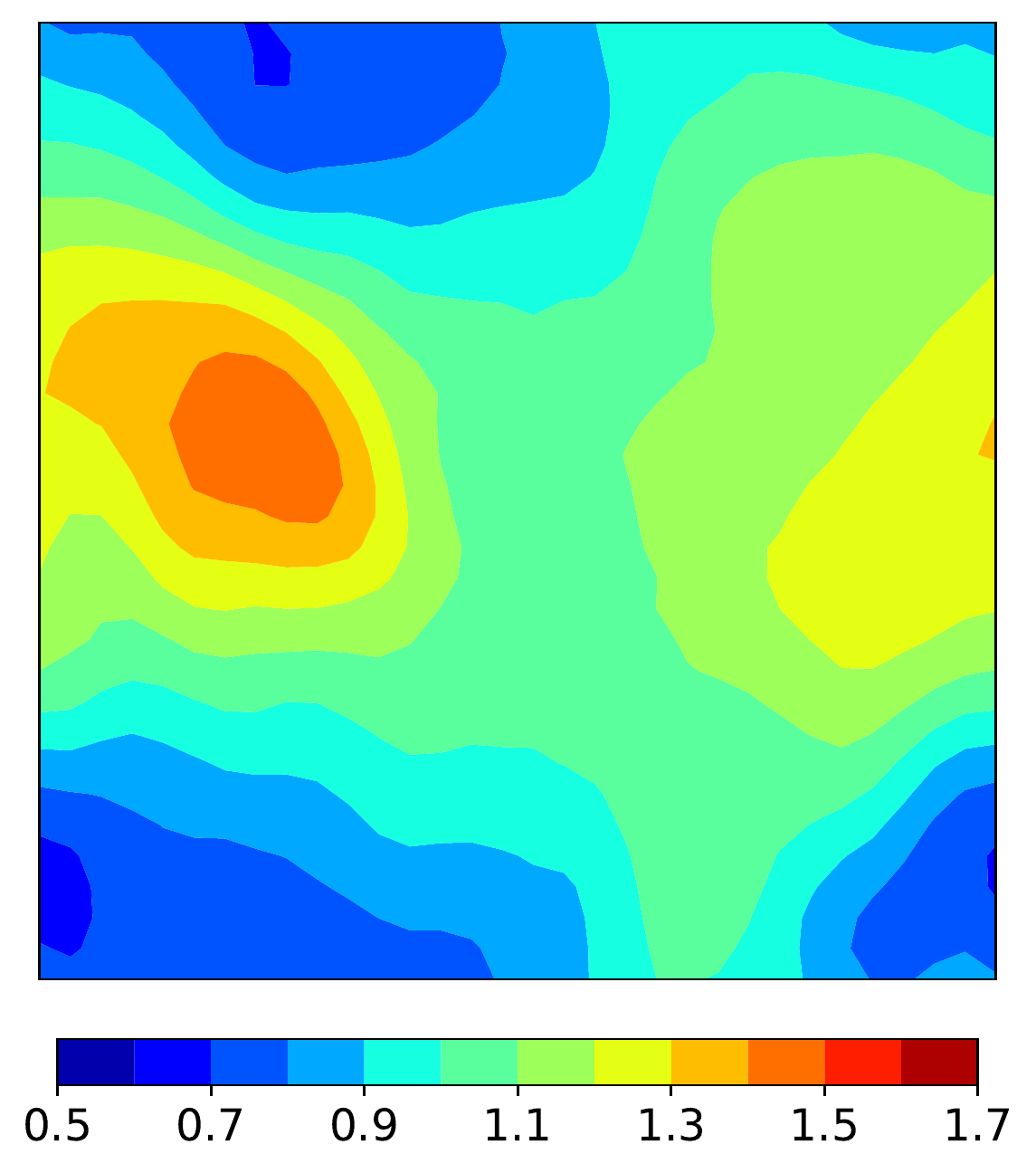}} & 
        \raisebox{-0.5\height}{\includegraphics[width=.20 \textwidth]{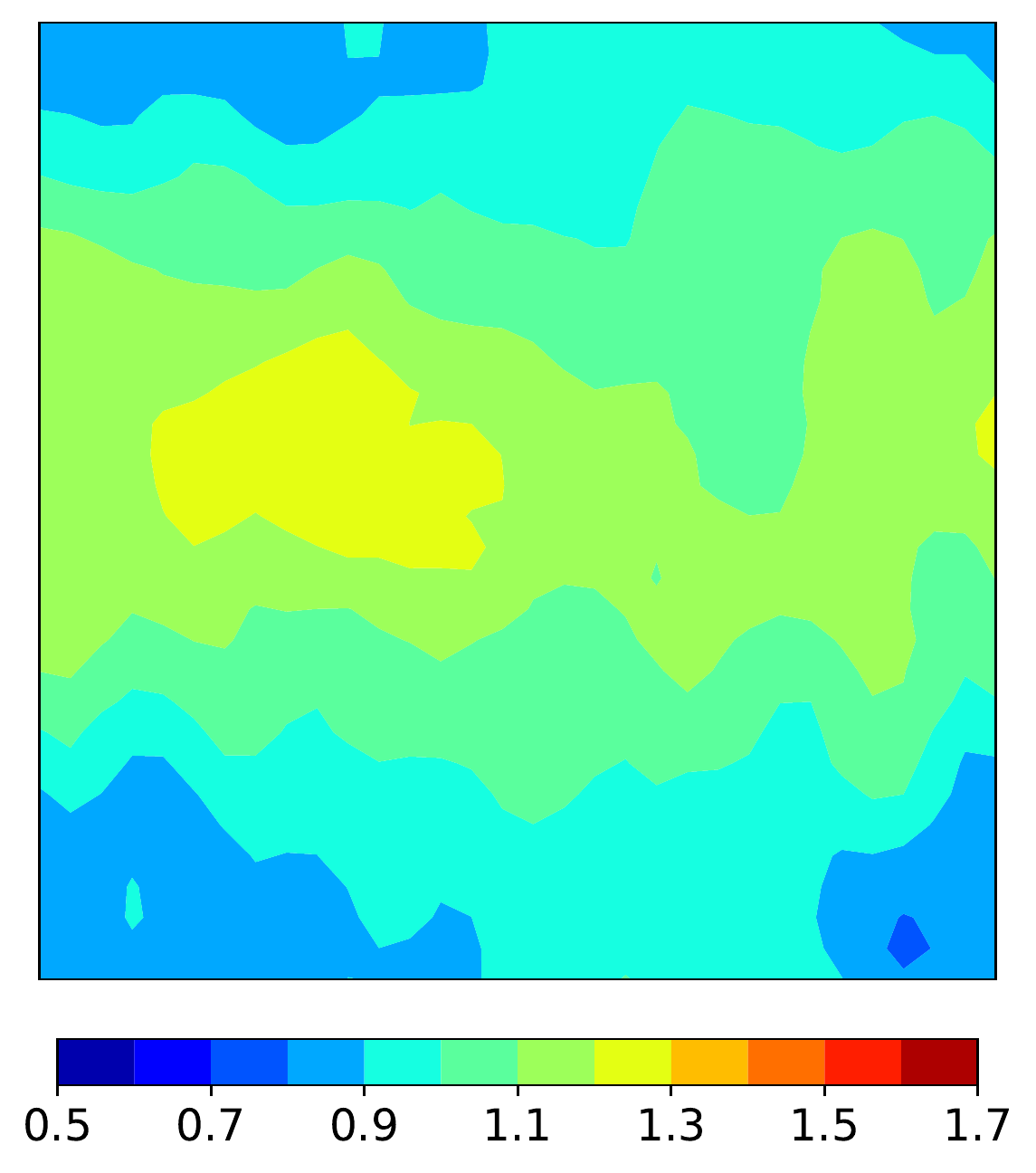}} 
        \\
        \centering
        \rotatebox[origin=c]{90}{\small $\alpha =0,\delta = 2\%$} &
        \raisebox{-0.5\height}{\includegraphics[width=.20 \textwidth]{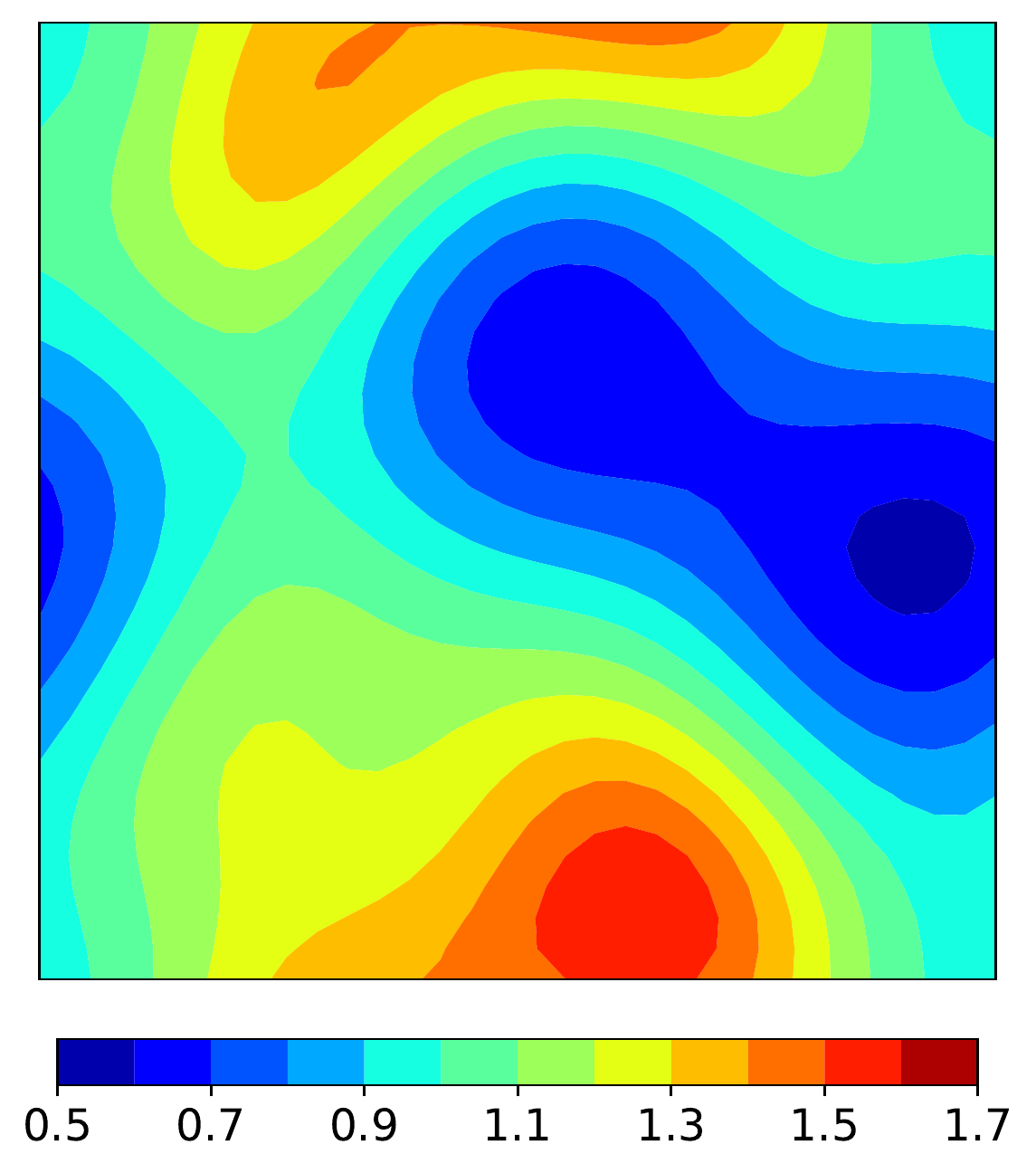}} &
        \raisebox{-0.5\height}{\includegraphics[width=.20 \textwidth]{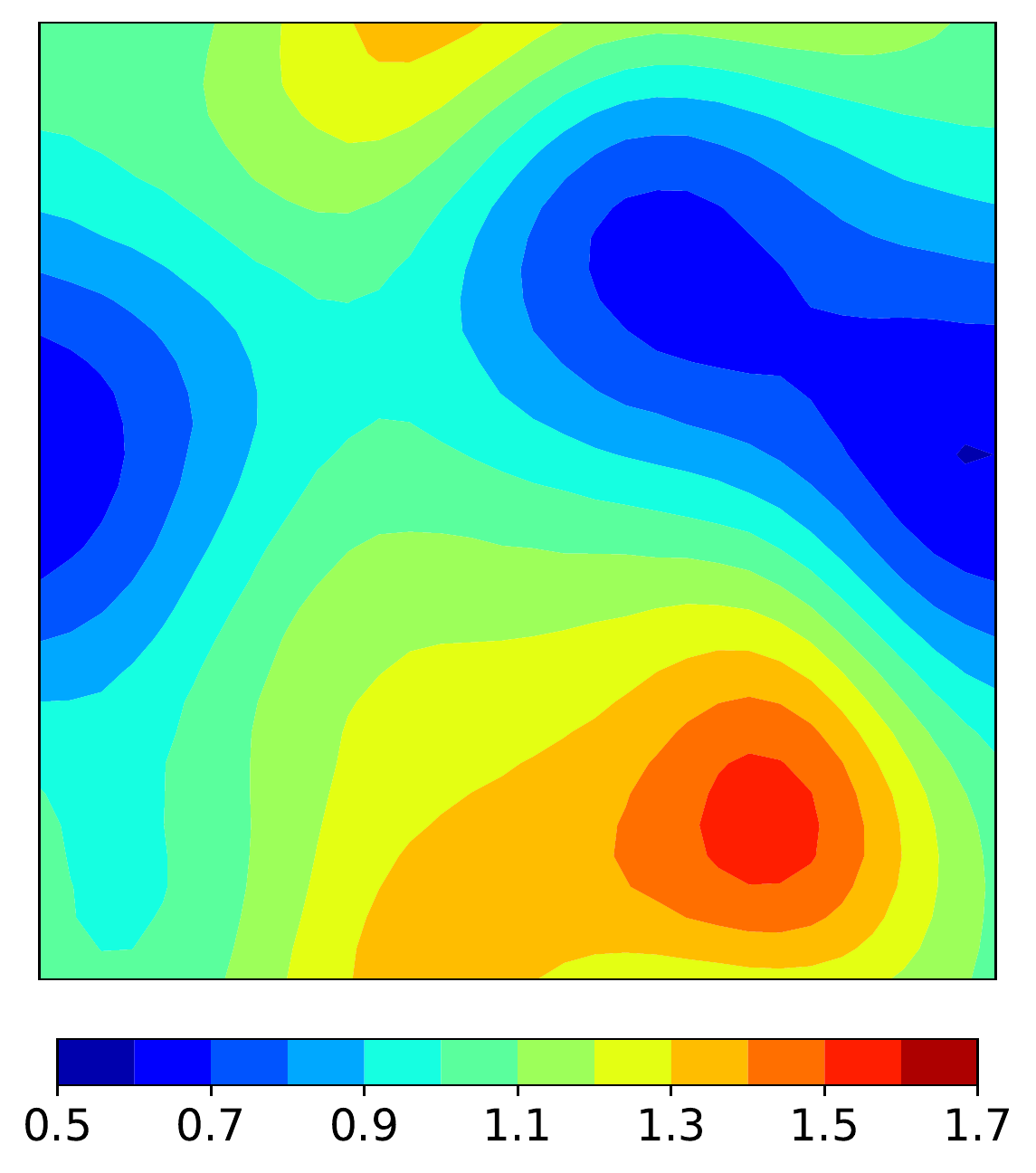}} &
        \raisebox{-0.5\height}{\includegraphics[width=.20 \textwidth]{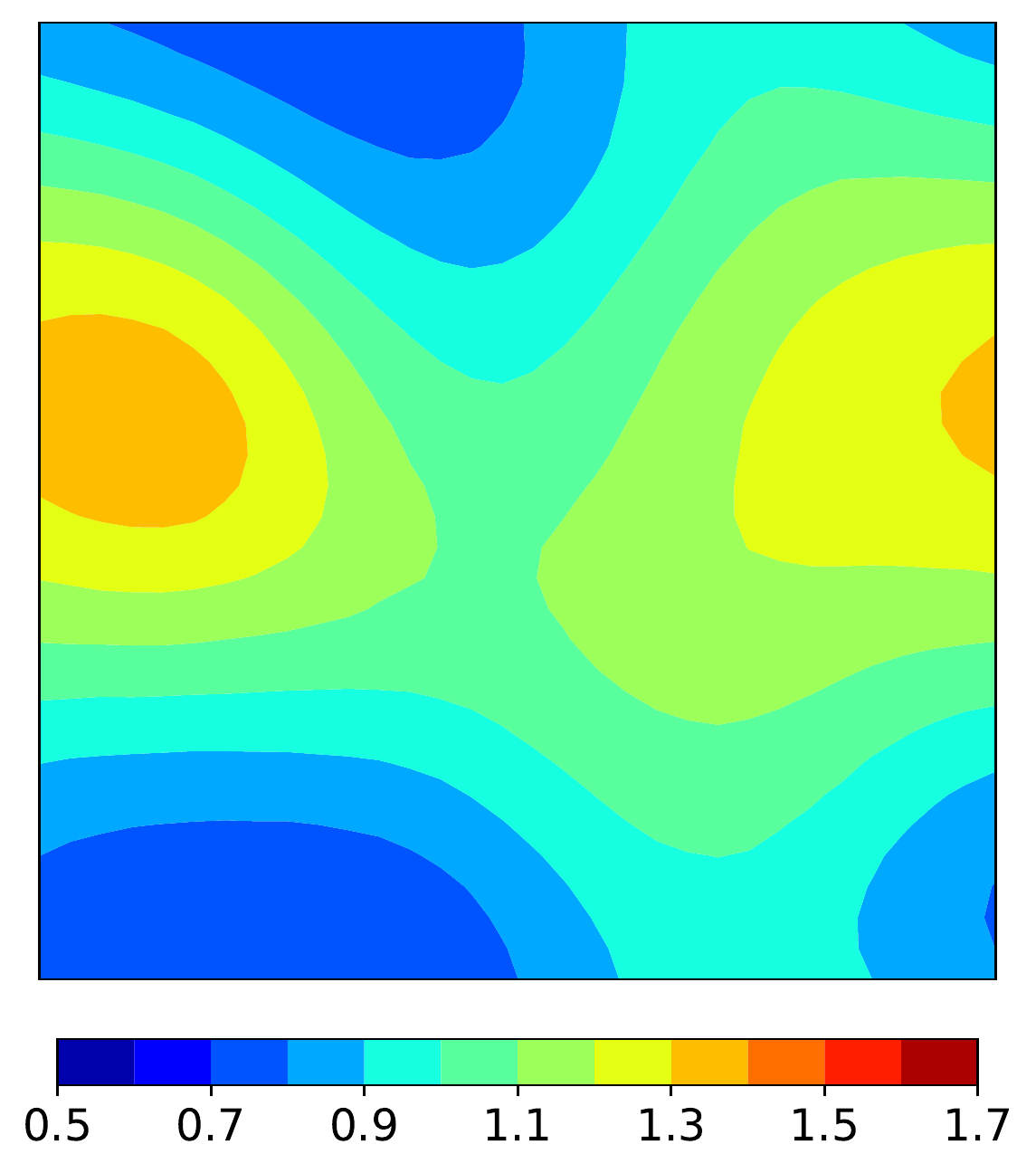}} & 
        \raisebox{-0.5\height}{\includegraphics[width=.20 \textwidth]{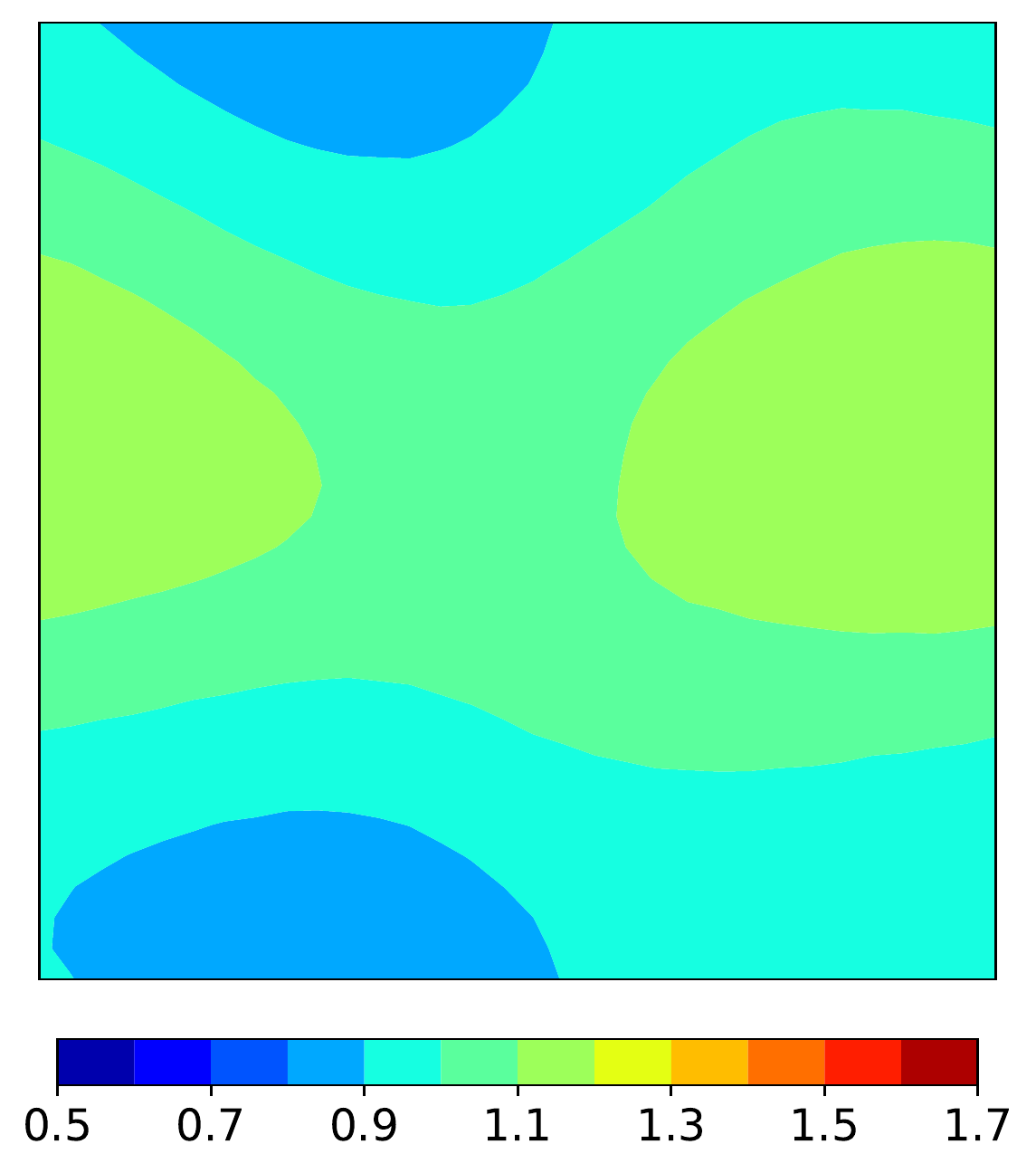}} 
        \\
        \centering
        \rotatebox[origin=c]{90}{\small $\alpha = 1e^5,\delta = 0\%$} &
        \raisebox{-0.5\height}{\includegraphics[width=.20 \textwidth]{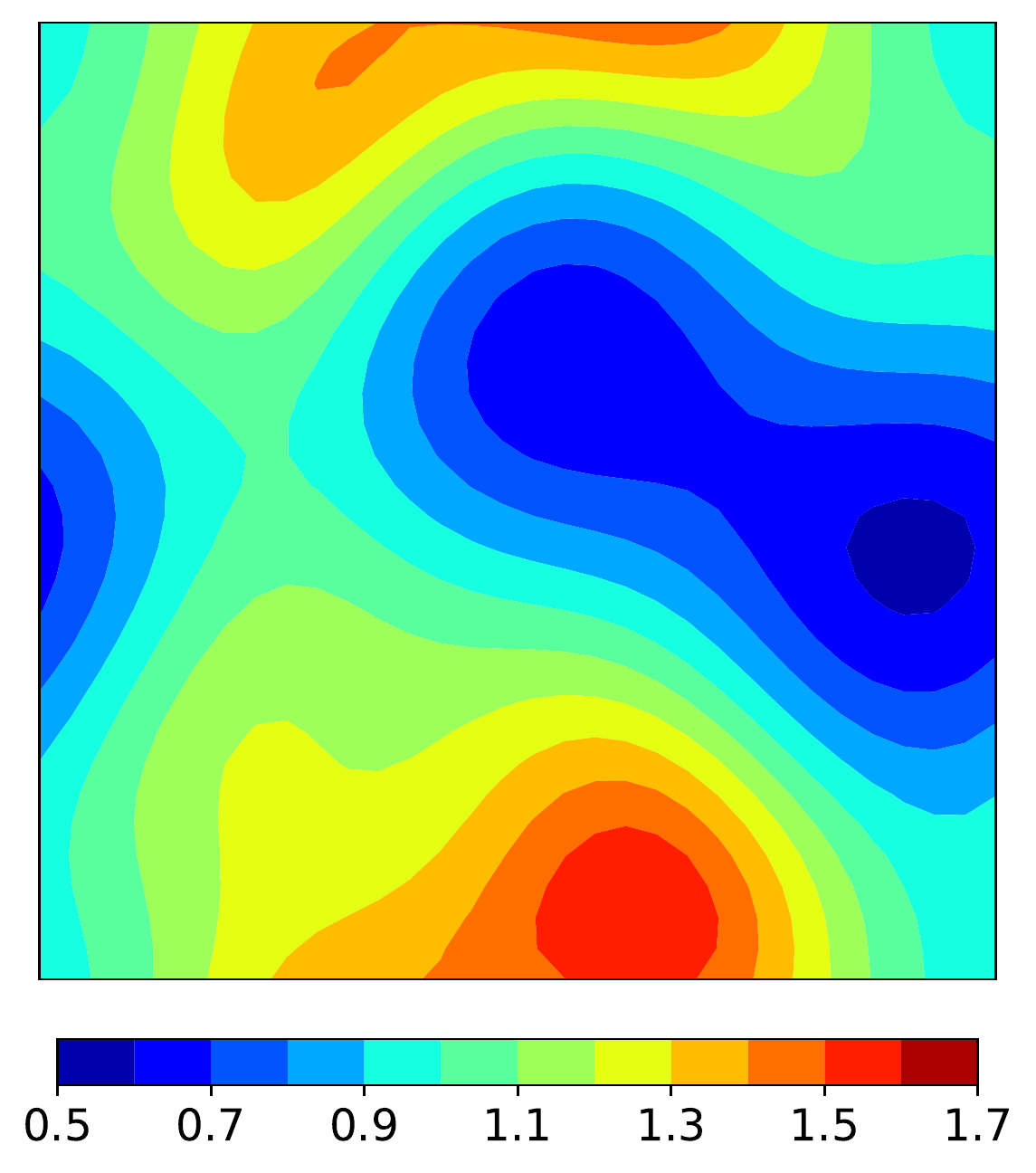}} &
        \raisebox{-0.5\height}{\includegraphics[width=.20 \textwidth]{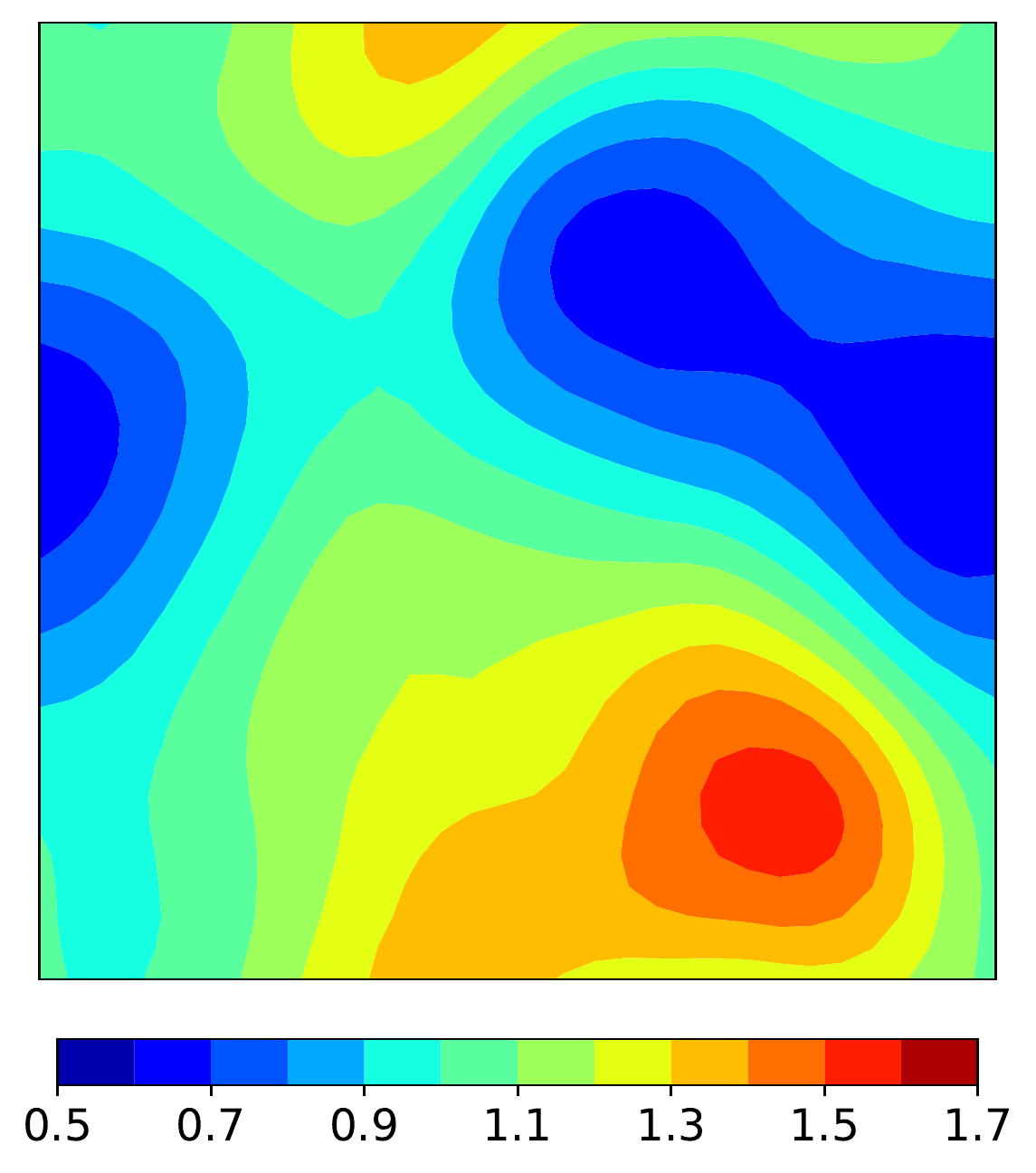}} &
        \raisebox{-0.5\height}{\includegraphics[width=.20 \textwidth]{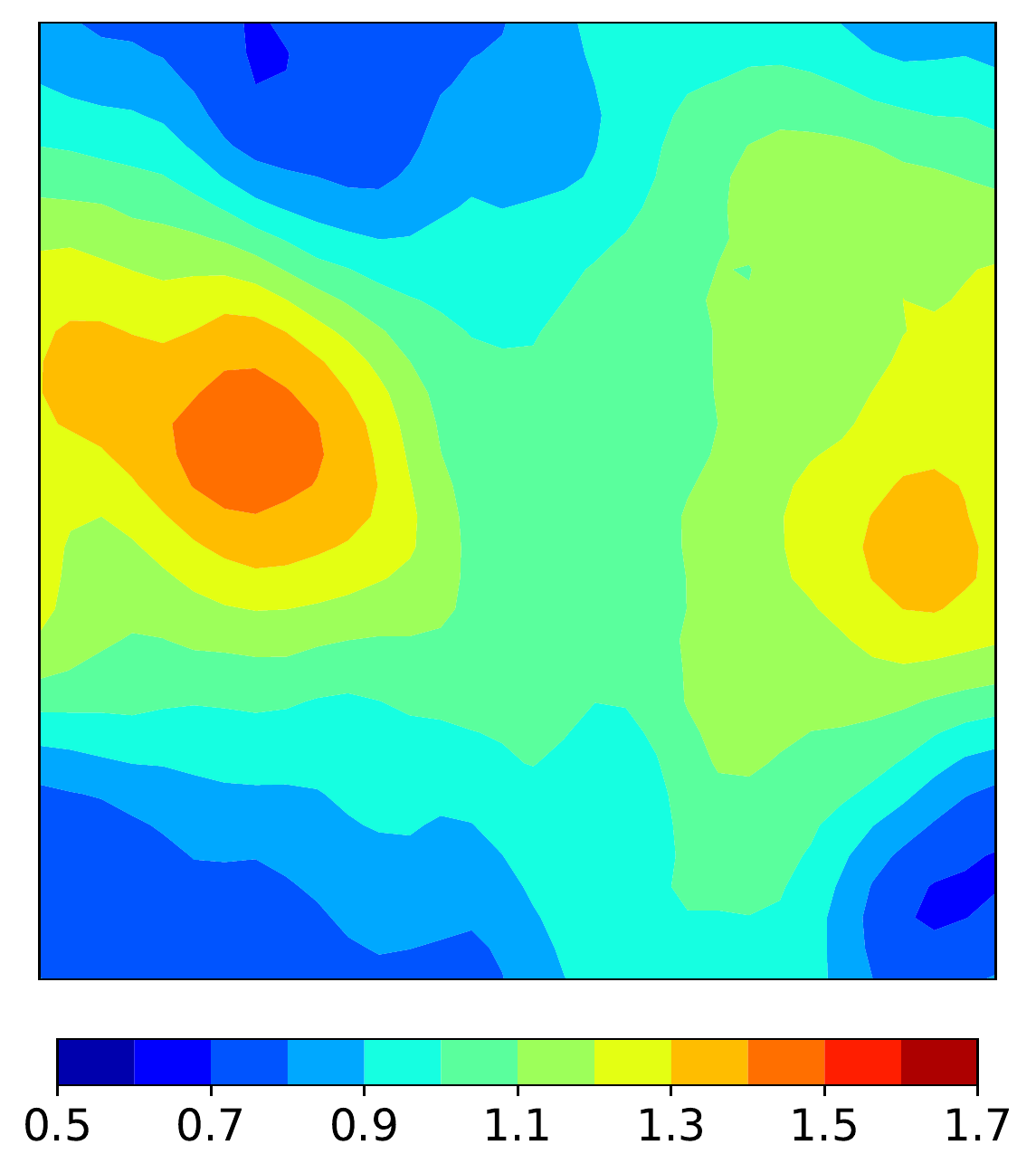}} & 
        \raisebox{-0.5\height}{\includegraphics[width=.20 \textwidth]{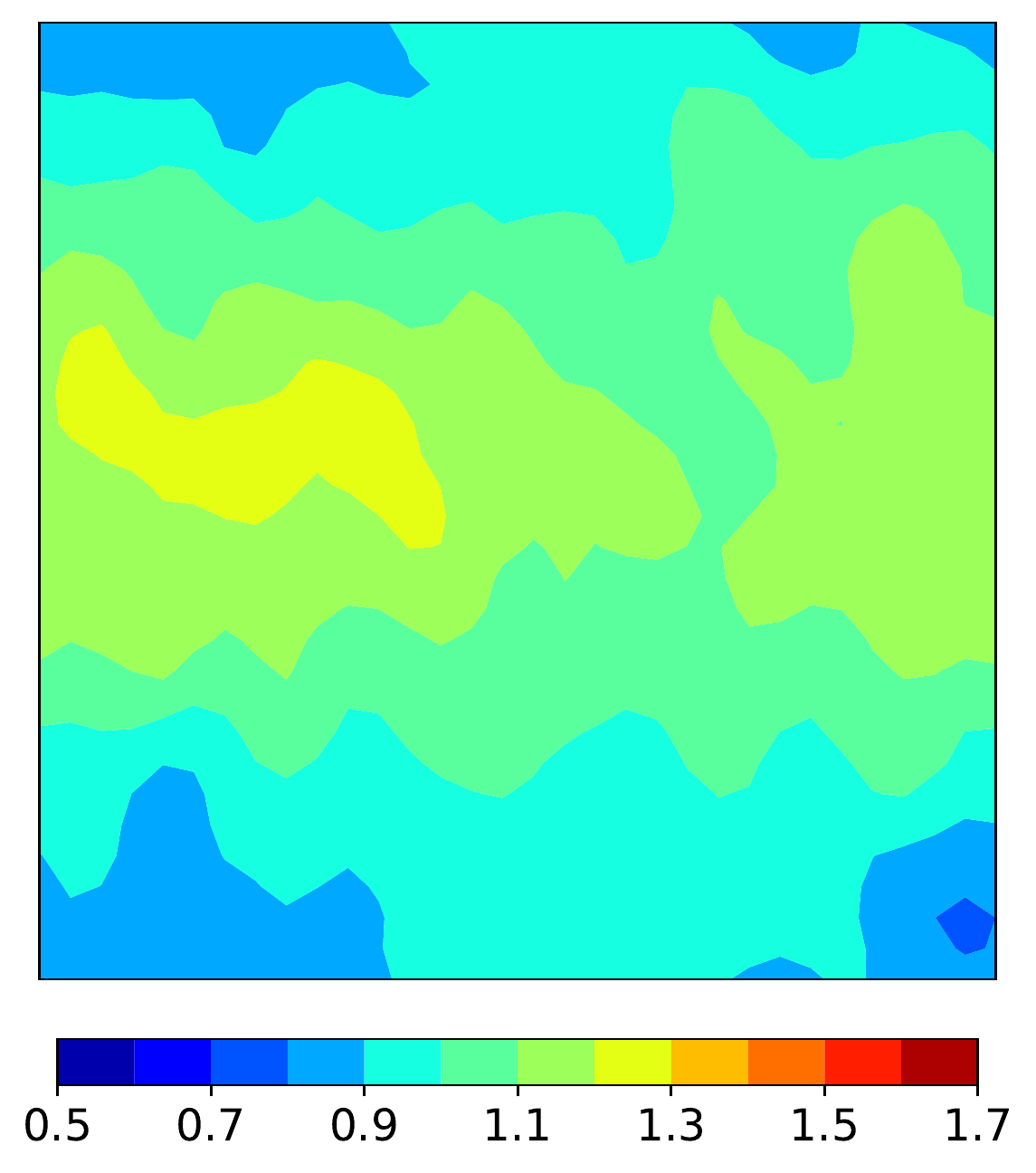}} 
        \\
        \centering
        \rotatebox[origin=c]{90}{\small $\alpha =1e^5,\delta = 2\%$} &
        \raisebox{-0.5\height}{\includegraphics[width=.20 \textwidth]{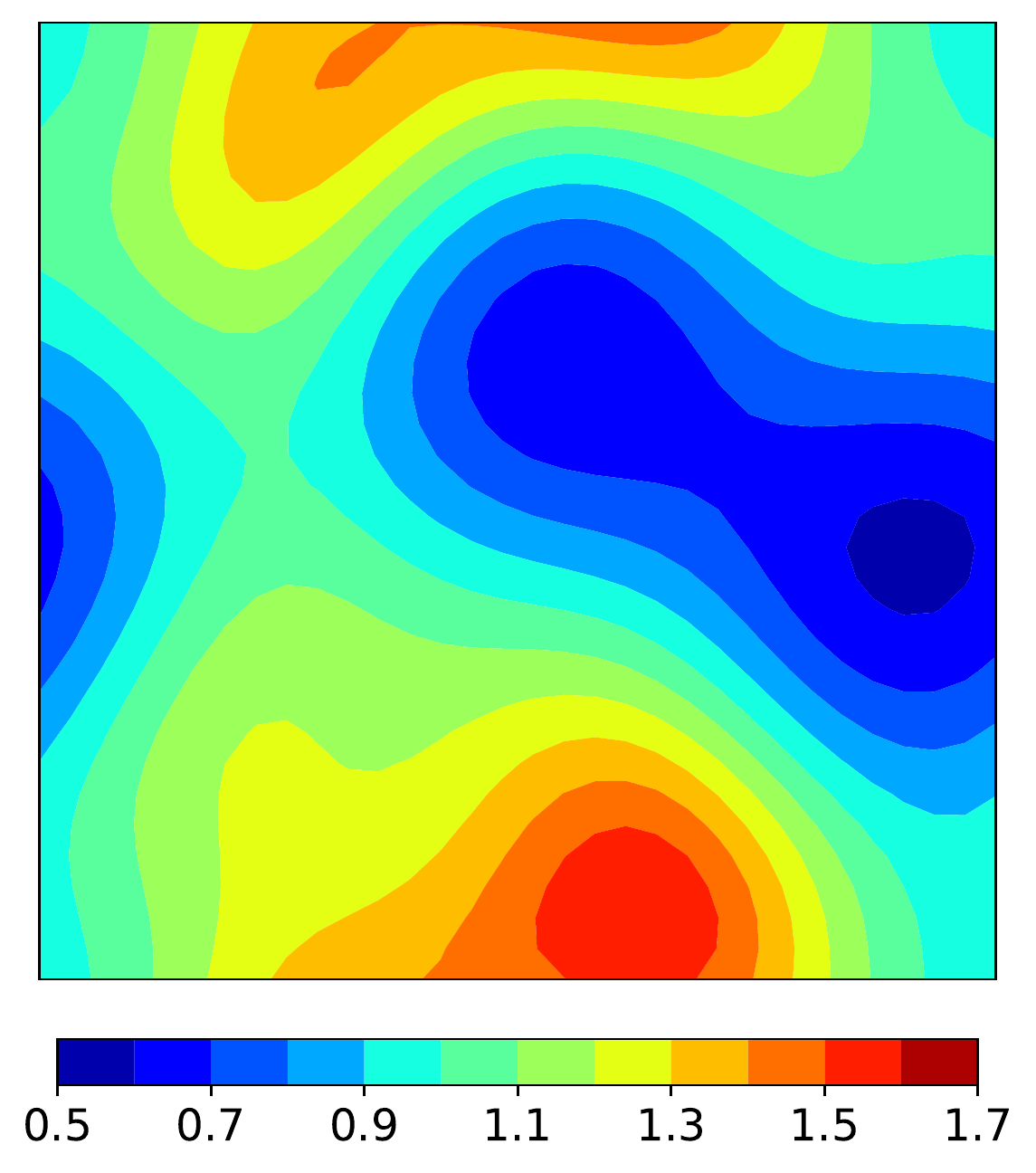}} &
        \raisebox{-0.5\height}{\includegraphics[width=.20 \textwidth]{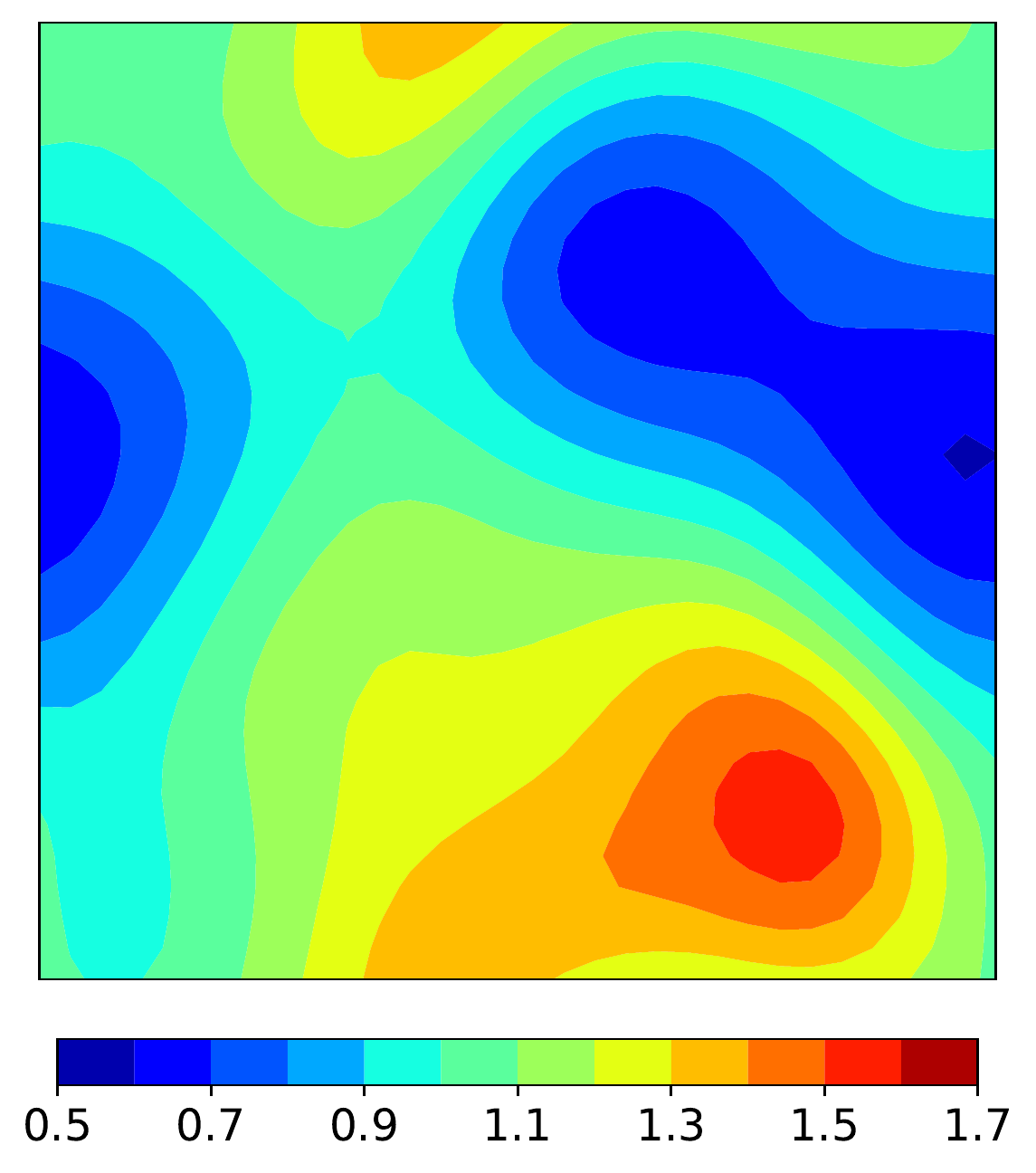}} &
        \raisebox{-0.5\height}{\includegraphics[width=.20 \textwidth]{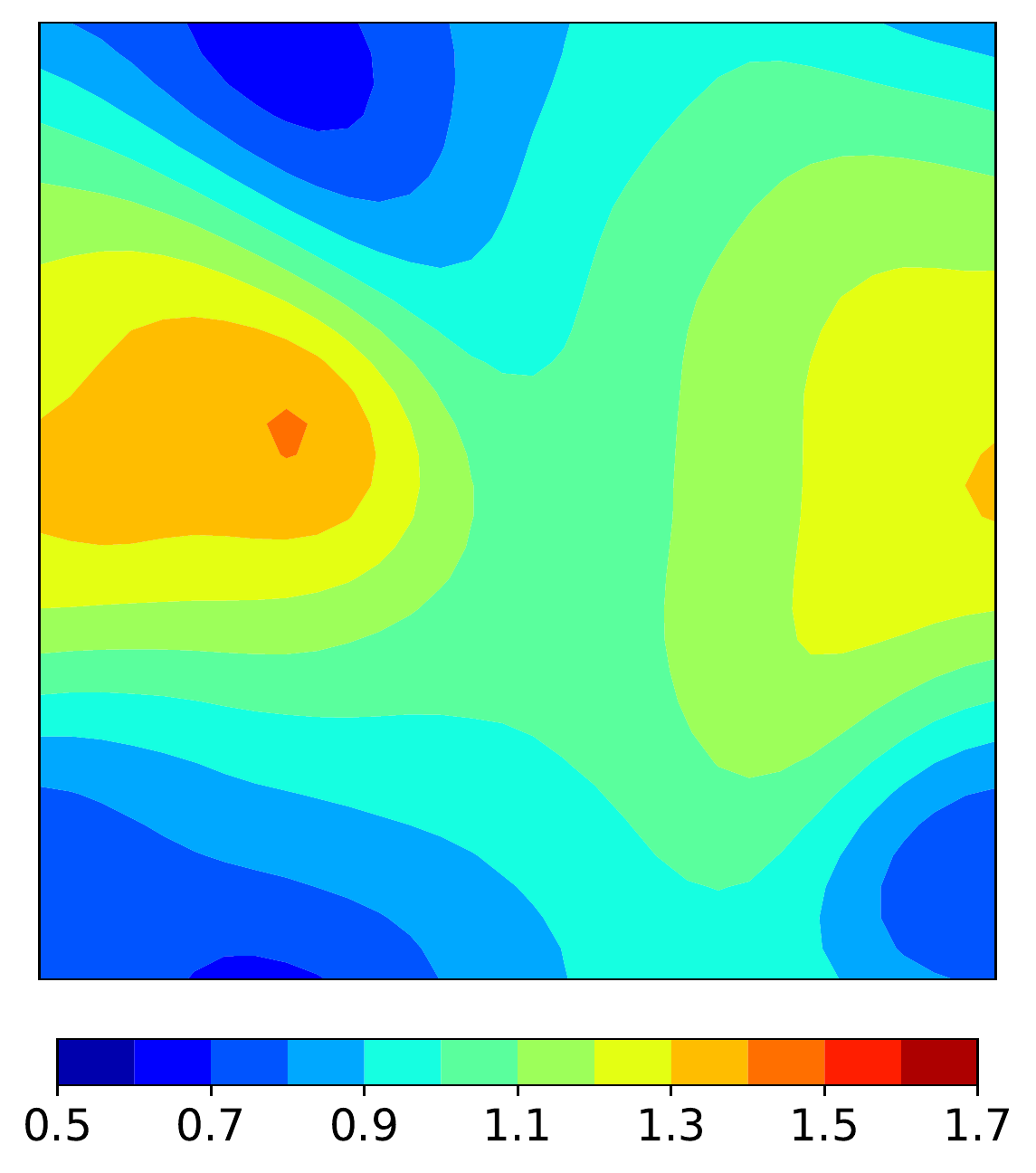}} & 
        \raisebox{-0.5\height}{\includegraphics[width=.20 \textwidth]{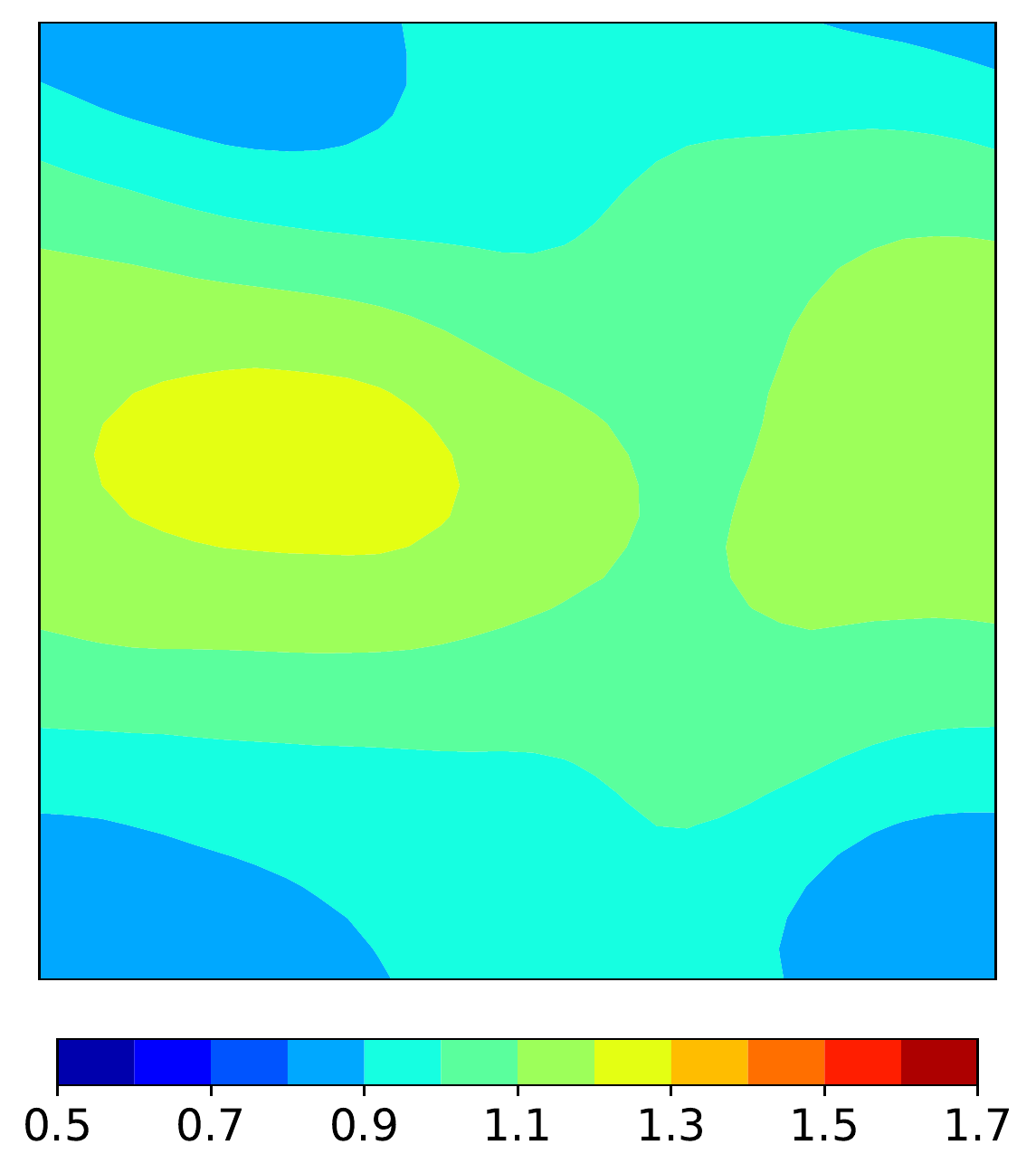}} 
    \end{tabular*}
    \caption{\textbf{Burger's equations}. Predicted solutions at different time steps ($n_t$) obtained by various learned neural network tangents with  $600$ training data samples and $S = 1$ and  $\dt = 10^{-3}$. {\em Top row}: True high-resolution solution; {\em Second row}: pure data-driven network without data randomization; {\em Third row}: pure data-driven  network with noisy data; {\em Fourth row}: model-constrained network without data randomization; {\em Fifth row}: model-constrained network with noisy data.} 
    \figlab{2D_Bur_samples}
\end{figure}

{\bf Long-time predictions with small and large training data sets.}
As discussed above, since long sequential machine learning training  does not  provide significant improvement, 
we consider $S = 1$ for  numerical results using large data sets
 in \cref{fig:2D_Bur_d200d600}. As can be seen, compared to $200$ data samples, training with $600$ data samples provides more accurate predictions. Moreover, model-constrained neural networks with randomized data are the most accurate among others (model-constrained with noise-free data and pure machine learning with/without randomized data). We can also observe that using more than $600$ data samples does not provide significant improvements but is more expensive. 
 %Indeed, we use $d1000$ data sets to achieve the network $\LRp{d1000, 2\%, 1,1,10^5}$, an inconsiderable increase in the accuracy level is observed as opposed to $\LRp{d600, 2\%, 1,1,10^5}$. 
 Unlike the case with $200$ data samples,  long and short sequential model-constrained trainings with $R = 5$ and $R = 1$, respectively,  provide similar results for $600$ data samples. This is expected as richer data reduces the significance of the model-constrained term.
 
 As shown in \cref{fig:2D_Bur_samples}, predicted solutions obtained by the model-constrained approach (the fifth row) with data randomization are in good agreement with the ground-truth counterparts. On the contrary, the pure data-driven approach with  data randomization (the third row) shows poor long-time predictions.  We also observe that
 both pure data-driven learning solutions and model-constrained  solutions (the second and fourth rows, respectively) without randomization  are unstable for long-time predictions. It is not surprising since both do not have sufficient regularizations compared to the randomized cases in which extra regularizations are implicitly performed (see \cref{sect:noise_data_sec}). Moreover,  regularizations induced by data randomization shown in \cref{sect:noise_data_sec} stabilize the network predictions and this can be clearly seen by comparing the third and the second rows for the pure data-driven learning approach, and by comparing the fourth and the fifth rows for the  model-constrained learning approach. 
 
 \cref{fig:2D_Bur_dU_samples} plots the contours of the learned and the true tangent slopes. Clearly, the learned model-constrained tangent slope with data randomization provides the best agreement with the true tangent slope. This is not surprising as both the governing equations (explicit via model-constrained term) and sufficient regularizations (implicit via data randomization) are incorporated.

\begin{figure}[htb!]
    \centering
    \begin{tabular*}{\textwidth}{c c c c c c}
        \centering
         &
        \raisebox{-0.5\height}{\small True $\F(\ub)$} &
        \raisebox{-0.5\height}{\small $\LRp{0, 0\%}$} &
        \raisebox{-0.5\height}{\small $\LRp{0, 2\%}$} &
        \raisebox{-0.5\height}{\small $\LRp{10^5, 0\%}$} & 
        \raisebox{-0.5\height}{\small $\LRp{10^5, 2\%}$} 
        \\
        \centering
        \rotatebox[origin=c]{90}{\small $n_t = 100$} &
        \raisebox{-0.5\height}{\includegraphics[width=.16 \textwidth]{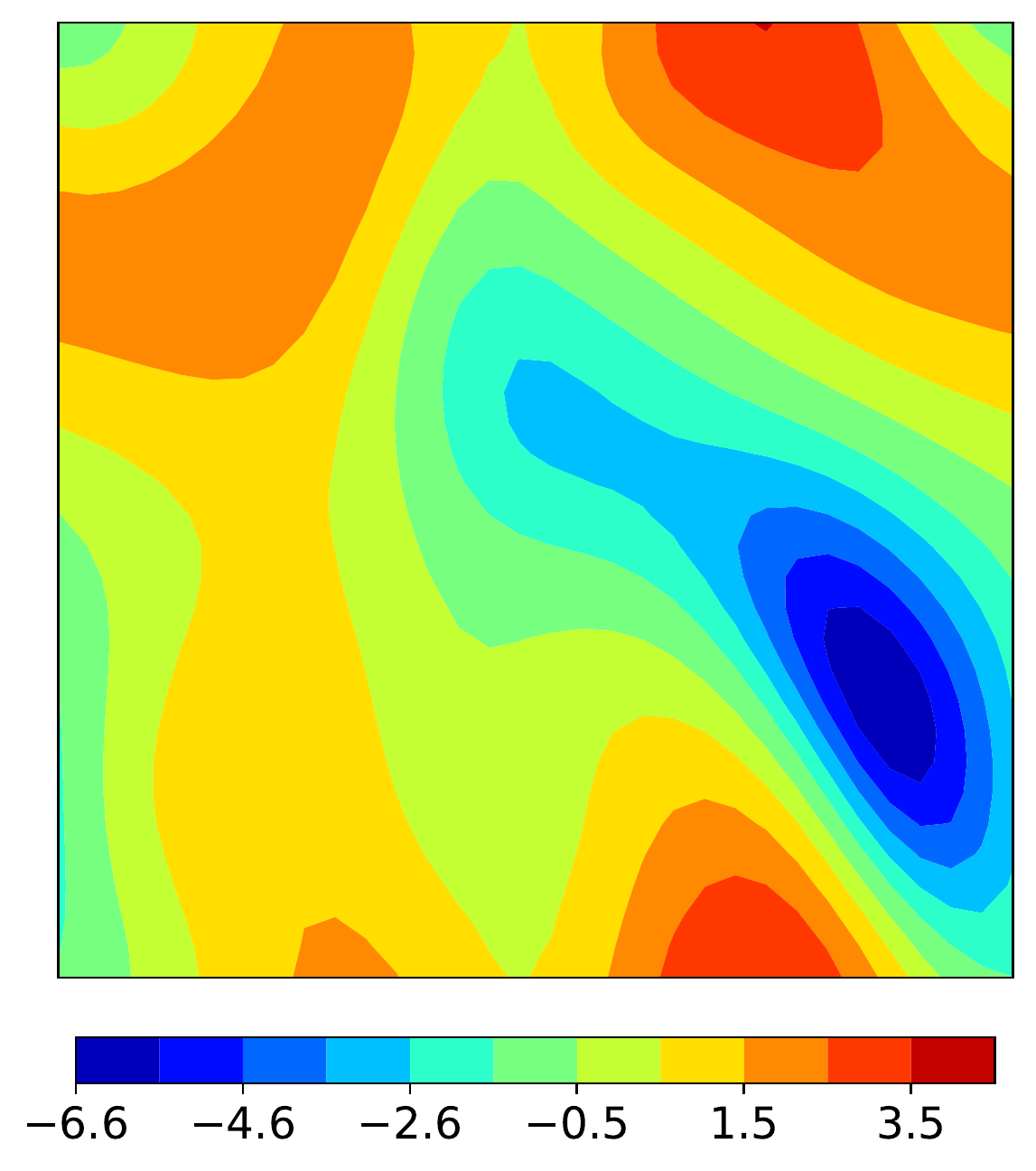}} &
        \raisebox{-0.5\height}{\includegraphics[width=.16 \textwidth]{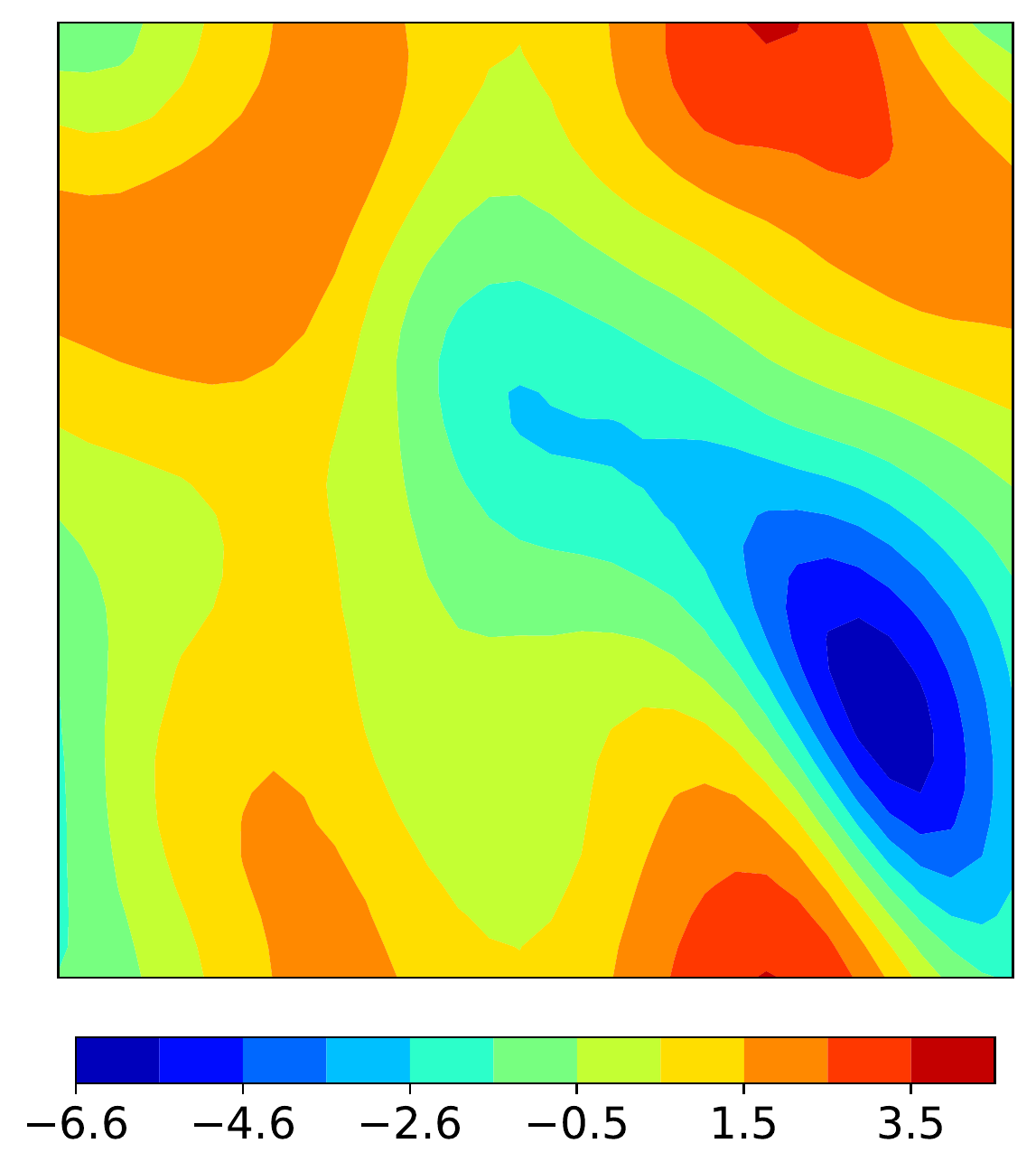}} &
        \raisebox{-0.5\height}{\includegraphics[width=.16 \textwidth]{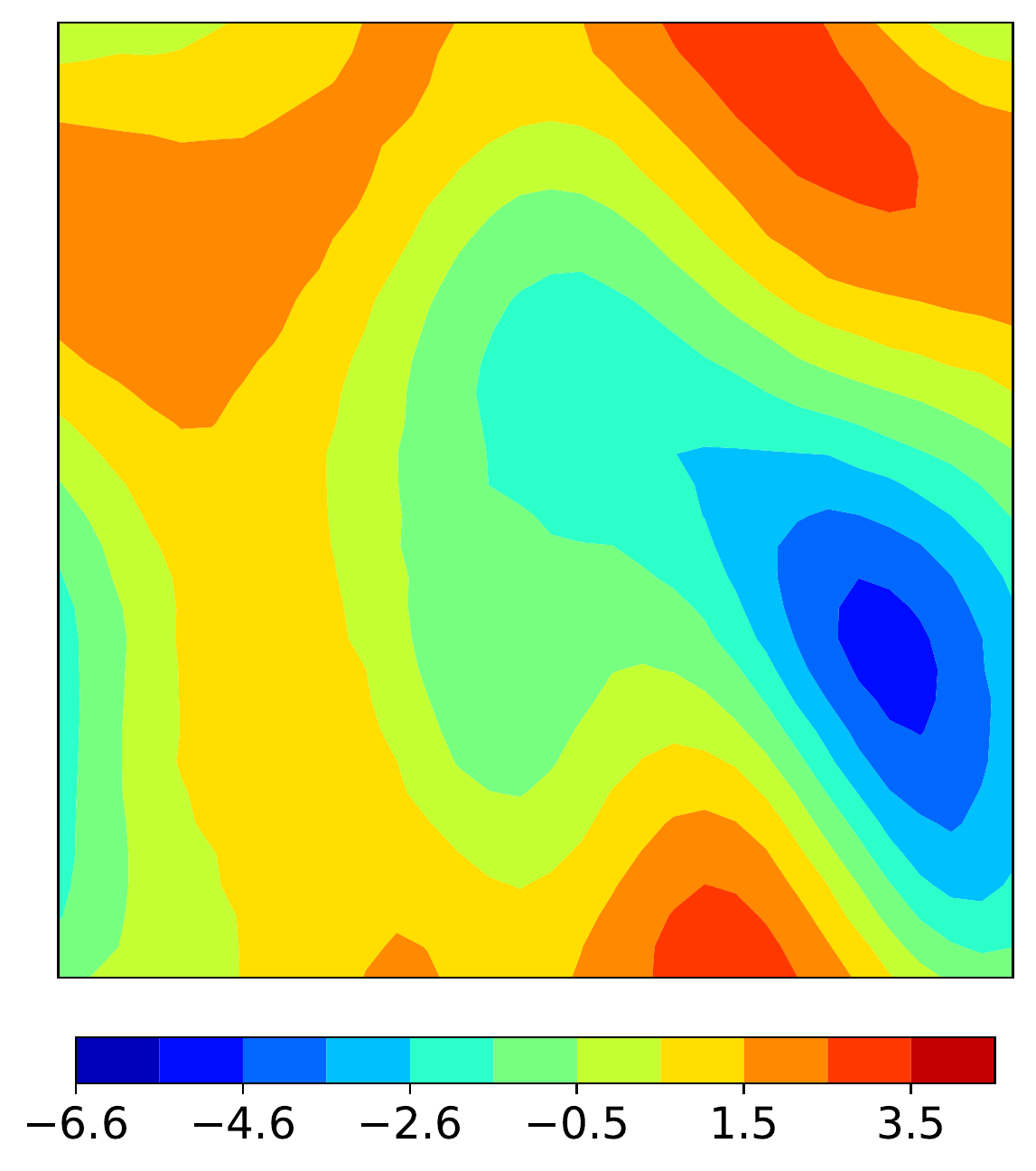}} &
        \raisebox{-0.5\height}{\includegraphics[width=.16 \textwidth]{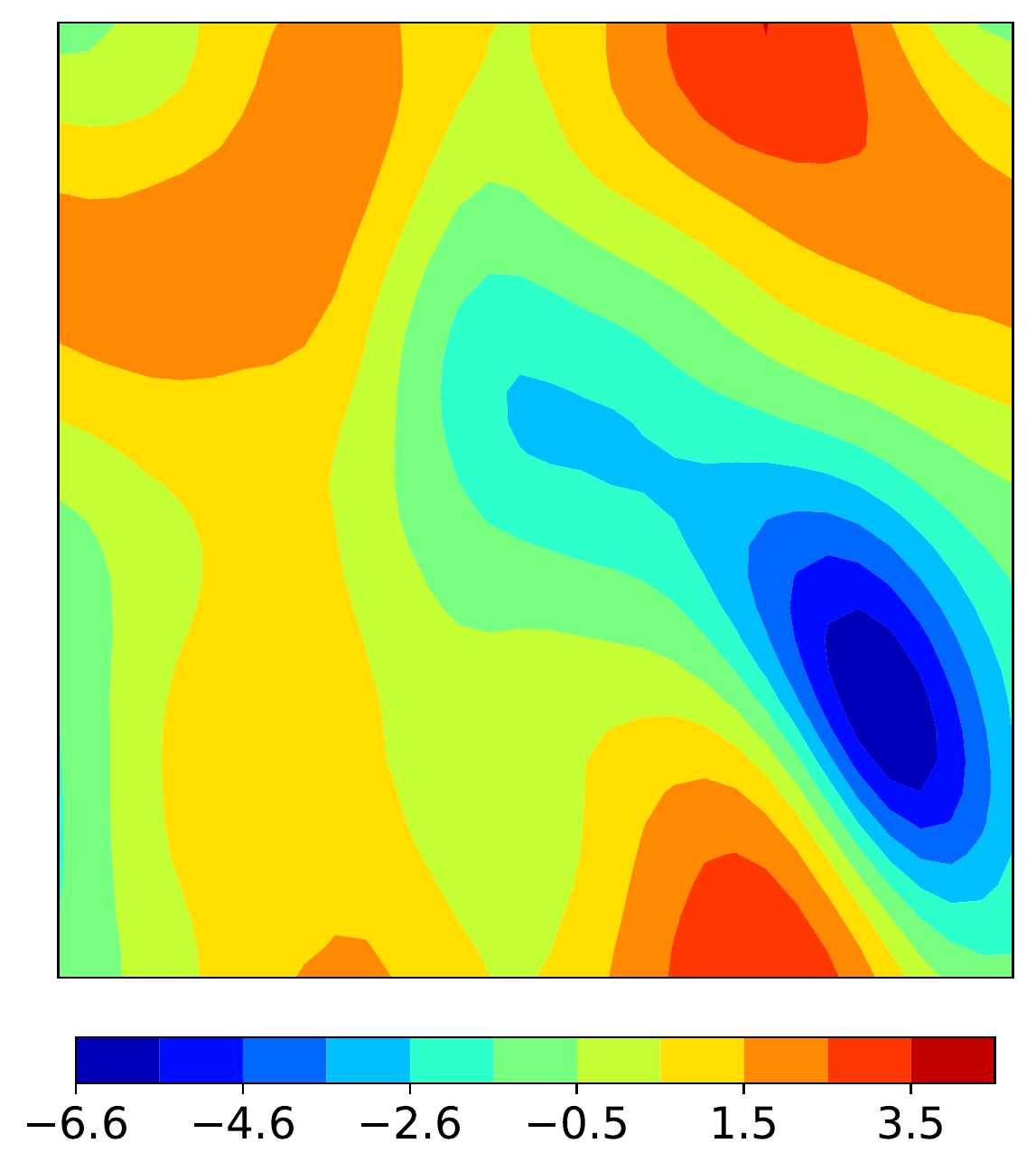}} & 
        \raisebox{-0.5\height}{\includegraphics[width=.16 \textwidth]{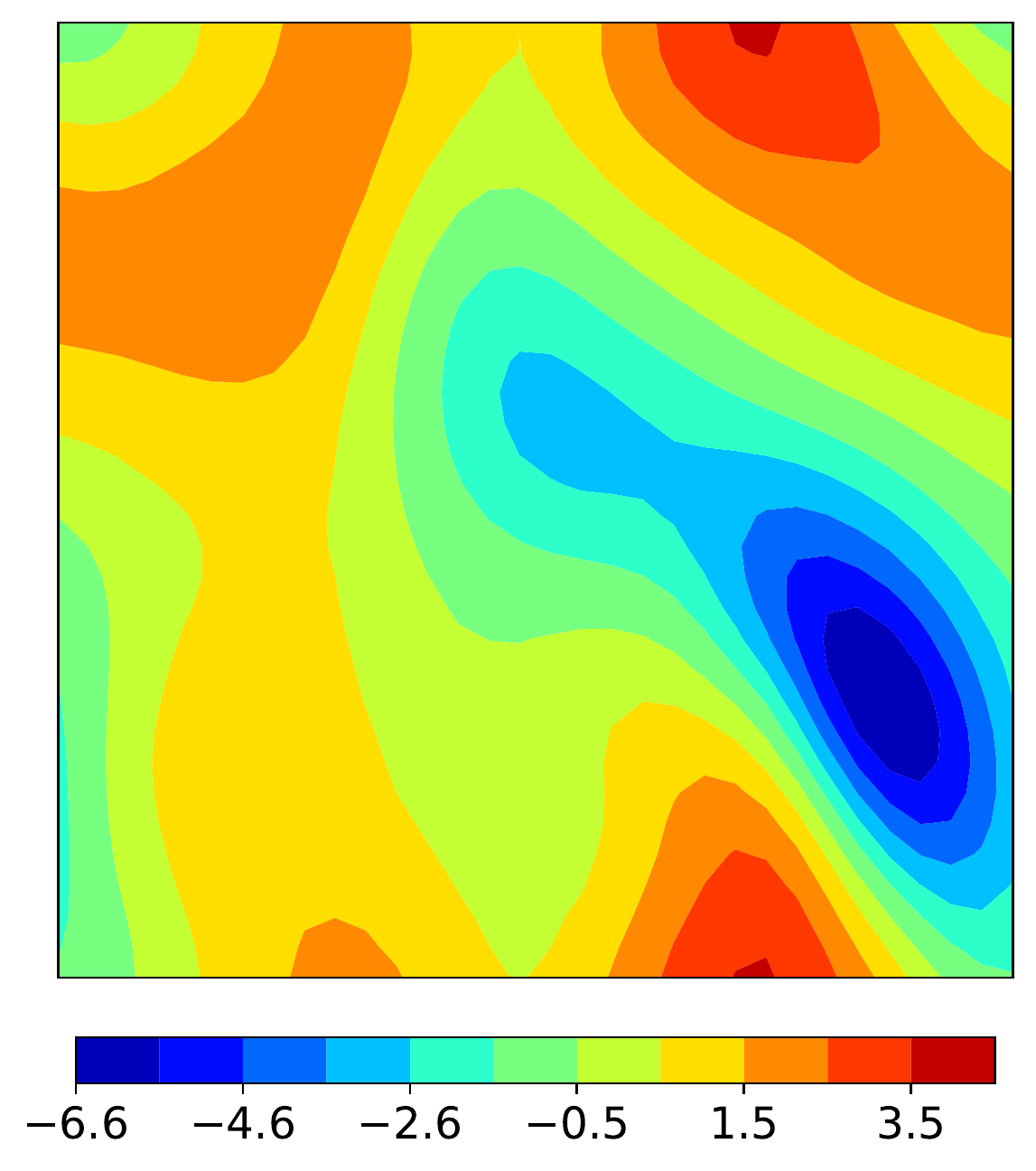}}
        \\
        \centering
        \rotatebox[origin=c]{90}{\small $n_t = 500$} &
        \raisebox{-0.5\height}{\includegraphics[width=.16 \textwidth]{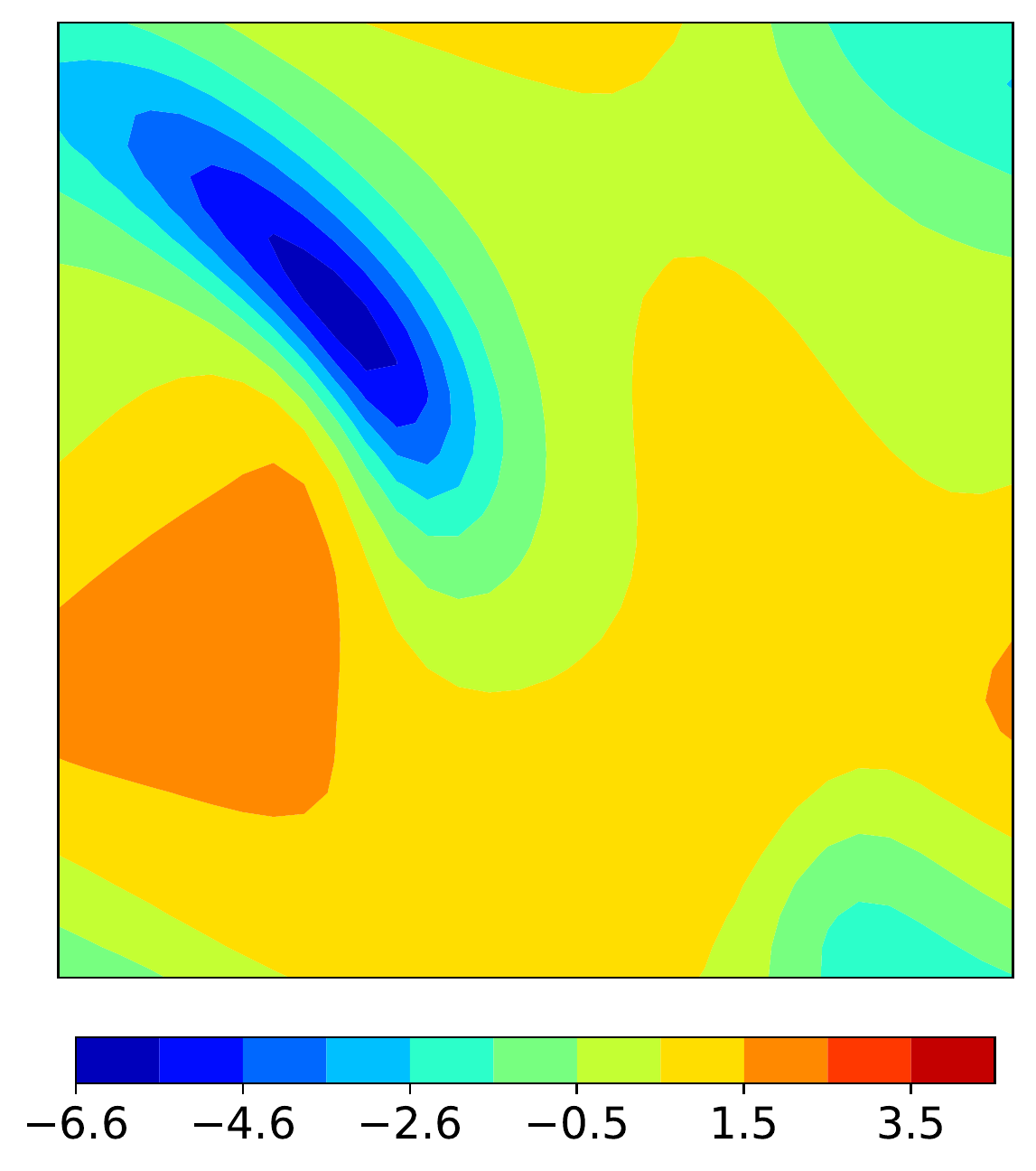}} &
        \raisebox{-0.5\height}{\includegraphics[width=.16 \textwidth]{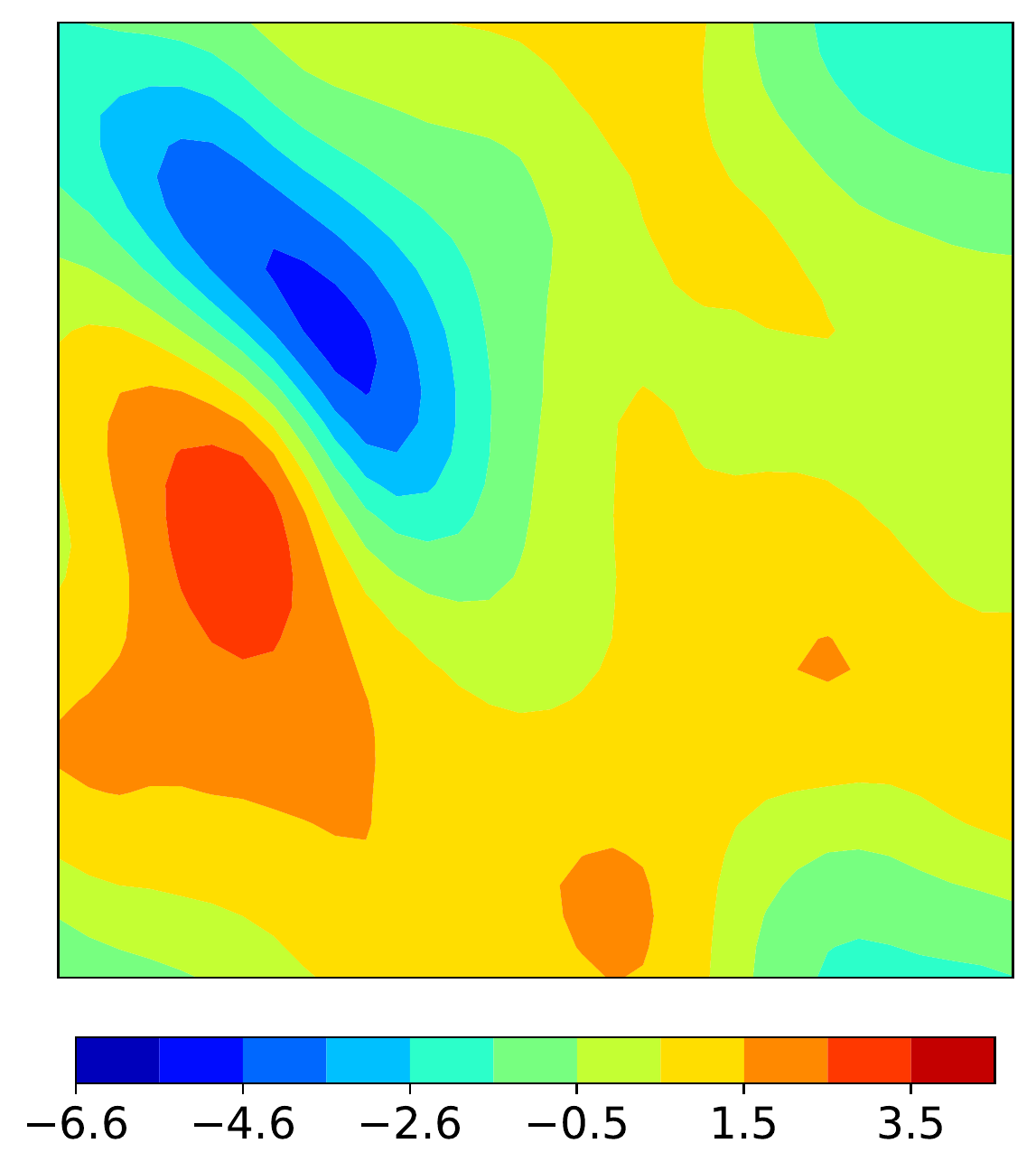}} &
        \raisebox{-0.5\height}{\includegraphics[width=.16 \textwidth]{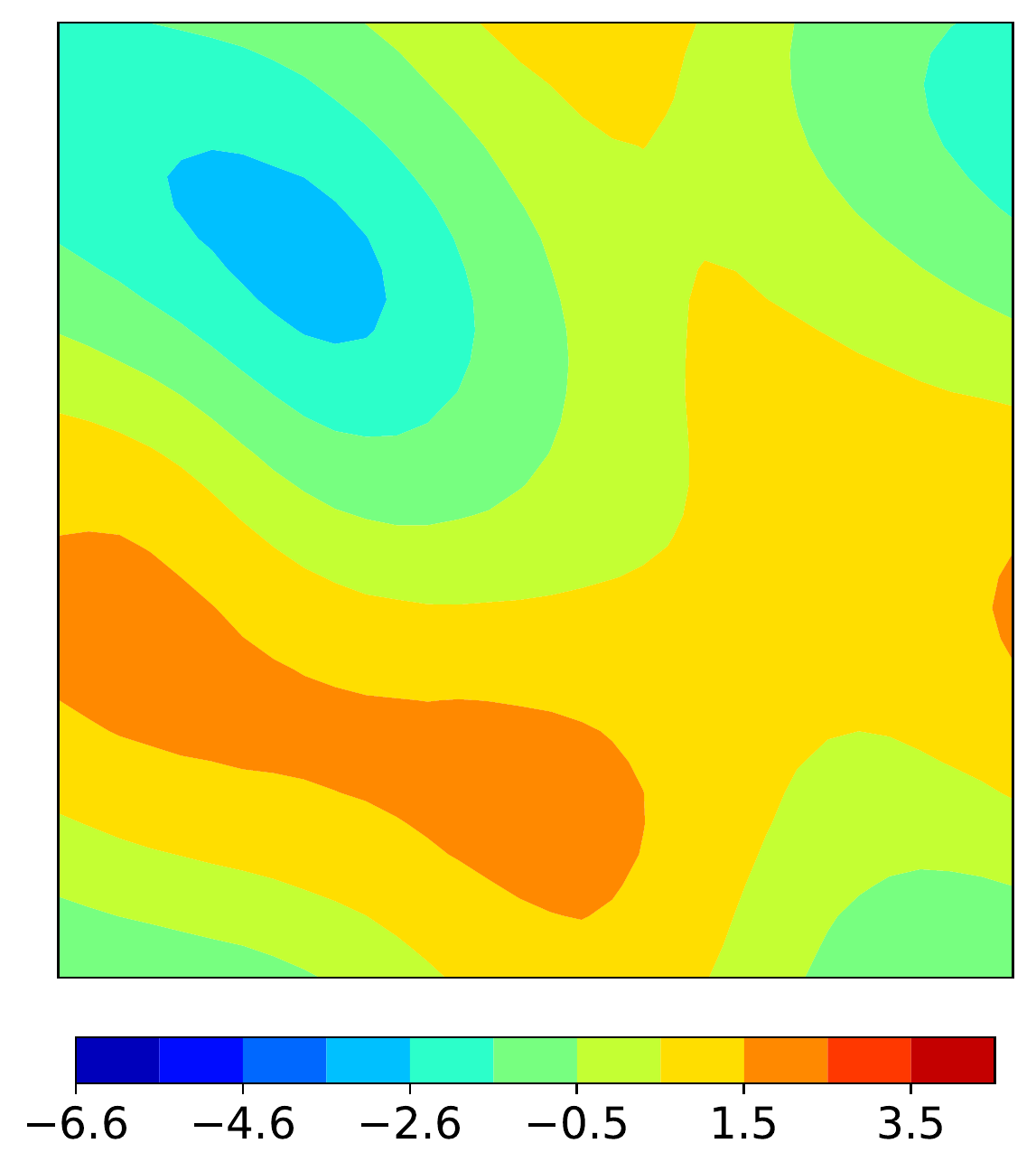}} &
        \raisebox{-0.5\height}{\includegraphics[width=.16 \textwidth]{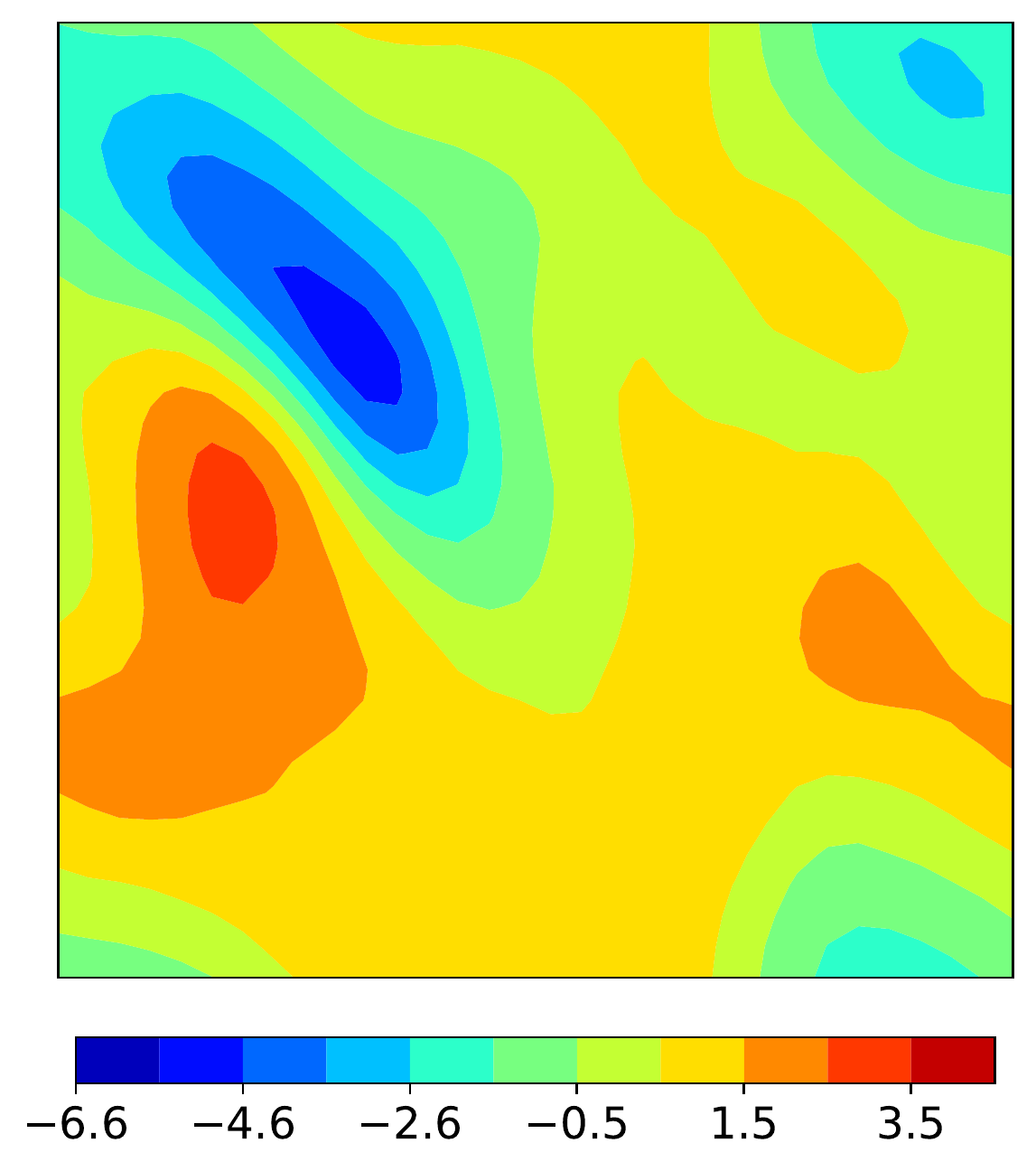}} & 
        \raisebox{-0.5\height}{\includegraphics[width=.16 \textwidth]{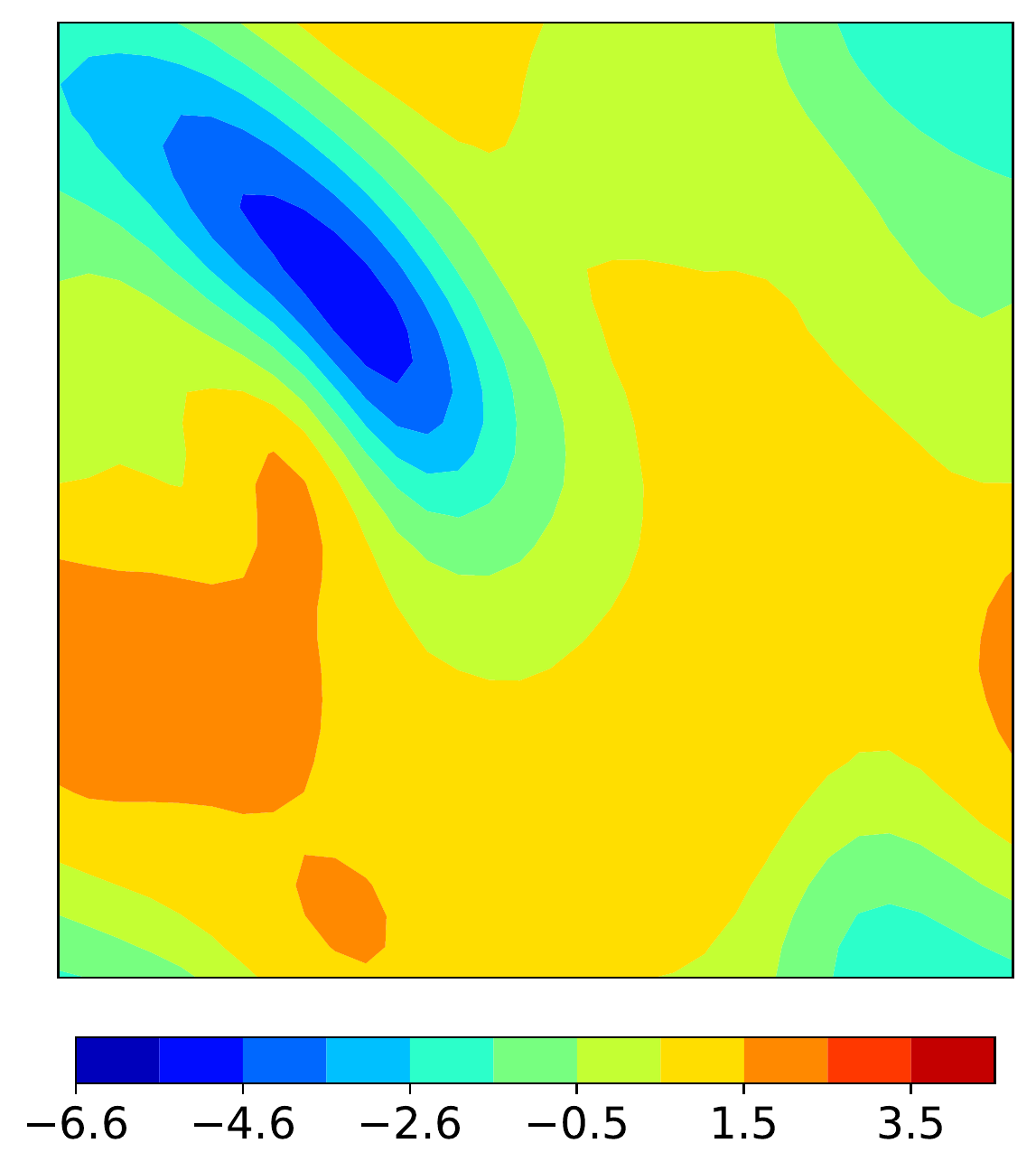}}
        \\
        \centering
        \rotatebox[origin=c]{90}{\small $n_t = 1500$} &
        \raisebox{-0.5\height}{\includegraphics[width=.16 \textwidth]{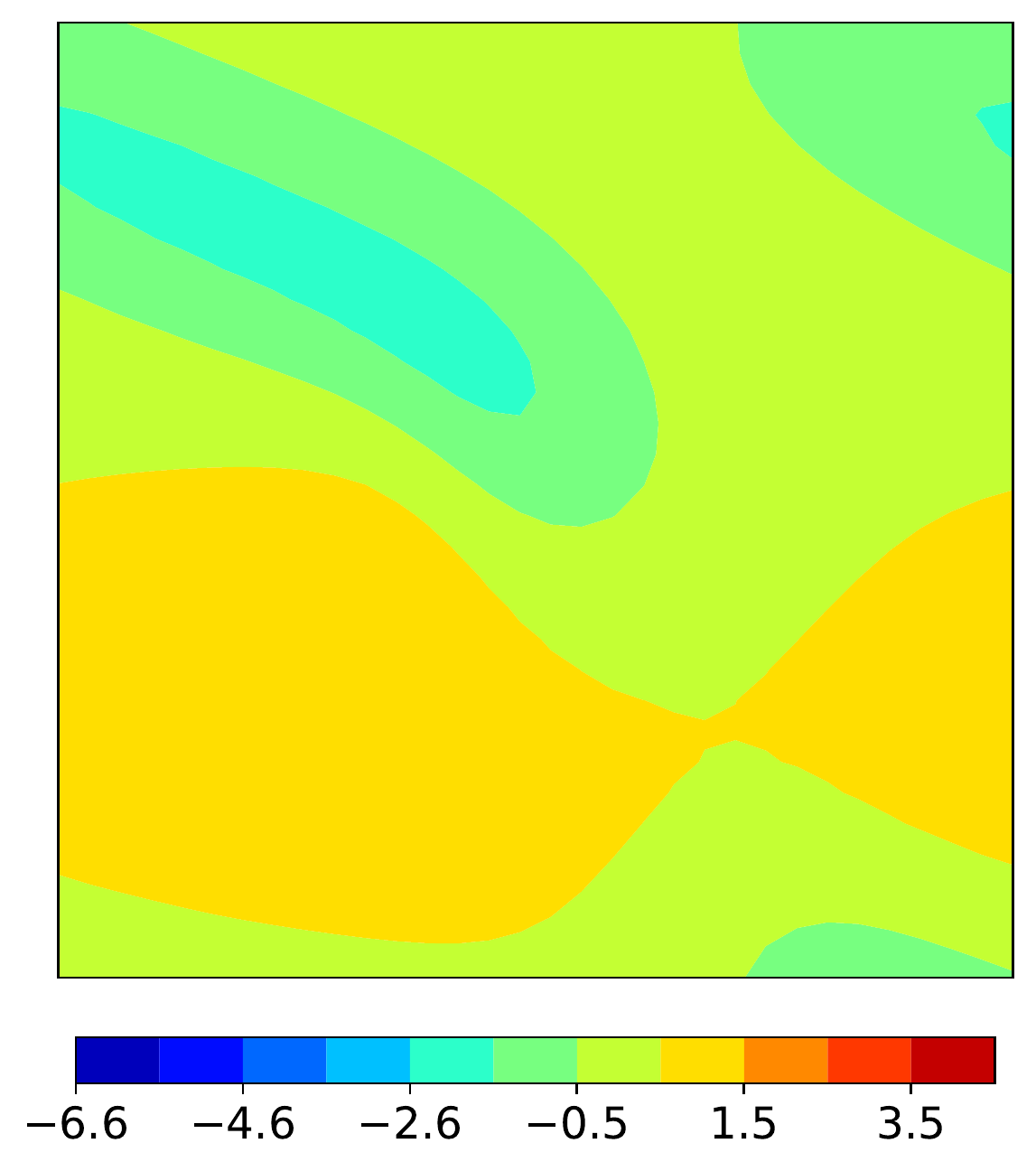}} &
        \raisebox{-0.5\height}{\includegraphics[width=.16 \textwidth]{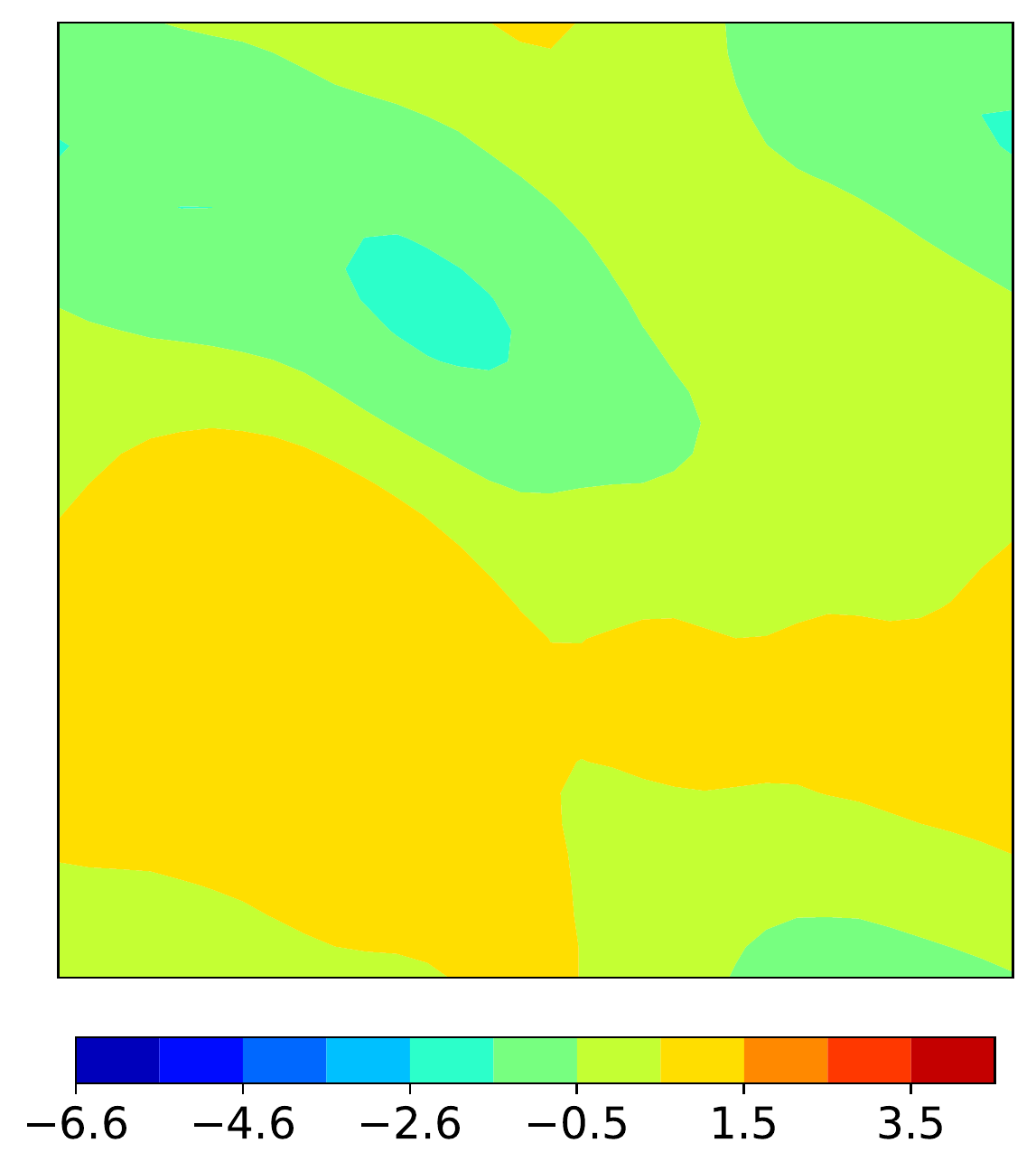}} &
        \raisebox{-0.5\height}{\includegraphics[width=.16 \textwidth]{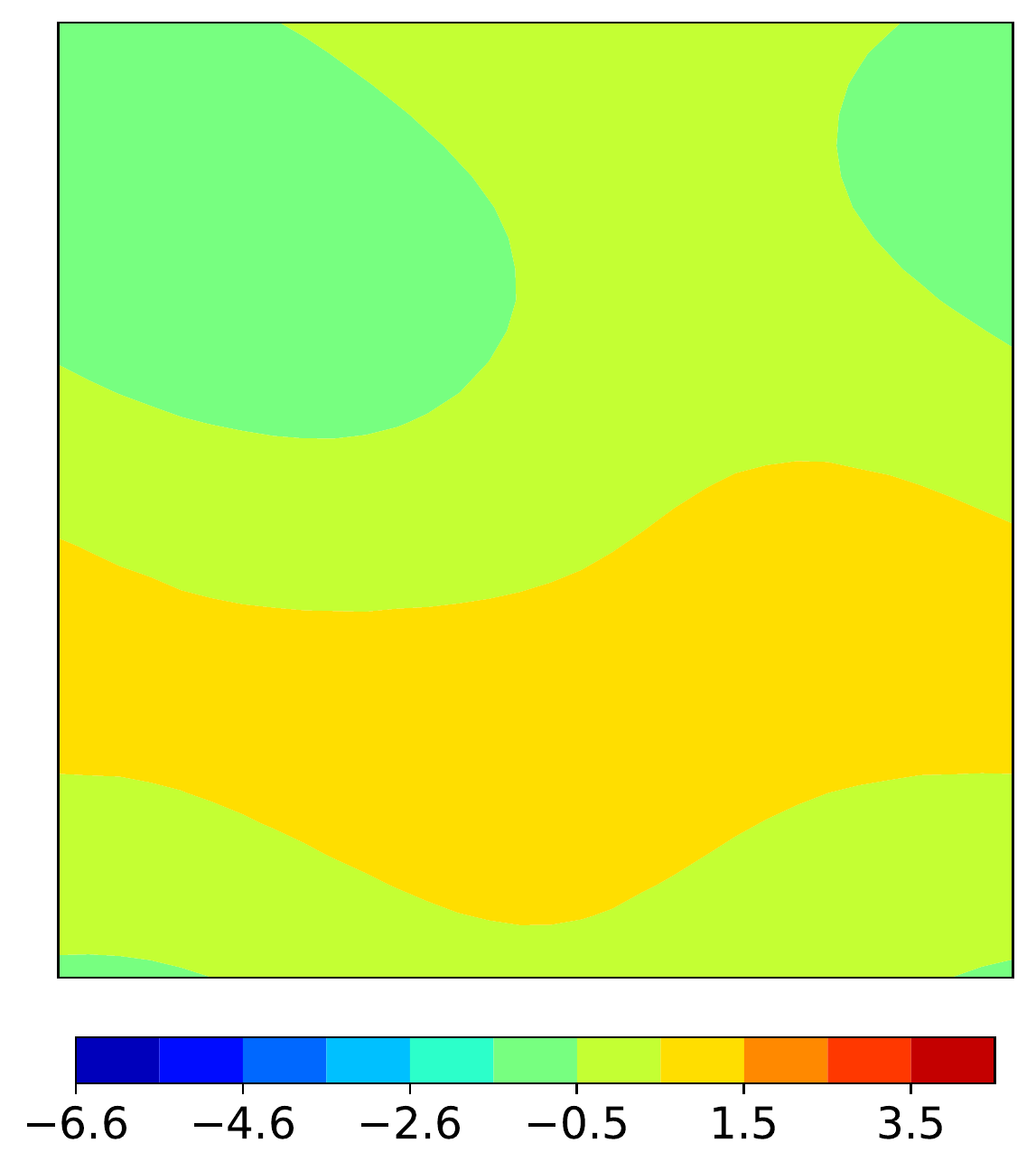}} &
        \raisebox{-0.5\height}{\includegraphics[width=.16 \textwidth]{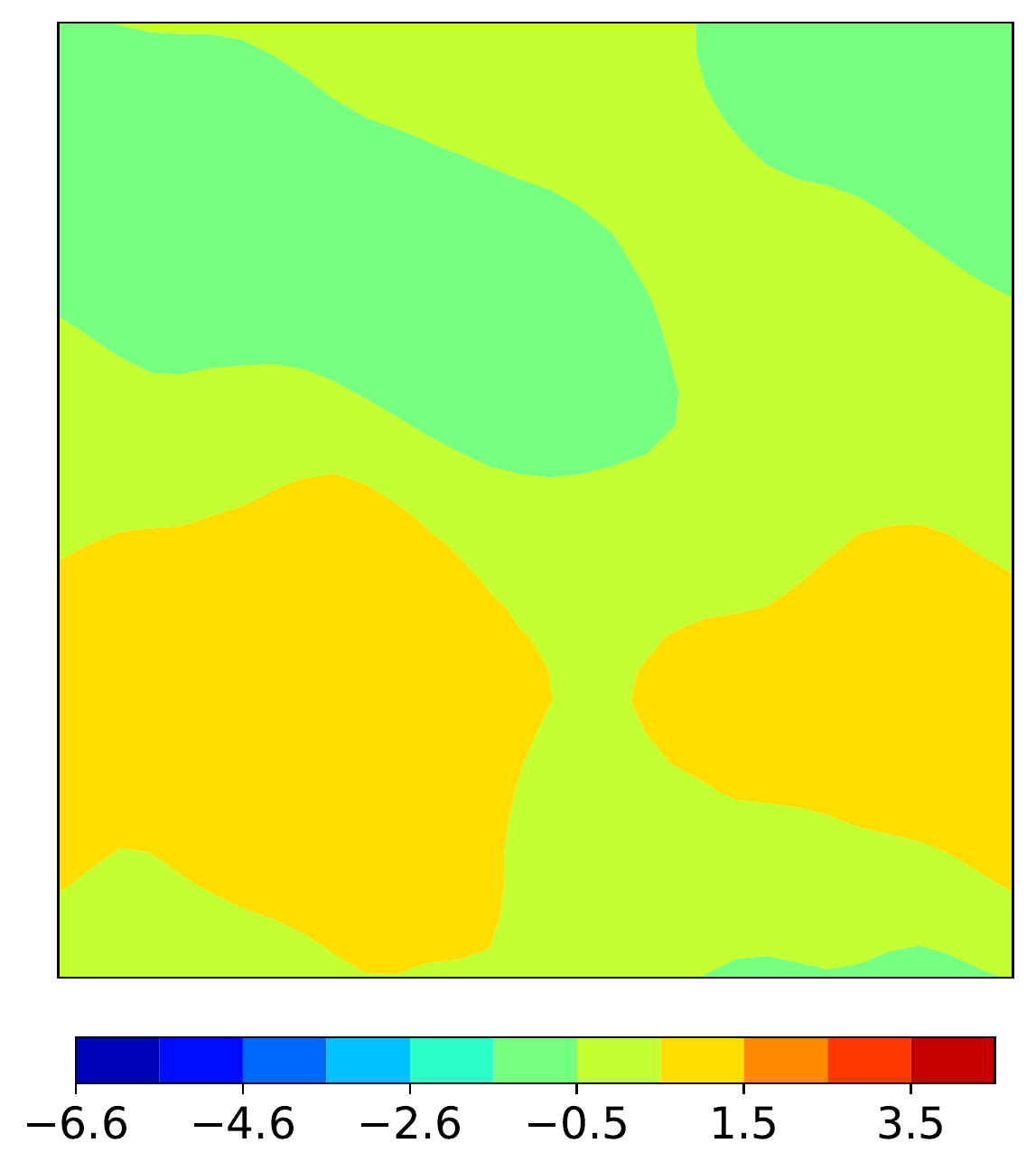}} & 
        \raisebox{-0.5\height}{\includegraphics[width=.16 \textwidth]{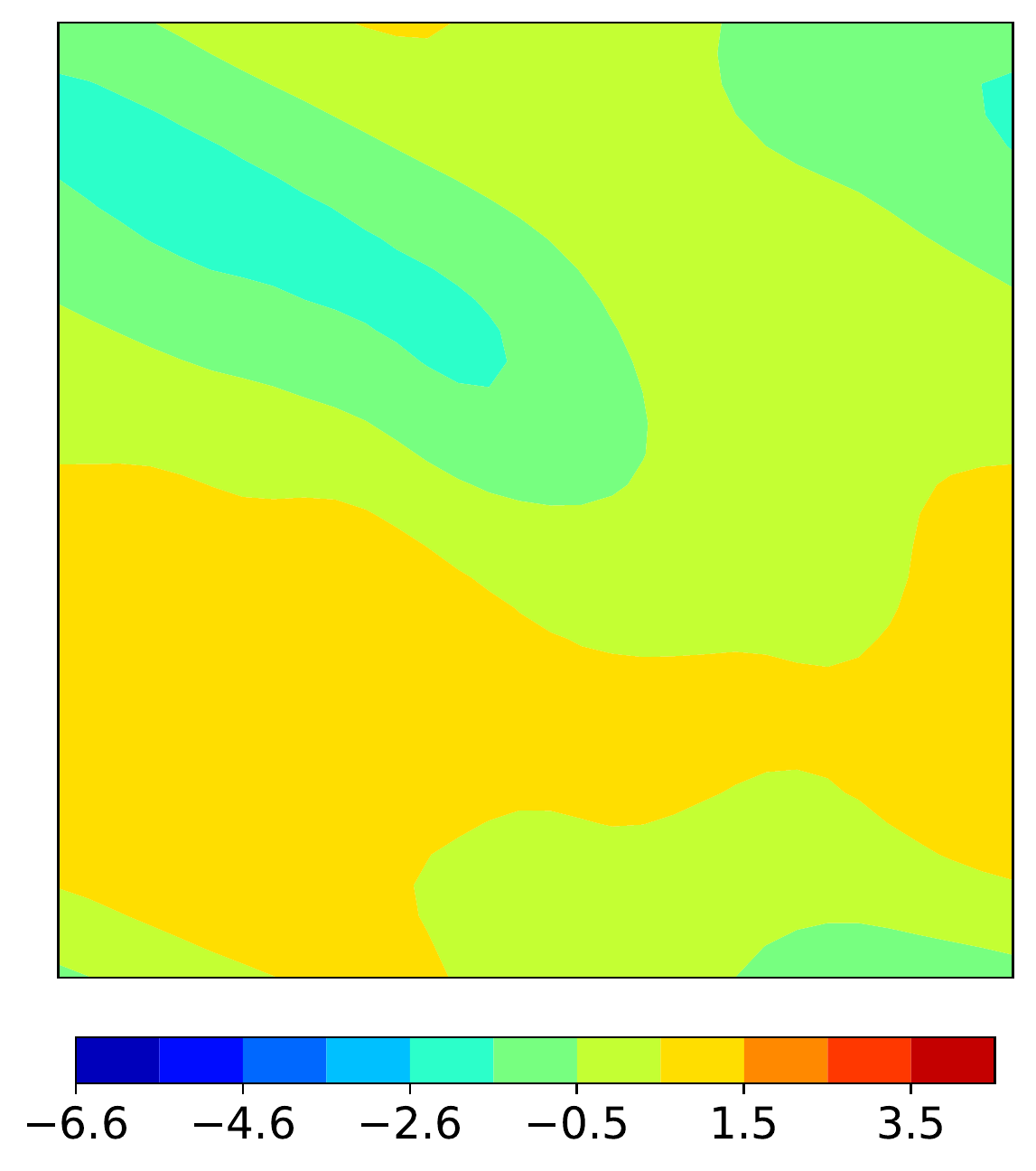}}
    \end{tabular*}
    \caption{\textbf{Burger's equations}. Contours of True and various learned tangent slopes. Contours are plotted at different time steps $n_t$ for four different combinations of regularization parameter $\alpha$ and noise level $\delta$. For all cases, we use  $600$ data samples, $S = 1$, and  $\dt = 10^{-3}$. \textit{First column}: True tangent slope ; \textit{Second column}: pure data-driven  tangent slope without data randomization $\LRp{0, 0\%}$; \textit{Third column}: pure data-driven  tangent slope with data randomization $\LRp{0, 2\%}$; \textit{Fourth column}: model-constrained tangent slope without data randomization $\LRp{10^5, 0\%}$; \textit{Fifth column}: model-constrained tangent slope with  data randomization $\LRp{10^5, 2\%}$.} 
    \figlab{2D_Bur_dU_samples}
\end{figure}

{\bf Predictive flexibility in time for \texttt{mcTangent} approach}
As discussed above, one appealing feature of tangent slope learning is that once trained it can be used to solve for approximate solutions with smaller or larger time stepsizes, despite the fact that it is trained based on a particular spatial discretization. On the contrary, direct learning is attached to a space-time discretization. \cref{fig:2D_Bur_smaller_dt_samples} shows the model-constrained tangent slope learning solutions and  contours of the corresponding learned tangent slope at various times  for the setting $\LRp{d600, 2\%, 1,1, 10^5}$. Here we use half of the training time stepsize $\dt'' = \frac{1}{2}\dt = 5 \times 10^{-4}$. It can be seen that these predictions are indistinguishable  from ones (the fifth row in \cref{fig:2D_Bur_samples} for prediction solutions and the fifth column in \cref{fig:2D_Bur_dU_samples} for predicted tangent slopes) obtained by using the training time stepsize $\dt = 10^{-3}$ with the same learned network.

\begin{figure}[htb!]
    \centering
    \begin{tabular*}{\textwidth}{c c c c c}
        \centering
         &
        \raisebox{-0.5\height}{\small $t = .1$} &
        \raisebox{-0.5\height}{\small $t = 0.5$} &
        \raisebox{-0.5\height}{\small $t = 0.5005$} & 
        \raisebox{-0.5\height}{\small $t = 1.5$} 
        \\
        \centering
        \rotatebox[origin=c]{90}{\small True $\ub$} &
        \raisebox{-0.5\height}{\includegraphics[width=.20 \textwidth]{figures/2D_Bur/Bur_FD_127x127_step_t_100.pdf}} &
        \raisebox{-0.5\height}{\includegraphics[width=.20 \textwidth]{figures/2D_Bur/Bur_FD_127x127_step_t_500.pdf}} &
        \raisebox{-0.5\height}{\includegraphics[width=.20 \textwidth]{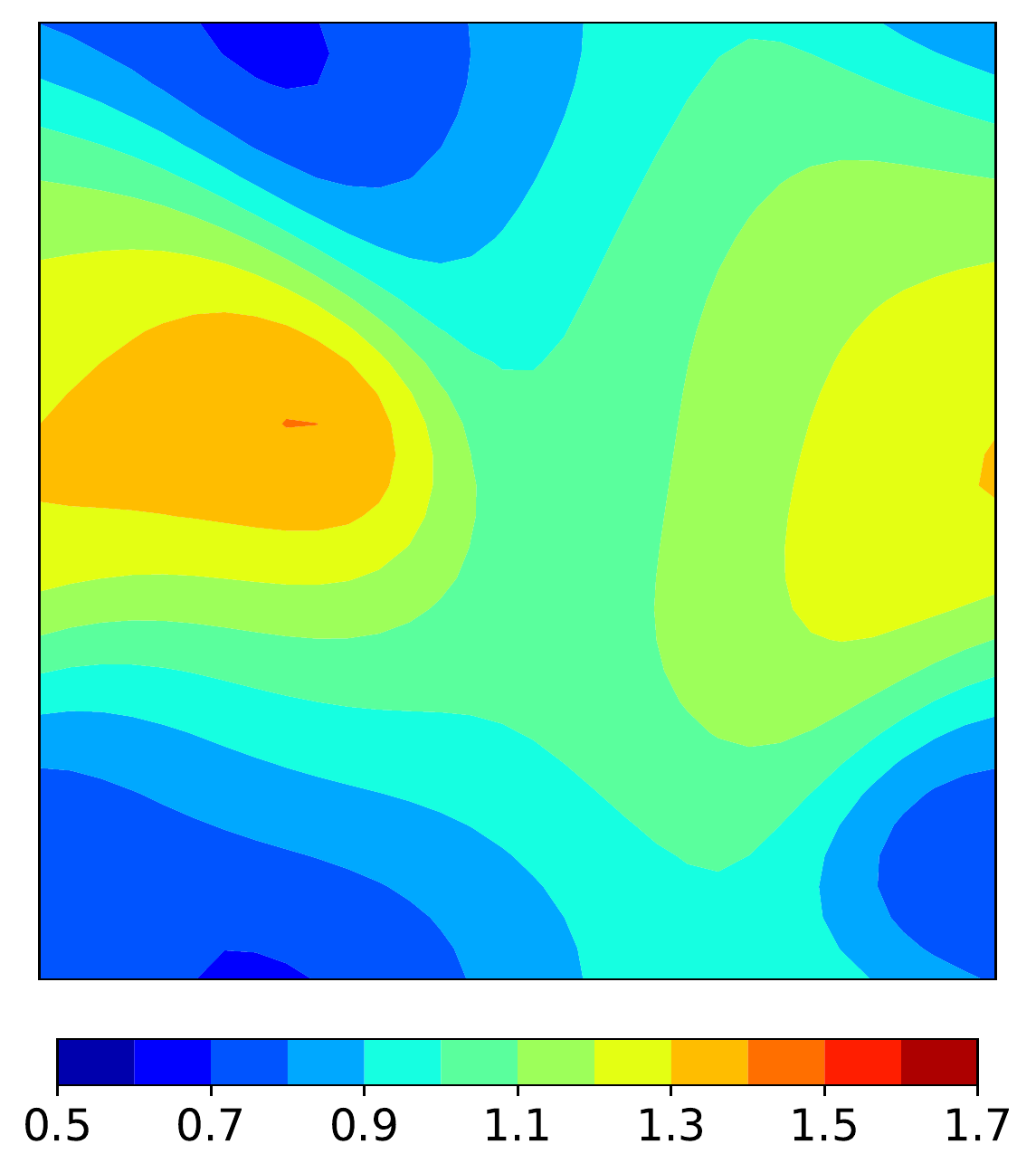}} & 
        \raisebox{-0.5\height}{\includegraphics[width=.20 \textwidth]{figures/2D_Bur/Bur_FD_127x127_step_t_1500.pdf}} 
        \\
        \centering
        \rotatebox[origin=c]{90}{\small True $\F(\ub)$} &
        \raisebox{-0.5\height}{\includegraphics[width=.20 \textwidth]{figures/2D_Bur/Bur_dU_FD_32x32_step_t_100.pdf}} &
        \raisebox{-0.5\height}{\includegraphics[width=.20 \textwidth]{figures/2D_Bur/Bur_dU_FD_32x32_step_t_500.pdf}} &
        \raisebox{-0.5\height}{\includegraphics[width=.20 \textwidth]{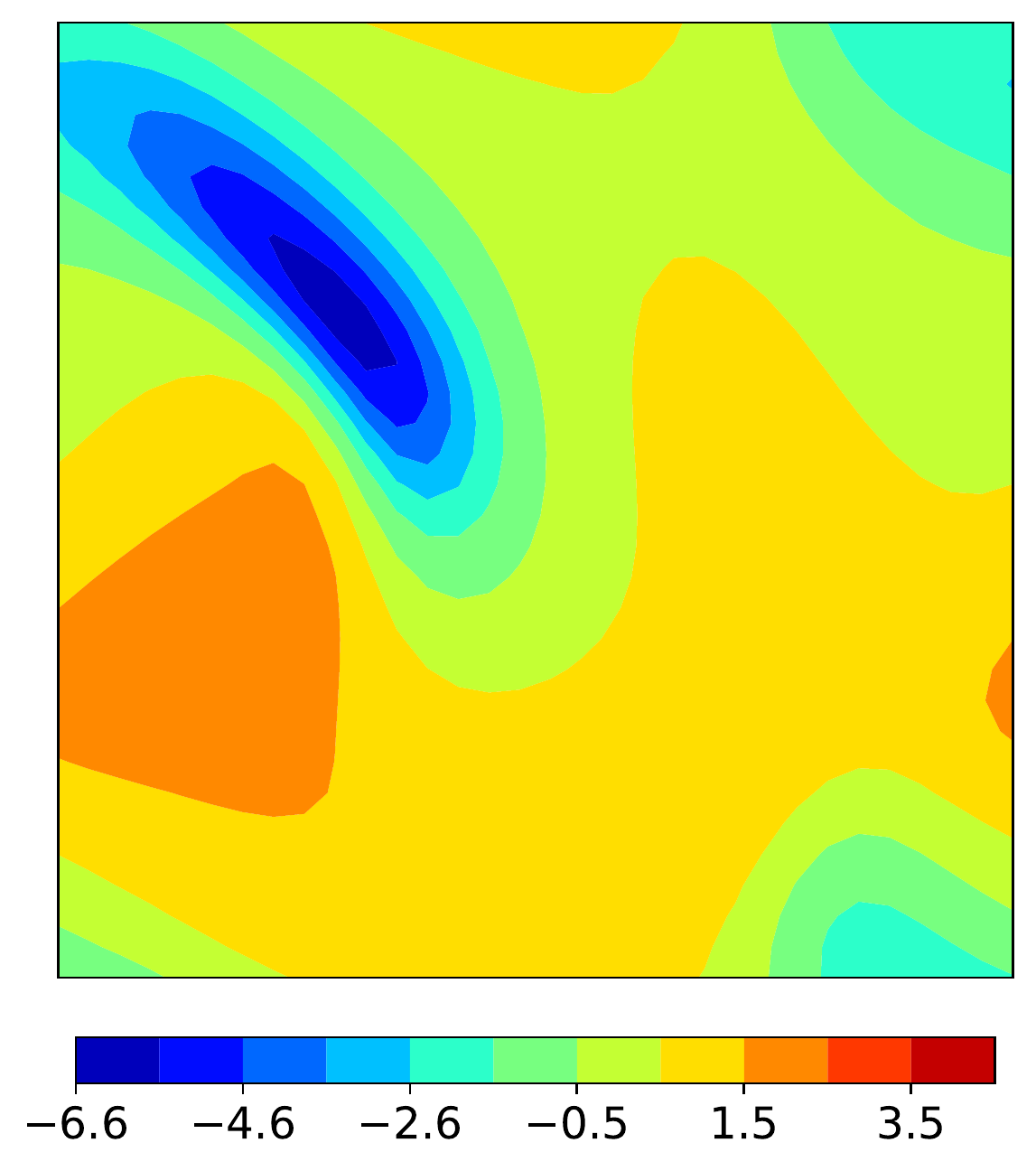}} & 
        \raisebox{-0.5\height}{\includegraphics[width=.20 \textwidth]{figures/2D_Bur/Bur_dU_FD_32x32_step_t_1500.pdf}} 
        \\
        \centering
        \rotatebox[origin=c]{90}{\small Predicted $\ub$} &
        \raisebox{-0.5\height}{\includegraphics[width=.20 \textwidth]{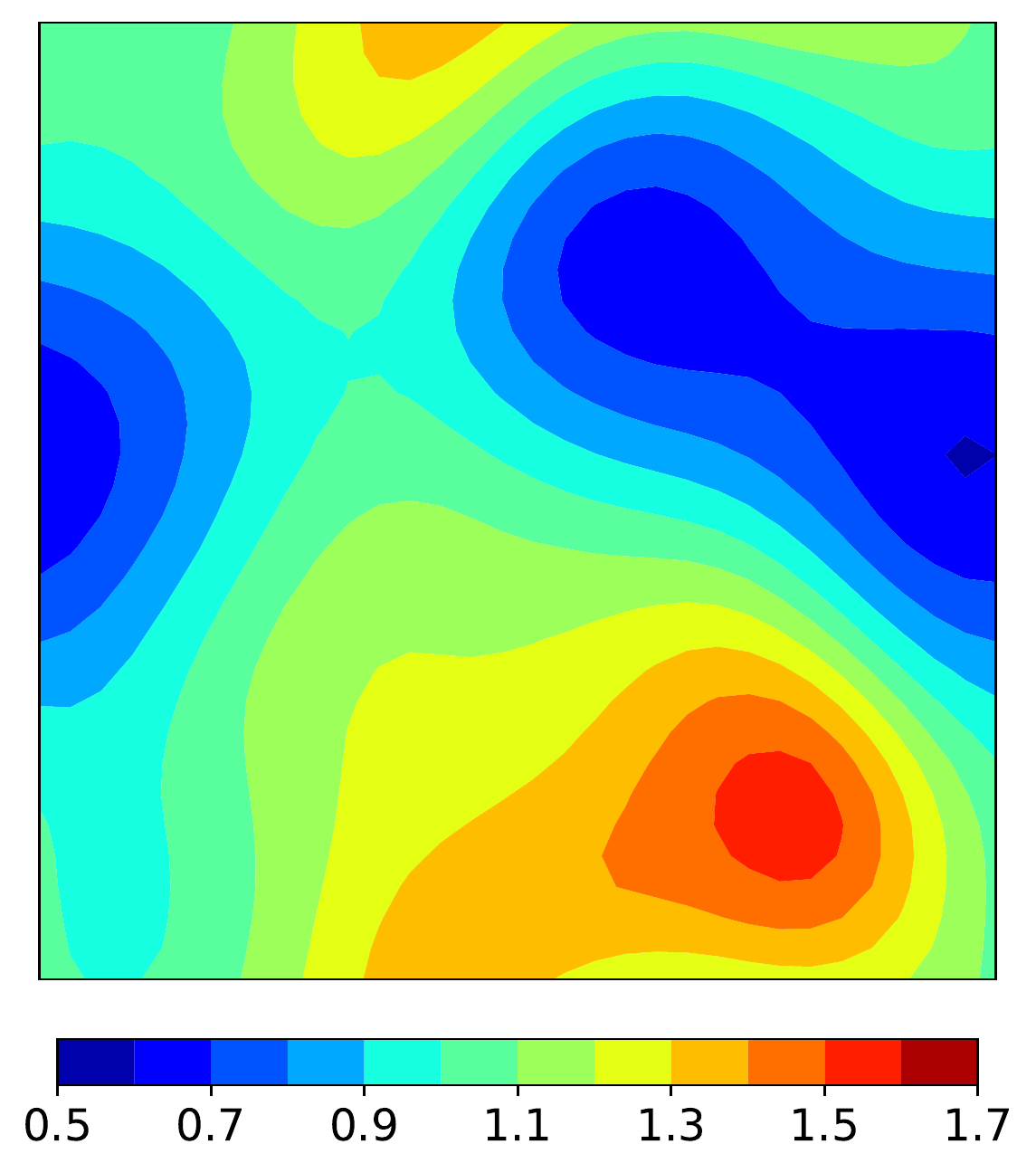}} &
        \raisebox{-0.5\height}{\includegraphics[width=.20 \textwidth]{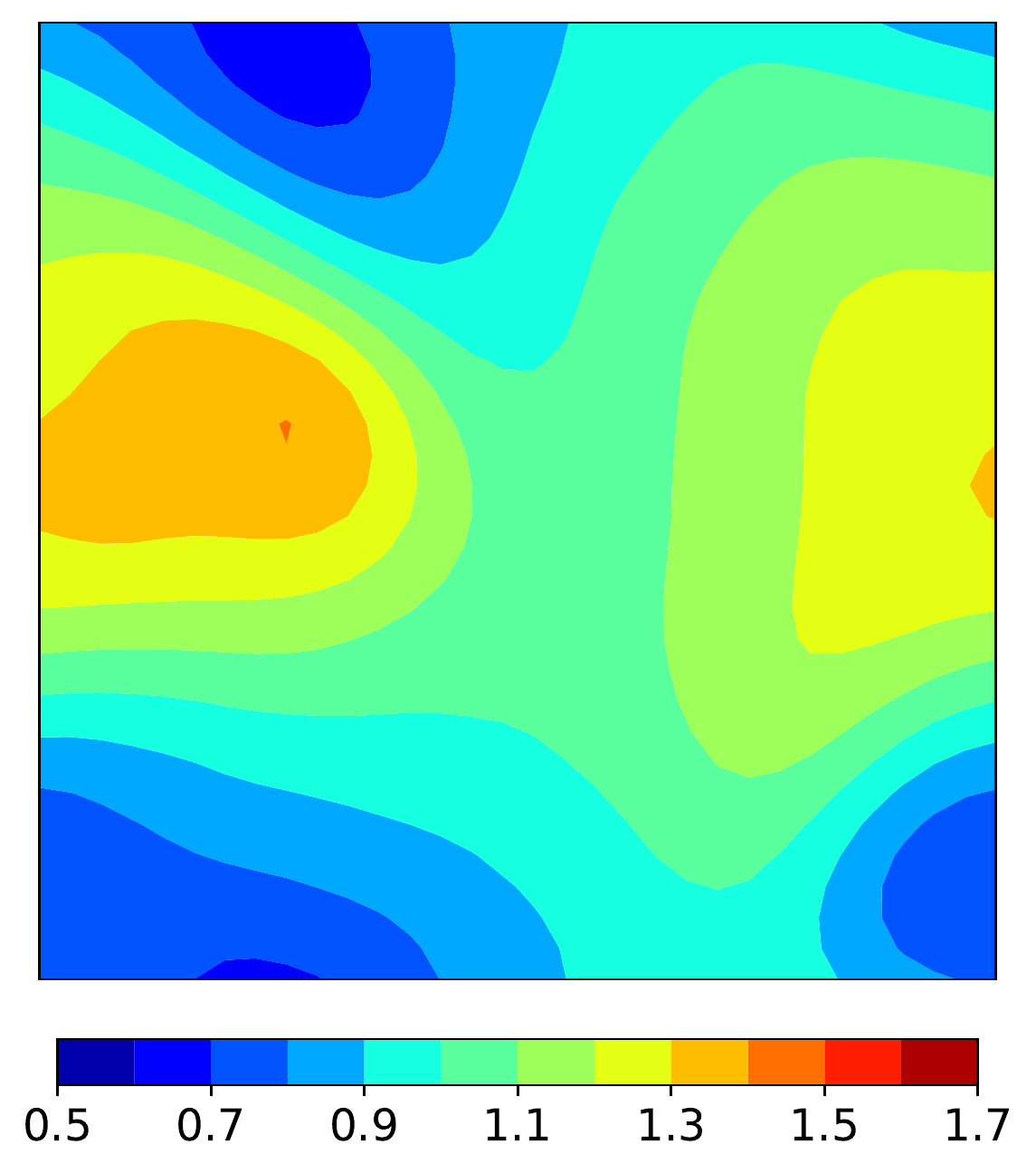}} &
        \raisebox{-0.5\height}{\includegraphics[width=.20 \textwidth]{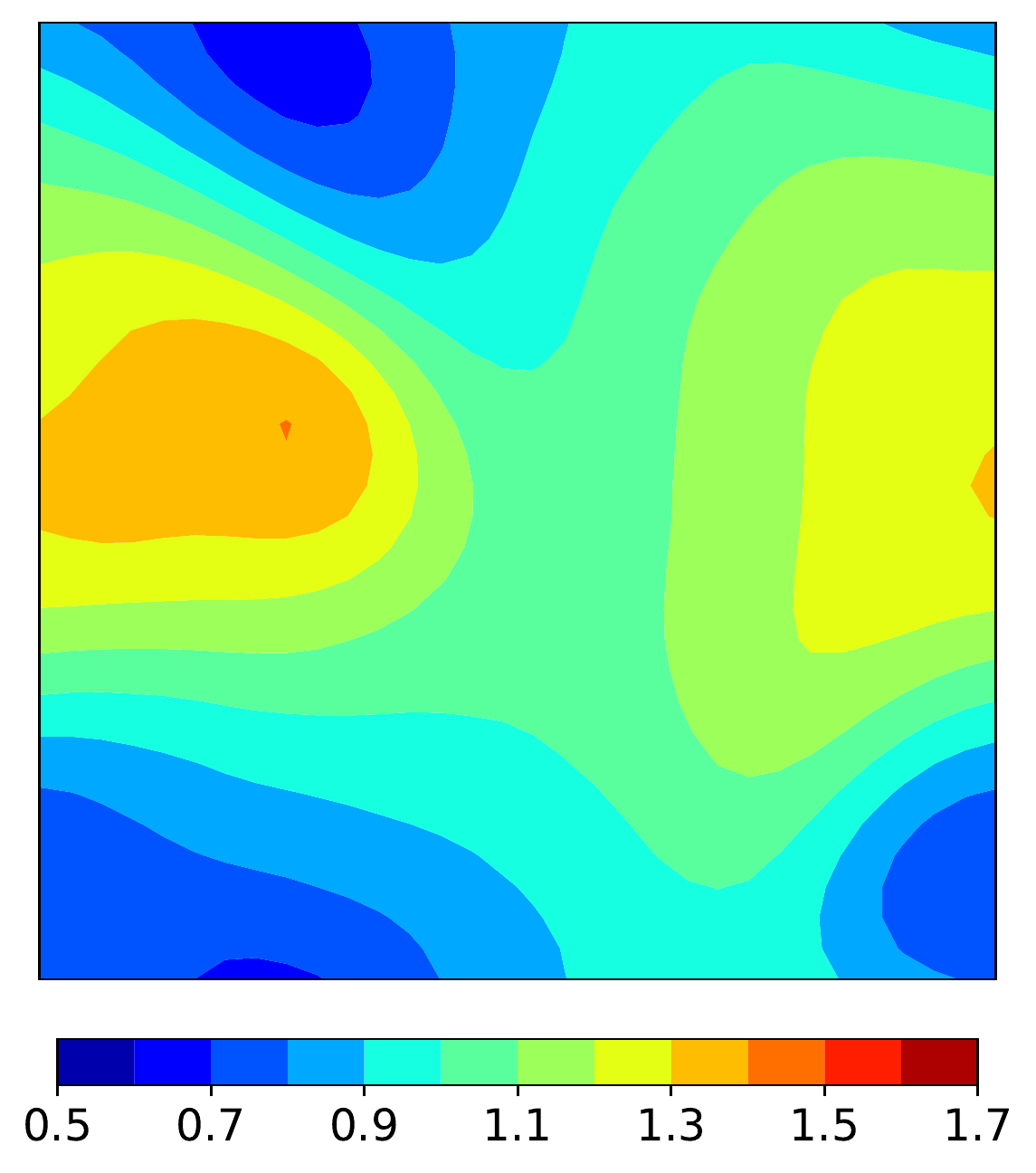}} & 
        \raisebox{-0.5\height}{\includegraphics[width=.20 \textwidth]{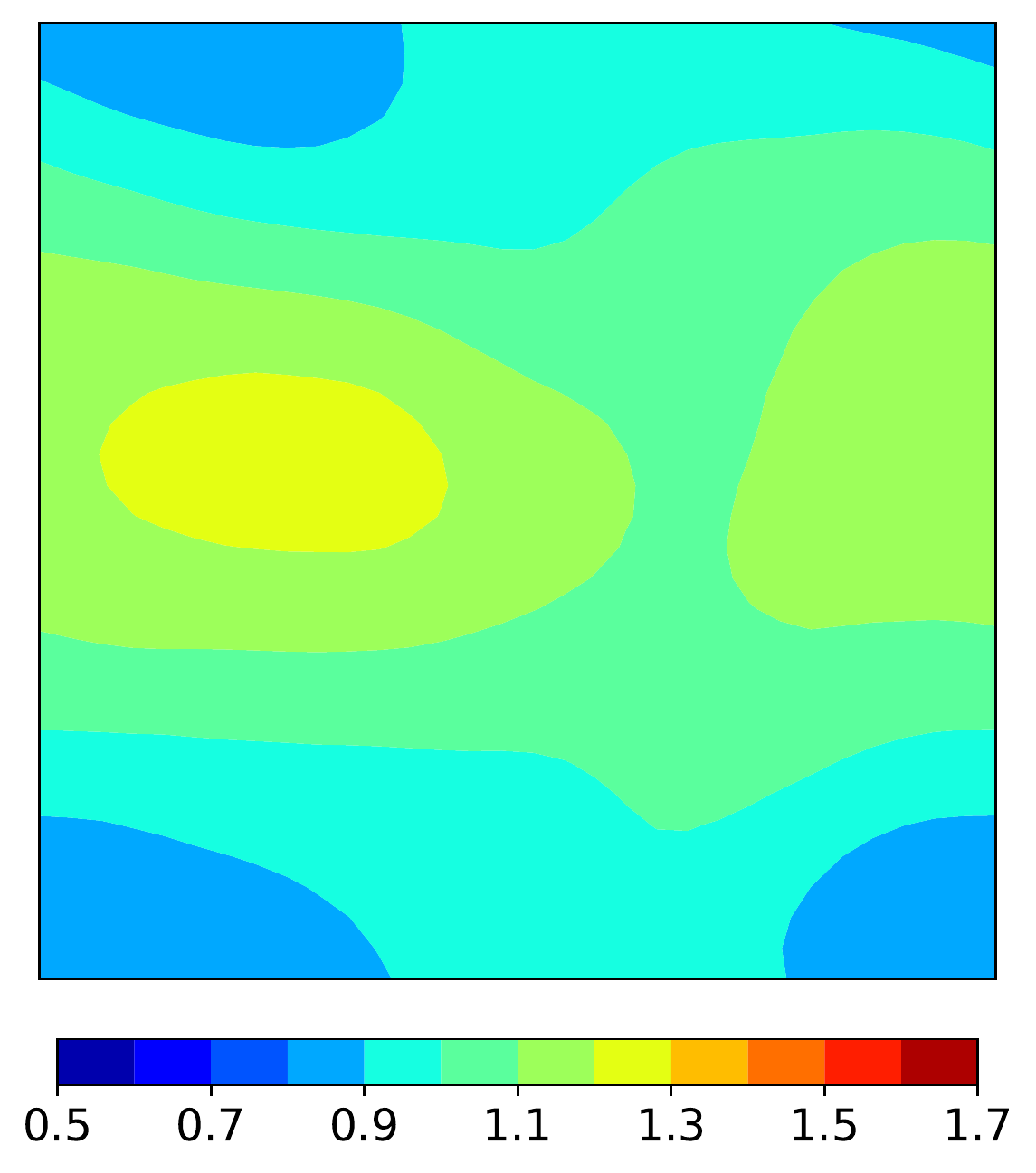}} 
        \\
        \centering
        \rotatebox[origin=c]{90}{\small Predicted $\Psi(\ub)$} &
        \raisebox{-0.5\height}{\includegraphics[width=.20 \textwidth]{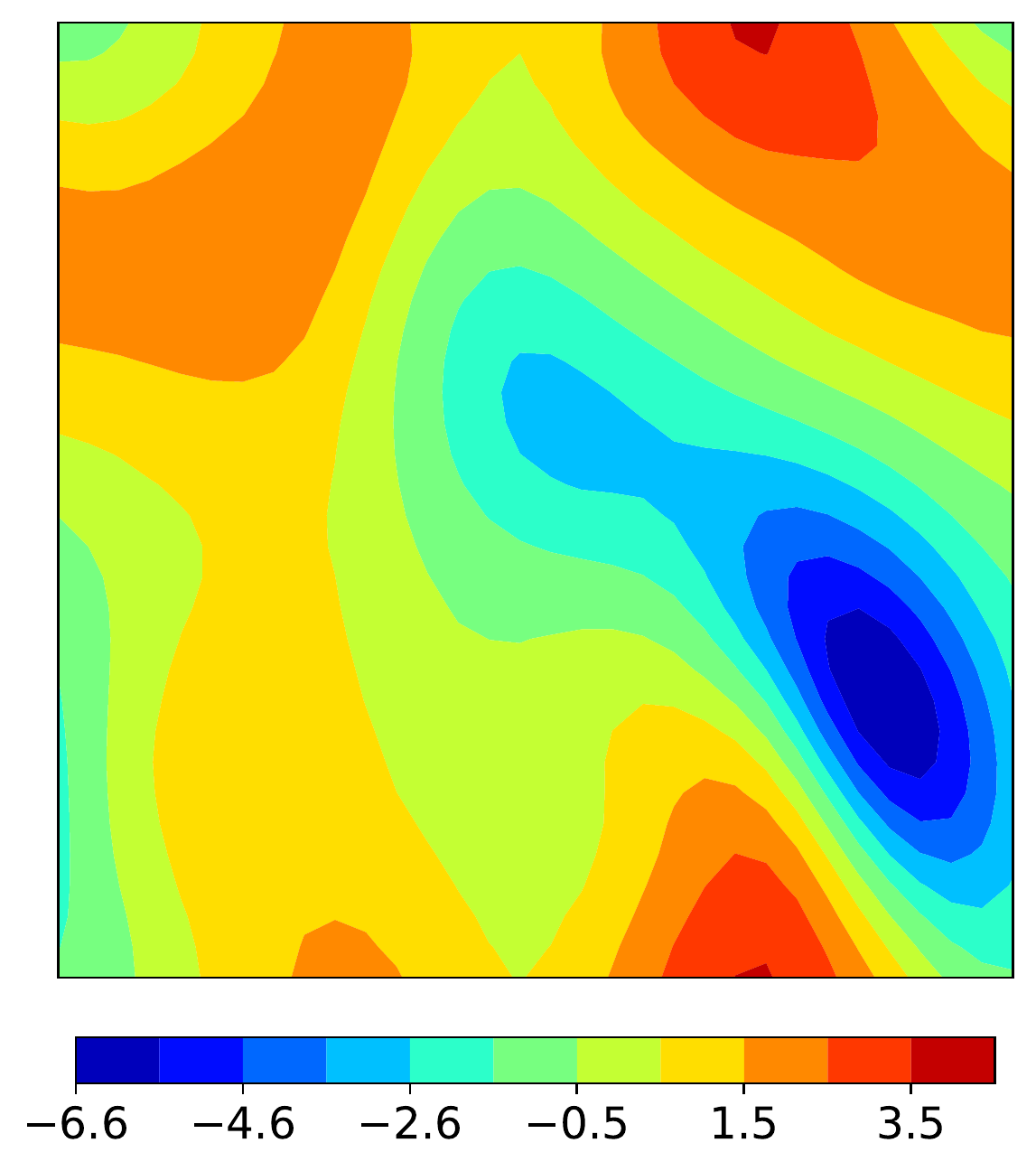}} &
        \raisebox{-0.5\height}{\includegraphics[width=.20 \textwidth]{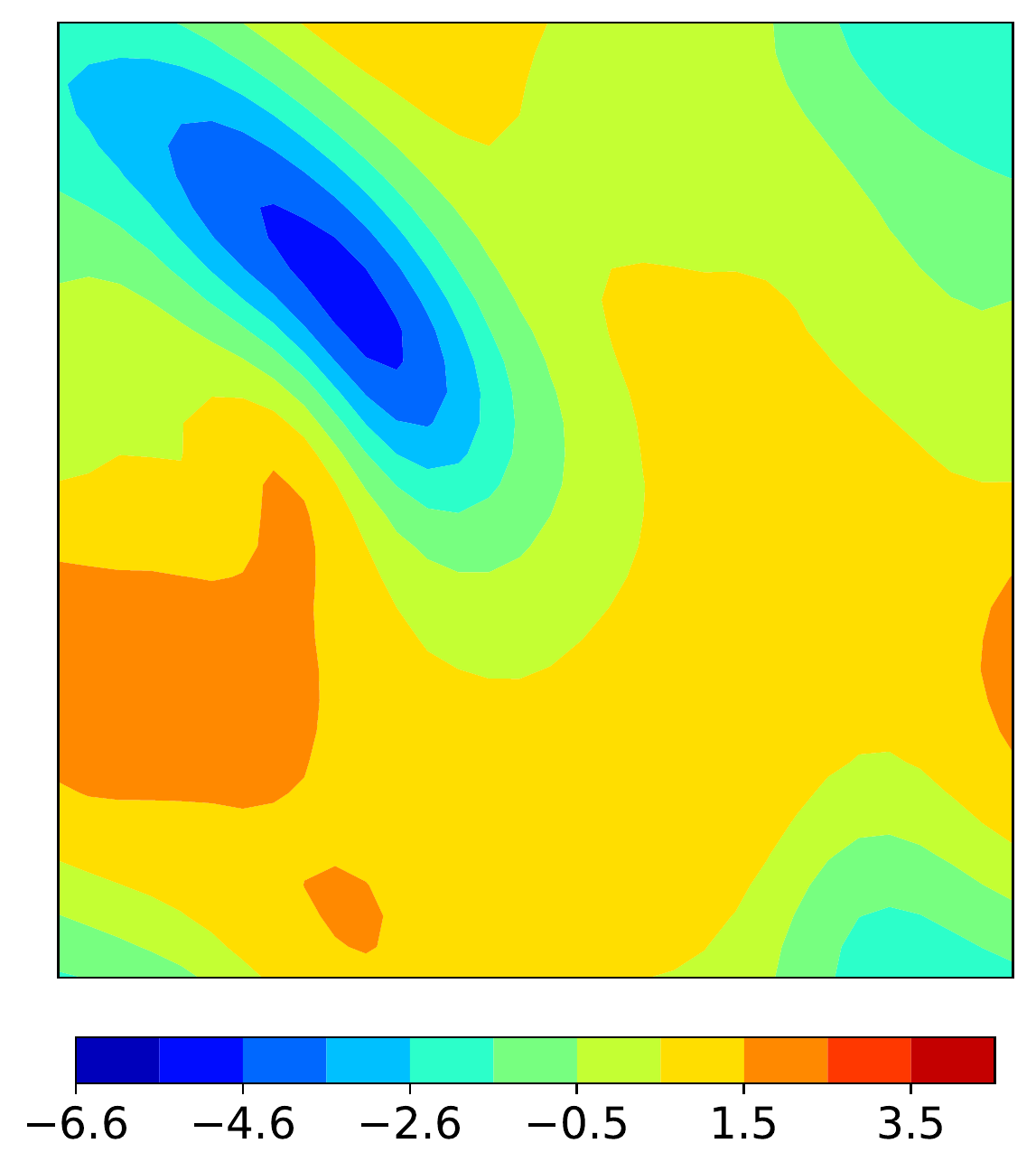}} &
        \raisebox{-0.5\height}{\includegraphics[width=.20 \textwidth]{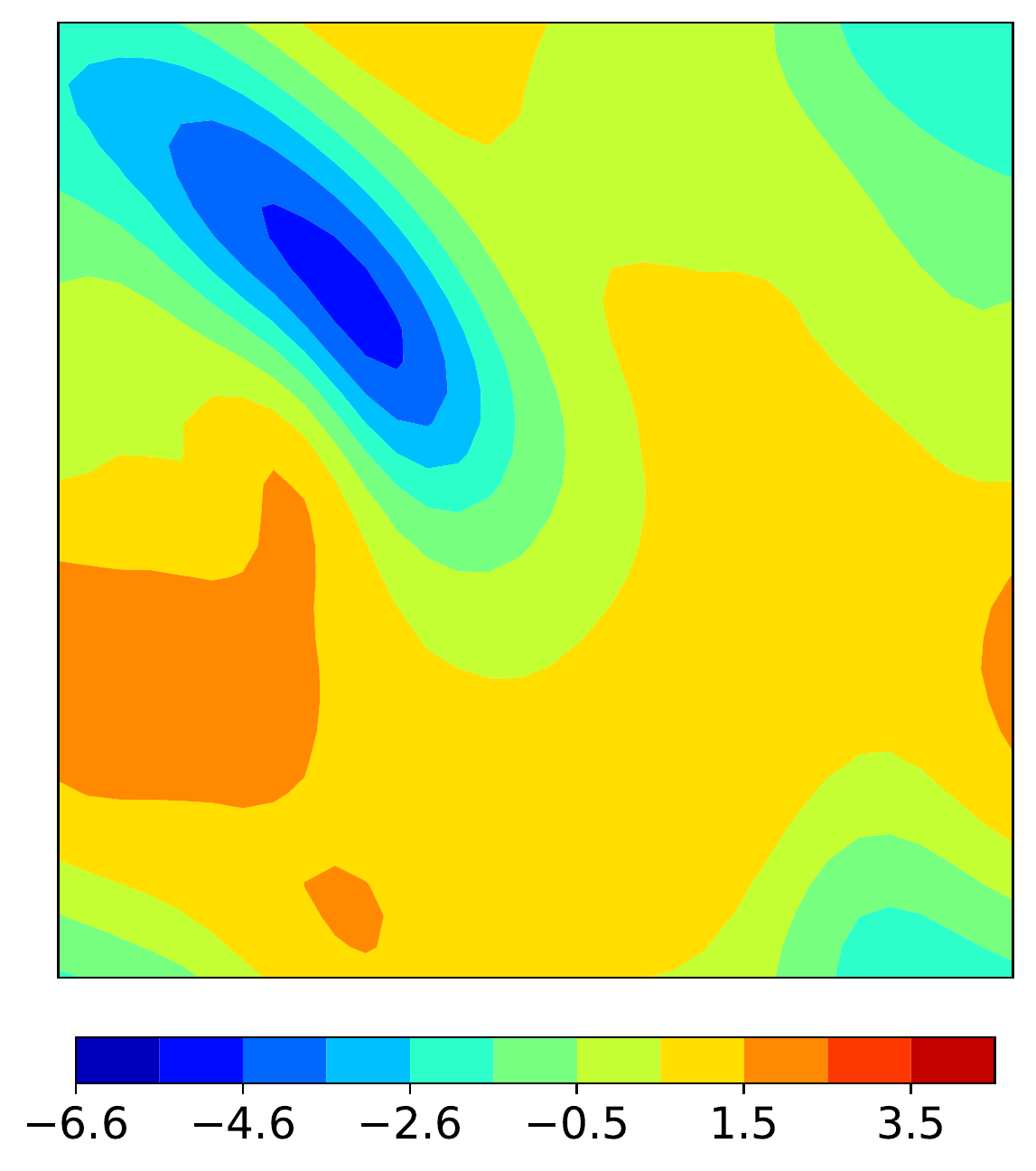}} & 
        \raisebox{-0.5\height}{\includegraphics[width=.20 \textwidth]{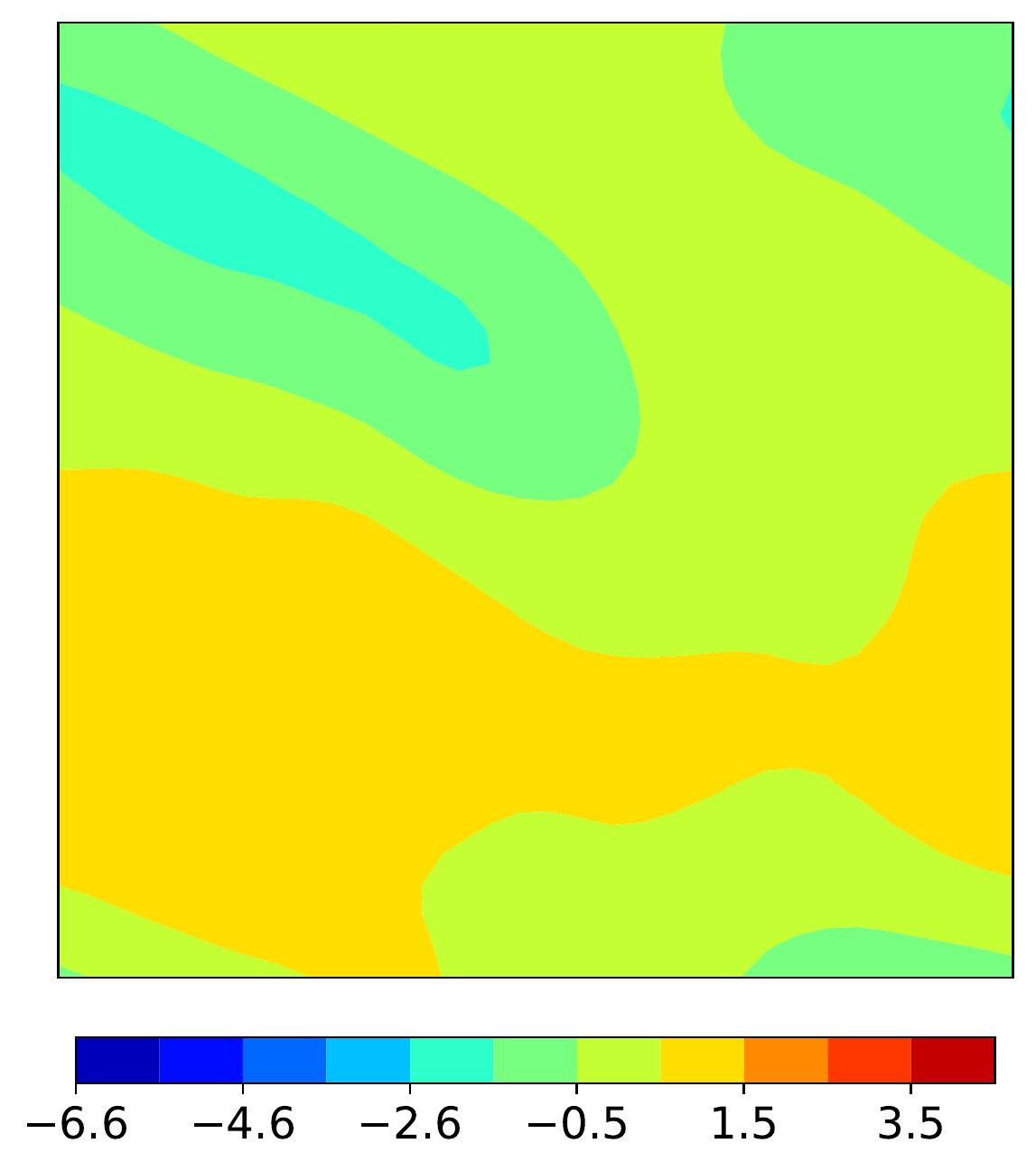}}  
    \end{tabular*}
    \caption{\textbf{Burger's equations}. Predicted solutions and tangent slope using by \texttt{mcTangent} neural networks with $\LRp{d600, 2\%, 1,1, 10^5}$, and time step $\dt'' = \frac{1}{2}\dt = 5 \times 10^{-4}$. \textit{First row}: True high-resolution solutions; \textit{Second row}: contours of True tangent slope, $\F(\ub)$; \textit{Third row}: Predicted \texttt{mcTangent}  solutions, $\ub$; \textit{Fourth row}: contours of \texttt{mcTangent} tangent slope, $\Psi(\ub)$.} 
    \figlab{2D_Bur_smaller_dt_samples}
\end{figure}

{\bf Implicit time integration with learned network.}
Another appealing feature of tangent slope learning is that once trained it can be deployed with any time discretization schemes. 
We use the learned network from the setting $\LRp{d600, 2\%, 1, 1, 10^5}$ together with the backward Euler method with a larger time stepsize $\dt' = 12.5 \dt = 1.25 \times 10^{-2}$, where $\dt = 10^{-3}$ is the training stepsize. Shown in \cref{fig:2D_Bur_implicit_samples} are predicted solutions at $t=\LRc{0, 0.1, 0.5, 1.5}$ corresponding to $0, 100, 500, 1500$th time steps. We observe that solutions using the forward Euler scheme, regardless of using the true tangent slope or learned one (second and third rows, respectively),  are unstable as the time stepsize $\dt'$ is too big for stability. On the contrary, using the backward Euler scheme, \texttt{mcTangent} solutions are comparable to the true counterparts (fourth and fifth rows, respectively). 
Clearly, due to large time stepsize, both are more diffusive compared to the true solutions with small time stepsize $\dt$ in the first row.
%
%It is worth noting that the discrepancy of implicit scheme solutions from the true solutions is due to the nature of the scheme, the error stems from  using the large time $\dt'$.

\begin{figure}[htb!]
    \centering
    \begin{tabular*}{\textwidth}{c c c c c}
        \centering
         &
        \raisebox{-0.5\height}{\small $t = 0$} &
        \raisebox{-0.5\height}{\small $t = 0.1$} &
        \raisebox{-0.5\height}{\small $t = 0.5$} & 
        \raisebox{-0.5\height}{\small $t = 1.5$} 
        \\
        \centering
        \rotatebox[origin=c]{90}{\small True solution $\ub$} &
        \raisebox{-0.5\height}{\includegraphics[width=.20 \textwidth]{figures/2D_Bur/Bur_FD_127x127_step_t_0.pdf}} &
        \raisebox{-0.5\height}{\includegraphics[width=.20 \textwidth]{figures/2D_Bur/Bur_FD_127x127_step_t_100.pdf}} &
        \raisebox{-0.5\height}{\includegraphics[width=.20 \textwidth]{figures/2D_Bur/Bur_FD_127x127_step_t_500.pdf}} & 
        \raisebox{-0.5\height}{\includegraphics[width=.20 \textwidth]{figures/2D_Bur/Bur_FD_127x127_step_t_1500.pdf}} 
        \\
        \centering
        \rotatebox[origin=c]{90}{FE - True $\F$} &
        \raisebox{-0.5\height}{\includegraphics[width=.20 \textwidth]{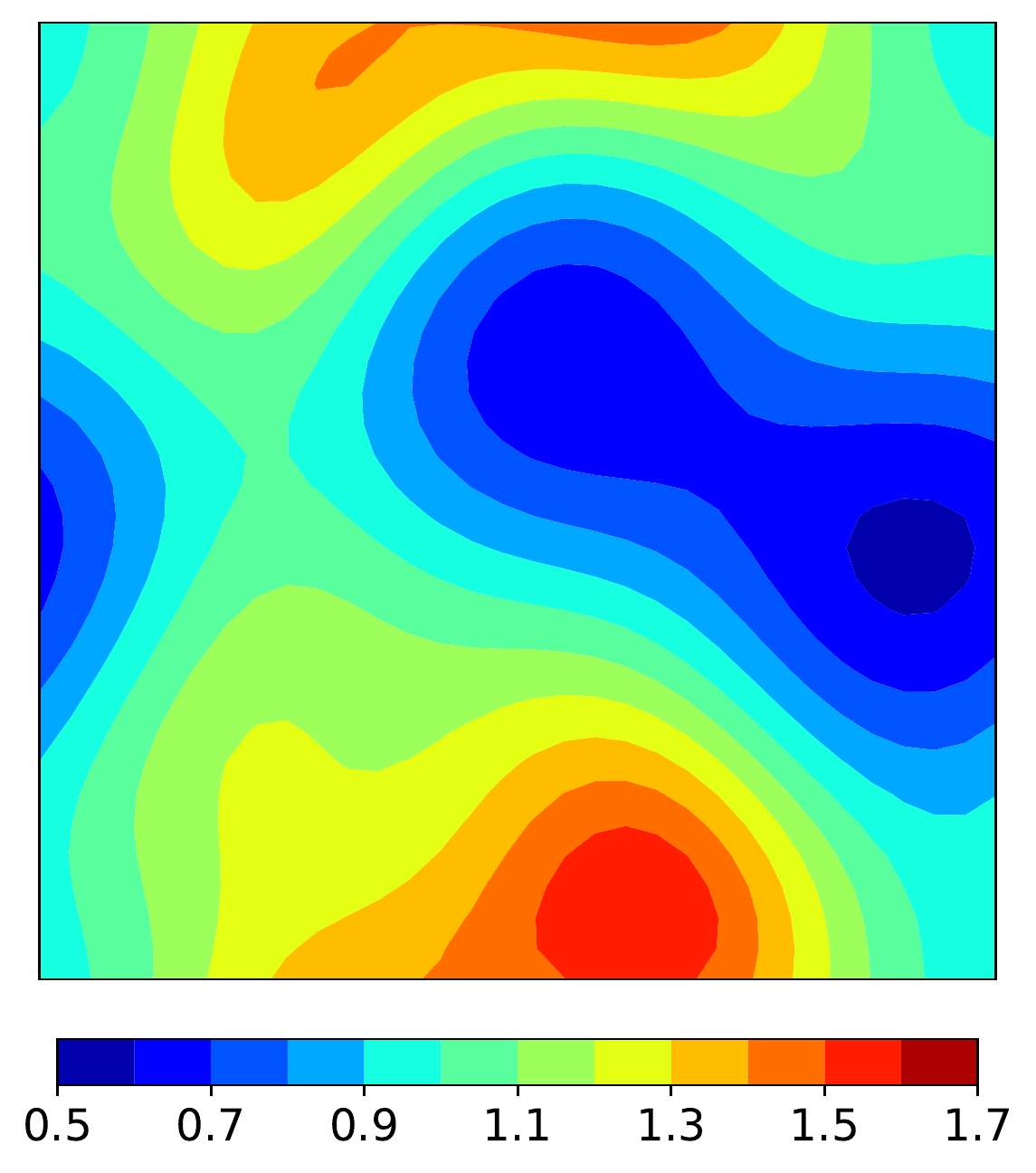}} &
        \raisebox{-0.5\height}{\includegraphics[width=.20 \textwidth]{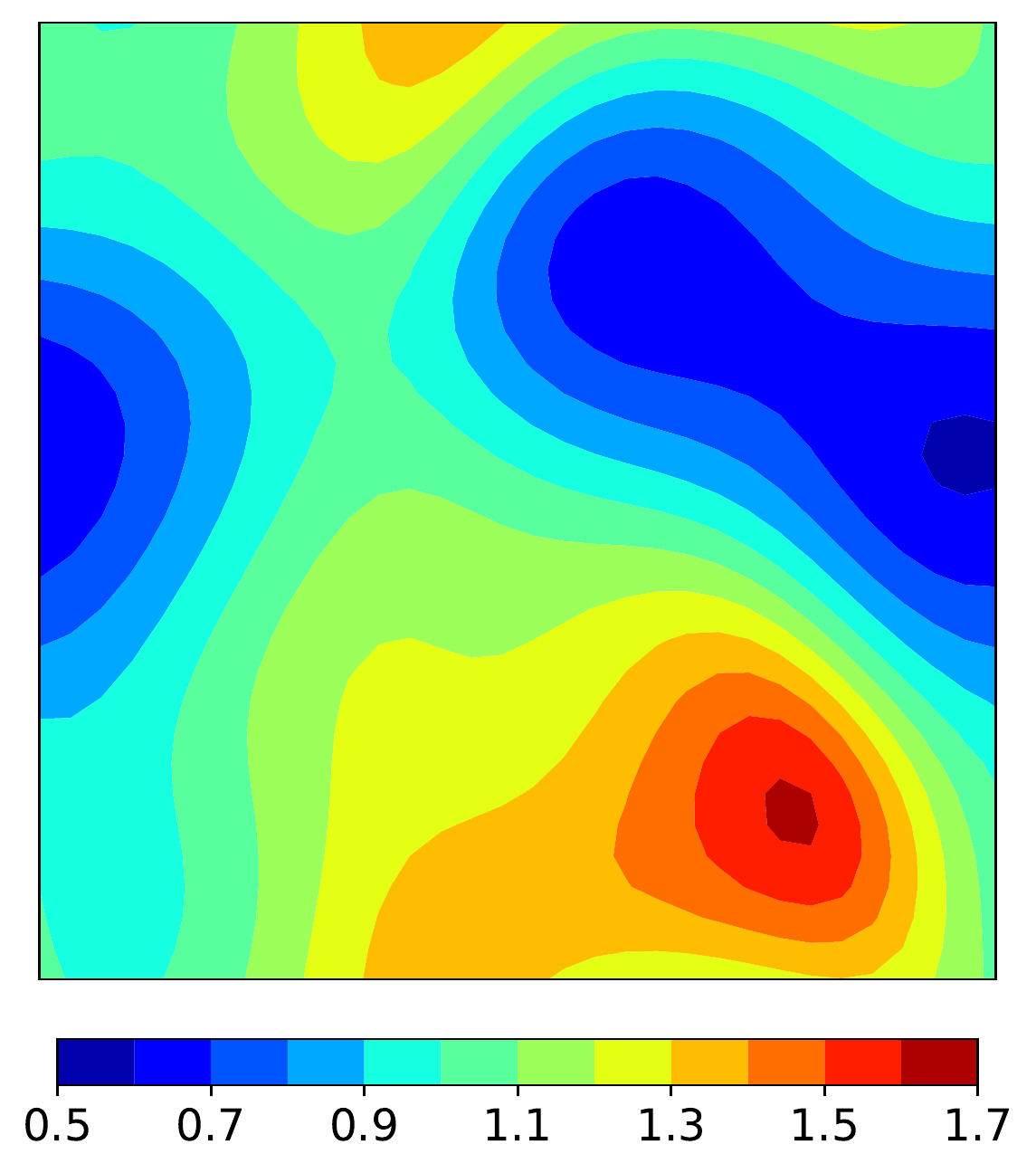}} &
        \raisebox{-0.5\height}{\includegraphics[width=.20 \textwidth]{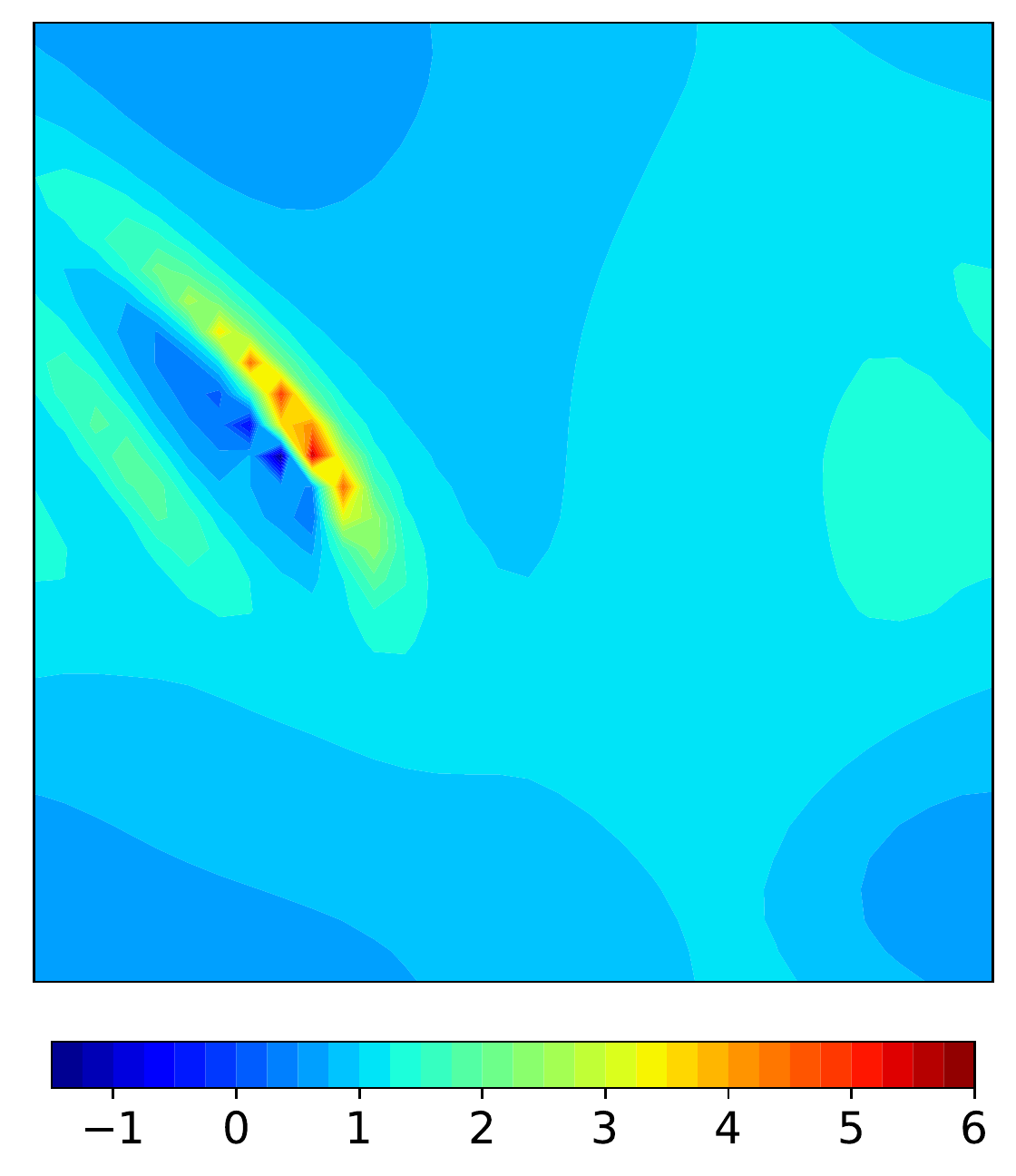}} & 
        \raisebox{-0.5\height}{\includegraphics[width=.20 \textwidth]{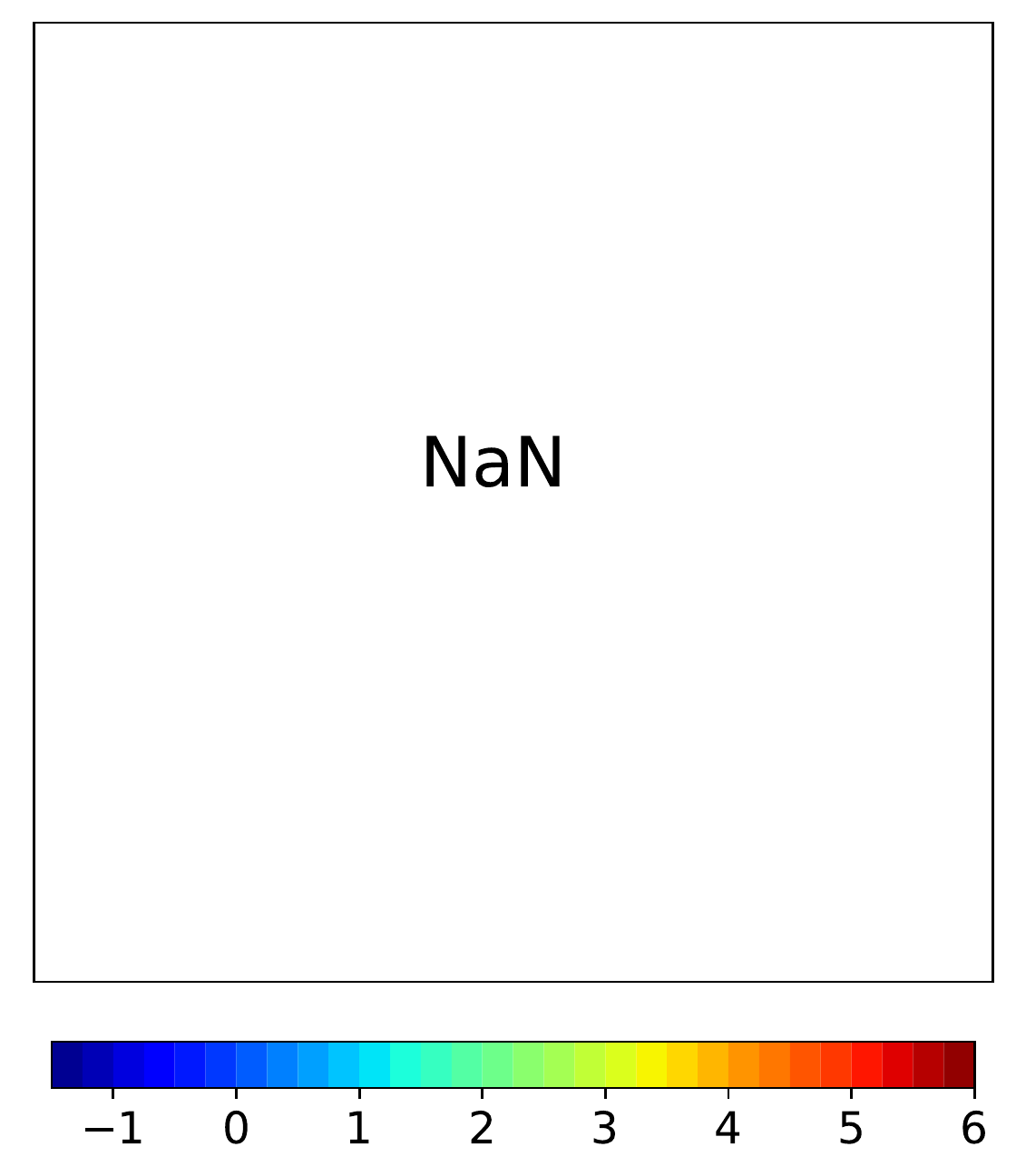}} 
        \\
        \centering
        \rotatebox[origin=c]{90}{FE - Learned $\Psi$} &
        \raisebox{-0.5\height}{\includegraphics[width=.20 \textwidth]{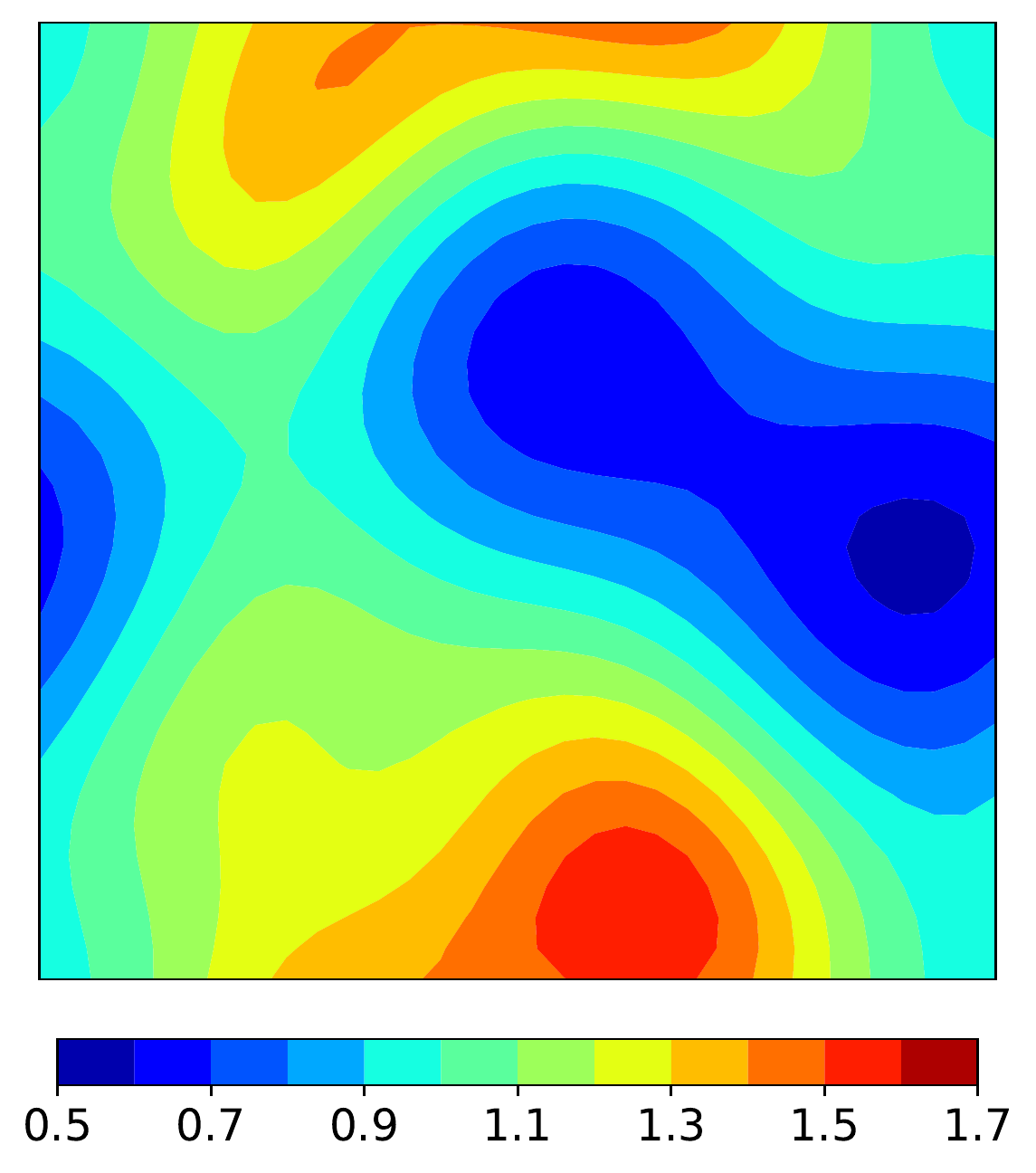}} &
        \raisebox{-0.5\height}{\includegraphics[width=.20 \textwidth]{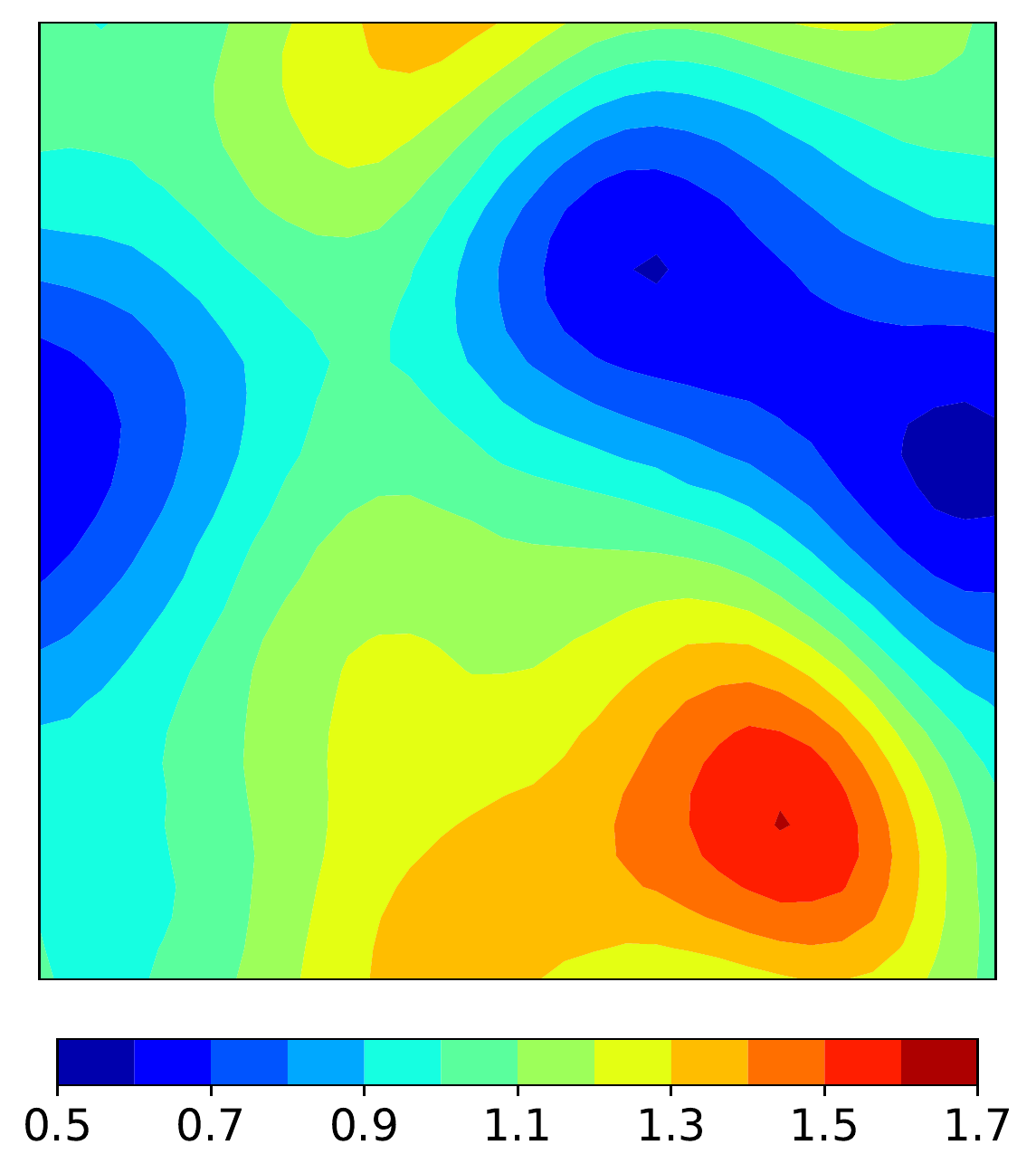}} &
        \raisebox{-0.5\height}{\includegraphics[width=.20 \textwidth]{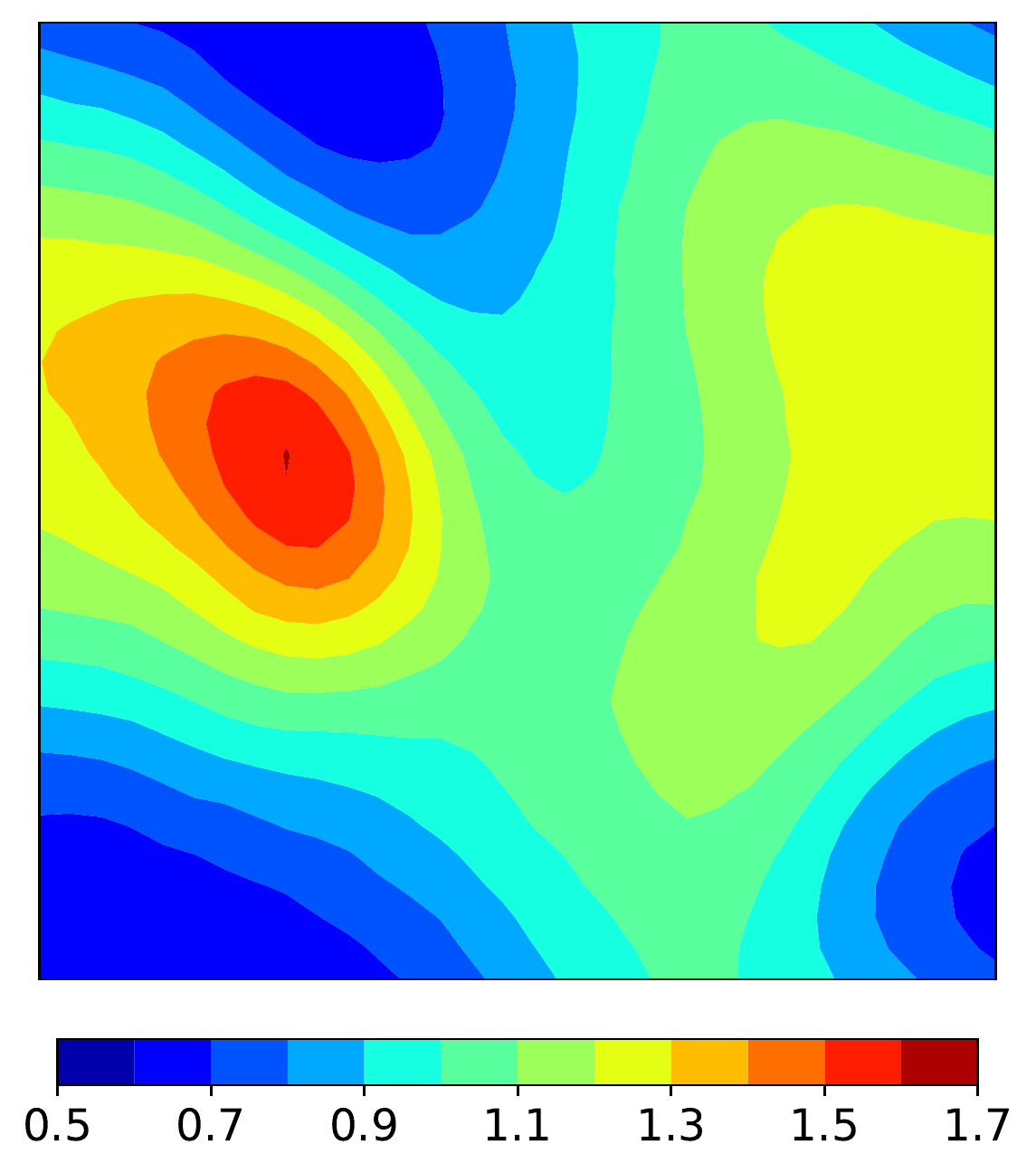}} & 
        \raisebox{-0.5\height}{\includegraphics[width=.20 \textwidth]{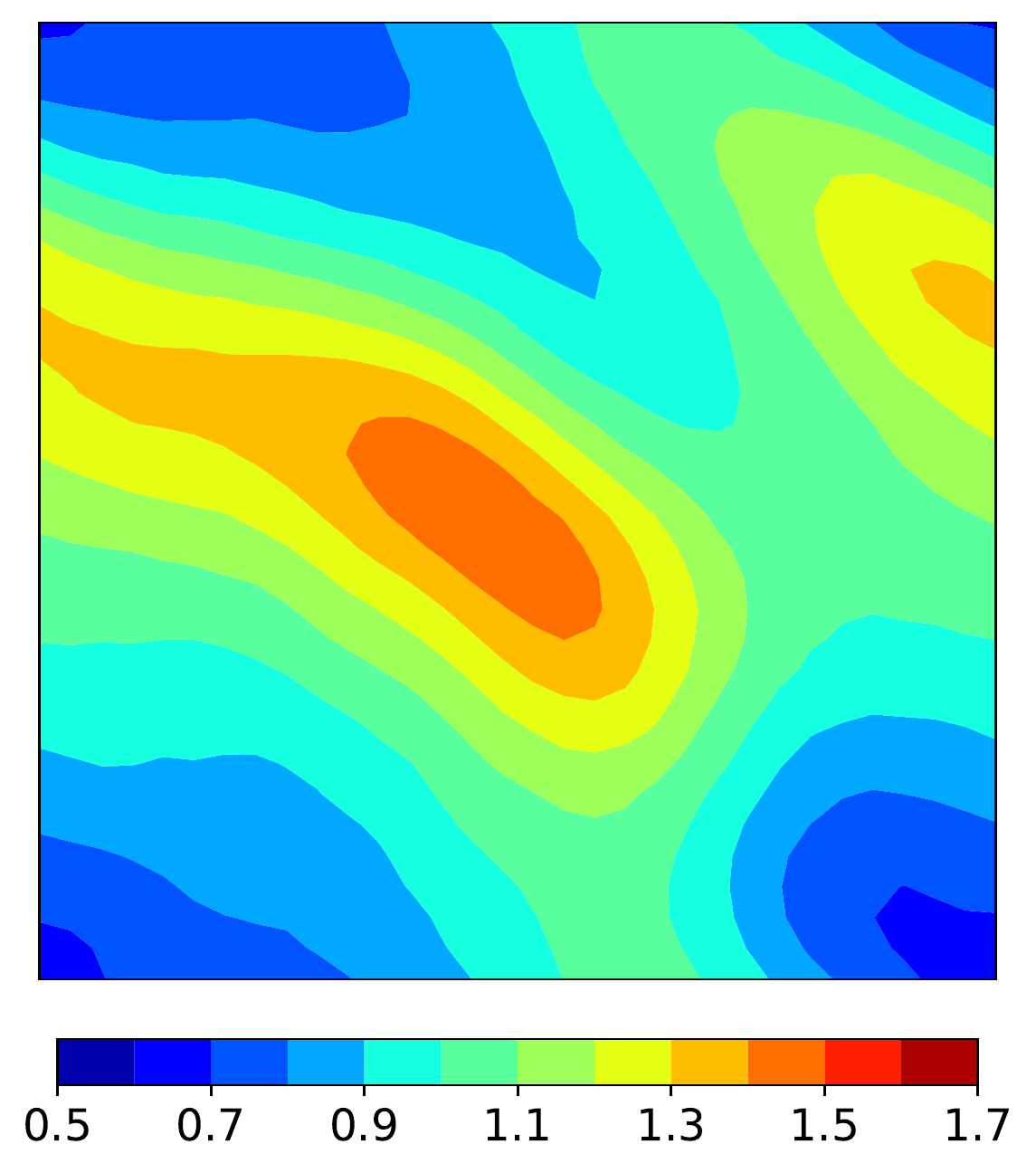}} 
         \\
        \centering
        \rotatebox[origin=c]{90}{\small BE - True $\F$} &
        \raisebox{-0.5\height}{\includegraphics[width=.20 \textwidth]{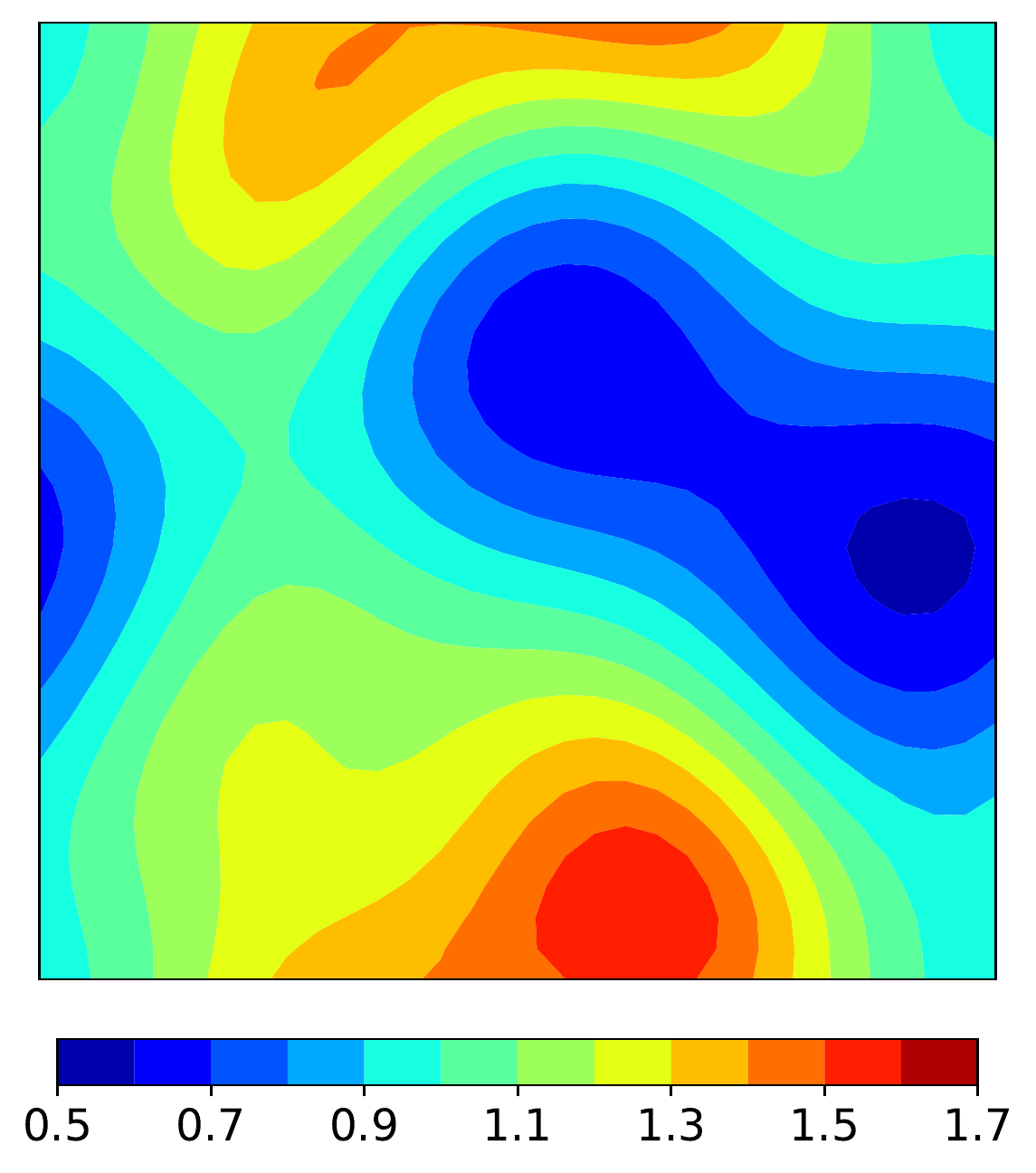}} &
        \raisebox{-0.5\height}{\includegraphics[width=.20 \textwidth]{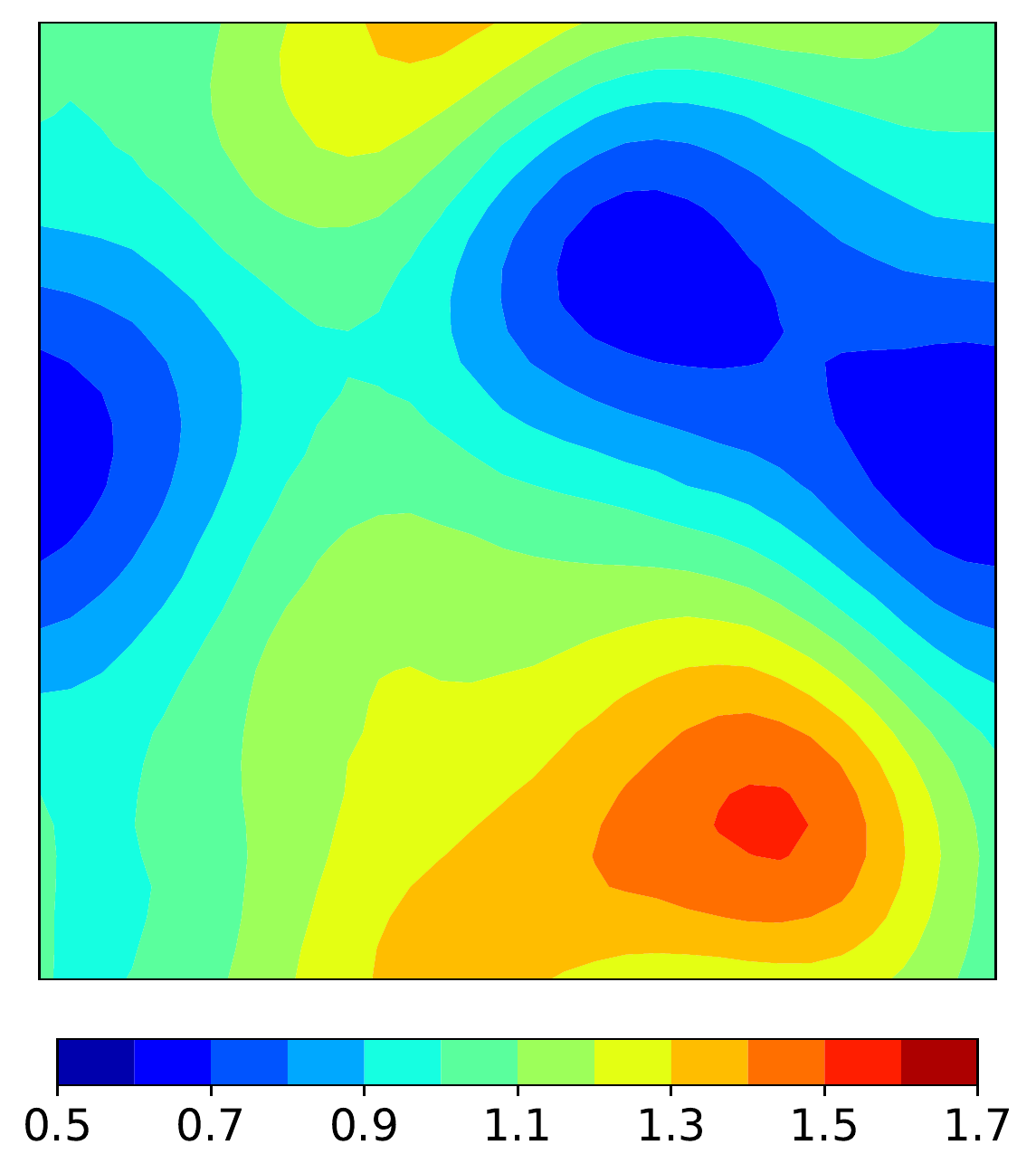}} &
        \raisebox{-0.5\height}{\includegraphics[width=.20 \textwidth]{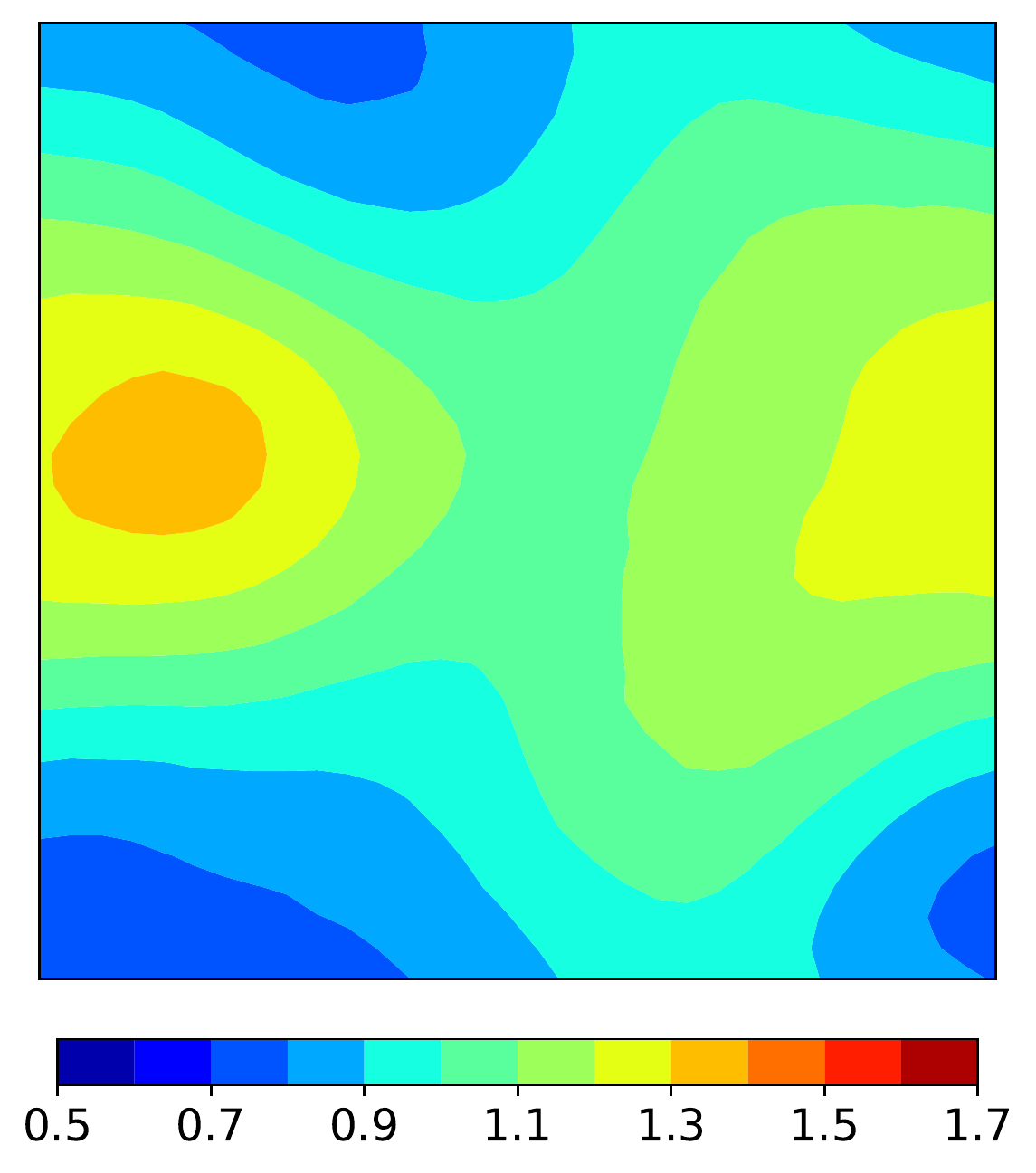}} & 
        \raisebox{-0.5\height}{\includegraphics[width=.20 \textwidth]{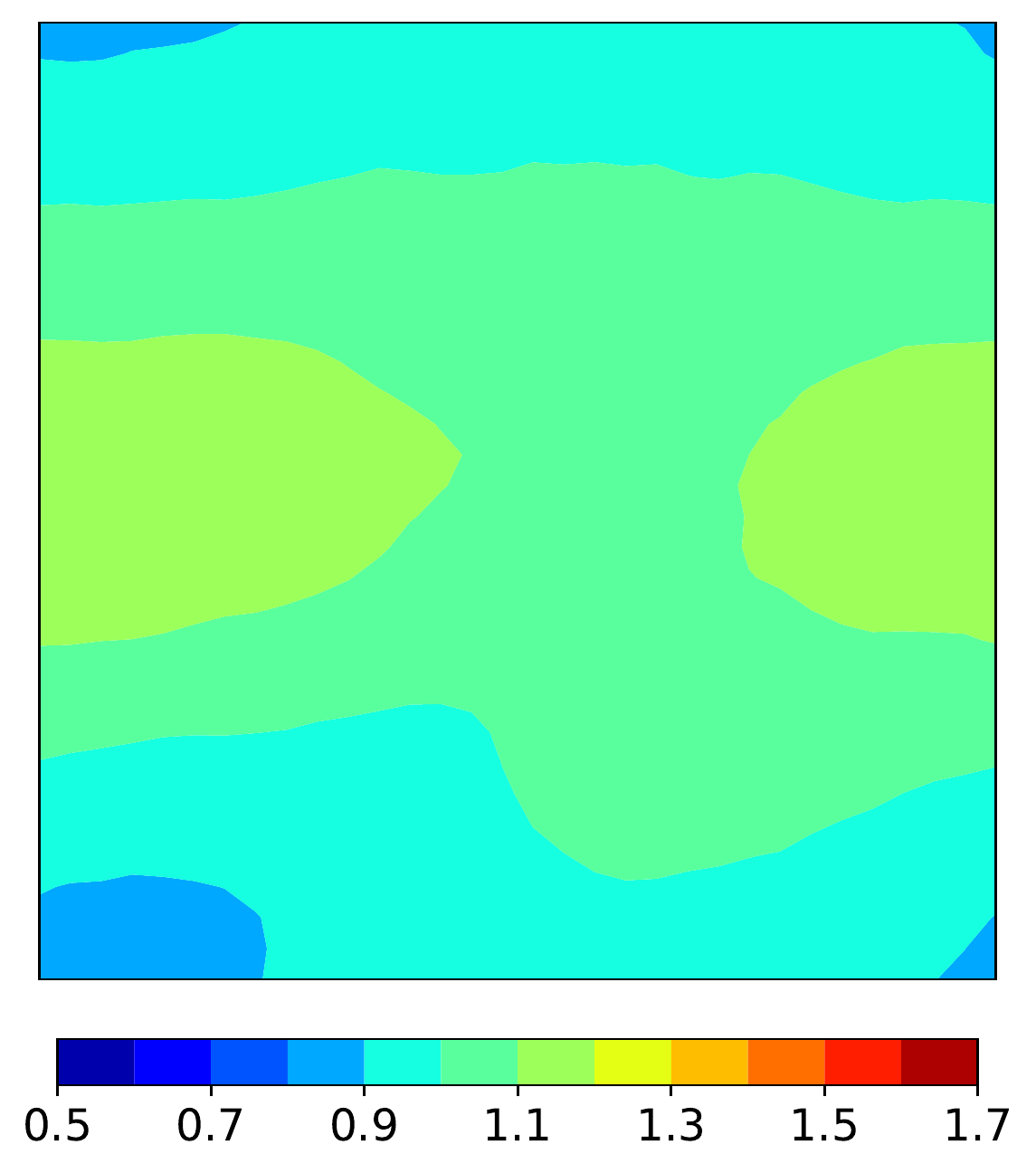}} 
        \\
        \centering
        \rotatebox[origin=c]{90}{\small BE - Learned $\Psi$} &
        \raisebox{-0.5\height}{\includegraphics[width=.20 \textwidth]{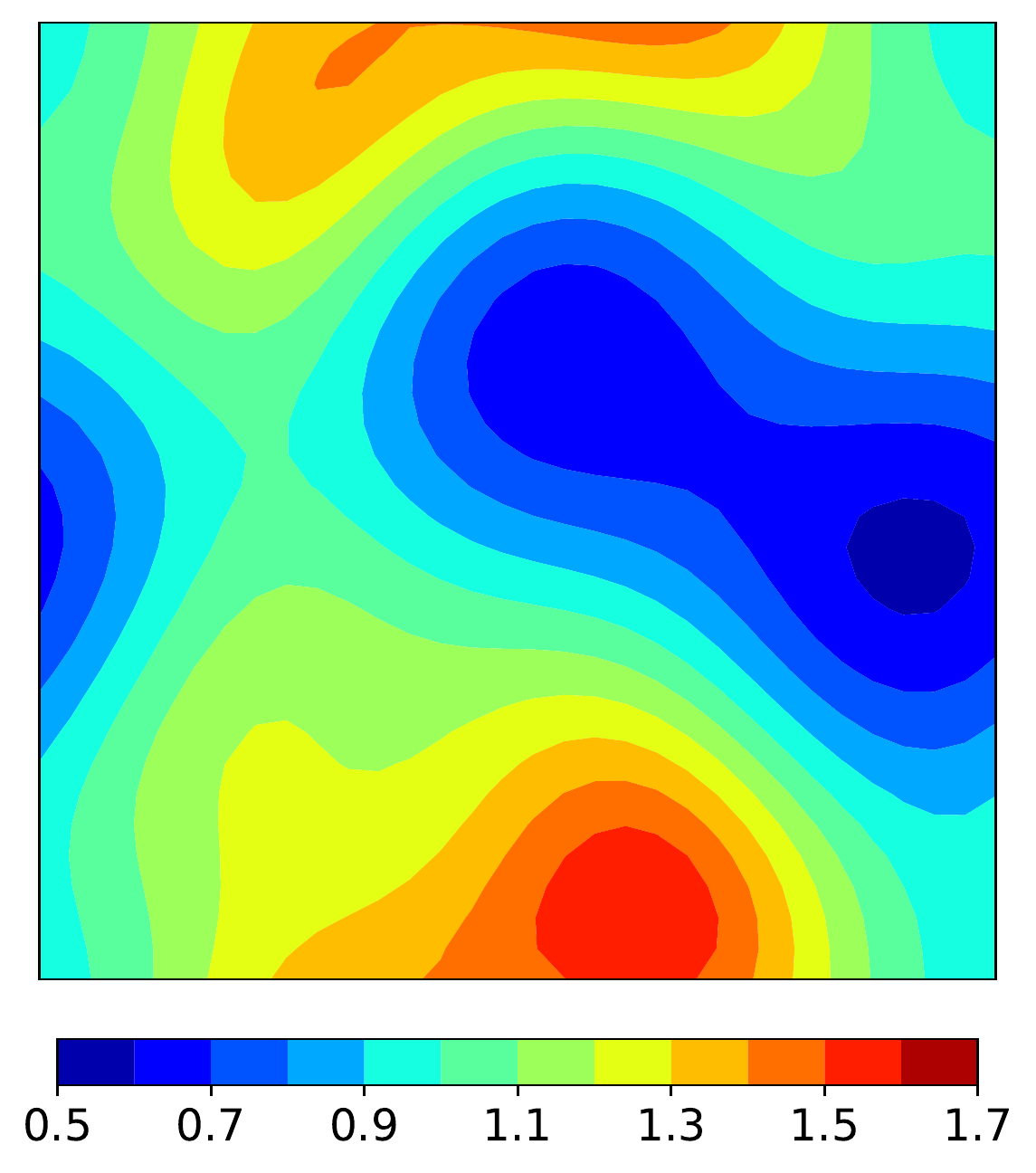}} &
        \raisebox{-0.5\height}{\includegraphics[width=.20 \textwidth]{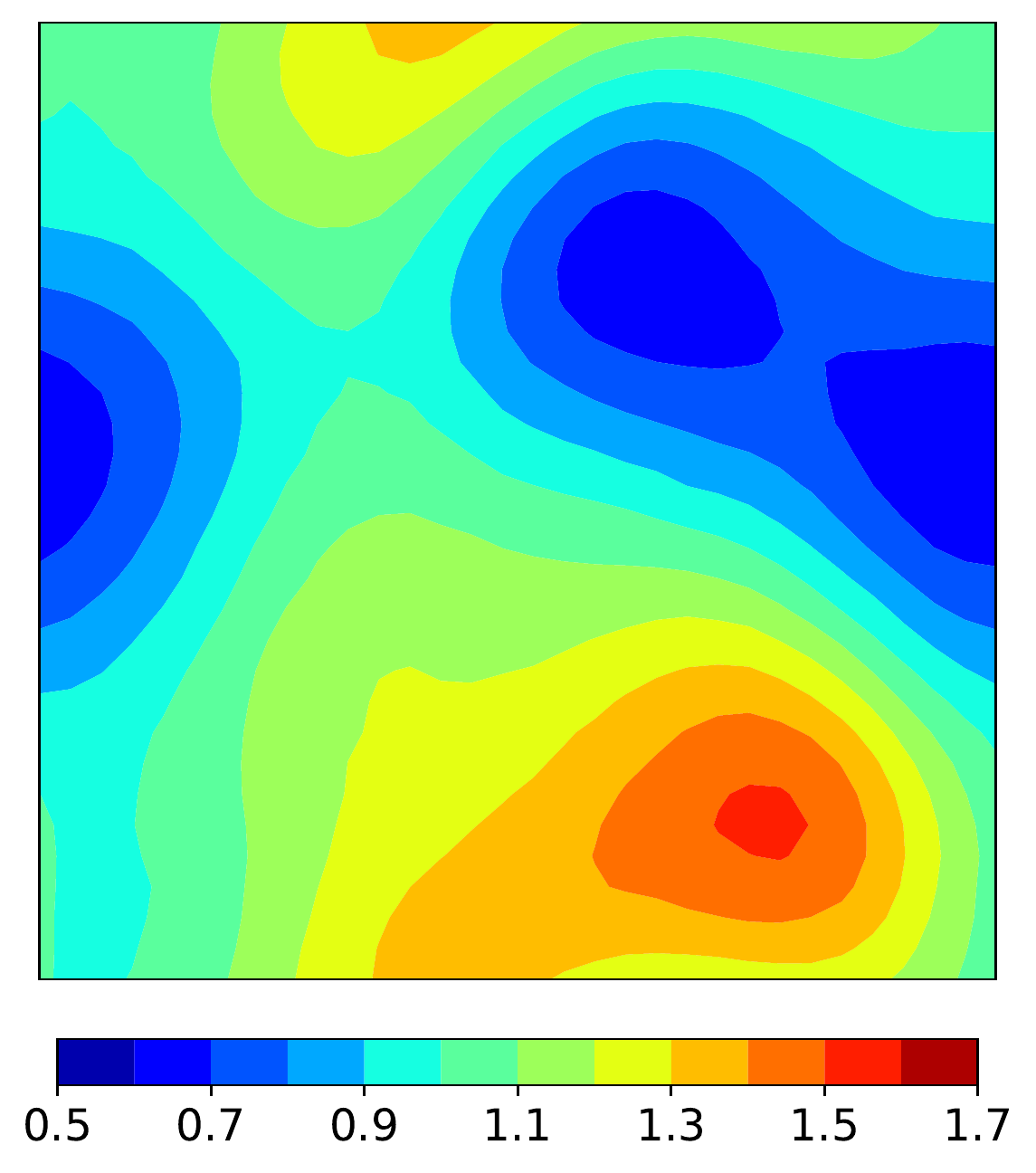}} &
        \raisebox{-0.5\height}{\includegraphics[width=.20 \textwidth]{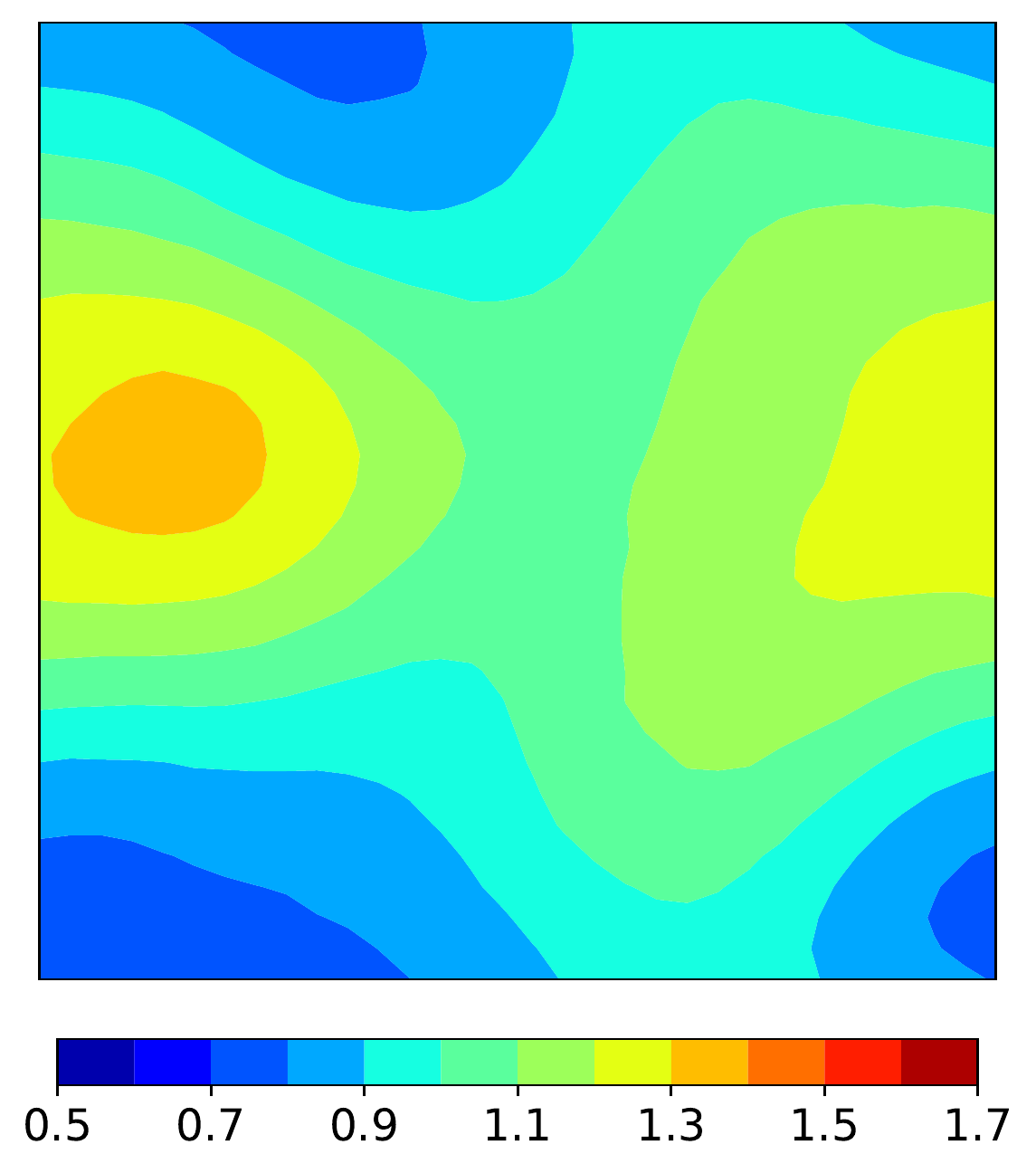}} & 
        \raisebox{-0.5\height}{\includegraphics[width=.20 \textwidth]{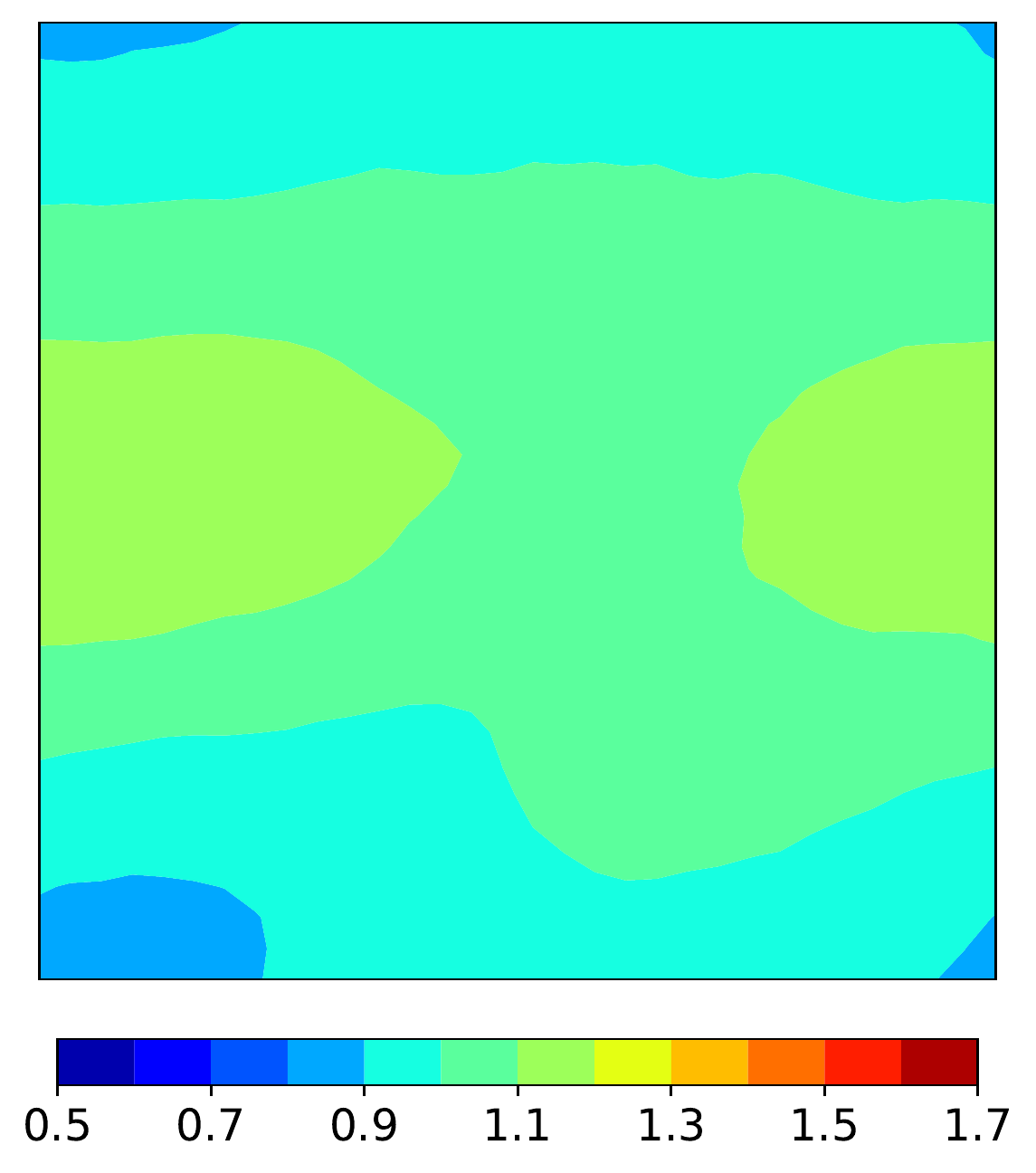}} 
    \end{tabular*}
    \caption{\textbf{Burger's equations}. Predicted solutions at different times obtained by Forward Euler (FE) scheme and Backward Euler (BE) scheme using large stepsize $\dt' = 12.5 \dt = 1.25 \times 10^{-2}$ with the true tangent slope $\F$ and the learned neural network $\Psi$ for the setting $\LRp{d600, 2\%, 1, 1, 10^5}$.}
    \figlab{2D_Bur_implicit_samples}
\end{figure}

{\bf Direct learning versus \texttt{mcTangent} slope  learning.}
Recall that by direct learning we mean learning the map from $\ui{i}$ to $\ui{i+1}$ for two consecutive time steps.
We investigate the difficulty and complexity of direct learning.
Specifically, we use a data set with $600$ samples with/without  data randomization to learn the neural network with one layer of 5000 neurons that maps velocities from one step to the next. As shown in \cref{fig:2D_Bur_Direct_EF_NNs}, the direct learning approach (with the best combination of hyperparameters) for the setting  $\LRp{d600, 2\%, 1,3,2}$ is less accurate for both short-time and long-time predictions compared to the tangent learning counterpart  with even a smaller data set of $200$ samples with the setting $\LRp{d200, 2\%, 1,1,10^5}$. 
Interestingly, unlike the tangent learning approach, the direct learning approach, both pure data-driven and model-constrained approaches, trained with randomized data is  less accurate compared to noise-free data in short sequential training $S = 1$. On the other hand, data randomization does not have visible benefits on long sequential training $S = 10$. Specifically, both $\LRp{d600, 0\%, 10,1,0}$ and $\LRp{d600, 2\%, 10,1,0}$ settings behave similarly.

Also seen in \cref{fig:2D_Bur_Direct_EF_NNs}, among pure data-driven networks ($\alpha = 0$) with direct learning,  long sequential machine learning training with $S = 10$ is the most  accurate.
 Model-constrained network with direct learning for the setting  $\LRp{d600, 2\%, 1, 1, 10}$ is much more accurate  compared to the pure data-driven network with direct learning for the same setting.
Moreover,
sequential model-constrained networks for $R = 2, 3$ corresponding to two settings $\LRp{d600,2\%, 1, 2, 2}$ and $\LRp{d600,2\%, 1, 3, 2}$ are comparable to  much longer sequential machine learning network with $S = 10$ for the setting $\LRp{d600,2\%, 10, 1, 0}$. 
 In the presented results,
it is important to point out that for direct learning, 
care must be taken in choosing a good regularization parameter $\alpha$. For example, $\alpha = 2$ is good for $R = 2, 3$, but $\alpha = 10$ is good for $R = 1$. On the contrary, tangent learning is more robust. In particular, a single $\alpha = 10^5$ works well for all settings.
Solutions predicted by direct and tangent learnings
(both with model-constrained terms) 
for $\LRp{d600, 2\%, 1, 3, 2}$  and $\LRp{d200, 2\%, 1, 1, 10^5}$, respectively, are shown in \cref{fig:2D_Bur_Direct_samples}. As can be observed, tangent learning solutions with even smaller data set $\LRp{d200, 2\%, 1, 1, 10^5}$  are much more accurate than the direct learning with $\LRp{d600, 2\%, 1, 3, 2}$. This is due to the fact that direct learning tries to learn a mixed space-time discretization, which is more difficult than learning only the spatial discretization in tangent learning.

%network $\LRp{d600, 2\%, 1, 3, 2}$ approximates the map mixing the tangent slope and temporal evolution operator, which is a complicated and challenging one. Meanwhile, the tangent network $\LRp{d200, 2\%, 1, 1, 10^5}$ learns only the tangent slope, the time propagation is taken into account by the forward Euler scheme.

\begin{figure}[htb!]
    \centering
    \includegraphics[width=\textwidth]{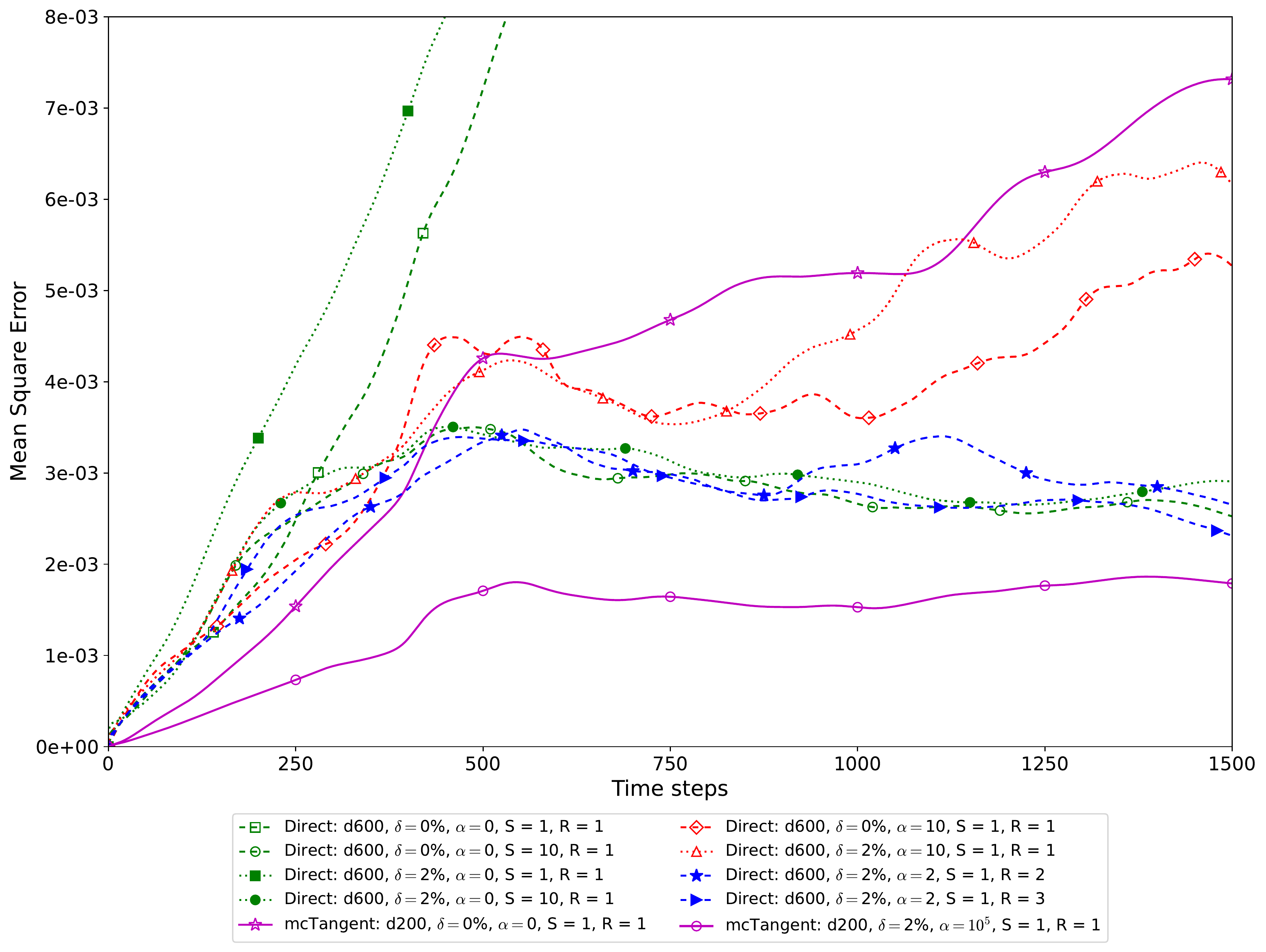}
    \caption{\textbf{Burger's equations}. The mean square error versus time steps obtained by the direct learning approach (Direct) using 600 training samples and the tangent learning approach (\texttt{mcTangent}) using 200 training samples.} 
    \figlab{2D_Bur_Direct_EF_NNs}
\end{figure}

\begin{figure}[htb!]
    \centering
    \begin{tabular*}{\textwidth}{c c c c c}
        \centering
         &
        \raisebox{-0.5\height}{\small $t = 0$} &
        \raisebox{-0.5\height}{\small $t = 0.1$} &
        \raisebox{-0.5\height}{\small $t = 0.5$} & 
        \raisebox{-0.5\height}{\small $t = 1.5$} 
        \\
        \centering
        \rotatebox[origin=c]{90}{\small True  $\ub$} &
        \raisebox{-0.5\height}{\includegraphics[width=.20 \textwidth]{figures/2D_Bur/Bur_FD_127x127_step_t_0.pdf}} &
        \raisebox{-0.5\height}{\includegraphics[width=.20 \textwidth]{figures/2D_Bur/Bur_FD_127x127_step_t_100.pdf}} &
        \raisebox{-0.5\height}{\includegraphics[width=.20 \textwidth]{figures/2D_Bur/Bur_FD_127x127_step_t_500.pdf}} & 
        \raisebox{-0.5\height}{\includegraphics[width=.20 \textwidth]{figures/2D_Bur/Bur_FD_127x127_step_t_1500.pdf}} 
        \\
        \centering
        \rotatebox[origin=c]{90}{\small \texttt{mcTangent} $\ub$, d200} &
        \raisebox{-0.5\height}{\includegraphics[width=.20 \textwidth]{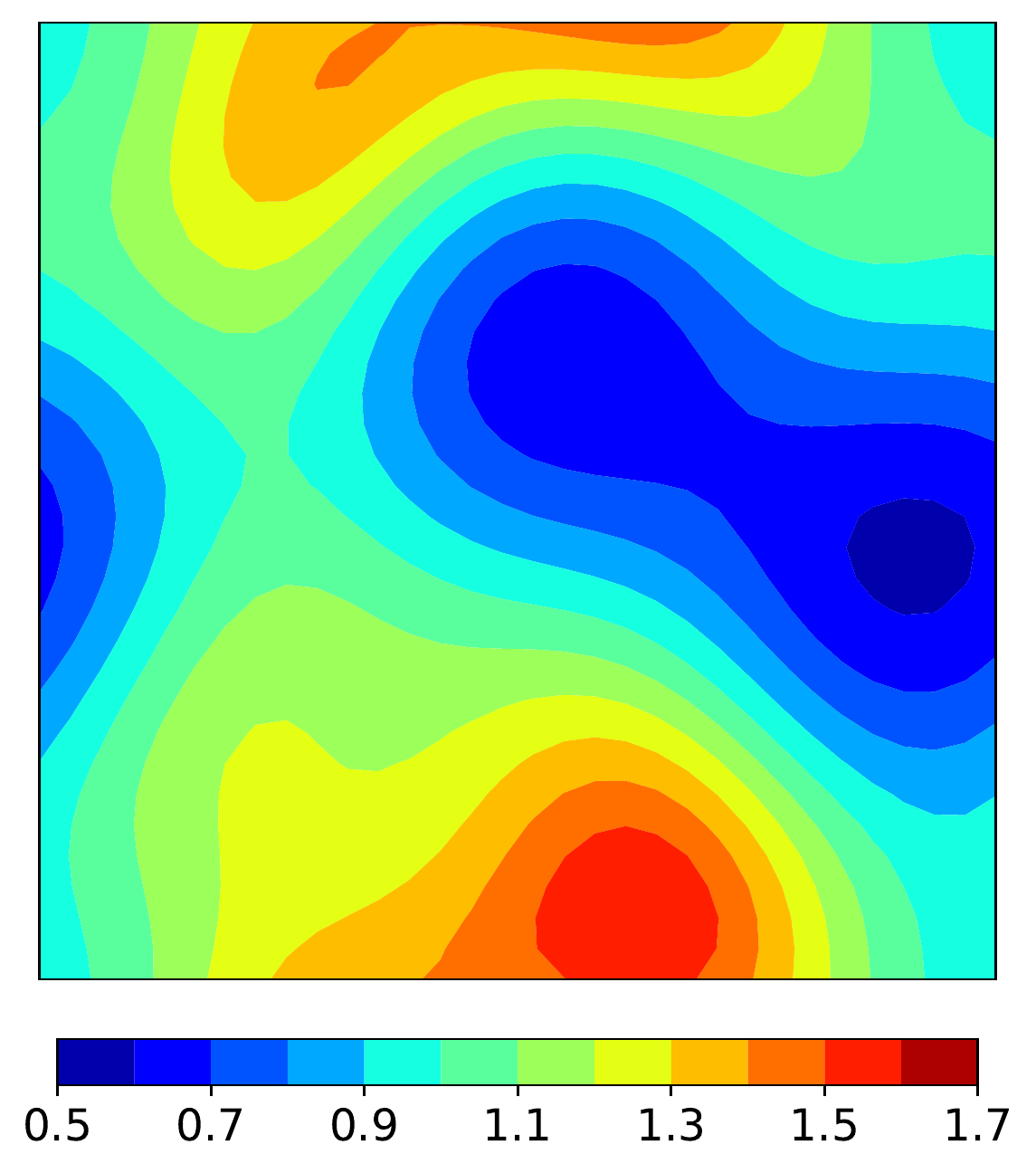}} &
        \raisebox{-0.5\height}{\includegraphics[width=.20 \textwidth]{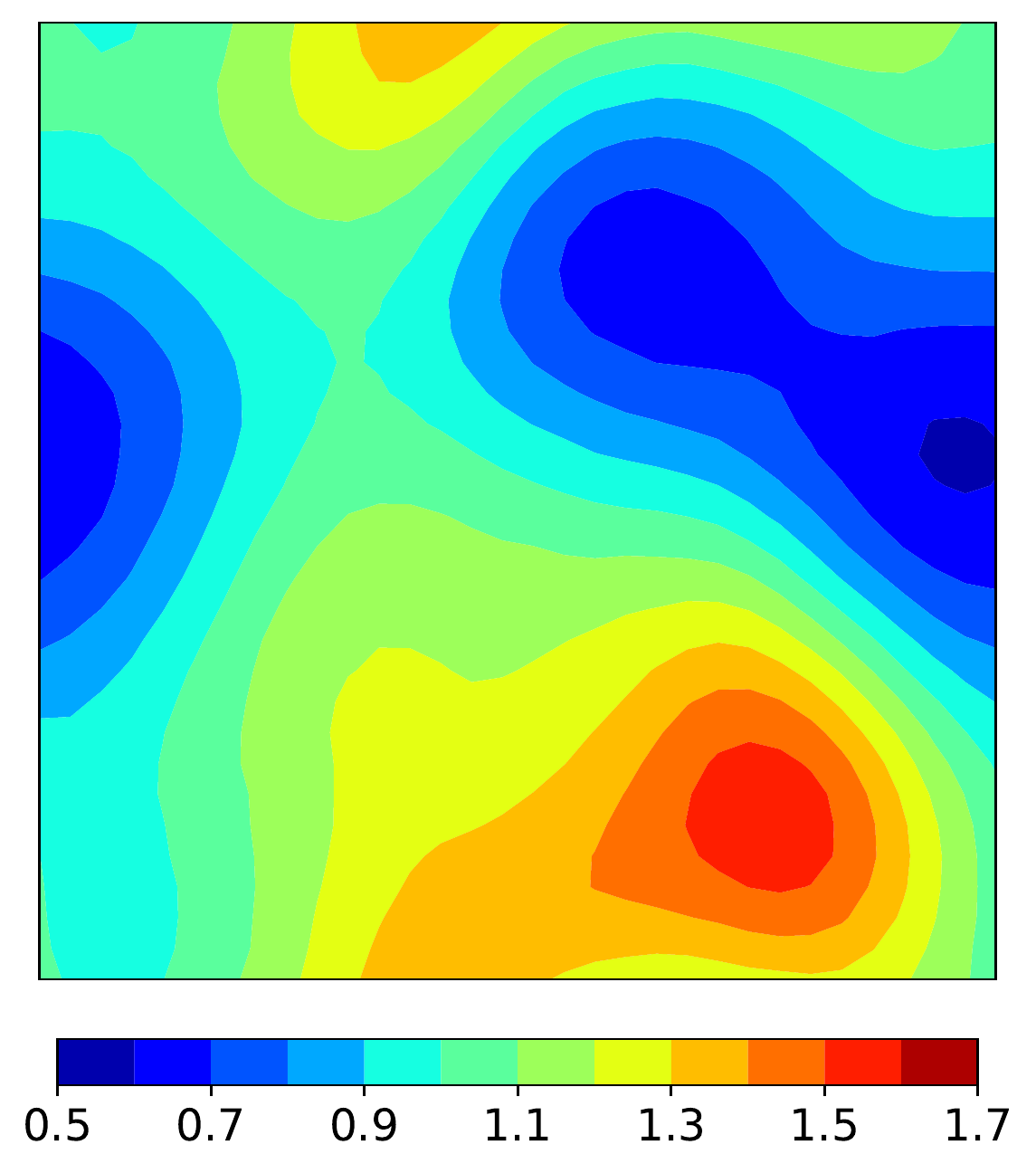}} &
        \raisebox{-0.5\height}{\includegraphics[width=.20 \textwidth]{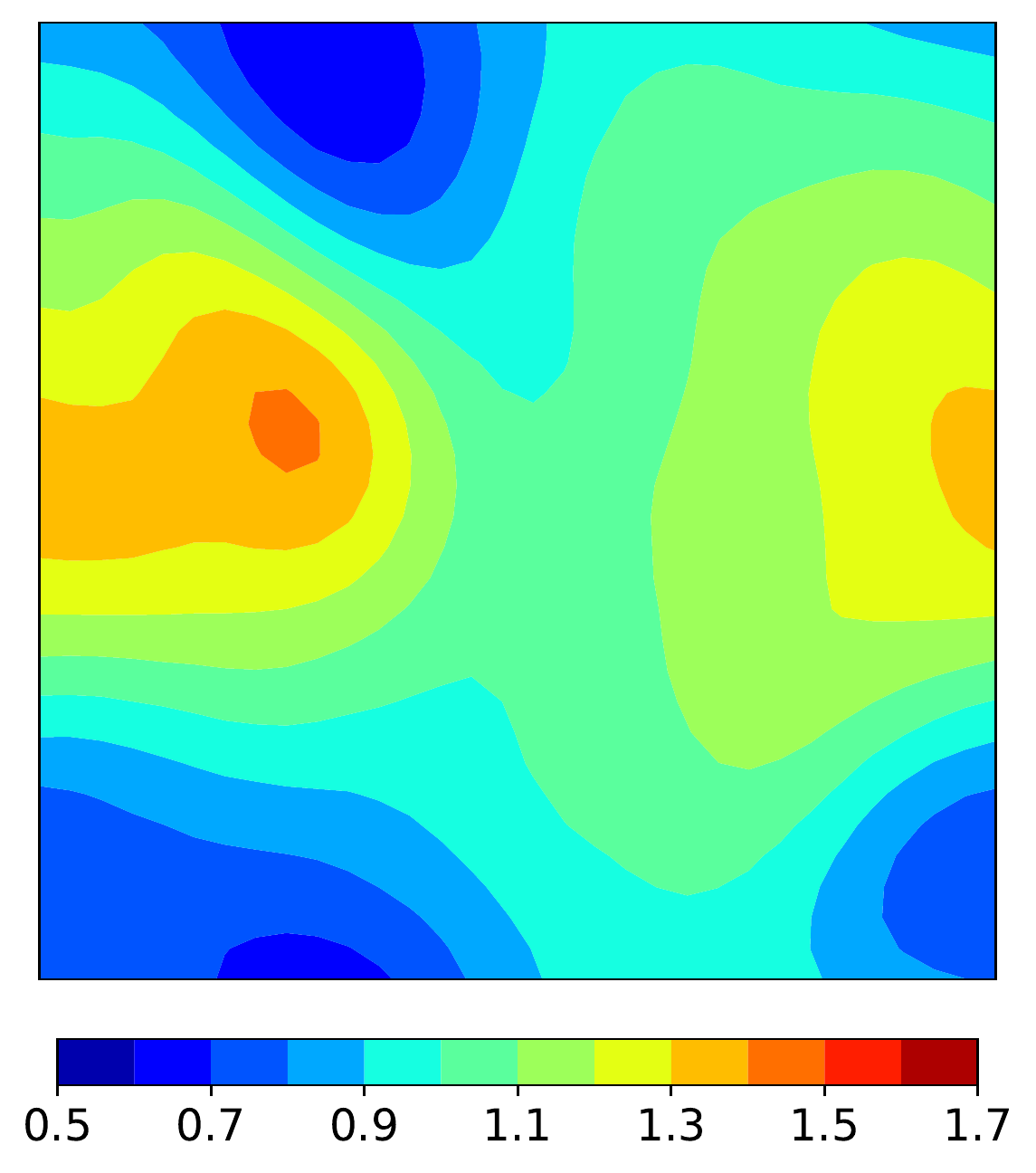}} & 
        \raisebox{-0.5\height}{\includegraphics[width=.20 \textwidth]{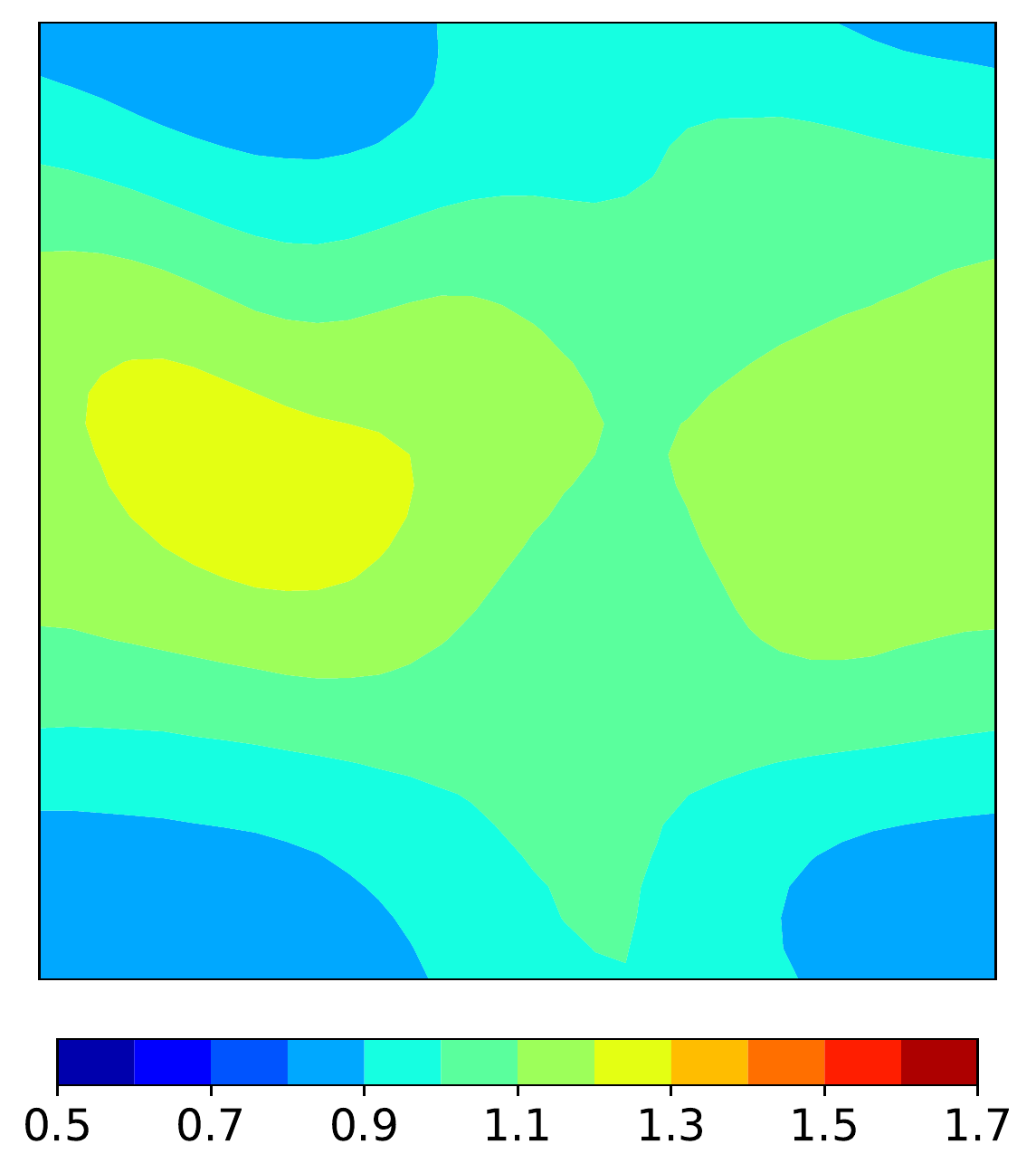}} 
        \\
        \centering
        \rotatebox[origin=c]{90}{Direct  $\ub$, d600} &
        \raisebox{-0.5\height}{\includegraphics[width=.20 \textwidth]{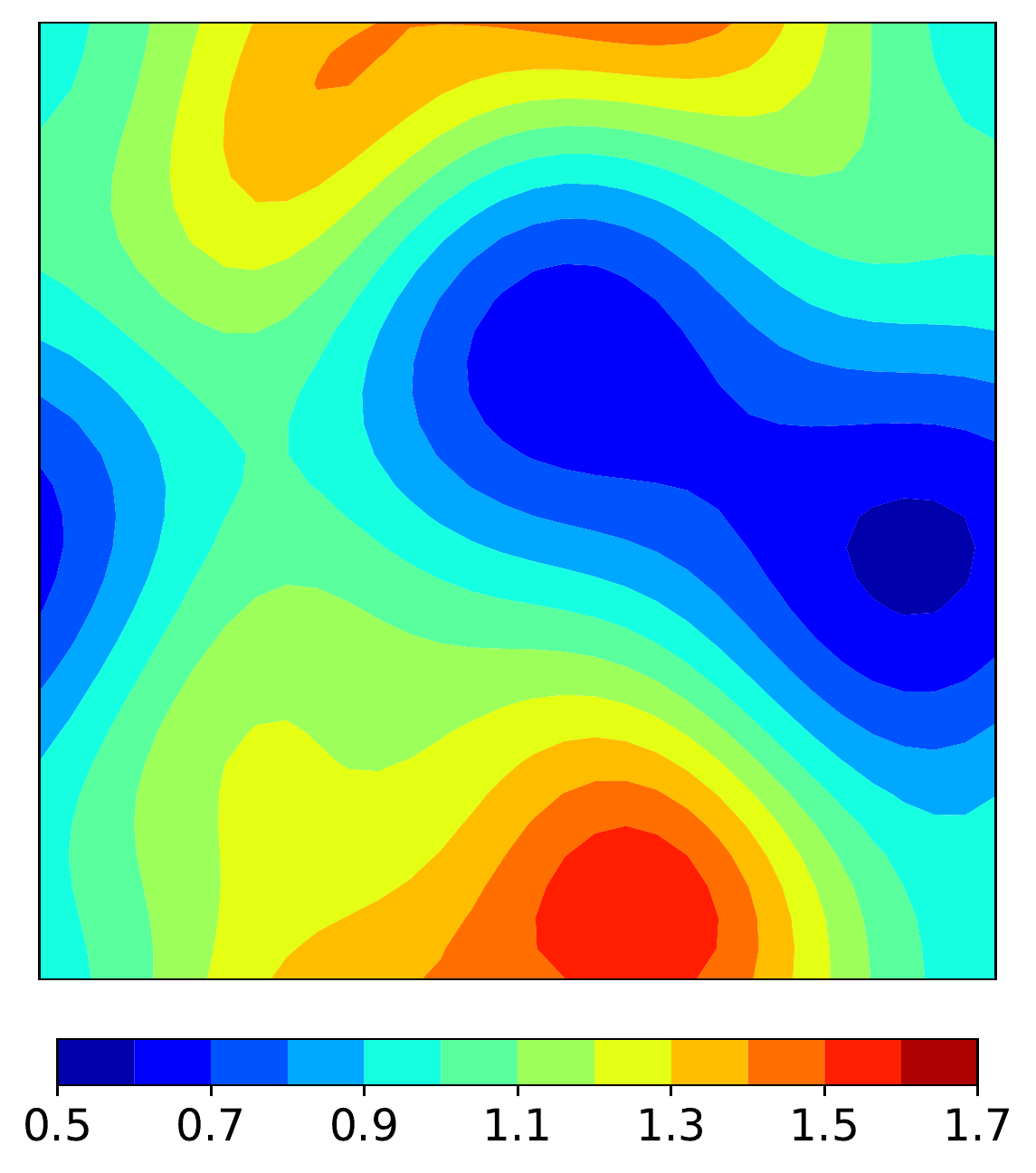}} &
        \raisebox{-0.5\height}{\includegraphics[width=.20 \textwidth]{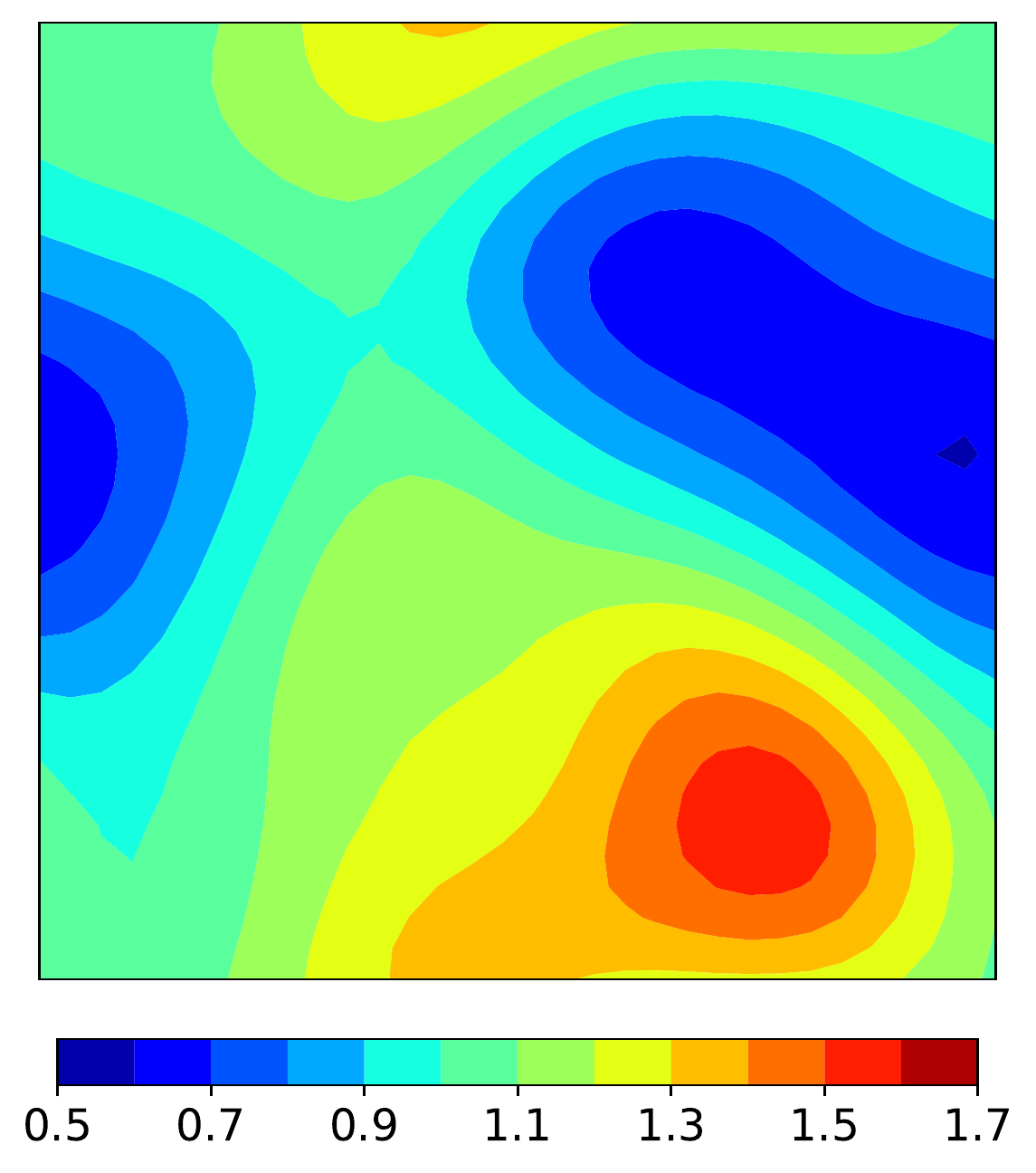}} &
        \raisebox{-0.5\height}{\includegraphics[width=.20 \textwidth]{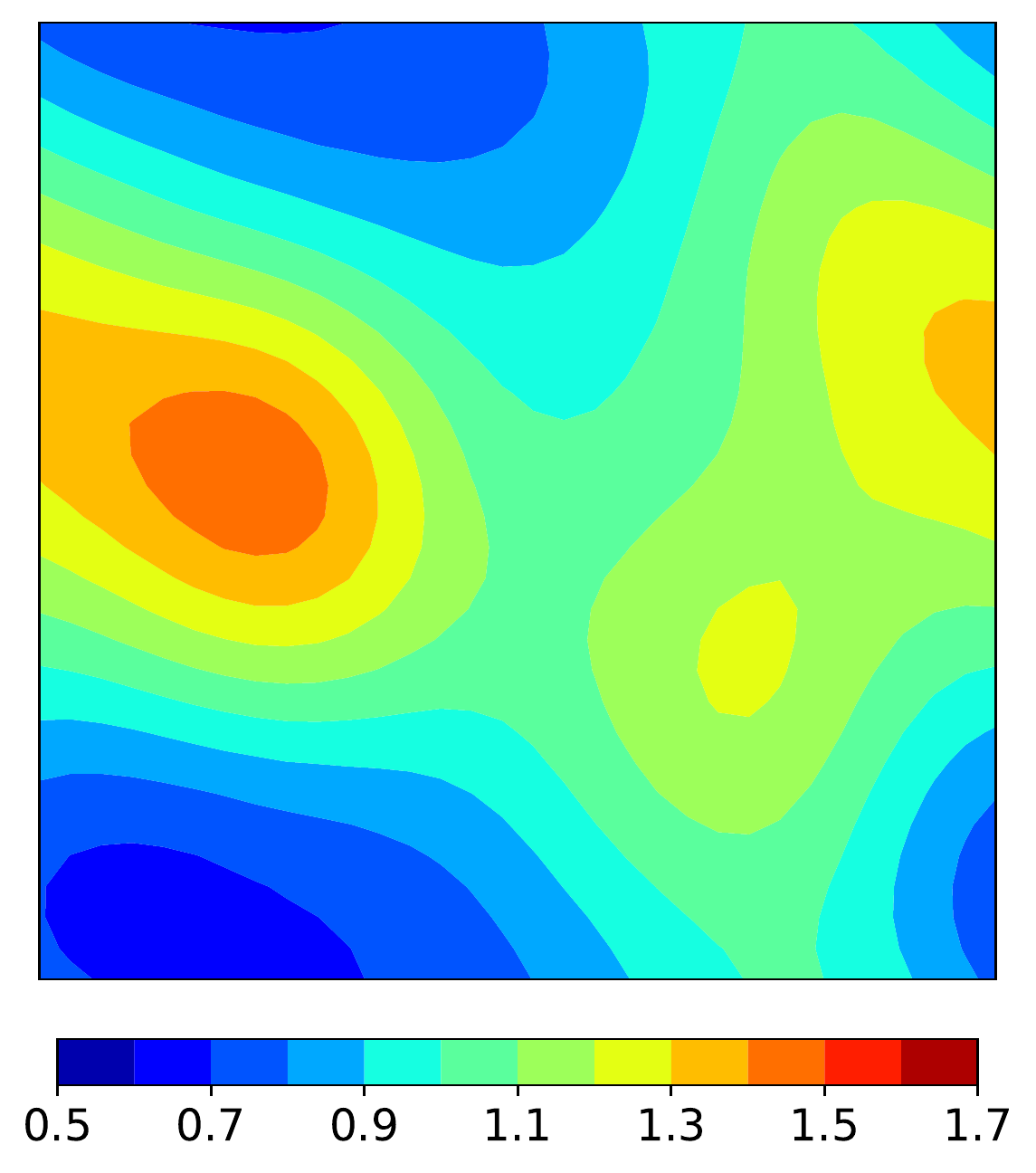}} & 
        \raisebox{-0.5\height}{\includegraphics[width=.20 \textwidth]{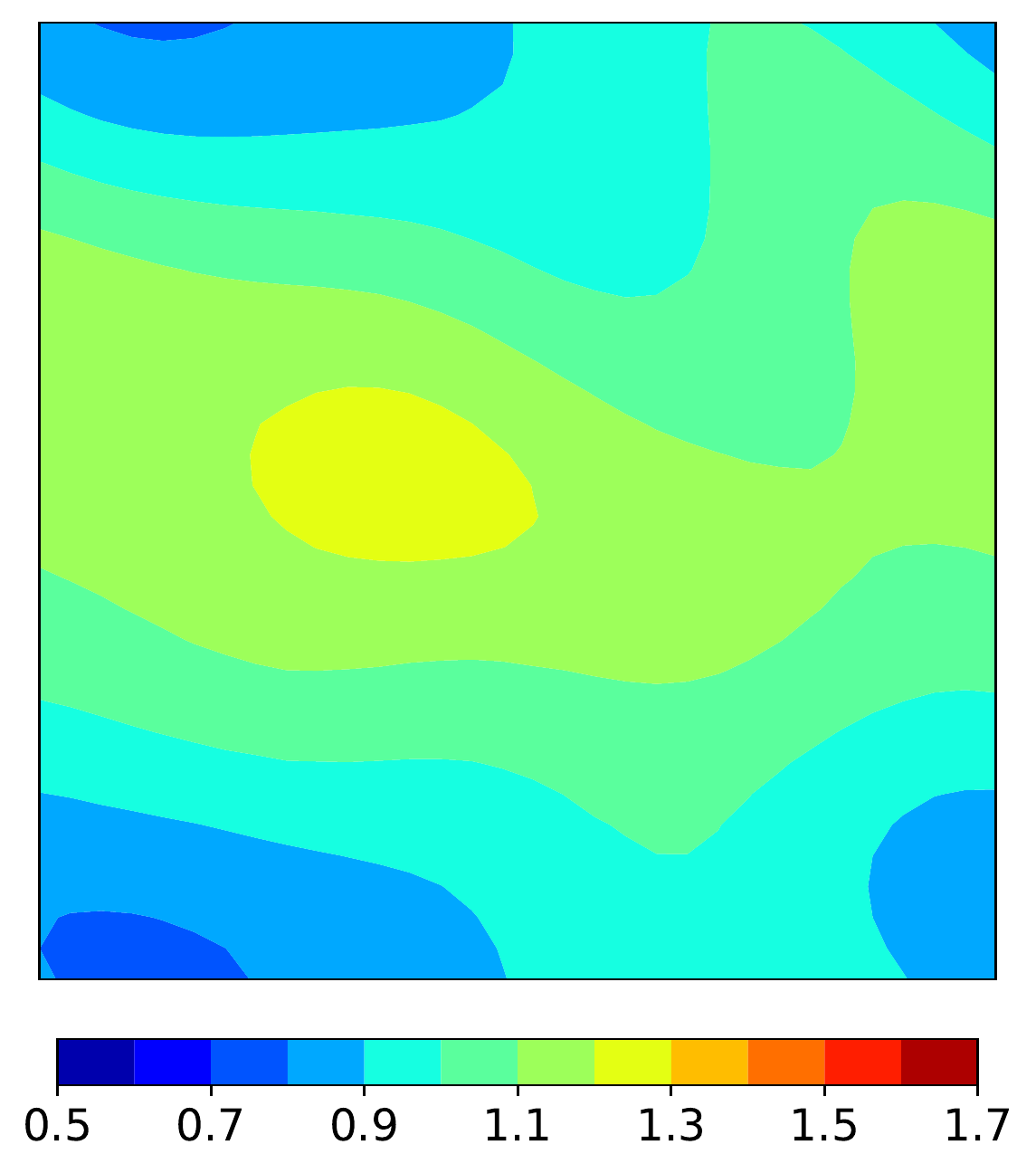}} 
        
    \end{tabular*}
    \caption{\textbf{Burger's equations}. Solutions at different times, \textit{First row}: True high-resolution solutions; \textit{Second row}: learned tangent slope neural network solutions with $\LRp{d200, 2\%, 1, 1, 10^5}$; \textit{Third row}: learned direct neural network solutions with $\LRp{d600, 2\%, 1, 3, 2}$.}
    \figlab{2D_Bur_Direct_samples}
\end{figure}

\subsection{Navier-Stokes equation}
\seclab{2D_NS}

The vorticity form of the 2D Navier-Stokes equation for viscous and incompressible fluid \cite{FouierOP} can be written as 
\begin{equation*}
    \begin{aligned}
        \partial_t u(x,t) + v(x,t) \cdot \nabla u(x,t) &= \nu \Delta u(x,t) + f(x), & \quad x \in \LRp{0,1}^2, t \in (0, T] \\
        \nabla \cdot v(x,t) & = 0, & \quad x \in \LRp{0,1}^2, t \in (0, T] \\
        u(x,0) & = u_0(x),  & \quad x \in \LRp{0,1}^2 
    \end{aligned}
\end{equation*}
where $v\LRp{x,t}$ is the velocity field, $u = \nabla \times v$ is the vorticity, $u_0 $ is the initial vorticity, $f(x) = 0.1 \LRp{\sin \LRp{2 \pi \LRp{x_1 + x_2}} + \cos\LRp{2 \pi \LRp{x_1 + x_2}}}$ is the forcing function and $\nu = 10^{-3}$ is the viscosity coefficient. Our goal is to solve for the vorticity  $u\LRp{x,y,t}$ given the initial condition $u_0$ at $t = 0$ by a trained tangent network $\Psi$ .

{\bf Data generation.} 
Data pair $\LRp{\ub, \yb}$ is generated by a similar procedure outlined for Burger's equation problem in \cref{sect:Burger_eq}. 
In particular, we draw samples of $\ub_0$ using the truncated Karhunen-Loève expansion
\[\ub_0 = \sum_{i=1}^{15} \sqrt{\lambda_i} \, {\bf \omega_i}(x) \, z_i,\]
where $z_i \sim \mathcal{N} \LRp{0, 1}, i = 1, \hdots, 15$, and $\LRp{\lambda, {\bf \omega}}$ is eigenpairs obtained by the eigendecomposition of the covariance operator $ 7^{\frac{3}{2}} \LRp{-\Delta + 49 \textbf{I}}^{-2.5}$ with periodic boundary conditions. Next, given initial vorticity $\ub_0$, we solve the Navier-Stokes equation by the stream-function formulation with a pseudospectral method \cite{FouierOP}. 
High resolution solutions are obtained on a uniform  $128 \times 128$ spatial mesh and uniform 1000 time steps in $\LRp{0, 2}$.
%{what is the time horizon? the time horizon for training data is $t \in \LRp{0, 2}$} 
%The solver is implemented in the high-resolution mesh grid of $128 \times 128$ in the spatial space of the unit square domain and 1000 equally discretized time points over the time domain $\LRs{0, 2}$. 
The high-resolution solutions are then down-sampled on a coarser mesh $32 \times 32$ in space and 200 uniform time steps, and they are used as the training data. To verify the accuracy of the learned neural network, we draw 10 test samples independently. It turns out that the Navier-stokes equation is much more challenging than Burger's equation, thus we use 200 time steps for each training data as opposed to 100 for the Burger equation.
Similar to the above, to challenge the learned network we use 1500 time steps for testing, and thus the testing time horizon is far beyond the training time horizon.

\begin{figure}[htb!]
    \centering
    \includegraphics[width=\textwidth]{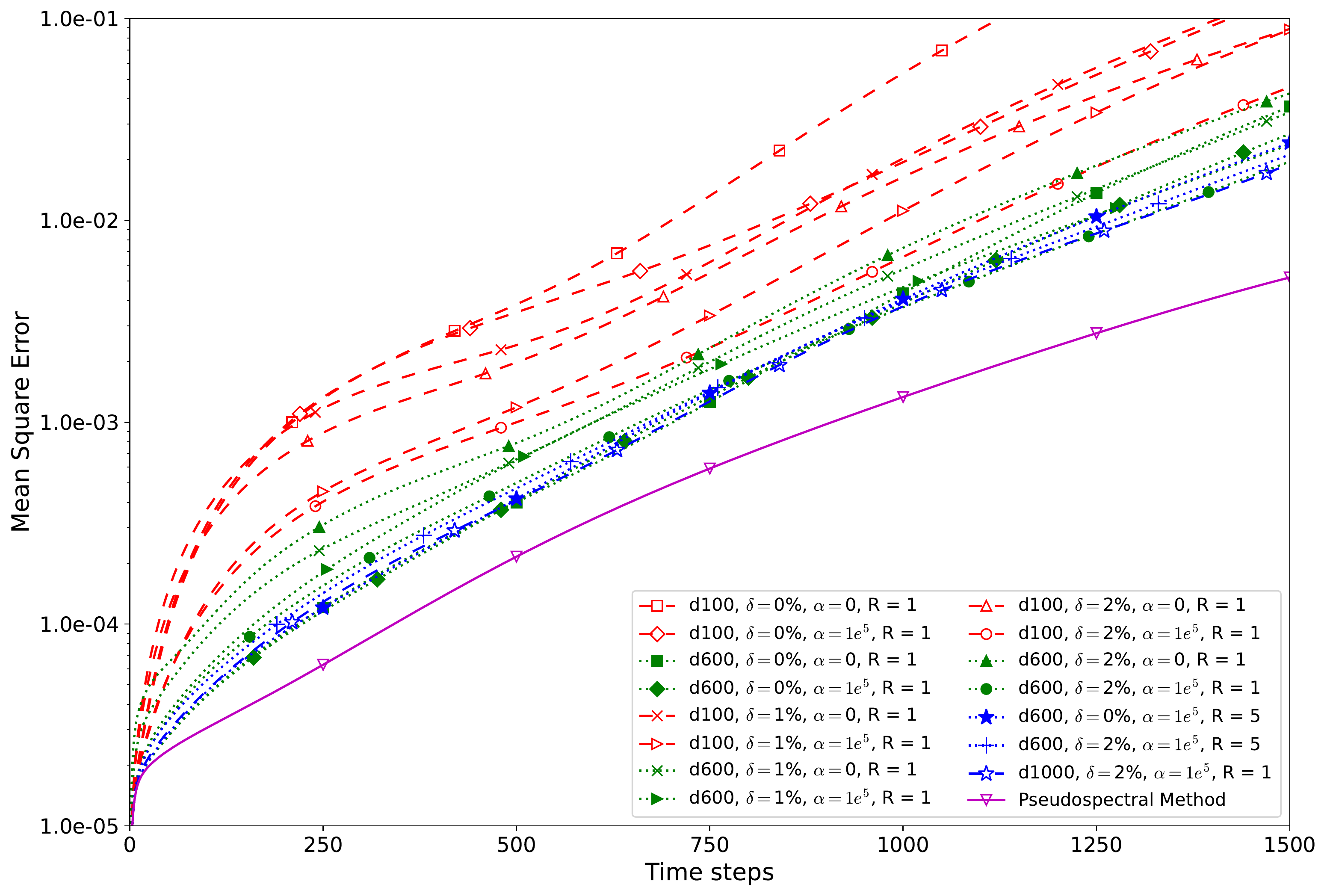}
    \caption{\textbf{Navier-Stokes equation}. The mean-square error versus time steps obtained by the various learned neural network with $S = 1$.} 
    \figlab{2D_NS_d200d600}
\end{figure}

{\bf Neural network architecture.}
With the same observation for the Burger equation in \cref{sect:Burger_eq}, we use a shallow network of one layer with  5000 neurons using ReLU activation function.
%since it works well for a wide training data size range. The 
\texttt{ADAM} optimizer with default parameters is used with the learning rate of $2 \times 10^{-4}$, while the training batch size is 2 samples. The chosen ``optimal" network is the one having the lowest accumulated mean square error after $1500$ time steps for 10 testing samples. Following the wave and Burger examples, we pick a relatively large value for the model-constrained regularization parameter $\alpha = 10^5$.
%and again almost similar results are obtained with larger values, $\alpha \leq 10^7$. 

{\bf Long-time predictions.}
\begin{figure}[htb!]
    \centering
    \begin{tabular*}{\textwidth}{c c c c c}
        \centering
         &
        \raisebox{-0.5\height}{\small $n_t = 0$} &
        \raisebox{-0.5\height}{\small $n_t = 200$} &
        \raisebox{-0.5\height}{\small $n_t = 1000$} & 
        \raisebox{-0.5\height}{\small $n_t = 1500$} 
        \\
        \centering
        \rotatebox[origin=c]{90}{\small True $\ub$} &
        \raisebox{-0.5\height}{\includegraphics[width=.20 \textwidth]{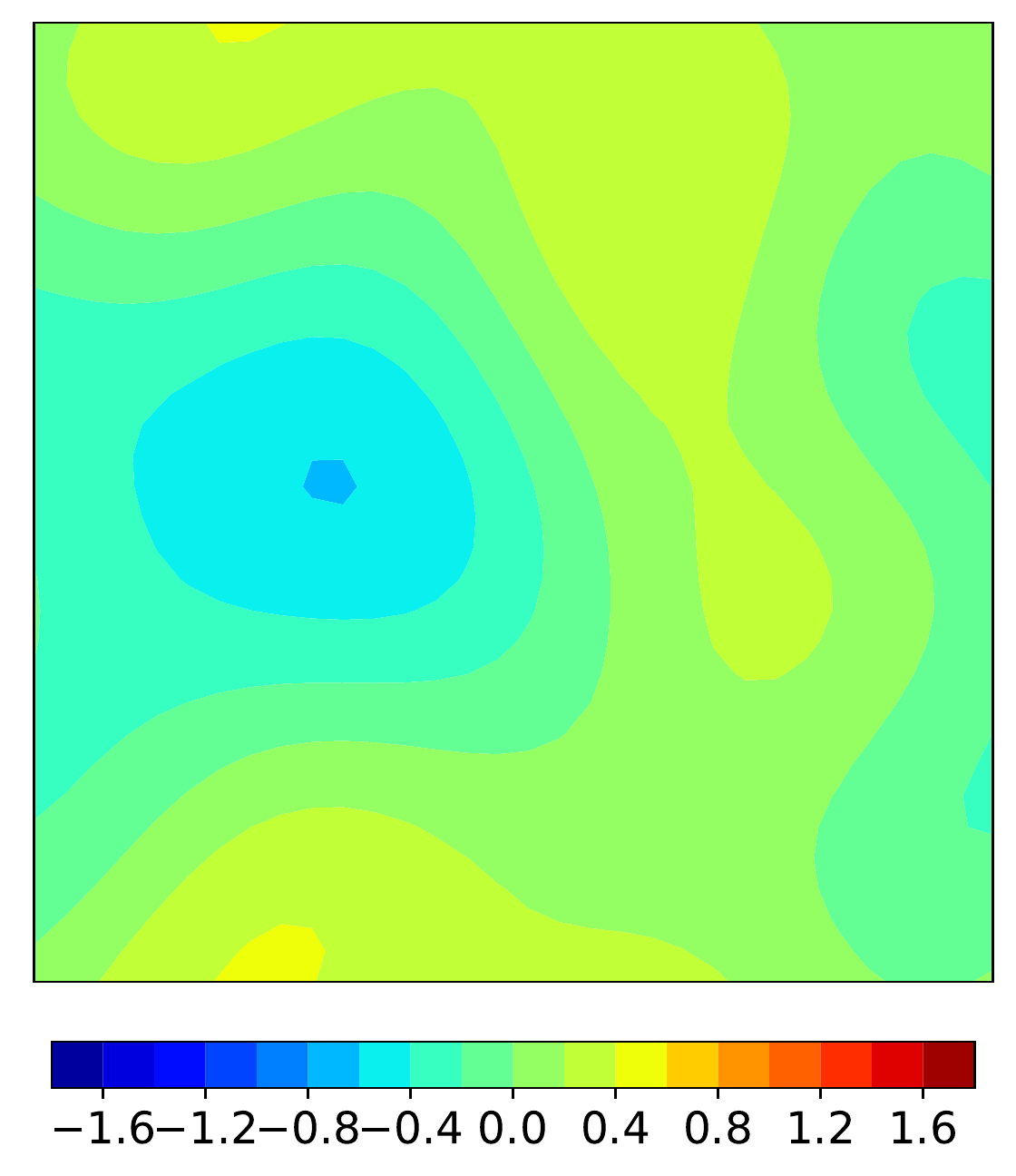}} &
        \raisebox{-0.5\height}{\includegraphics[width=.20 \textwidth]{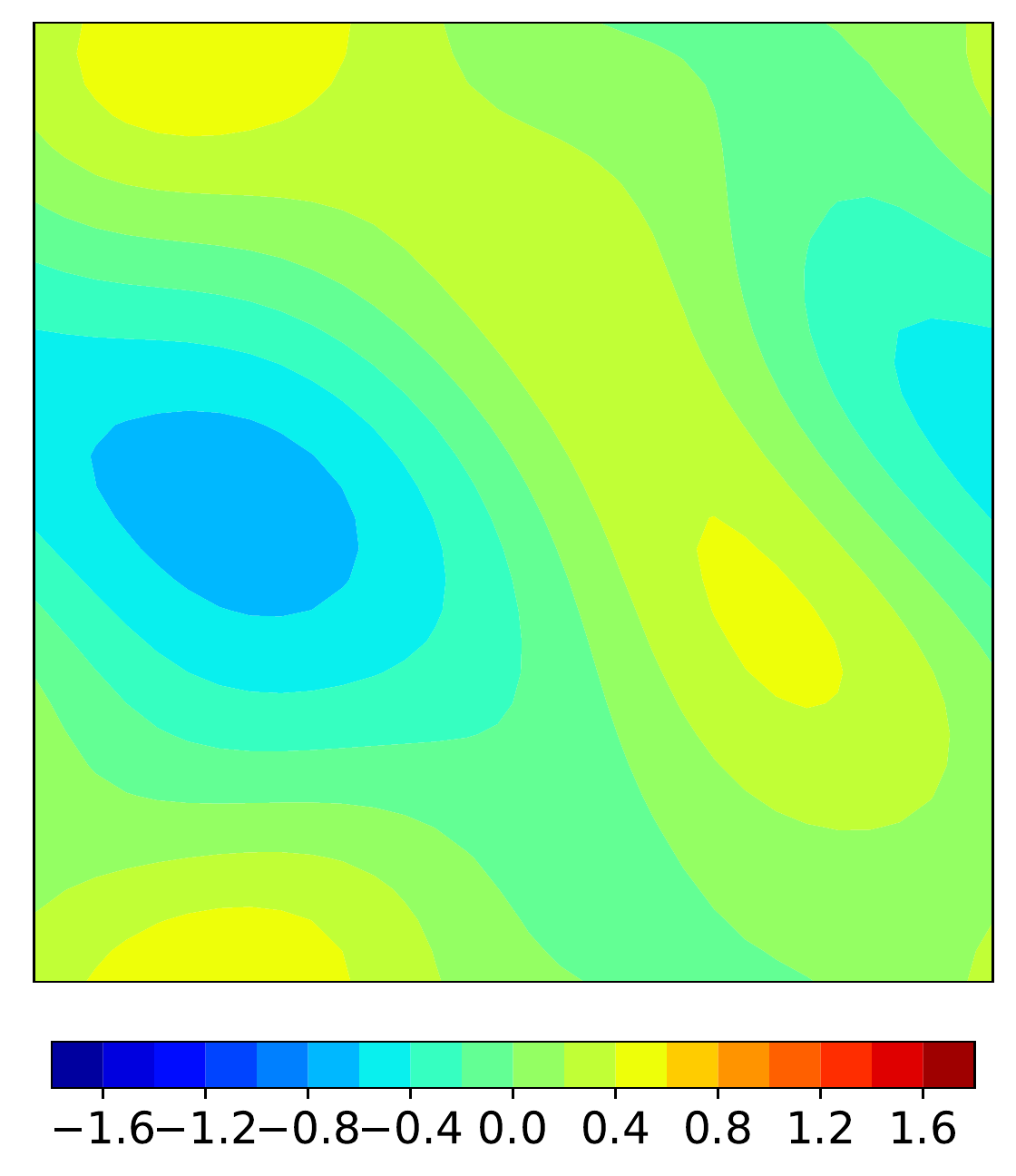}} &
        \raisebox{-0.5\height}{\includegraphics[width=.20 \textwidth]{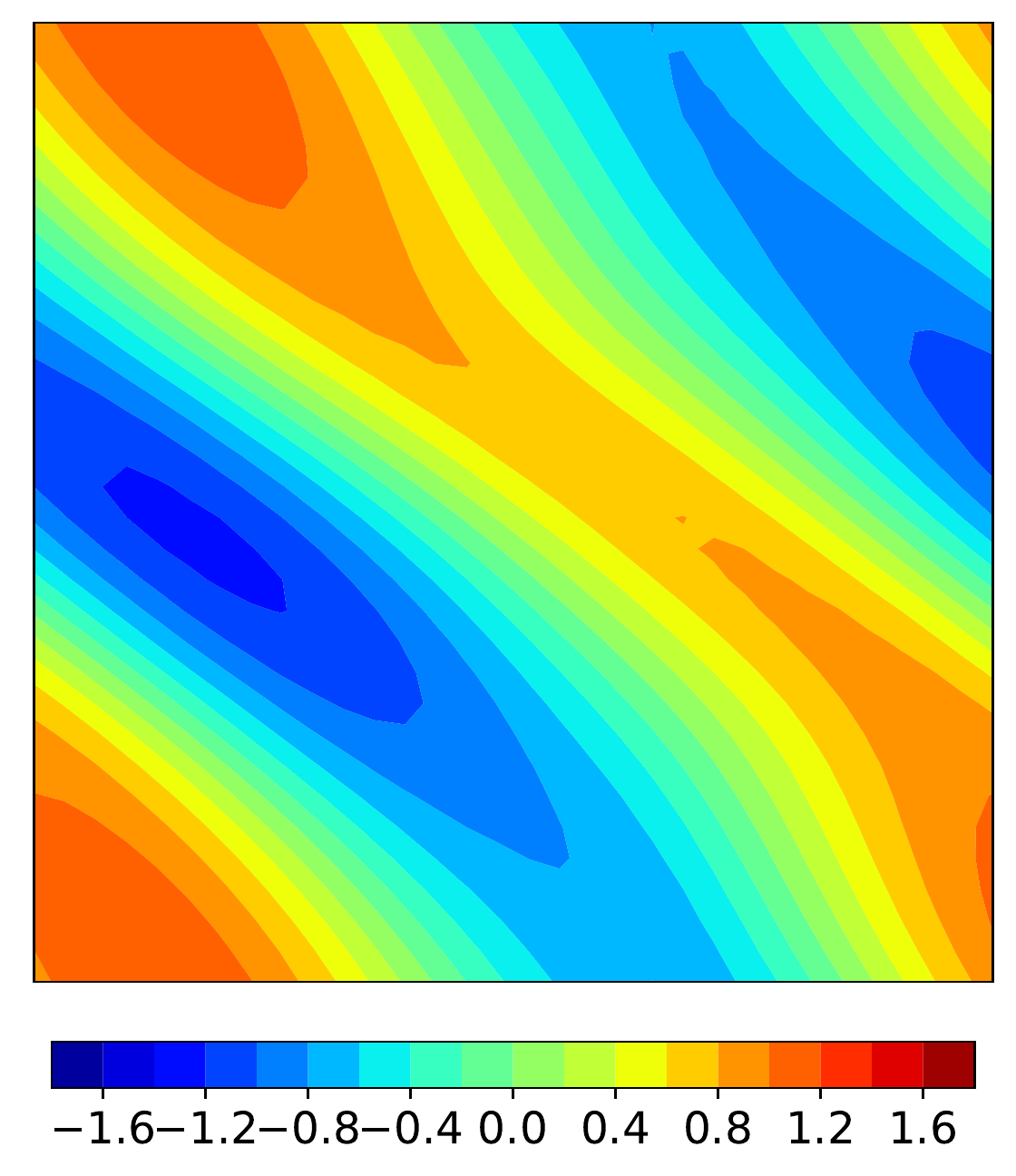}} & 
        \raisebox{-0.5\height}{\includegraphics[width=.20 \textwidth]{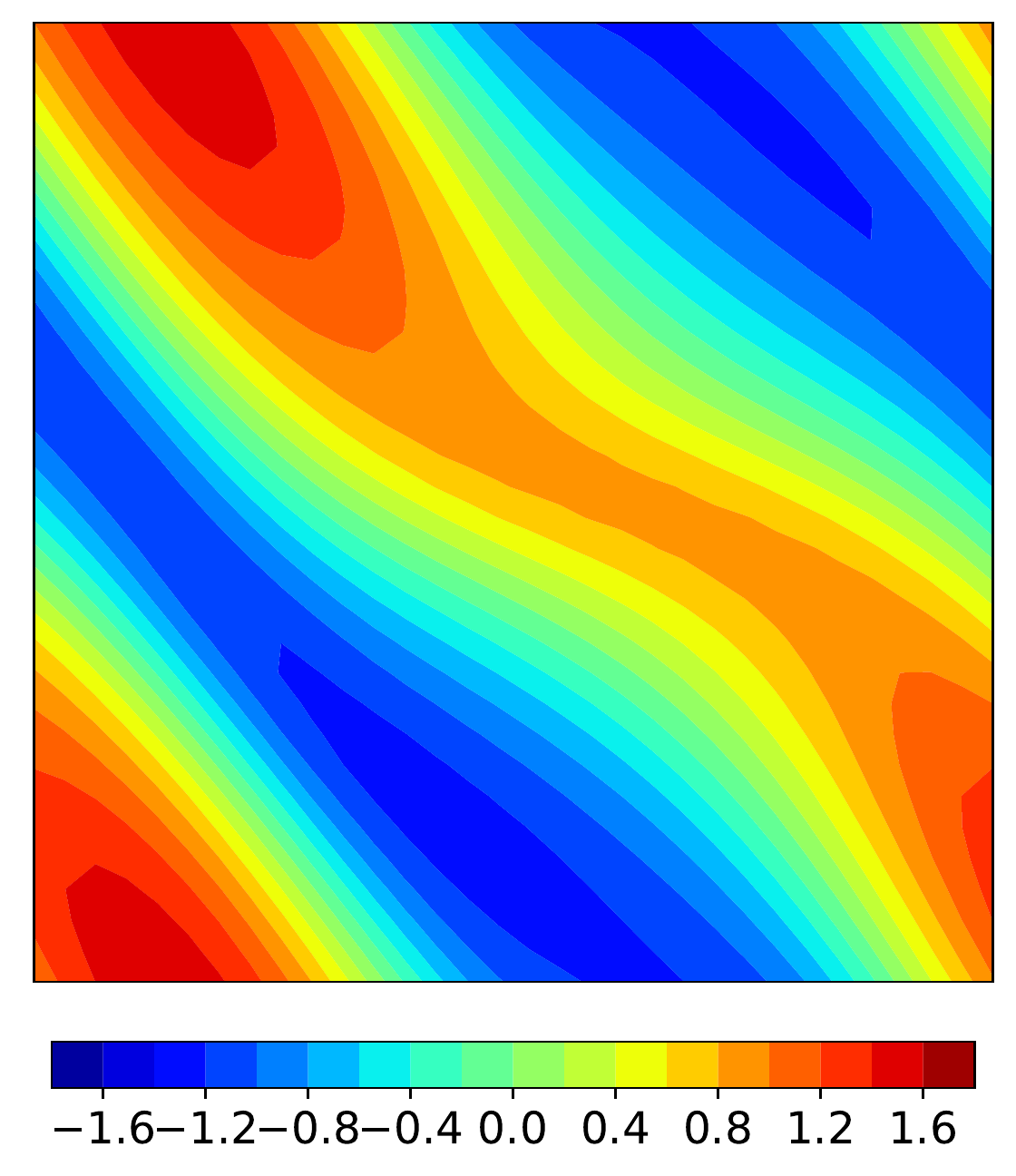}} 
        % \\
        % \centering
        % \rotatebox[origin=c]{90}{\small $\F$ - 32x32} &
        % \raisebox{-0.5\height}{\includegraphics[width=.20 \textwidth]{figures/2D_NS/NS_FD_32x32_step_t_0.pdf}} &
        % \raisebox{-0.5\height}{\includegraphics[width=.20 \textwidth]{figures/2D_NS/NS_FD_32x32_step_t_200.pdf}} &
        % \raisebox{-0.5\height}{\includegraphics[width=.20 \textwidth]{figures/2D_NS/NS_FD_32x32_step_t_1000.pdf}} & 
        % \raisebox{-0.5\height}{\includegraphics[width=.20 \textwidth]{figures/2D_NS/NS_FD_32x32_step_t_1500.pdf}} 
        \\
        \centering
        \rotatebox[origin=c]{90}{$\alpha =0,\delta = 0\%$} &
        \raisebox{-0.5\height}{\includegraphics[width=.20 \textwidth]{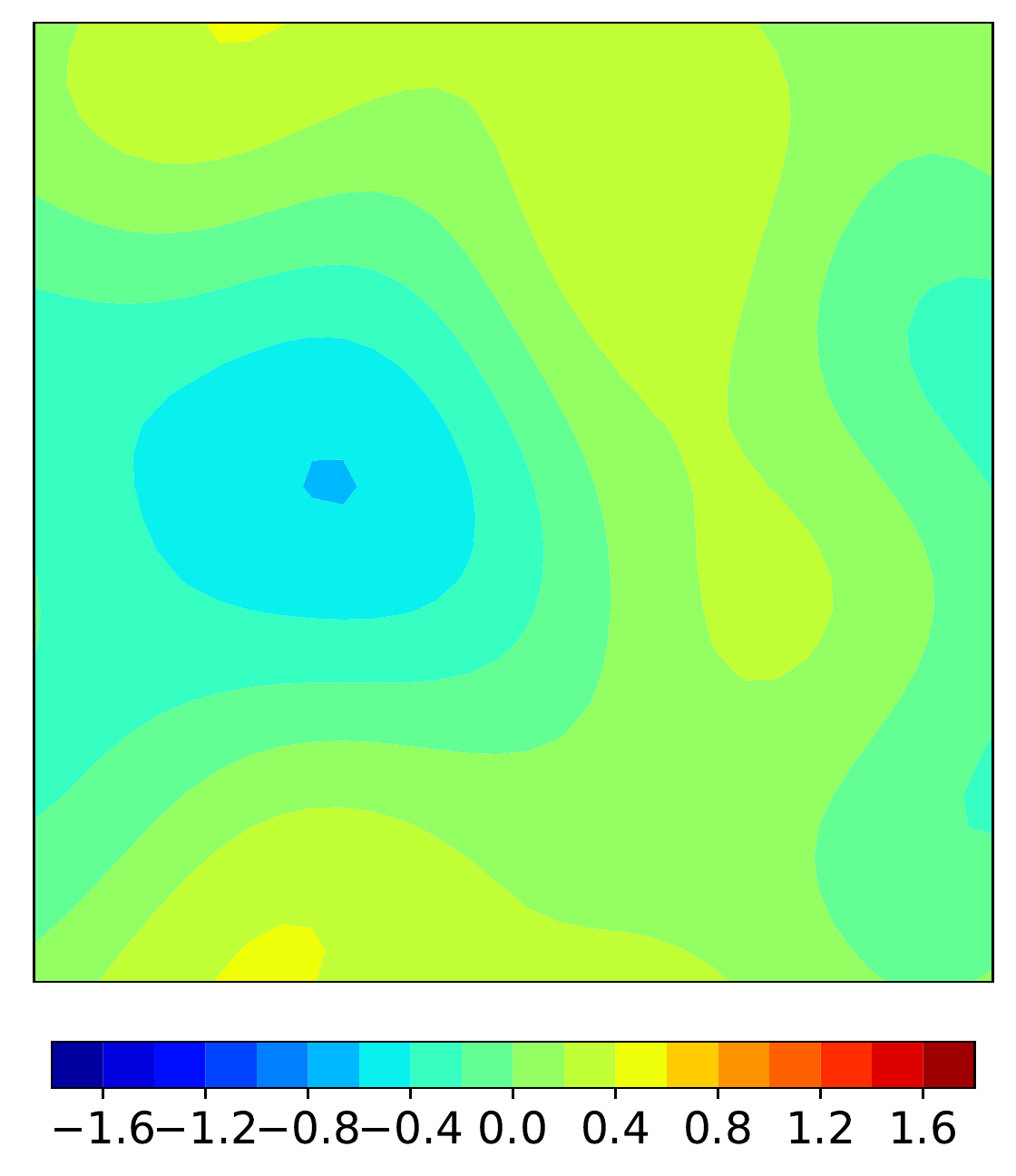}} &
        \raisebox{-0.5\height}{\includegraphics[width=.20 \textwidth]{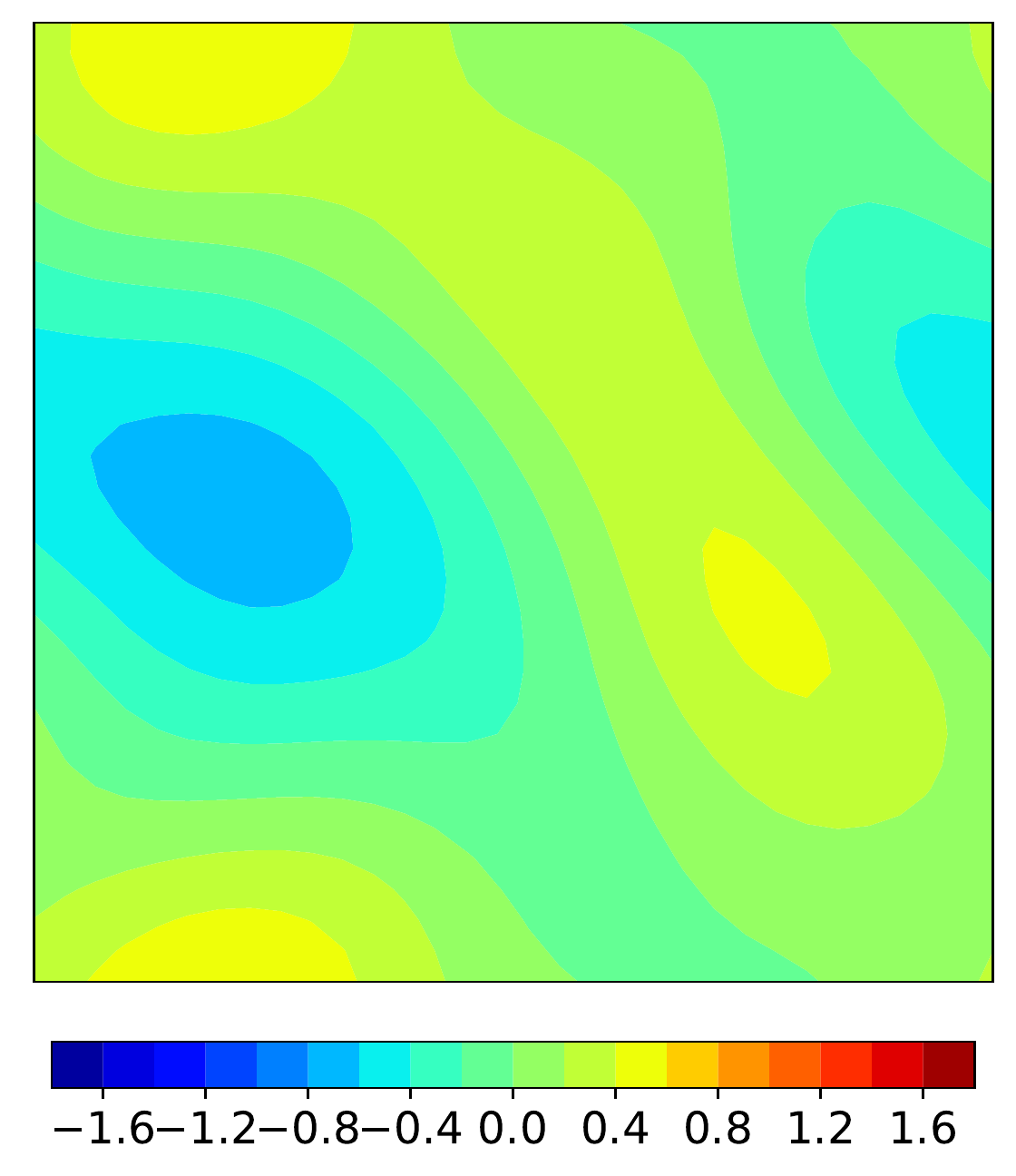}} &
        \raisebox{-0.5\height}{\includegraphics[width=.20 \textwidth]{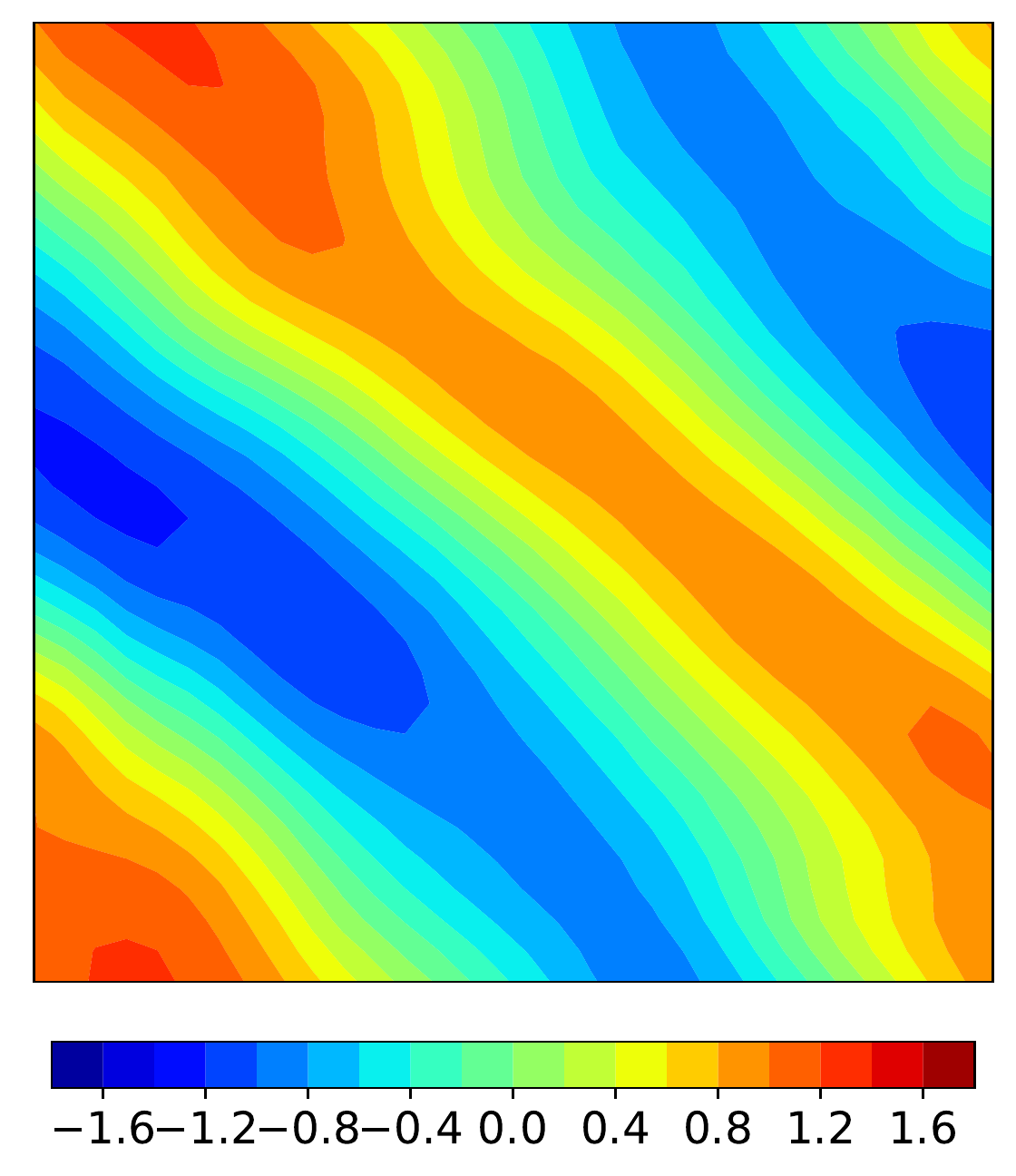}} & 
        \raisebox{-0.5\height}{\includegraphics[width=.20 \textwidth]{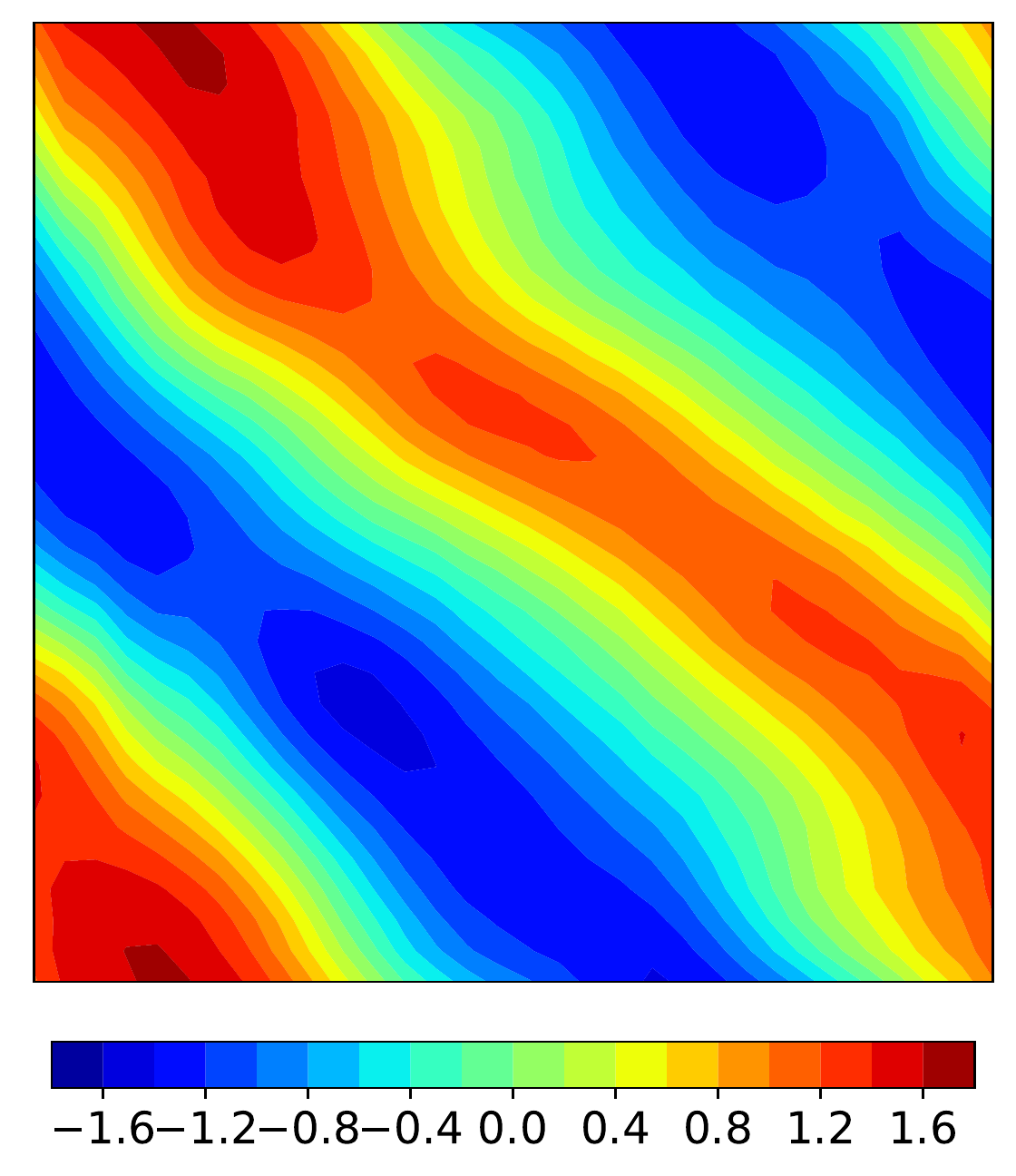}} 
        \\
        \centering
        \rotatebox[origin=c]{90}{$\alpha =0,\delta = 2\%$} &
        \raisebox{-0.5\height}{\includegraphics[width=.20 \textwidth]{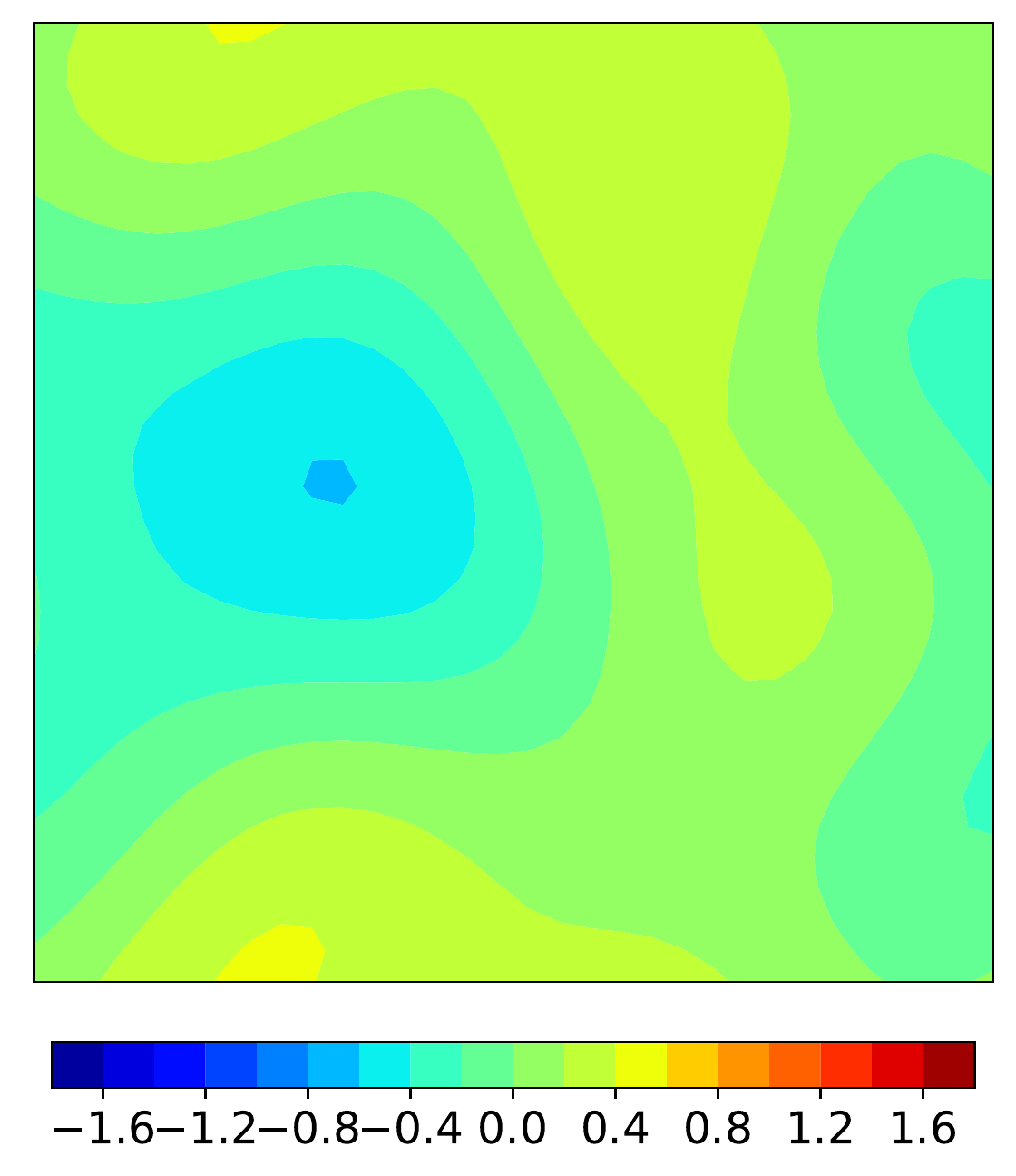}} &
        \raisebox{-0.5\height}{\includegraphics[width=.20 \textwidth]{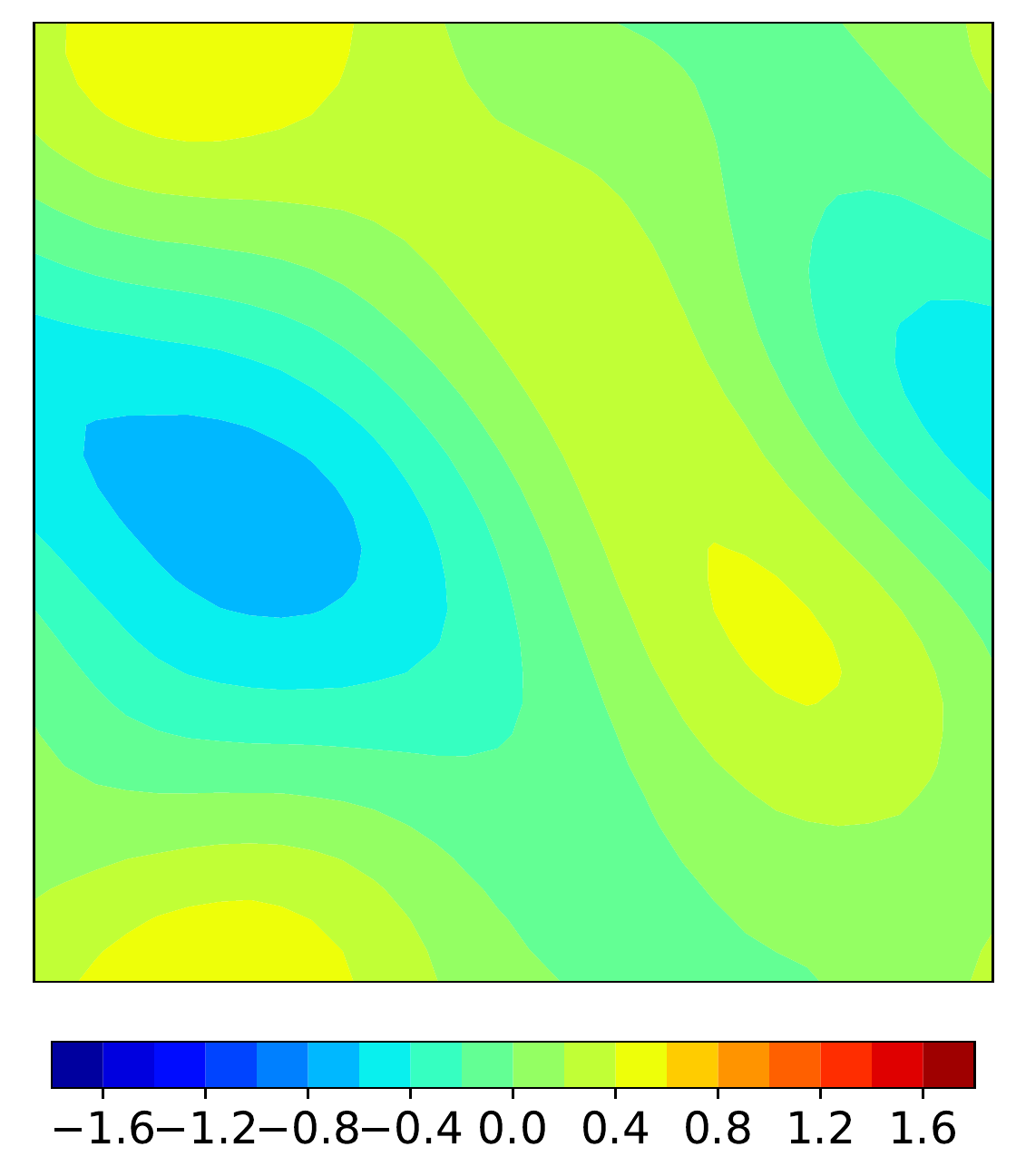}} &
        \raisebox{-0.5\height}{\includegraphics[width=.20 \textwidth]{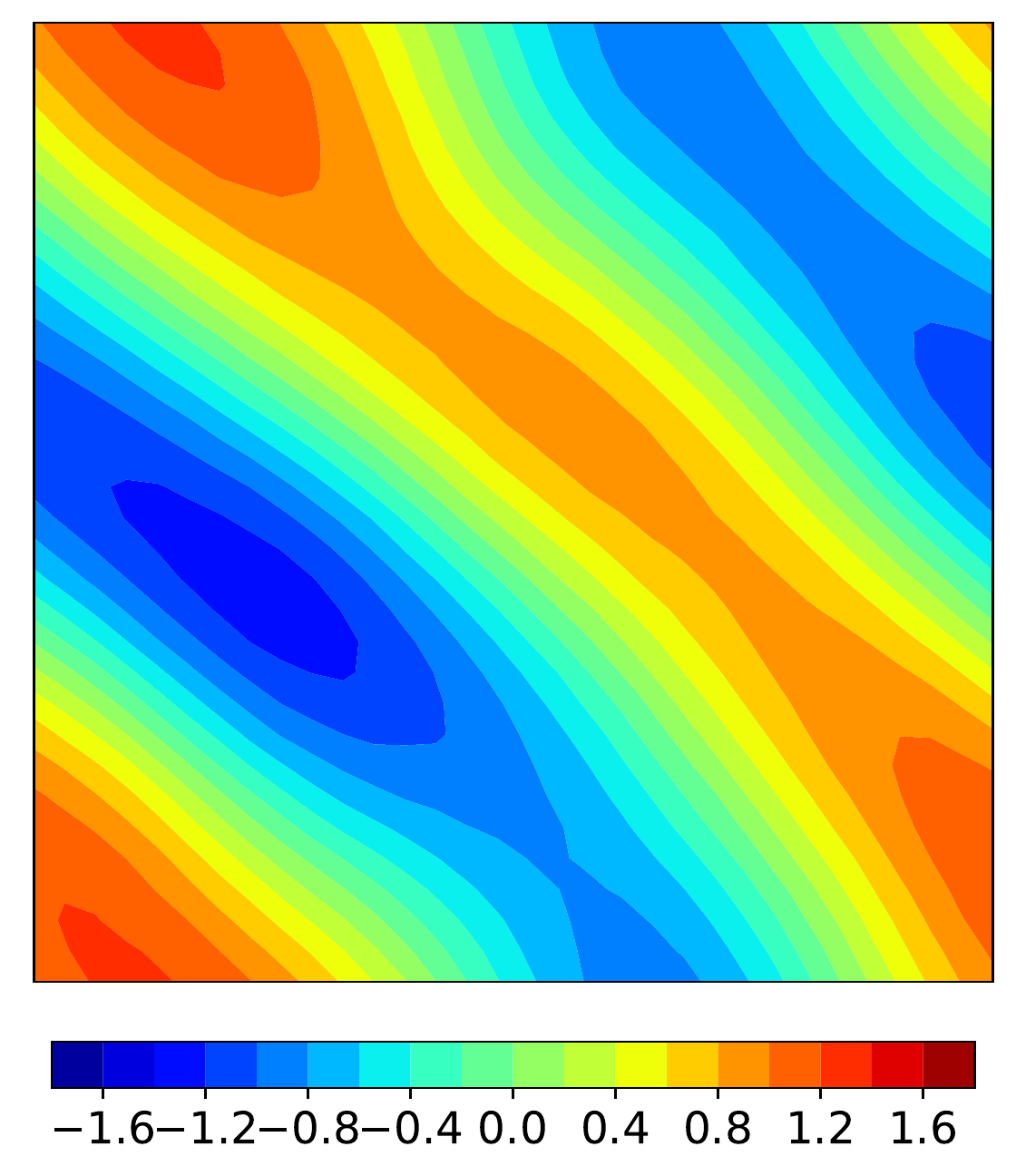}} & 
        \raisebox{-0.5\height}{\includegraphics[width=.20 \textwidth]{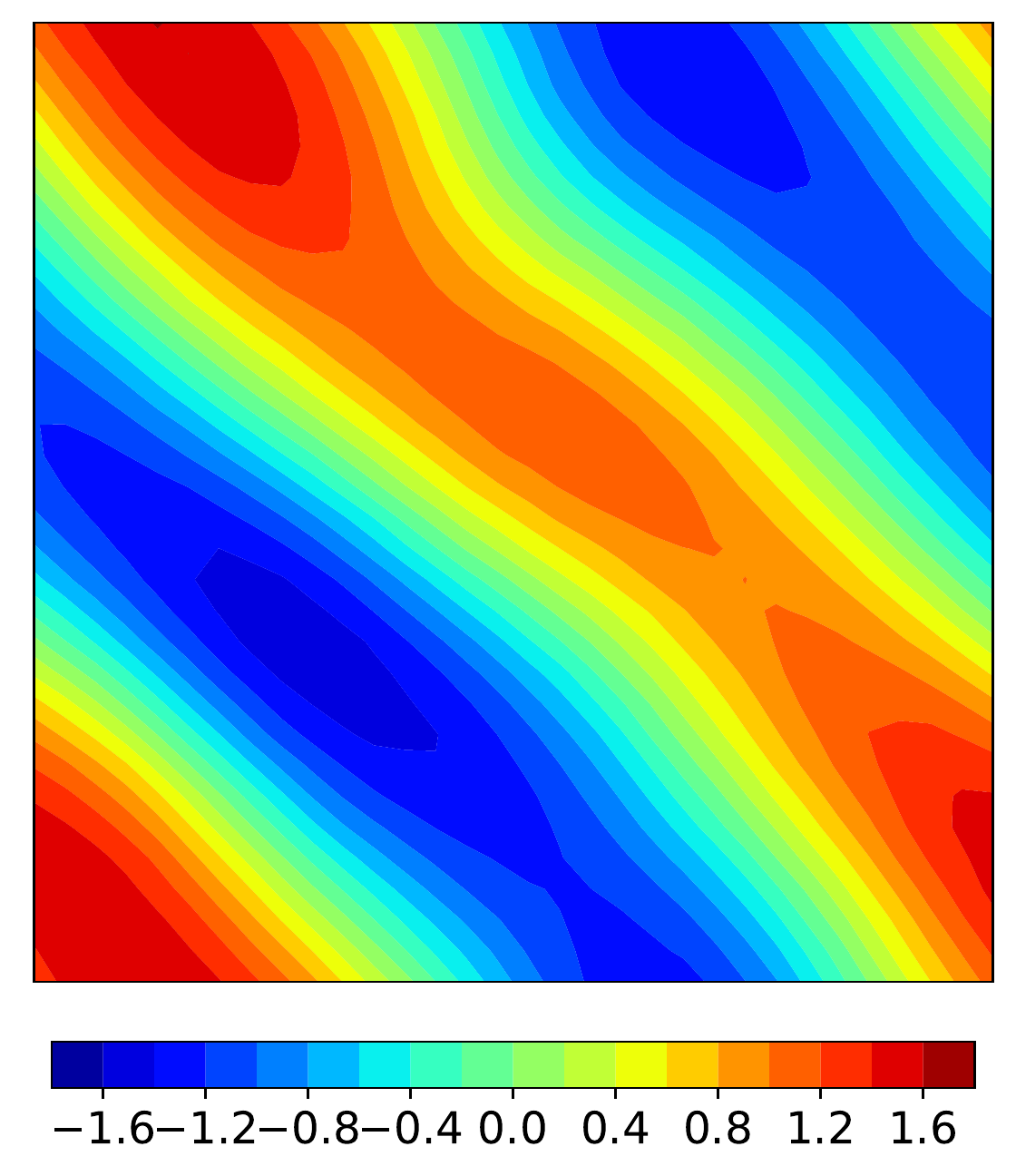}} 
        \\
        \centering
        \rotatebox[origin=c]{90}{$\alpha =1e^5,\delta = 0\%$} &
        \raisebox{-0.5\height}{\includegraphics[width=.20 \textwidth]{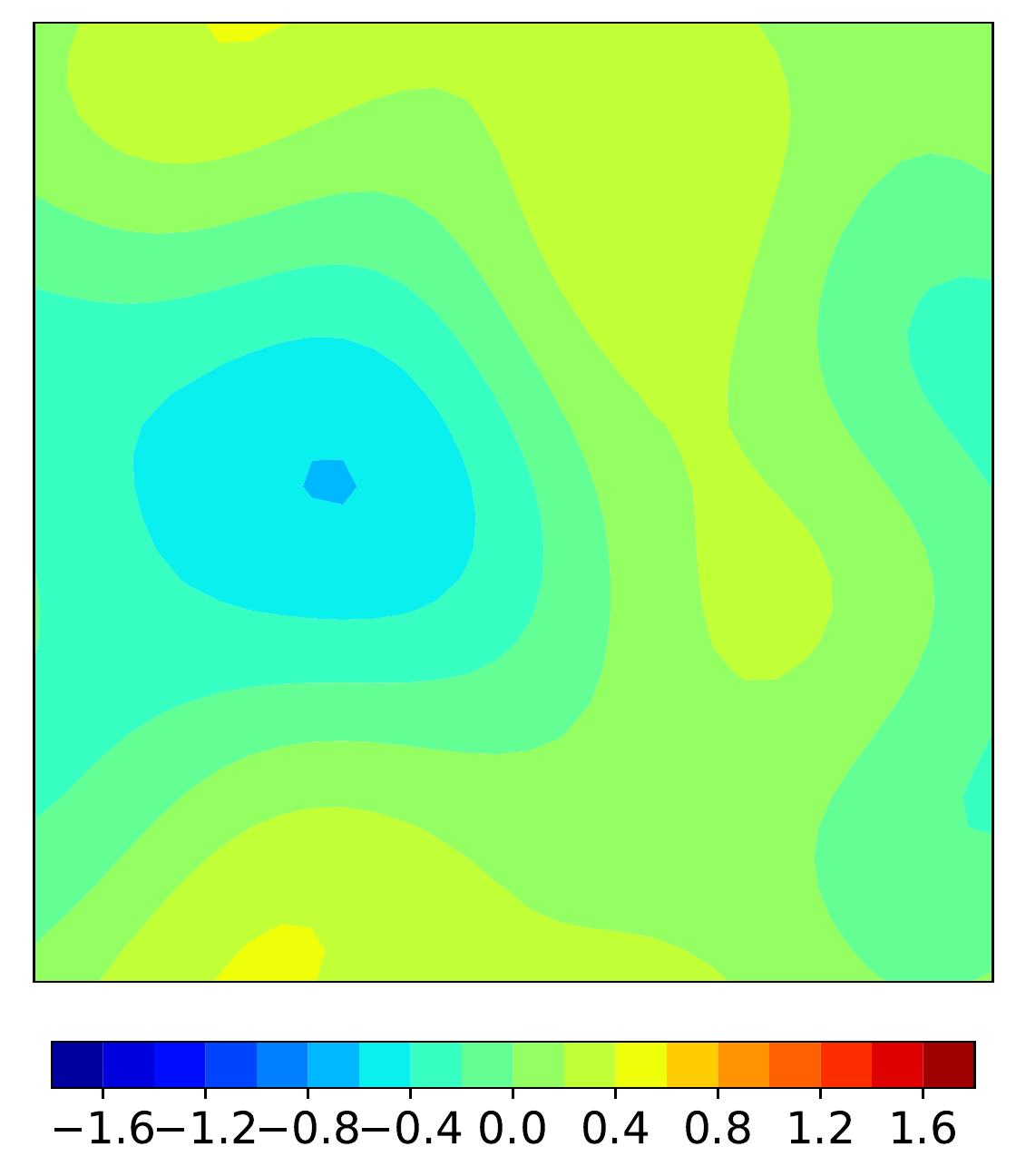}} &
        \raisebox{-0.5\height}{\includegraphics[width=.20 \textwidth]{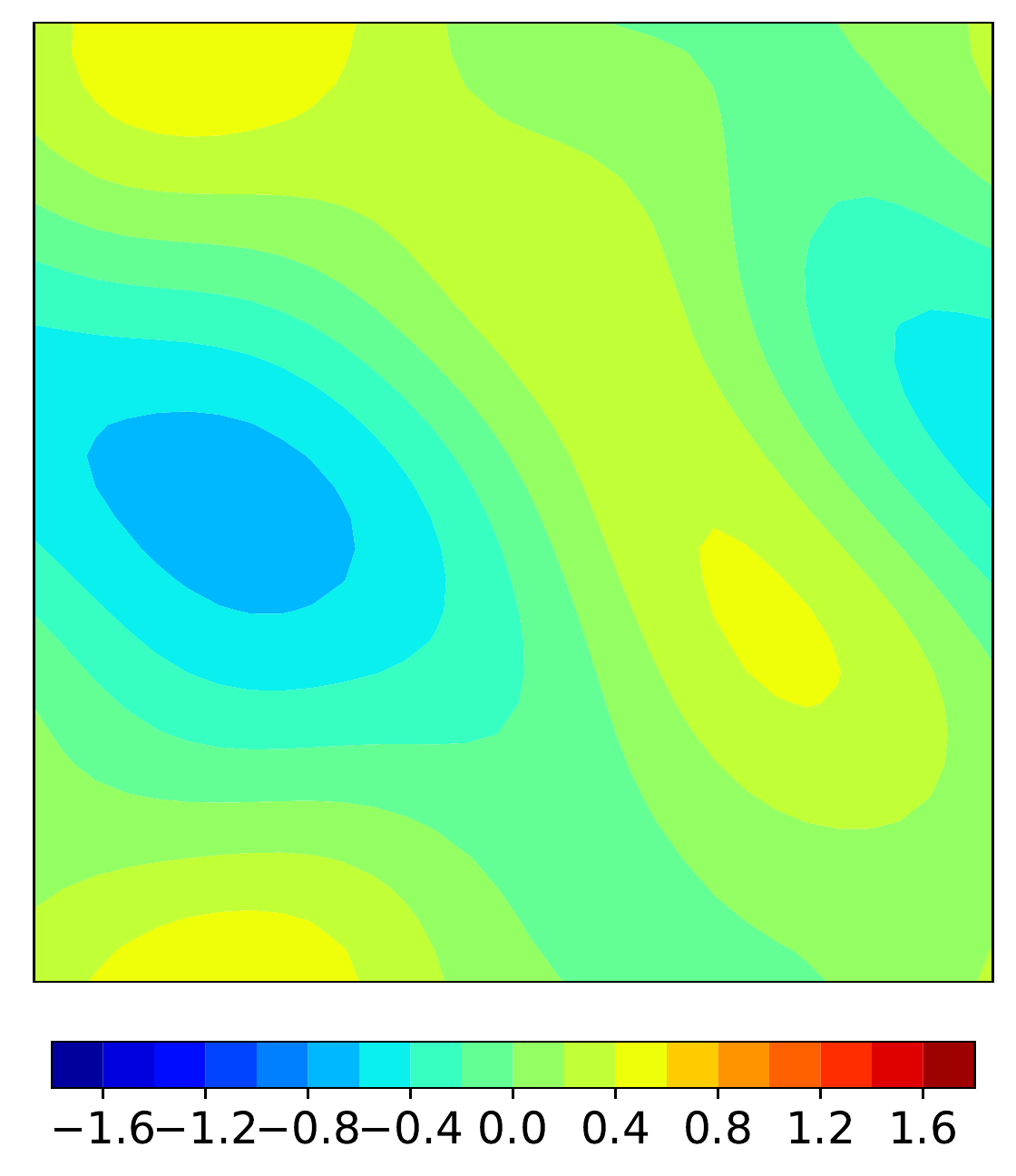}} &
        \raisebox{-0.5\height}{\includegraphics[width=.20 \textwidth]{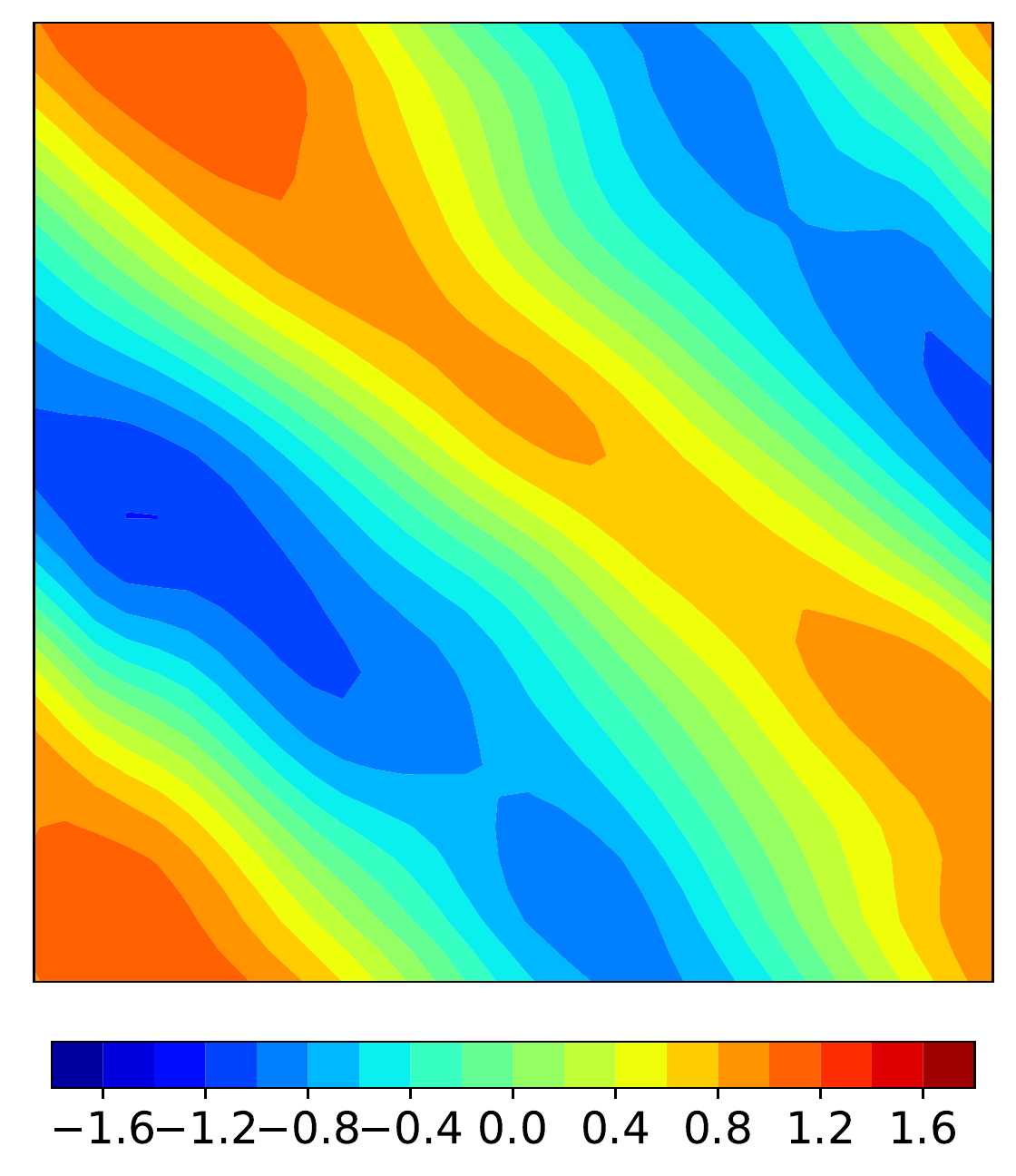}} & 
        \raisebox{-0.5\height}{\includegraphics[width=.20 \textwidth]{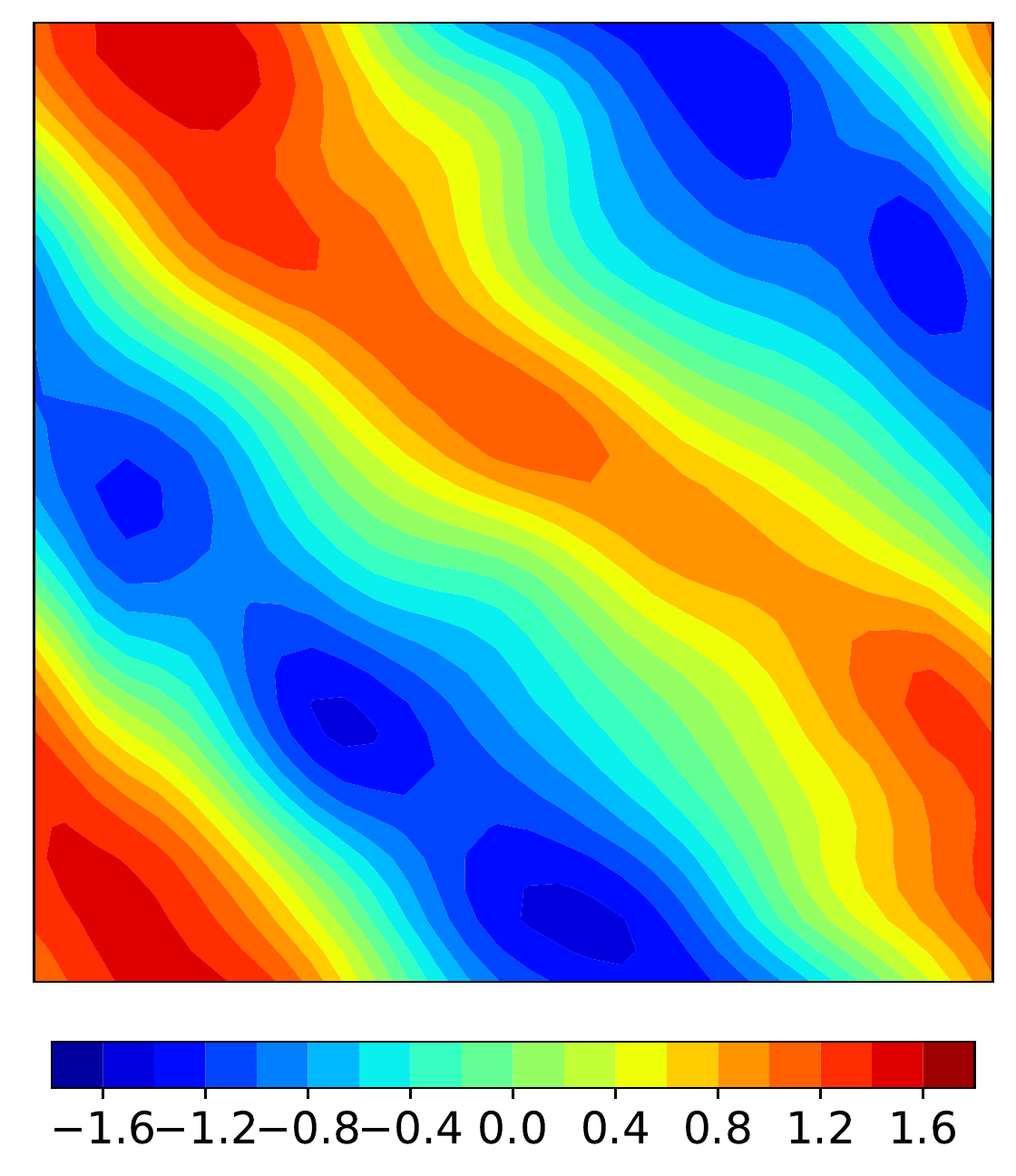}}
        \\
        \centering
        \rotatebox[origin=c]{90}{$\alpha =1e^5,\delta = 2\%$} &
        \raisebox{-0.5\height}{\includegraphics[width=.20 \textwidth]{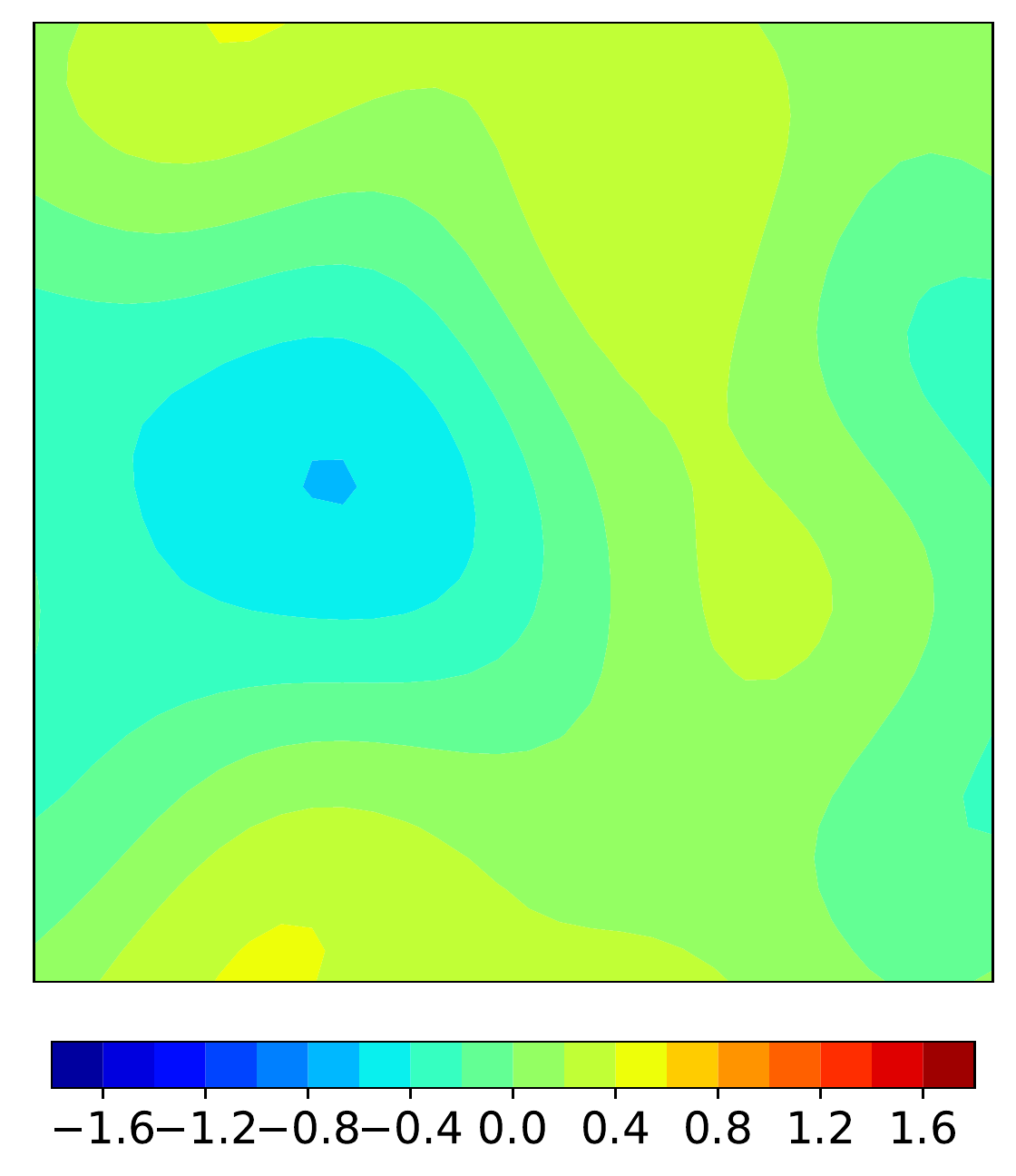}} &
        \raisebox{-0.5\height}{\includegraphics[width=.20 \textwidth]{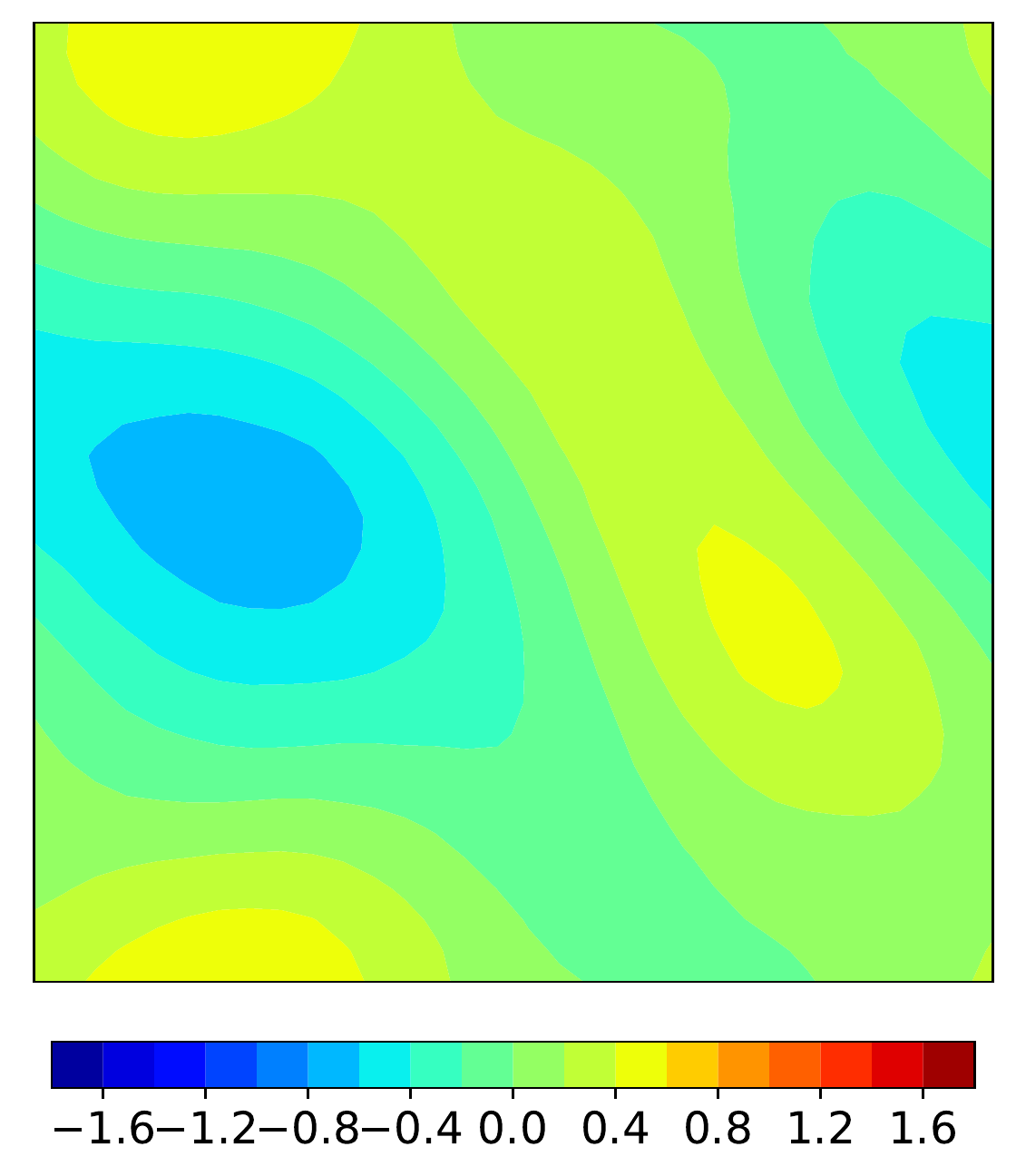}} &
        \raisebox{-0.5\height}{\includegraphics[width=.20 \textwidth]{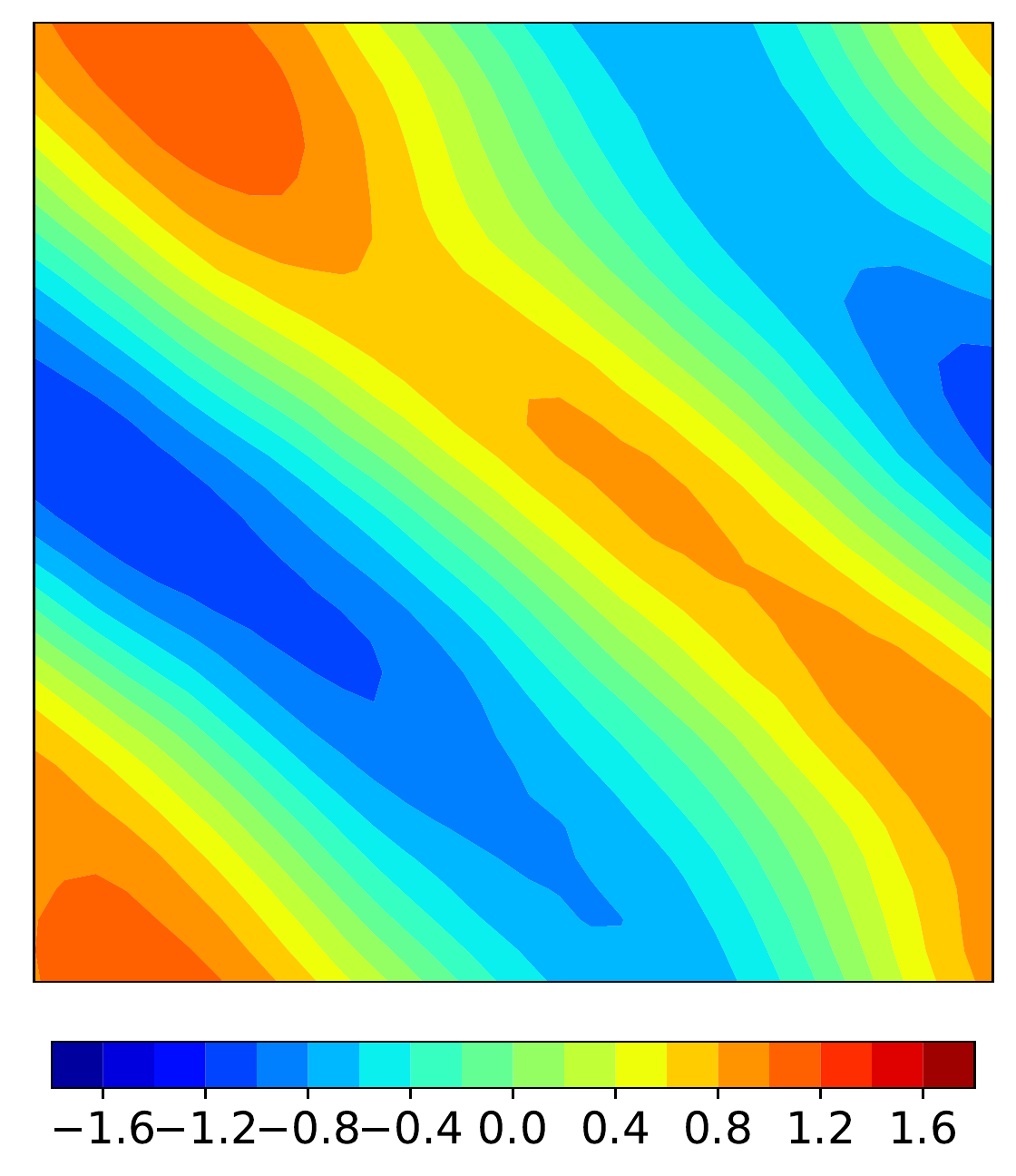}} & 
        \raisebox{-0.5\height}{\includegraphics[width=.20 \textwidth]{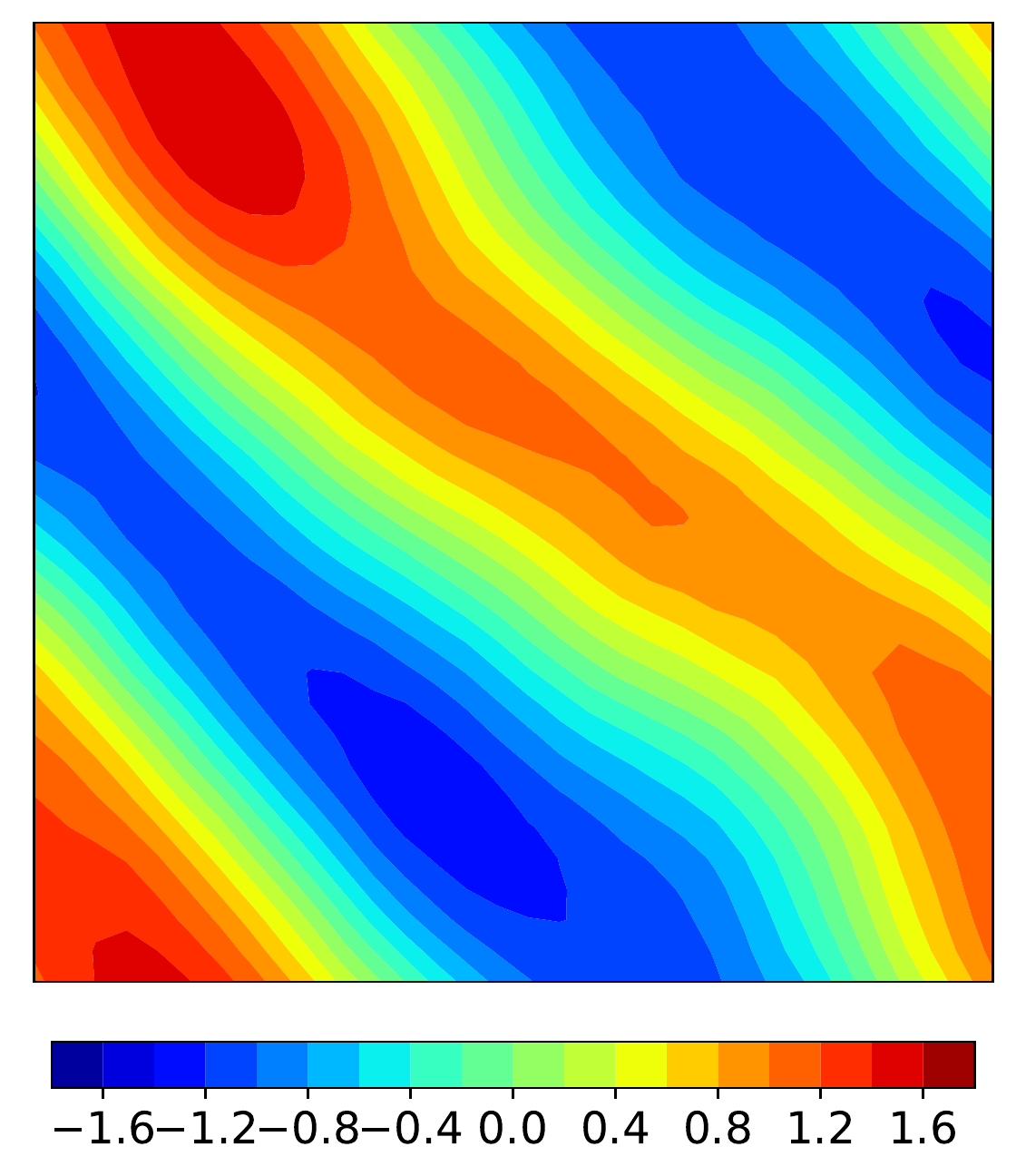}} 
    \end{tabular*}
    \caption{\textbf{Navier-Stokes equation}. Predicted solutions at different time steps obtained by various trained networks with 600 data samples. \textit{First row}: ground truth; \textit{Second row}: pure data-driven network solutions with noise-free data; \textit{Third row}: pure data-driven network solutions with randomized data; \textit{Fourth row}:   model-constrained network solutions with noise-free data; \textit{Fifth row}:  model-constrained network solutions with randomized data.} 
    \figlab{2D_NS_samples}
\end{figure}
\cref{fig:2D_NS_d200d600} shows  the mean-square error of predictions and ground truth solutions as a function of time steps. It can be seen that training with a large data set with $600$ samples provides much more accurate solutions  than with small data set with $100$ samples. On the one hand, among learned neural networks trained with $100$ data samples, the model-constrained network with data randomization for $\LRp{d100,2\%, 1,1,10^5}$ setting is far closer to the true solution than the other networks. This implies that the model-constrained approach has a significant contribution to producing accurate predictions in the context of  small data. 
In the case of richer data set with 600 samples,  networks with two settings $\LRp{d600,0\%, 1,1,0}$ and $\LRp{d600,0\%, 1,1,10^5}$ trained with noise-free data show a good performance in the short time predictions, while the long-time predictions deteriorate. Noticeably, between these two networks, the model-constrained one has more accurate predictions starting from the $500$th time step. 
In the meantime, with the same data set with 600 samples, pure data-driven neural networks trained with higher noise level data give a higher error, for example, $\LRp{d600,2\%, 1,1,0}$ neural network predictions are less accurate than those obtained from $\LRp{d600,1\%, 1,1,0}$. In contrast, model-constrained network with 2\% noise level $\LRp{d600,2\%, 1,1,10^5}$ is superior to 1\% noise level $\LRp{d600,1\%, 1,1,10^5}$. 
Another point is that as we increase the sequential model-constrained value to $R=5$, we obtain good predictions in both short-time and long-time intervals. 
Two model-constrained networks with $\LRp{d600,0\%, 1,5,10^5}$ and $\LRp{d600,2\%, 1,5,10^5}$ are comparable to the network with much larger data set $\LRp{d1000,0\%, 1,1,10^5}$ without randomization. \begin{figure}[htb!]
    \centering
    \begin{tabular*}{\textwidth}{c c c c c}
        \centering
         &
        \raisebox{-0.5\height}{\small $t = 0$} &
        \raisebox{-0.5\height}{\small $t = 2$} &
        \raisebox{-0.5\height}{\small $t = 10$} & 
        \raisebox{-0.5\height}{\small $t = 15$} 
        \\
        \centering
        \rotatebox[origin=c]{90}{True solutions $\ub$} &
        \raisebox{-0.5\height}{\includegraphics[width=.20 \textwidth]{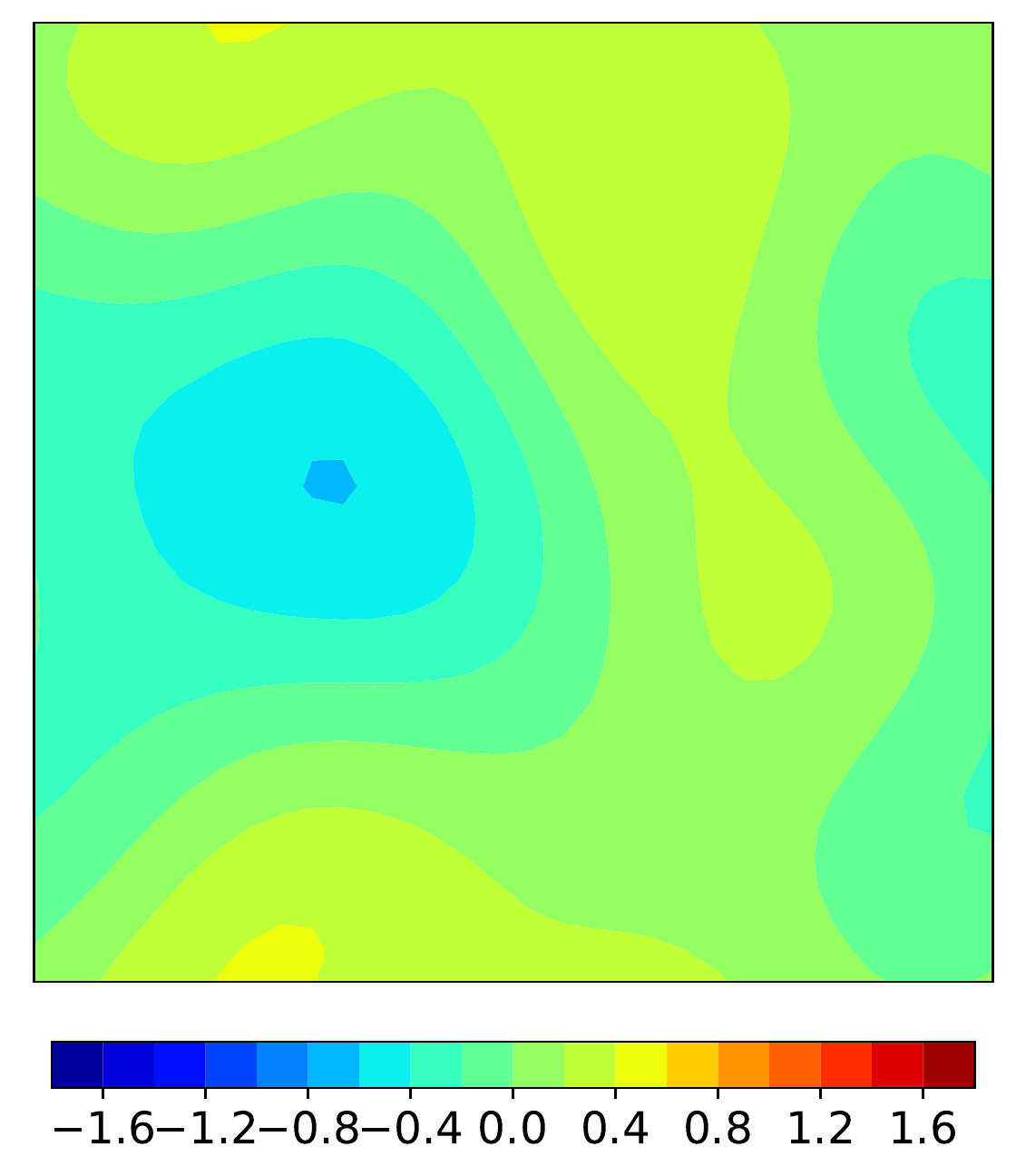}} &
        \raisebox{-0.5\height}{\includegraphics[width=.20 \textwidth]{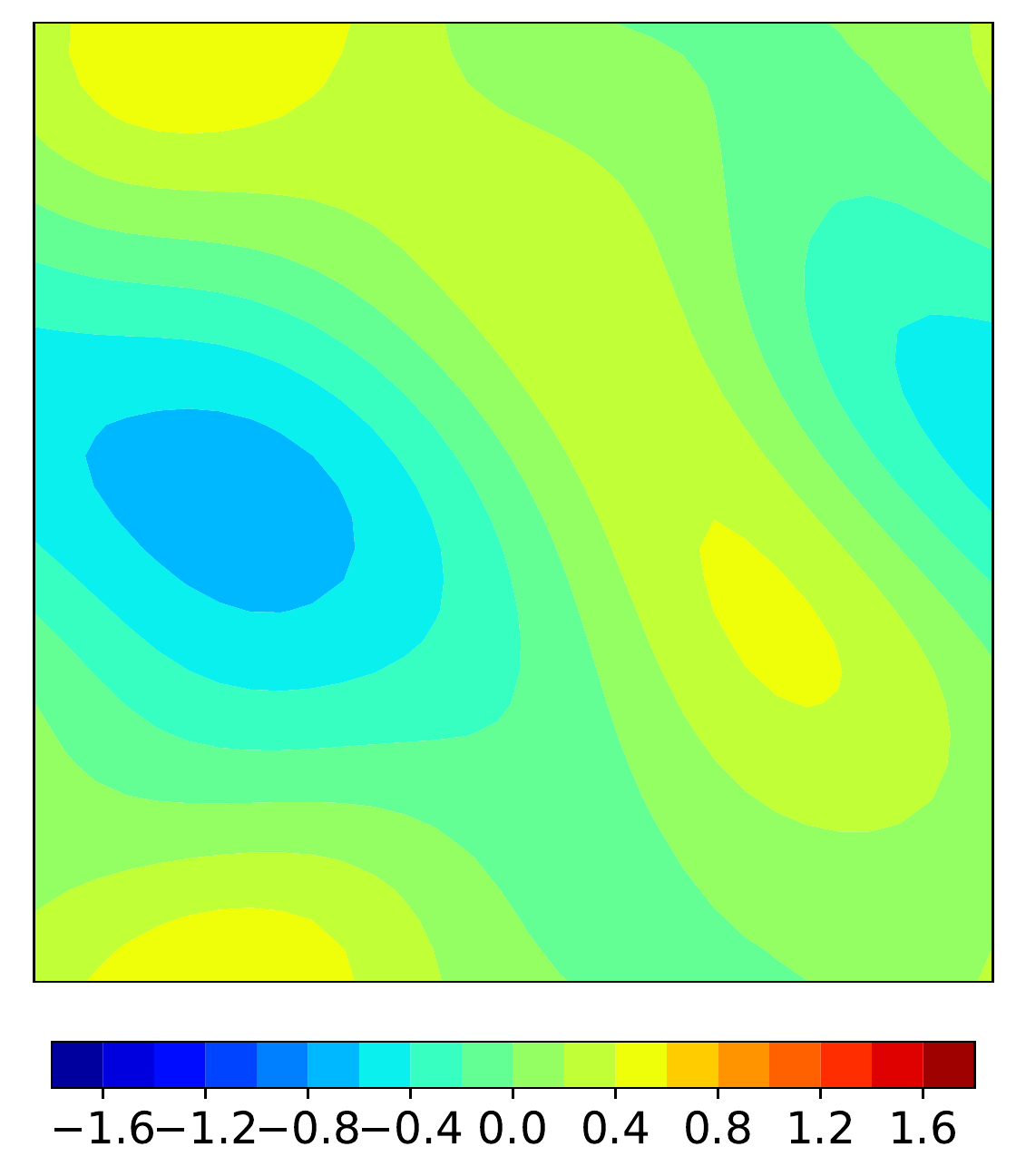}} &
        \raisebox{-0.5\height}{\includegraphics[width=.20 \textwidth]{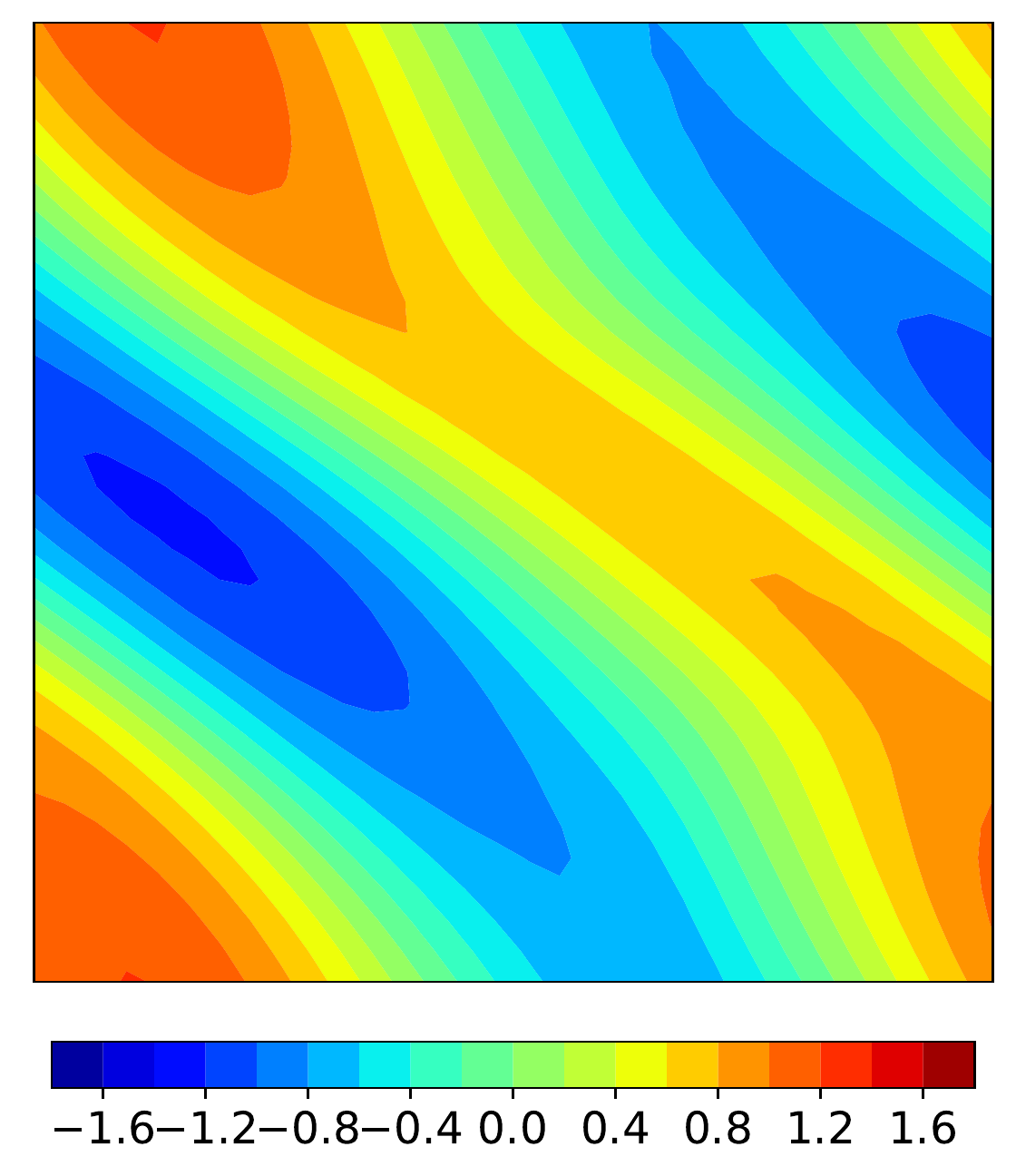}} & 
        \raisebox{-0.5\height}{\includegraphics[width=.20 \textwidth]{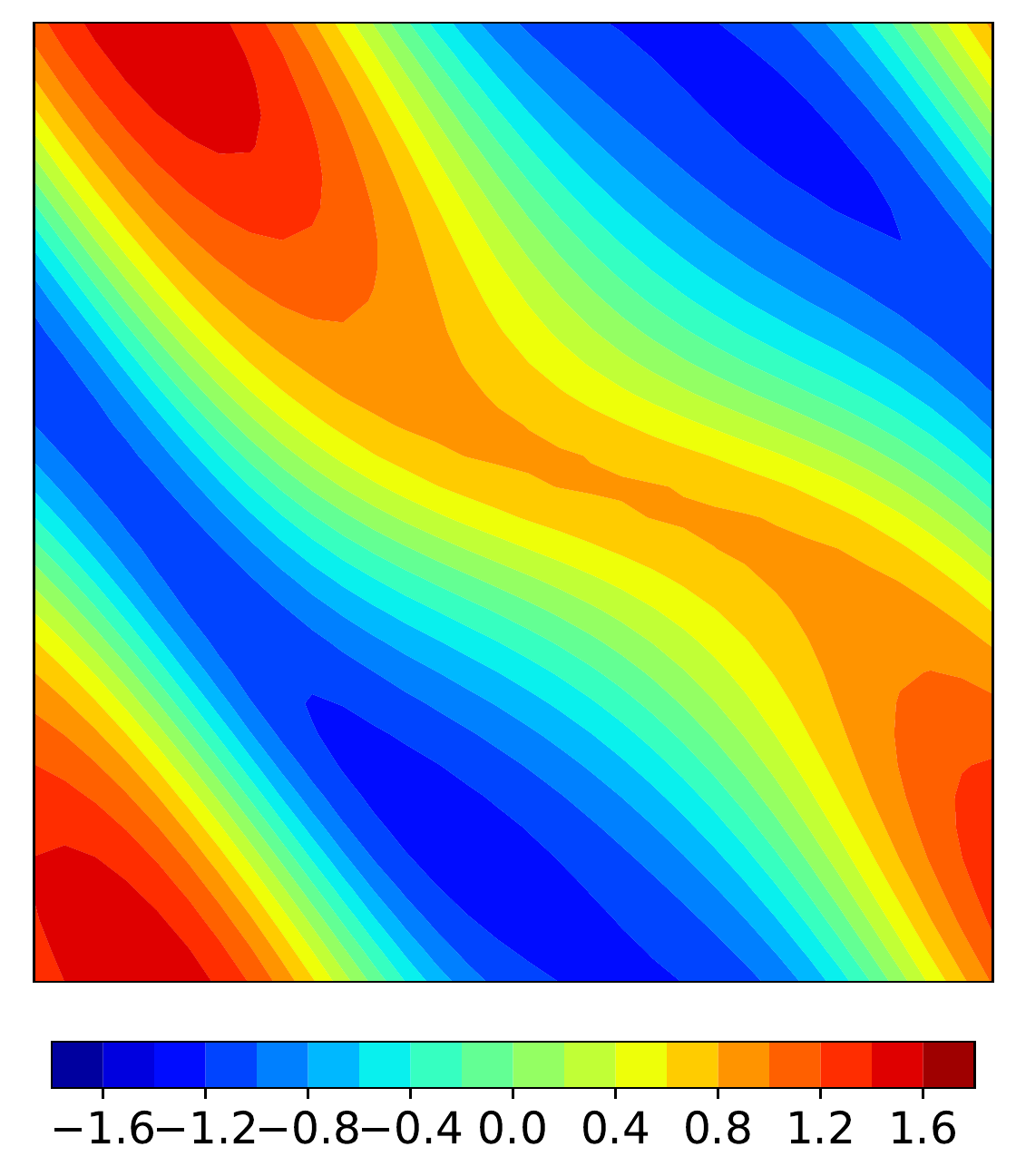}} 
        \\
        \centering
        \rotatebox[origin=c]{90}{FE - Learned $\Psi$} &
        \raisebox{-0.5\height}{\includegraphics[width=.20 \textwidth]{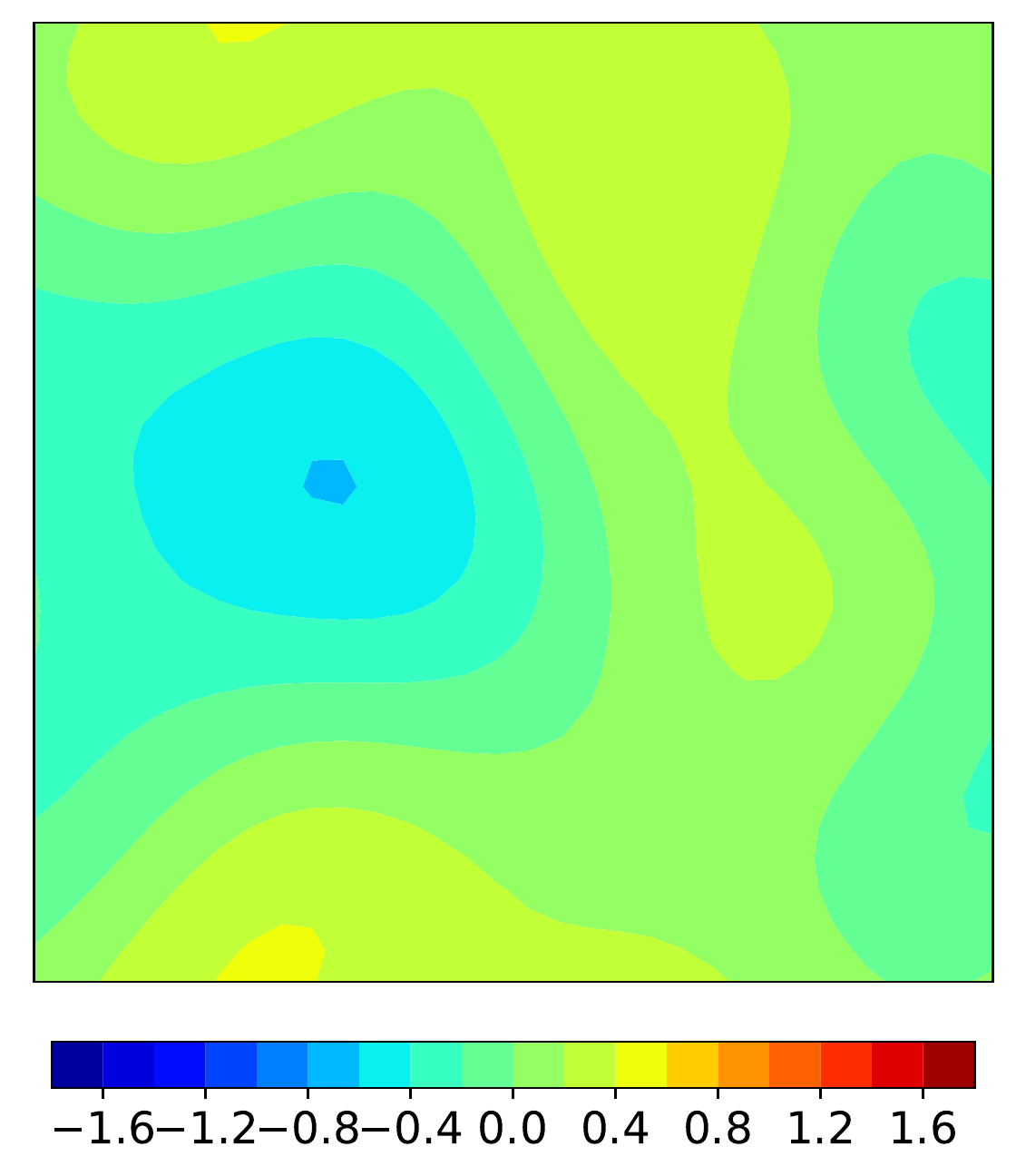}} &
        \raisebox{-0.5\height}{\includegraphics[width=.20 \textwidth]{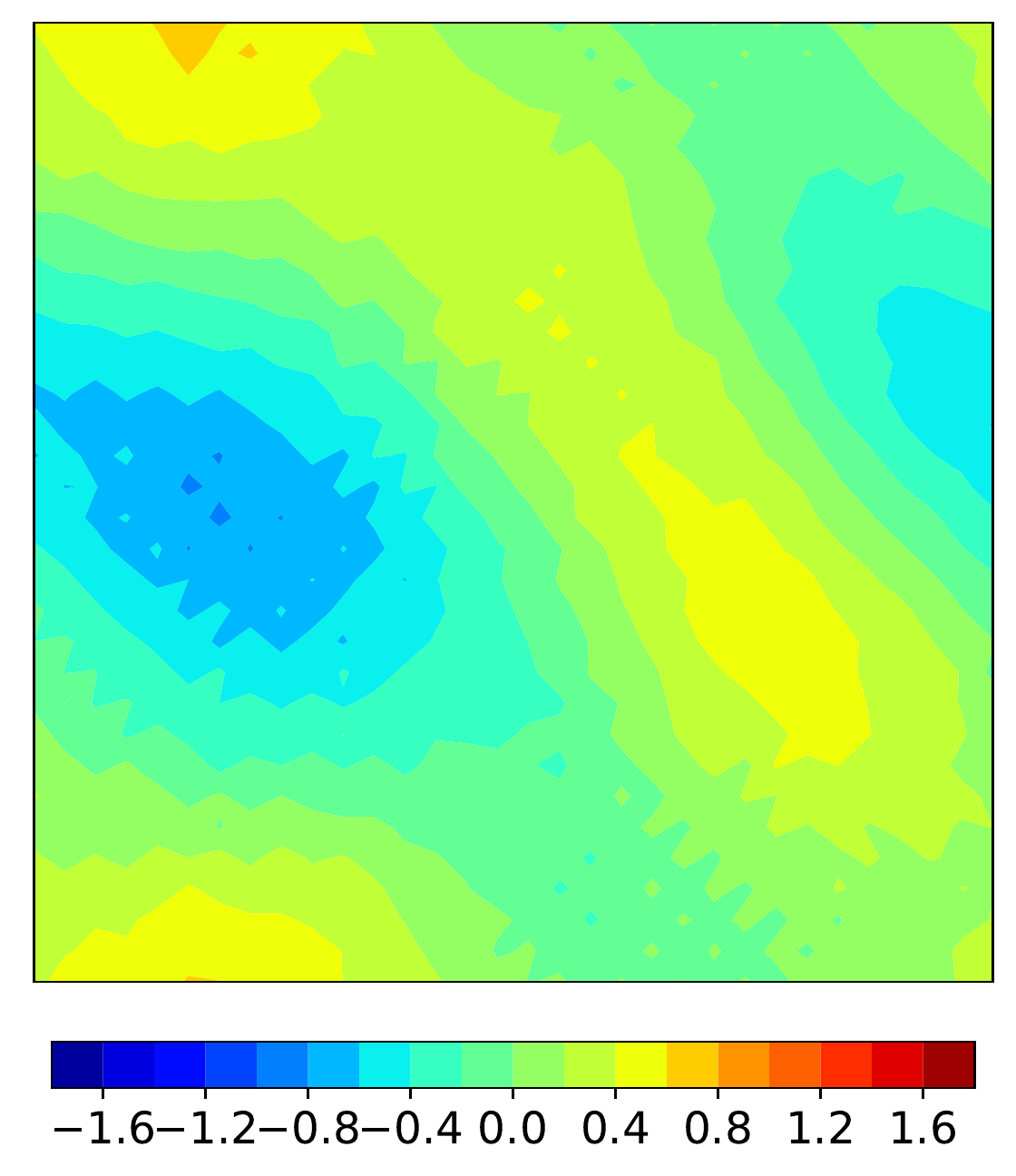}} &
        \raisebox{-0.5\height}{\includegraphics[width=.20 \textwidth]{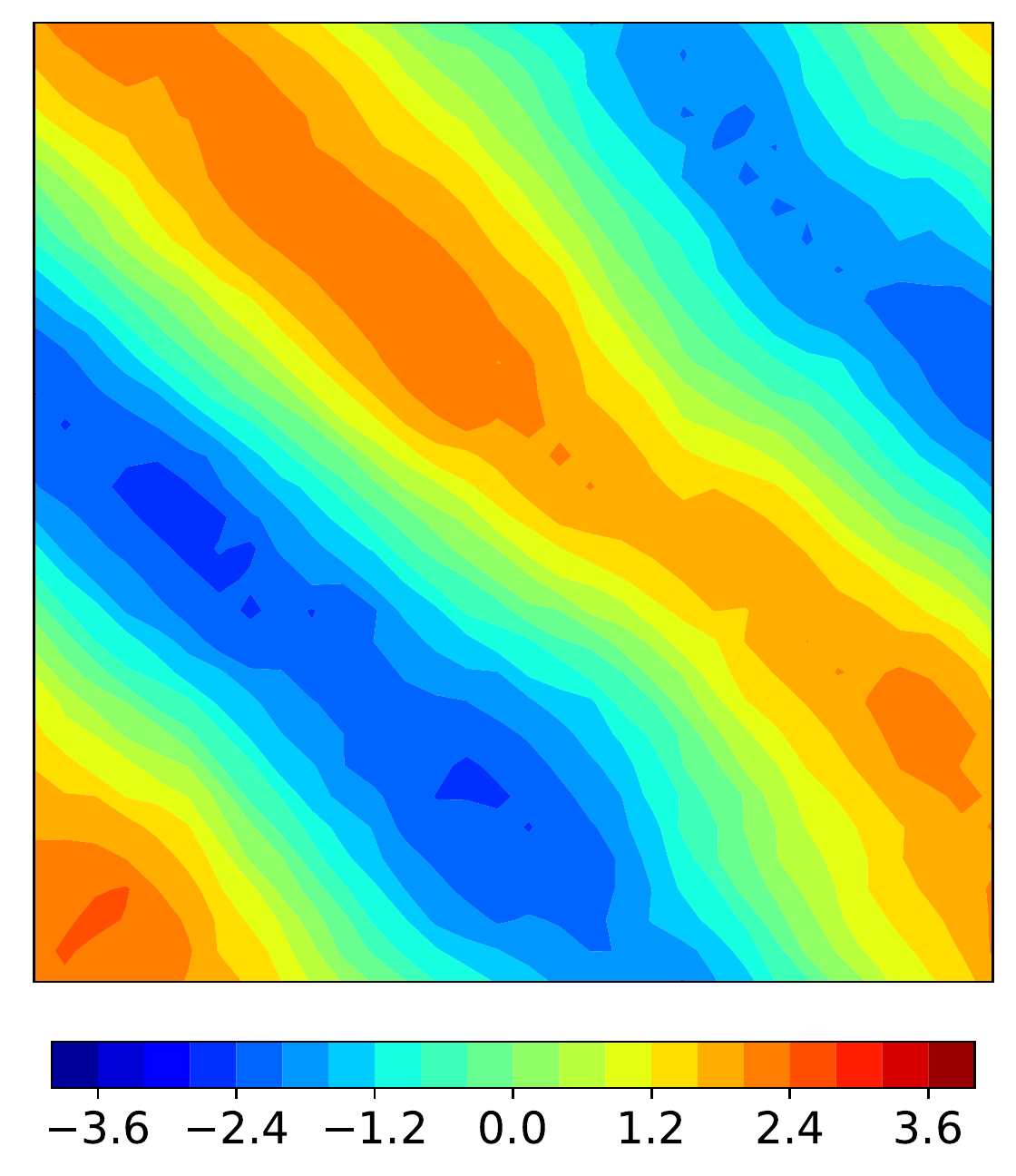}} & 
        \raisebox{-0.5\height}{\includegraphics[width=.20 \textwidth]{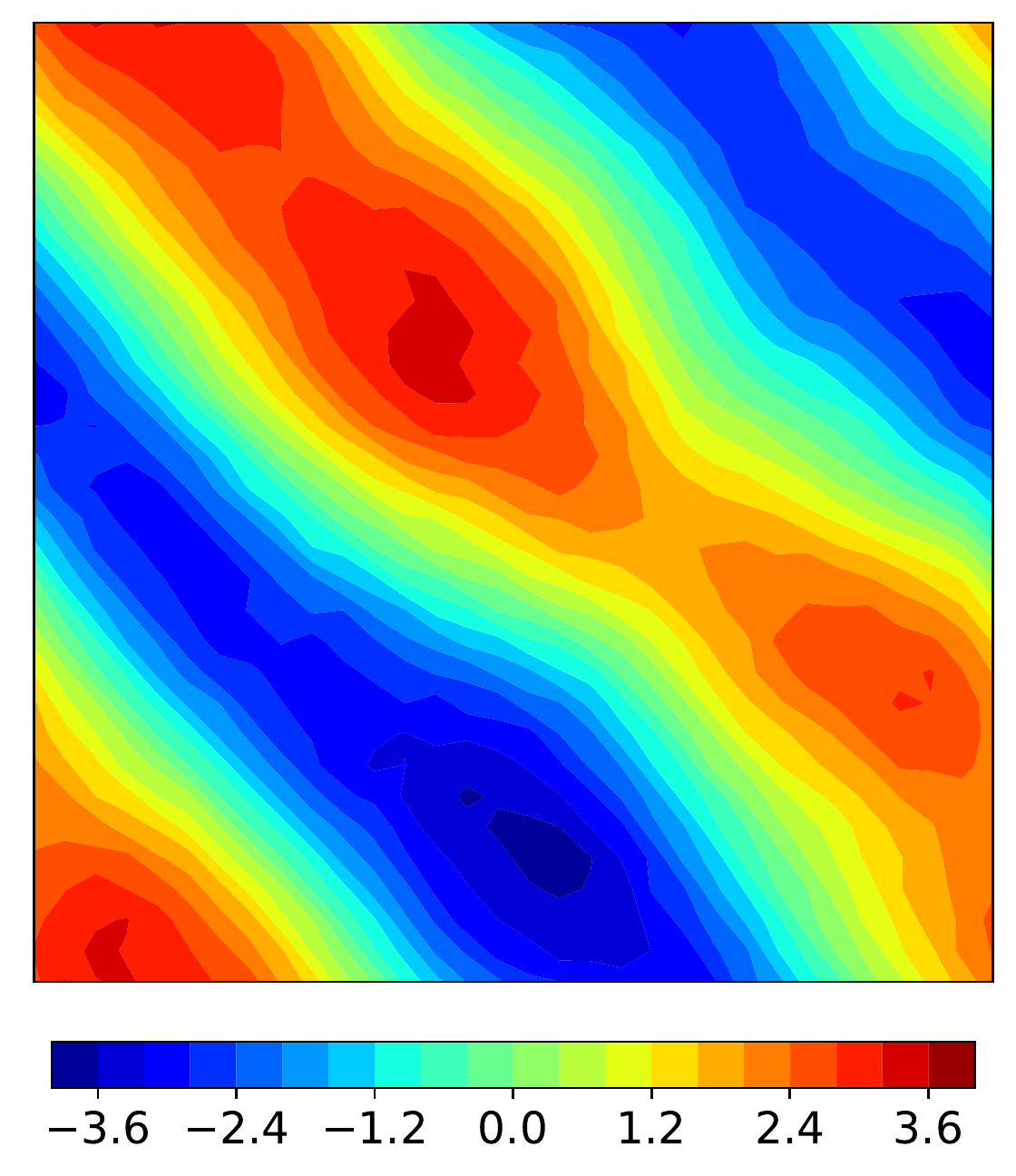}} 
        \\
        \centering 
        \rotatebox[origin=c]{90}{\small BE - Learned $\Psi$} &
        \raisebox{-0.5\height}{\includegraphics[width=.20 \textwidth]{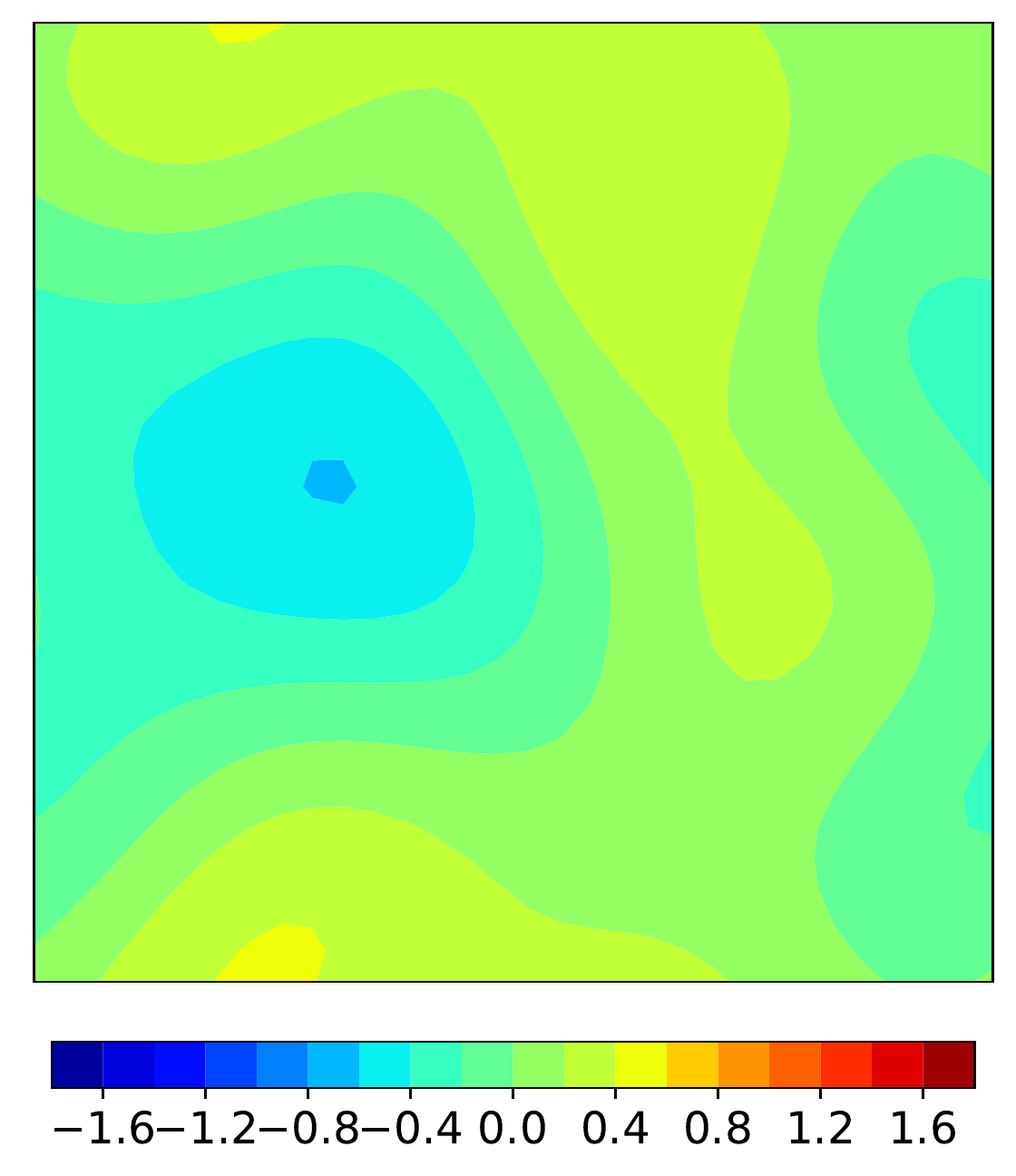}} &
        \raisebox{-0.5\height}{\includegraphics[width=.20 \textwidth]{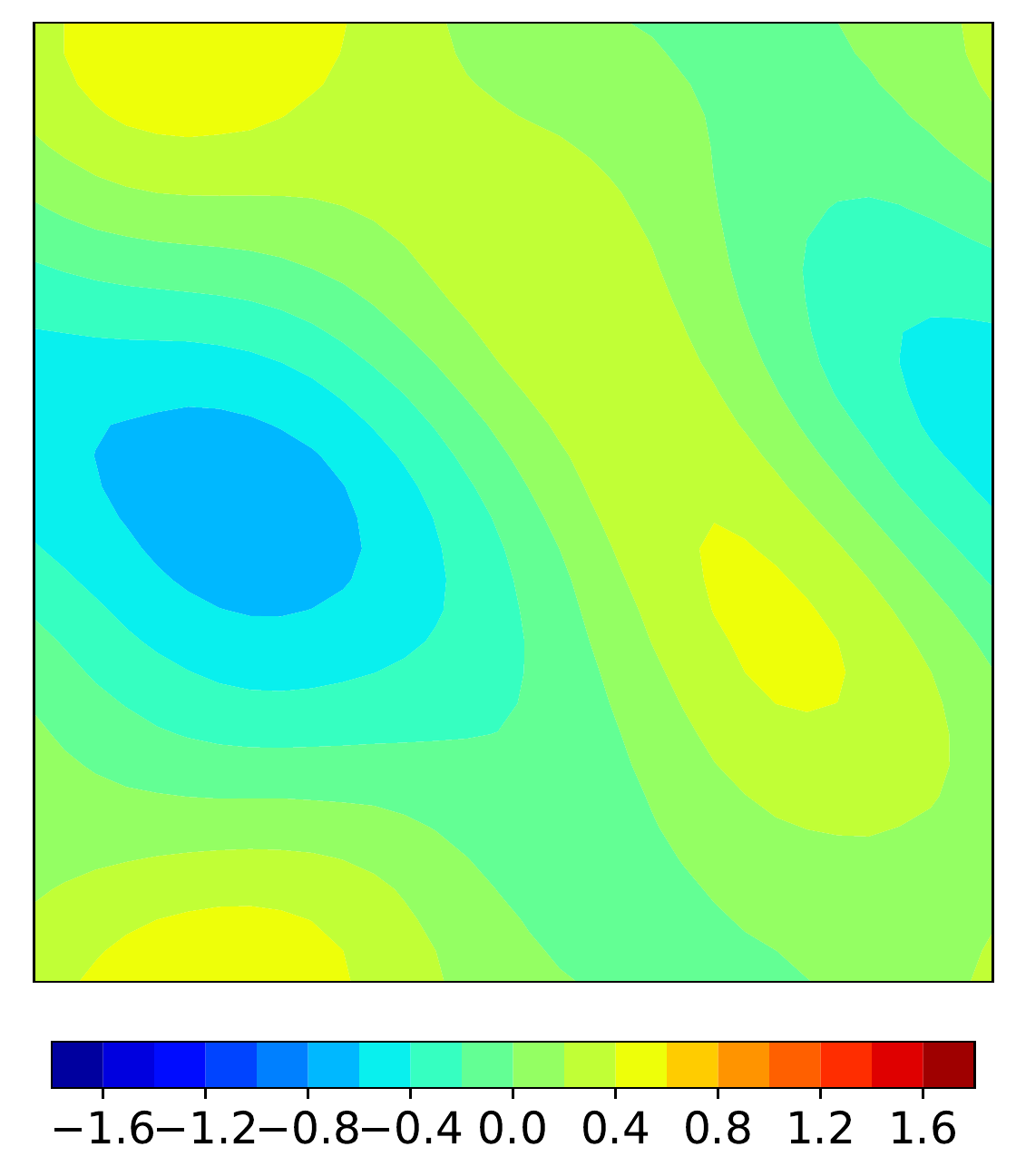}} &
        \raisebox{-0.5\height}{\includegraphics[width=.20 \textwidth]{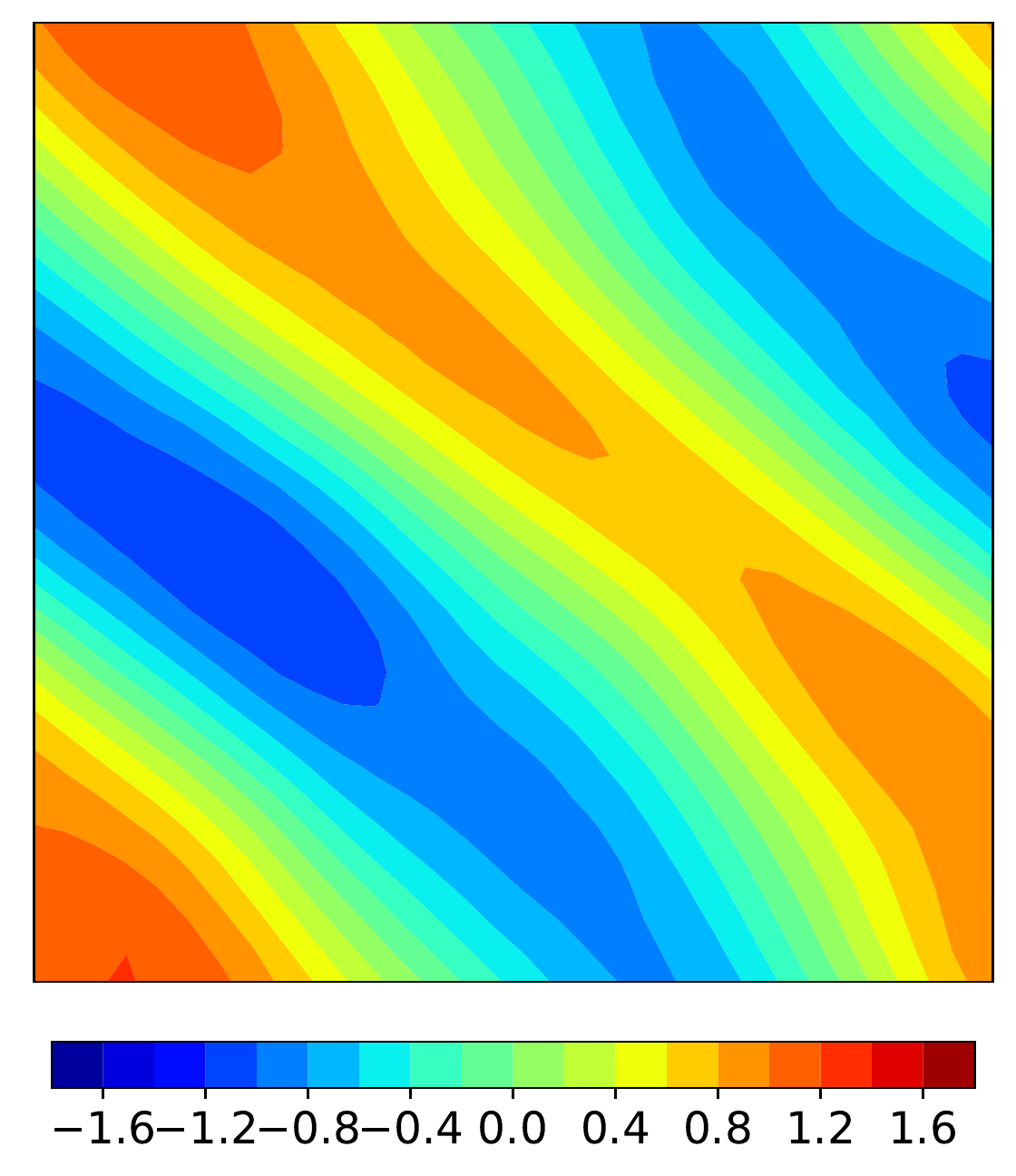}} & 
        \raisebox{-0.5\height}{\includegraphics[width=.20 \textwidth]{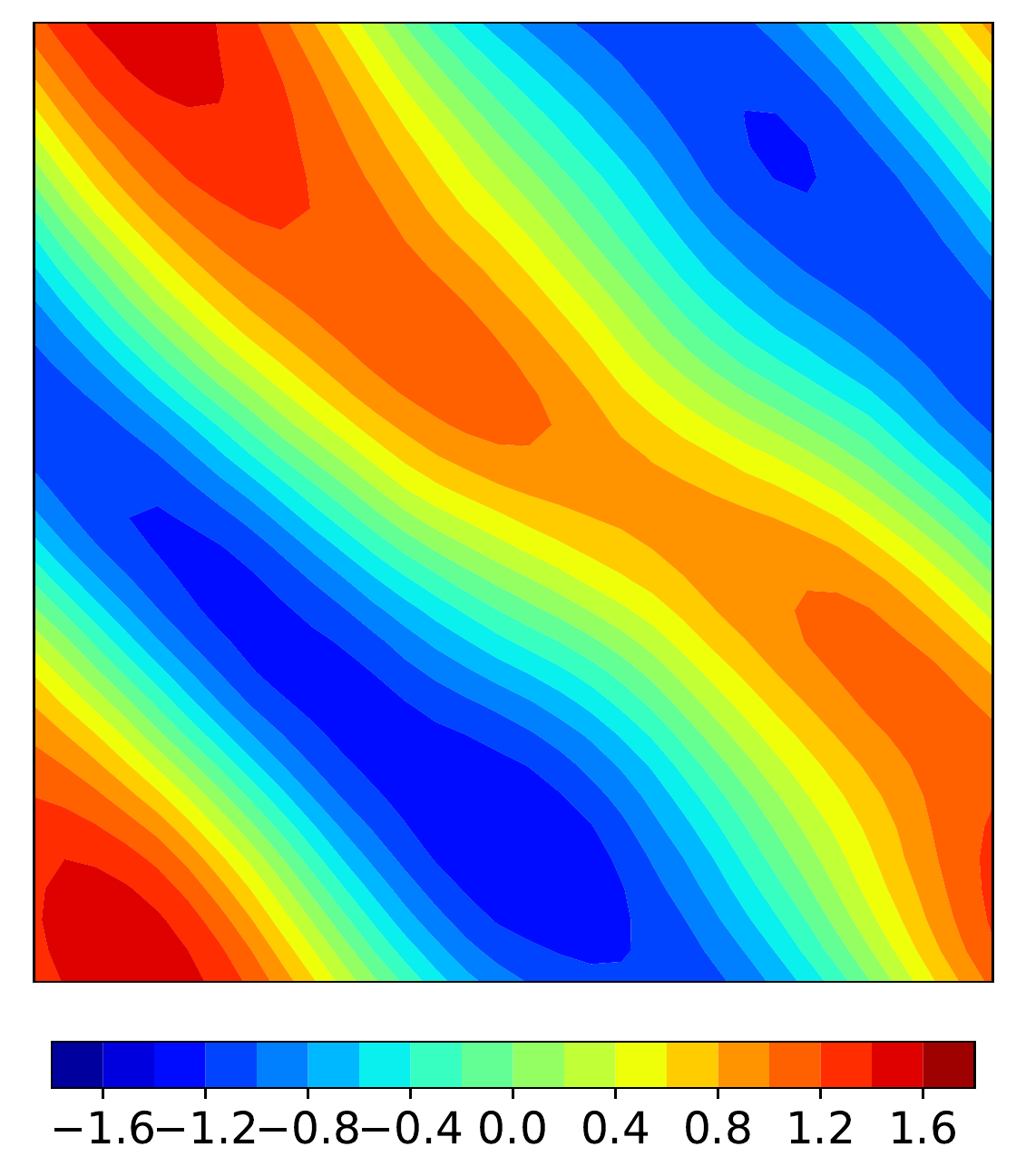}} 
    \end{tabular*}
    \caption{\textbf{Navier-Stokes equation}. Comparison of various neural network solutions. \textit{Top row}: True solutions with spectral method (Crank–Nicolson time integration scheme) with the true tangent slope; \textit{Second row}: Forward Euler (FE) scheme (a different scale bar in the third and fourth columns); \textit{Third row}: Backward Euler (BE) scheme with the learned neural network for 20 times larger time stepsize $\dt' = 20 \dt = 0.2$.} 
    \figlab{2D_NS_implicit_samples}
\end{figure}
However, the noisy data network $\LRp{d600,2\%, 1,5,10^5}$ outperforms the noise-free one $\LRp{d600,0\%, 1,5,10^5}$ in the long-time predictions. 
In summary,  model-constrained network with data randomization outperforms all other networks. 
Given a test initial vorticity, the plots of predicted solutions  obtained by different learned networks are shown in \cref{fig:2D_NS_samples}. As can be seen, the  model-constrained network with the setting $\LRp{d600,2\%, 1,1,10^5}$ provides the most accurate solutions as opposed to others trained from the same data set.

{\bf Implicit time integration with learned network.}
We used the learned network for  backward Euler scheme with 20 times larger time stepsize, $\dt' = 20 \dt = 0.2$, compared to training stepsize $\dt = 0.01$. As shown in  \cref{fig:2D_NS_implicit_samples},
%the true map is able to give accurate solutions compared to the case of using a small time stepsize $\dt = 0.01$. The reason is that the forward map uses the spectral method with the Crank–Nicolson update scheme. 
forward Euler scheme with the learned network shows severe instability, while the backward Euler scheme with the learned network solutions is in good agreement with the spectral solution with Crank–Nicolson scheme with a much smaller time stepsize.

\subsection{Information on parameter tuning, randomness, and cost}
\seclab{Comp_cost}

\subsubsection{Parameter tuning}
\seclab{tuning}
The purpose of this section is to determine a good set of hyperparameters including the learning rate, batch size, the number of layers, and the number of neurons on each layer. To that end, we set the random seed to 0 in order to have a fair initialization for all networks. For initialization, weights are drawn randomly from zero-mean Gaussian distribution with a variance of 0.01, while biases are set to zero. For all cases,  we take $S = 1, R = 1$, and noise-free data set with $600$ samples. We carry out the tuning process manually for only Burgers and Navier-Stokes examples as the transport example admits an analytical solution.  We pick the learning rate in  $\LRc{10^{-4},2 \times 10^{-4}, 5 \times 10^{-4}, 10^{-3}}$, batch size in $\LRc{2, 10, 40, 100}$, the number of layers in $\LRc{1,2,3}$, the number of neurons per layers in $\LRc{50, 200, 1000, 5000, 10000}$, and the model-constrained regularization parameter $\alpha$ in $\LRc{10, 10^3, 10^5, 10^7}$. We pick the combination of parameters that provides the best testing accuracy (see also  \cref{sect:Burger_eq} and \cref{sect:2D_NS} for the discussion on testing accuracy) in each numerical problem. The chosen parameter set is then used for training different values of $S, R$ and noise level $\delta$.

% We have tested different cases of different initializers, and the difference is insignificant. We note that we initialize for a larger number of parameters (weights and biases), $>5e9$, with very small random number. The randomness of network is not a problem. On the other hand, we also changed the random seed for adding noise for data randomization, there is insignificant change. We note that we have used 600 initial conditions and the renumbering data indexes increases the cases for data randomization up to $600 \times (N_t-s)$. It can be observed from the scenario when we use 1000 initial condition samples for training, there is almost no improvement.
    
\subsubsection{Robustness with random initializations and data randomization}
In this section, we study the effect of weights/biases random initialization and data randomization on the performance of  the chosen neural network architectures in \cref{sect:tuning}. We provide the study for Burger's equations in \cref{sect:Burger_eq} since the result for the Navier-Stokes equation in \cref{sect:2D_NS} would be similar. 
% We initialize the neural network with 32 different random seeds ranging from 0 to 31. For each random seed, we use the same set of learning parameters, which are tuned with random seed of 0 in \cref{sect:tuning}. 
{For random initialization of weights/biases, we initialize the neural network with 32 different random seeds ranging from 0 to 31. For each random seed, we use the same set of  hyper-parameters found in \cref{sect:tuning}. As shall be shown, our model-constrained approach is robust in random initialization, that is, all random seeds work equally well.
Thanks to this robustness,  we simply initialize weights/biases with random seed 0 and study the effect of  32 different noise random seeds ranging from 0 to 31 for data randomization.}
As an example, we compare the mean and variance of the mean square error between the pure data-driven machine learning case  $\LRs{d600, 2\%, 0, 1, 0}$ and the corresponding model-constrained case $\LRs{d600, 2\%, 10^5, 1, 1}$. The mean and variance results in \cref{fig:Burger_randomization} and \cref{fig:Burger_noise}  show that  \texttt{mcTangent} networks are not only  accurate but also more reliable with a smaller variance compared to the pure data-driven counterparts. Consequently\textemdash again thanks to the model-constrained term\textemdash the performance of \texttt{mcTangent} networks are robust to both
weights/biases random initialization  and data randomization.

\begin{figure}[htb!]
    \centering
    \includegraphics[width=\textwidth]{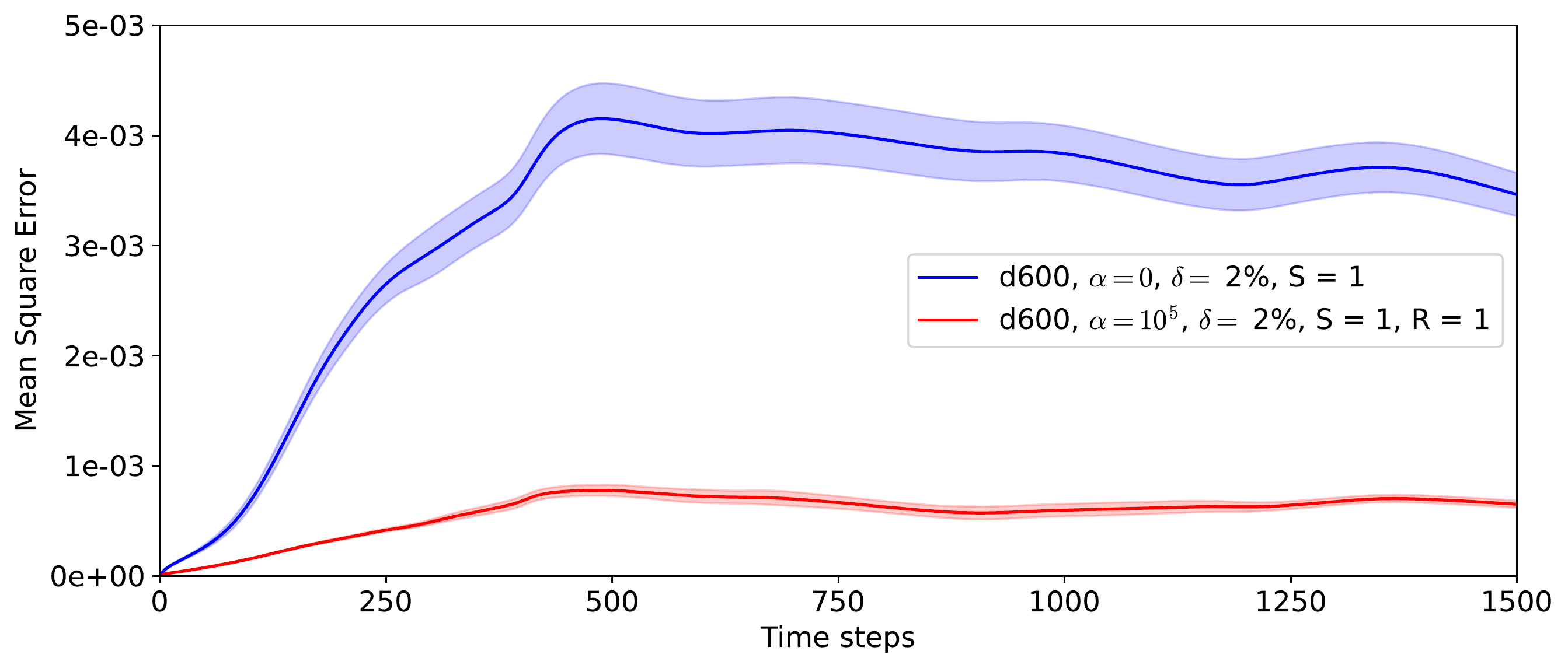}
    \caption{The mean and variance of mean square error of predictions for \texttt{mcTangent} and pure data-driven machine learning  approaches, obtained by 32 different neural networks corresponding to 32 different weights/biases random initializations.} 
    \figlab{Burger_randomization}
\end{figure}

\begin{figure}[htb!]
    \centering
    \includegraphics[width=\textwidth]{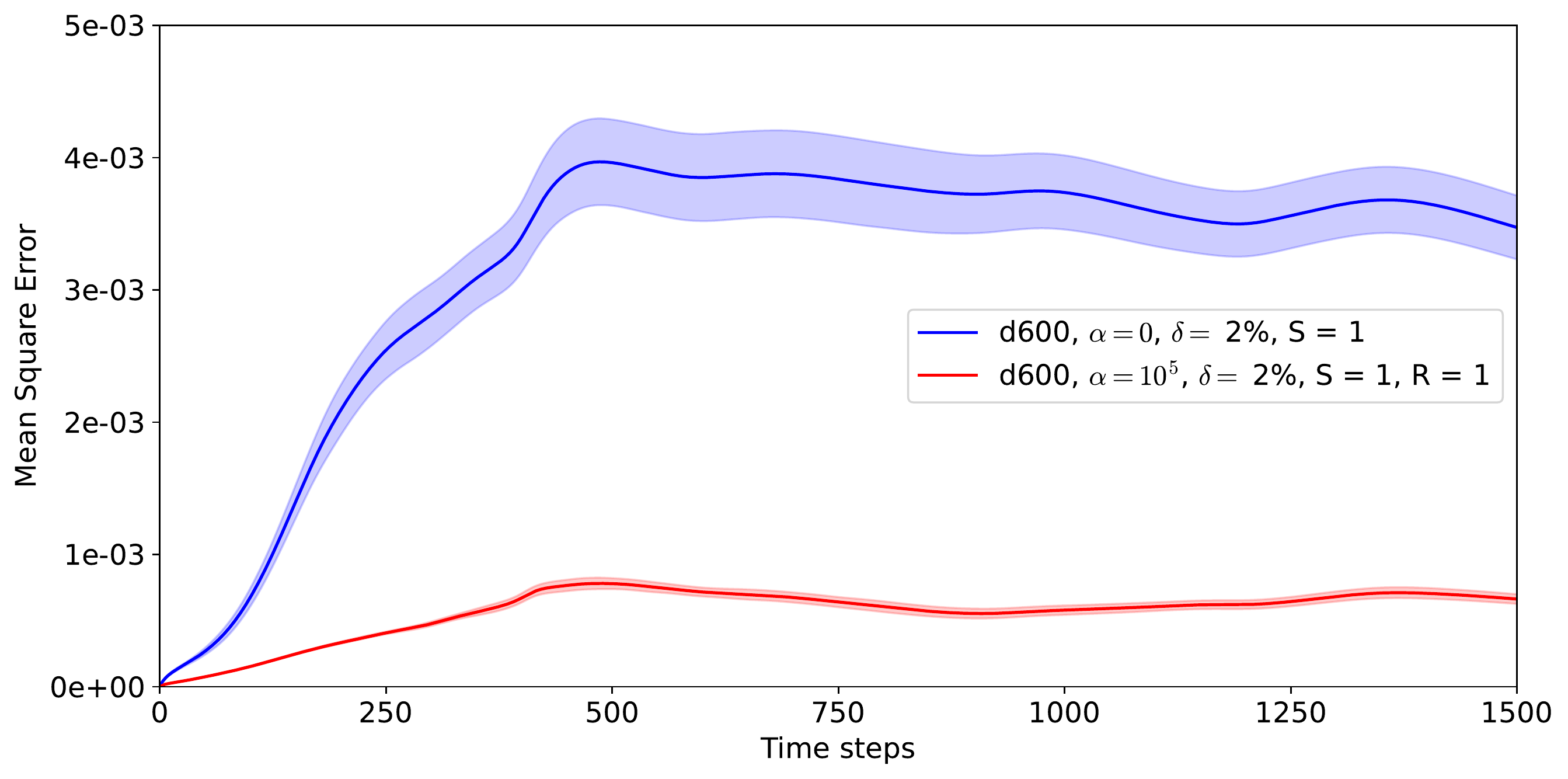}
    \caption{The mean and variance of mean square error of predictions for \texttt{mcTangent} and pure data-driven machine learning approaches, obtained by training a neural network (with weights/biases being initialized with random seed 0) with 32 random realizations of data  corresponding to 32 different random seeds.}
    \figlab{Burger_noise}
\end{figure}

\subsubsection{Training and testing cost}
\cref{tab:train_cost} presents the training computational cost for Burgers' problem using different values of $S$ and $R$. The \texttt{mcTangent} neural network is learned with 200 training samples with $2\%$ additive noise, $\alpha = 10^5$, learning rate $10^{-4}$, and batch size $40$. It can be seen that in the purely data-driven approach, i.e. $S= 1, R= 0$, the computational time per epoch is small, but the number of required epochs for convergence is larger. It is not surprising that adding larger $R$ and $S$ leads to a significant increase in computational cost per epoch. However, in this problem, since the overall convergence rate (measured in the terms of the number of epochs) is faster, the total amount of time for training model-constrained neural networks are at most three times larger than the pure machine learning method. To be more specific, $S = 1, R = 0$ network requires  $11.67$ hours compared to $22.22$ hours for $S = 1, R = 1$ network. On the other hand, $35$ hours and $34$ hours are needed to train the cases $S = 10, R = 1$ and $S = 1, R = 5$, respectively. We note that all model-constrained networks corresponding to $S=1, R = 1$, $S =10, R =1$ and $S = 1, R =5$ provide comparable accuracy levels which are significantly better than that obtained by the pure data-driven machine learning approach corresponding to $S = 1, R = 0$, and this  is shown in \cref{fig:2D_Bur_compare_all}.

To verify the computational benefits in the prediction stage, we compare the computational time between the ground truth solution using the truth tangent slope $\F\LRp{\ut{}}$ and  \texttt{mcTangent} tangent slope  $\NN{\ut{}}$ in \cref{tab:test_cost}. It can be seen that the \texttt{mcTangent} tangent slope is much faster  (more than $10$ times faster) than the truth tangent slope for the 2D Navier-Stokes problem. For the 2D Burgers' problem, the \texttt{mcTangent} tangent slope evaluation is negligibly faster than the truth.
That is, even with small-scale 2D problems with fast finite difference evaluations, the neural network is still faster. It is important to point out that the computational cost for \texttt{mcTangent} neural network remains unchanged, $2\times 10^{-4}$ seconds, for either Burgers or Navier-Stokes equations. We expect the computational gain is much more notable for 3D complex problems where the evaluation of the truth tangent slope is much more demanding. The gain is even more significant for implicit methods as in these cases not only the evaluation of the tangent slope but also the evaluation of its Jacobian is needed. This poses great challenges for traditional numerical methods, but for the  \texttt{mcTangent} approach, the evaluations of a feed-forward network and its Jacobian are trivial and fast.

% On the other hand, since we use the finite difference method for solving the Burger's equation,  computing $\F$ is not expensive. However, if more complicated method, i.e., finite element method or finite volume method are used, $\F$ is much more expensive. In that case, the \texttt{mctangent} network is more accurate (since the accuracy of $\Psi$ vastly depends on the accuracy of $\F$), but still evaluated with the same amount. A good example is the Navier-Stoke equation problem where $\F$ is more complicated and taken much more time, $7\times 10^{-3}$, to estimate, while the $\texttt{mctangent}$ just requires $2\times 10^{-4}$ seconds for per forward operator.

\begin{table}[htb!]
\caption{Training cost for Burgers' equations using different values of $S$ and $R$. \texttt{mcTangent} neural network is learned with 200 samples with $2\%$ additive noise, $\alpha = 10^5$, learning rate $10^{-4}$, and batch size $40$.}
\tablab{train_cost}
\centering
\begin{tabular}{|c|c|c|c|}
\hline
                & 1 Epoch (seconds)  & Number of Epoch & Training time (hours) \\ \hline
$S=1, R = 0$    & 0.07                & $6.0 \times 10^{5}$          & 11.67                  \\ \hline
$S = 1, R = 1$  & 0.20                & $4.0 \times 10^{5}$           & 22.22                 \\ \hline
$S = 10, R = 1$ & 0.86                & $1.5 \times 10^{5}$           & 35.83                 \\ \hline
$S = 1, R = 5$  & 0.62                & $2.5 \times 10^{5}$          & 34.44                 \\ \hline
\end{tabular}
\end{table}

\begin{table}[htb!]
\caption{Computational cost of ground truth tangent slope $\F(\ut{i})$ (mesh grid: $32 \times 32$), and trained neural network $\Psi\LRp{\ut{i}}$}
\tablab{test_cost}
\centering
\begin{tabular}{|c|c|c|}
\hline
                       & $\G (\ui{})$ (seconds)   & $\Psi(\ui{})$ (seconds) \\ \hline
Burgers' equation      & $2.1 \times 10^{-4}$ & $2 \times 10^{-4}$  \\ \hline
Navier-Stokes equation & $7.0 \times 10^{-3}$ & $2 \times 10^{-4}$  \\ \hline
\end{tabular}
\end{table}

\section{Conclusions}
\seclab{conclusions}
We have presented a model-constrained tangent slope learning (\texttt{mcTangent}) approach to simulate dynamical systems in real-time. At the heart of \texttt{mcTangent} is a careful craft  synergizing several desirable strategies: i) a tangent slope learning to take advantage of the neural network speed and time-accurate nature of the method of lines;
ii) a model-constrained approach to encode the neural network tangent slope with the underlying physics; iii) sequential learning strategies to promote long-time stability and accuracy; and iv) data randomization approach to implicitly regularize the smoothness of the neural network tangent and its likeliness to the truth tangent up second order derivatives in order to further enhance the stability  and accuracy of \texttt{mcTangent} solutions. Rigorous results are provided to analyze and justify the proposed approach. Several numerical results  for  transport equation, viscous Burgers equation, and Navier-Stokes equation are presented to study and demonstrate the capability of the proposed \texttt{mcTangent} learning approach. Further theoretical analysis of \texttt{mcTangent} with both sequential learning strategies is ongoing to provide a deeper understanding of the approach.  Strategies to improve the accuracy and to strongly encode the underlying governing equations are also part of future work.

\section*{Disclosure statement}
No potential conflict of interest was reported by the author(s).

% \section*{Funding}
% This work was supported by ...

\bibliographystyle{plain}
% \bibliography{references,referencesNew}
\bibliography{references, referencesNew}

\end{document}